\documentclass[twoside,11pt]{article}
\usepackage{jair, theapa, rawfonts}

\usepackage{amsmath}
\usepackage{graphicx}
\usepackage{multirow}
\usepackage{longtable}
\usepackage{amssymb}
\usepackage{mathtools}
\DeclarePairedDelimiter{\abs}{\lvert}{\rvert}
\usepackage{amsthm}
\usepackage{url}

\usepackage[table]{xcolor}
\definecolor{Gray}{gray}{0.9}

\usepackage{caption}
\usepackage{subcaption}

\newtheorem{lemma}{Lemma}
\newtheorem{theorem}{Theorem}

\newtheorem{definition}{Definition}
\newtheorem{example}{Example}

\newcommand{\inc}{\ensuremath{\mathcal{I}}}
\newcommand{\kb}{\ensuremath{\mathcal{K}}}
\newcommand{\allkbs}{\ensuremath{\mathbb{K}}}
\newcommand{\atoms}{\ensuremath{\mathsf{At}}}
\newcommand{\occs}{\ensuremath{\mathsf{Occ}}}
\newcommand{\forget}{\ensuremath{\mathrm{forget}}} 
\newcommand{\forgetting}{\ensuremath{\xrightarrow[]{\text{f}}}} 
\newcommand{\forgetfunc}{\ensuremath{\mathcal{F}}}

\newcommand{\np}{\ensuremath{\mathsf{NP}}}
\newcommand{\valuei}{\ensuremath{\textsc{value}_{\inc}}} %value_I
\newcommand{\upperi}{\ensuremath{\textsc{upper}_{\inc}}} %upper_I
\newcommand{\upperiarg}[1]{\ensuremath{\textsc{upper}_{#1}}}  %upper_{ARG}
\newcommand{\upperiforget}{\ensuremath{\textsc{upper}_{\iforget}}} 
\newcommand{\upperic}{\ensuremath{\textsc{upper}_{\icont}}} %upper_I
\newcommand{\valueic}{\ensuremath{\textsc{value}_{\icont}}} %value_I

\newcommand{\realPos}{\ensuremath{\mathbb{R}^\infty_{\geq 0}}}
\newcommand{\natnumszero}{\ensuremath{\mathbb{N}_0}}
\newcommand{\natPosInfty}{\ensuremath{\mathbb{N}^\infty_{0}}}

\newcommand{\varAt}{\ensuremath{X}}
\newcommand{\varSub}{\ensuremath{\Phi}}
\newcommand{\varSubSub}{\ensuremath{\Psi}}
\newcommand{\varFor}{\ensuremath{A}}
\newcommand{\varLit}{\ensuremath{L}}
\newcommand{\varTV}{\ensuremath{\theta}}

\newcommand{\icont}{\ensuremath{\inc{_{\mathrm{c}}}}} %\icont
\newcommand{\ihs}{\ensuremath{\inc{_{\mathrm{h}}}}} %I_hs
\newcommand{\iforget}{\ensuremath{\inc_{\mathrm{f}}}} %I_forget
\newcommand{\isdalal}{\ensuremath{\inc_{\mathrm{d}}^{\Sigma}}} %I_dalal^sum
\newcommand{\imdalal}{\ensuremath{\inc_{\mathrm{d}}^{\max}}}   %I_dalal^max
\newcommand{\ihdalal}{\ensuremath{\inc_{\mathrm{d}}^{\text{hit}}}} %I_dalal^hit

\newcommand{\satc}{\ensuremath{S_{\mathrm{c}}}}
\newcommand{\satf}{\ensuremath{S_{\mathrm{f}}}}
\newcommand{\sath}{\ensuremath{S_{\mathrm{h}}}}
\newcommand{\satmax}{\ensuremath{S_{\mathrm{d}}^{\max}}}
\newcommand{\satsum}{\ensuremath{S_{\mathrm{d}}^{\Sigma}}}
\newcommand{\sathit}{\ensuremath{S_{\mathrm{d}}^{\mathrm{hit}}}}

\newcommand{\head}{\ensuremath{\mathsf{head}}} % head of a rule
\newcommand{\body}{\ensuremath{\mathsf{body}}} % body of a rule
\newcommand{\termASP}{\ensuremath{\mathsf{t}}}

\newcommand{\pc}{\ensuremath{P_{\mathrm{c}}}}
\newcommand{\pf}{\ensuremath{P_{\mathrm{f}}}}
\newcommand{\ph}{\ensuremath{P_{\mathrm{h}}}}
\newcommand{\pmax}{\ensuremath{P_{\mathrm{d}}^{\max}}}
\newcommand{\psum}{\ensuremath{P_{\mathrm{d}}^{\Sigma}}}
\newcommand{\phit}{\ensuremath{P_{\mathrm{d}}^{\mathrm{hit}}}}

\newcolumntype{C}[1]{>{\centering\arraybackslash}p{#1}}
\newcolumntype{L}[1]{>{\raggedright\arraybackslash}p{#1}}
\newcolumntype{R}[1]{>{\raggedleft\arraybackslash}p{#1}}

\newcommand{\defArrow}{\ensuremath{\xRightarrow{\text{Def.}}{}}}


\jairheading{82}{2025}{563-685}{07/2024}{02/2025}
\ShortHeadings{Comparison of SAT-Based and ASP-Based Algorithms for Inconsistency Measurement}
{Kuhlmann, Gessler, Laszlo \& Thimm}
\firstpageno{563}

\begin{document}

\title{Comparison of SAT-Based and ASP-Based Algorithms for Inconsistency Measurement}

\author{\name Isabelle Kuhlmann \email isabelle.kuhlmann@fernuni-hagen.de \\
        \addr University of Hagen, Germany
        \AND
       % \name Anna Gessler\thanks{Contributions performed while being affiliated with the University of Koblenz.} \email anna.gessler.work@gmail.com \\   
       \name Anna Gessler \email anna.gessler.work@gmail.com \\   
       \addr University of Koblenz, Germany
       \AND
       \name Vivien Laszlo \email vivien.laszlo@fernuni-hagen.de \\
       % \addr University of Hagen, Germany
       % \AND
       \name Matthias Thimm \email matthias.thimm@fernuni-hagen.de \\
       \addr University of Hagen, Germany}

% For research notes, remove the comment character in the line below.
% \researchnote

\maketitle

\begin{abstract}
    We present algorithms based on satisfiability problem (SAT) solving, as well as answer set programming (ASP), for solving the problem of determining inconsistency degrees in propositional knowledge bases. 
    We consider six different inconsistency measures whose respective decision problems lie on the first level of the polynomial hierarchy.
    % Namely, these are the contension inconsistency measure, the forgetting-based inconsistency measure, the hitting set inconsistency measure, the max-distance inconsistency measure, the sum-distance inconsistency measure, and the hit-distance inconsistency measure.
    Namely, these are the contension, forgetting-based, hitting set, max-distance, sum-distance, and hit-distance inconsistency measures.
    In an extensive experimental analysis, we compare the SAT-based and ASP-based approaches with each other, as well as with a set of naive baseline algorithms.
    Our results demonstrate that, overall, both the SAT-based and the ASP-based approaches clearly outperform the naive baseline methods in terms of runtime.
    The results further show that the proposed ASP-based approaches perform superior to the SAT-based ones with regard to all six inconsistency measures considered in this work.
    Moreover, we conduct additional experiments to explain the aforementioned results in greater detail.
\end{abstract}

\section{Introduction}

The handling of conflicting information is a substantial problem in symbolic approaches to Artificial Intelligence.
For instance, different expert opinions could (partially) contradict each other, rule mining algorithms could yield conflicting rules, or data gathered from sensors could be noisy or otherwise distorted.
Thus, inconsistencies can occur in virtually any area of application, and require to be detected and handled.
The field of \textit{inconsistency measurement} \cite{Grant:2018,thimm2019b} provides an analytical perspective on this matter by facilitating the quantitative assessment of the severity of inconsistency in formal knowledge representation formalisms.
Representing the degree of inconsistency as a numerical value may assist automatic reasoning mechanisms on the one hand, and human modellers who aim to identify and compare multiple alternative formalizations on the other hand.
Moreover, such an analysis can be used to identify conflicts, and consequently also help to restore consistency to an inconsistent knowledge base. % while trying to keep information loss as low as possible. 
Inconsistency measures have been used to estimate reliability of agents in multi-agent systems \cite{Cholvy:2017}, to analyze inconsistencies in news reports \cite{Hunter:2006d}, to support collaborative software requirements specifications \cite{Martinez:2004a}, to allow for inconsistency-tolerant reasoning in probabilistic logic \cite{Potyka:2017}, to handle inconsistencies in business processes \cite{Corea:2021,Corea:2022}, and to monitor and maintain quality in database settings \cite{Decker:2017,Bertossi:2018}.
For a general overview of the subject, see the seminal work by Grant \citeyear{Grant:1978} and the edited collection by Grant \& Martinez \citeyear{Grant:2018}. 

In the literature, a multitude of different inconsistency measures have been introduced. 
Some approaches conceptually rely on minimal inconsistent sets or maximal consistent sets (see, e.\,g., \cite{hunter2008measuring,jabbour2014inconsistency,ammoura2015measuring}), others rely on non-classical semantics (see, e.\,g., \cite{grant2011measuring,ma2009,knight2002measuring}), and yet others utilize further properties (see, e.\,g., \cite{thimm2016}).
However---and despite the above list of applications and consequent need for practical working solutions---algorithmic approaches to inconsistency measurement have received only little attention so far.
Ma et al.\ \citeyear{ma2009} propose an algorithm that approximates the inconsistency value of a newly proposed inconsistency measure and evaluate it with respect to computational complexity. 
Likewise, Xiao~\& Ma \citeyear{xiao2012} present two new inconsistency measures and perform a complexity analysis on their decision problems, both of which are found to be on the second level of the polynomial hierarchy.
They also develop and evaluate a practically feasible anytime algorithm. 
\mbox{McAreavey et al.}\ \citeyear{mcAreavey2014} note that there is a lack of practical implementations for inconsistency measures that employ minimal unsatisfiable subsets, and develop and evaluate an algorithm for enumerating such subsets. 
Jabbour \& Sais \citeyear{Jabbour2016} describe two algorithms for their newly introduced inconsistency measure, but do not evaluate them with regard to performance or complexity. 
Thimm \citeyear{thimm2016} designs and evaluates evolutionary algorithms for two inconsistency measures. 
Bertossi \citeyear{Bertossi2018} proposes an inconsistency measure for databases that can be computed using answer set programming and analyzes its complexity. 

As the above overview shows, most algorithmic studies of inconsistency measurement focus on individual inconsistency measures and have failed to address systematic comparisons of algorithms and complexities. 
In response, one of the contributions of the survey by Thimm \& Wallner \citeyear{thimm2019a} was to determine the complexity levels of a large number of inconsistency measures. 
The authors concluded that the problem of inconsistency measurement is hard in general, but that there are certain measures which are more suitable candidates for practical applications due to their complexity class.
Based on these findings, for a selection of three different inconsistency measures whose corresponding decisions problems---i.\,e., deciding whether a certain value is an upper or lower bound of the inconsistency value, or whether it corresponds exactly to the inconsistency value---were found to be on the first level of the polynomial hierarchy, a set of algorithms based on Answer Set Programming (ASP) has been introduced in \cite{kuhlmann2020algorithm} and \cite{kuhlmann2021algorithms}. 
Namely, these measures are the \textit{contension} inconsistency measure \cite{grant2011measuring}, the \textit{forgetting-based} inconsistency measure \cite{besnard2016forget}, and the \textit{hitting set} inconsistency measures \cite{thimm2016} (for their formal definitions, see Section~\ref{sec:preliminaries}).
Those three measures were not only selected because of their associated complexity class, but also because they each give a different perspective on the inconsistencies in a given knowledge base, meaning that they could each provide different information that could be used to subsequently restore consistency.
The contension inconsistency measure tells us which propositions are involved in a conflict, while the forgetting-based inconsistency measure tells us which occurrences of each proposition are involved in a conflict, and can therefore also point to specific formulas. 
In addition, the forgetting-based measure is, to the best of our knowledge, the only inconsistency measure in the literature that is based on the notion of atom occurrences.
The contension inconsistency measure serves as a representative example of inconsistency measures based on non-classical semantics (see, e.\,g.,~\cite{ma2009,knight2002measuring} for other examples of such measures).
% With regard to inconsistency measures based on non-classical logics, such as the contension measure, there exist multiple approaches (see, e.\,g.,~\cite{ma2009,knight2002measuring})
% - außerdem ist contension ein beispielhafter Vertreter für Maße, die auf non-classical semantics basieren
%    -> generell sind paraconsistent logics eine wichtige Perspektive, wenn es um den Umgang mit Inkonsistenzen geht
The hitting set inconsistency measure offers a further perspective by considering how many different interpretations are minimally needed to (individually) satisfy all formulas.
This measure was originally designed for streaming-based applications, such as Linked Open Data.
With regard to such applications, the development of efficient algorithms is of great interest. 

In \cite{kuhlmann2021algorithms}, the three ASP-based approaches were implemented and compared to naive baseline implementations in an experimental evaluation. 
As anticipated, the result of the study was that the ASP-based implementations were clearly superior.
Furthermore, in \cite{kuhlmann2022comparison}, a revised version of the ASP-based approach for the contension inconsistency measure is proposed, in addition to an approach based on \textit{satisfiability problem} (SAT) solving.
The latter is widely used in applications such as hardware verification \cite{biere1999symbolic,vizel2015}, electronic design automation \cite{marques2000boolean}, or cryptanalysis \cite{mironov2006applications,nejati2020}.
This, in addition to the fact that there exist highly optimized SAT solvers (see the results of the annual SAT competition\footnote{\url{http://www.satcompetition.org/}} for an overview), makes it a natural approach for computing inconsistency measures on the first level of the polynomial hierarchy.
Moreover, SAT and ASP have been compared wrt.\ other applications before (see, e.\,g., \cite{banbara2015aspartame,eyupoglu2021stable}).
The results of the study show that both the ASP and the SAT approach clearly outperform the naive baseline method, but ultimately the ASP approach performs superior to the SAT approach.

In the work at hand, we follow up on \cite{kuhlmann2022comparison} by greatly extending the scope of the considered measures and the depth of the experimental evaluation. 
For that, we revisit the proposed SAT-based and ASP-based approaches for the contension inconsistency measure and we present SAT-based approaches for the forgetting-based and the hitting set inconsistency measure, as well as revised versions of the corresponding ASP-based approaches introduced in \cite{kuhlmann2021algorithms}.
Moreover, we propose both a SAT-based and an ASP-based approach for each of three different variations of the distance-based inconsistency measure \cite{grant2017}, which are likewise on the first level of the polynomial hierarchy.
Further, the distance-based approach offers yet another perspective on the notion of an inconsistency than the three previously discussed measures---Grant~\& Hunter \citeyear{grant2017} view the models of the formulas in a knowledge base as points in Euclidian space. 
The authors also point out how these measures can be used in applications such as the evaluation of violations of integrity constraints in databases.
% useful in applications such as evaluating violations of integrity constraints in databases and for deciding how to act on inconsistency.
% Namely,
We consider the \textit{max-distance}, \textit{sum-distance}, and \textit{hit-distance} inconsistency measures. 
Their formal definitions follow, along with those of the other measures, in Section \ref{sec:preliminaries}.

Hence, we examine a total of six different inconsistency measures in this work (the contension, the forgetting-based, the hitting set, the max-distance, the sum-distance, and the hit-distance inconsistency measure), and present one SAT-based and one ASP-based approach for each of them. 
With regard to each measure, we compare the SAT and the ASP approach to each other, as well as to a naive baseline method in an experimental evaluation.
More precisely, we conduct our experiments on a total of five different data sets.
Whilst two of them were used in the literature before \cite{kuhlmann2021algorithms,kuhlmann2022comparison}, the three other data sets are newly introduced and made publicly available.
Further, we investigate the runtime composition of the SAT-based and the ASP-based approaches in more detail.
In addition to the runtimes, we also record the actual inconsistency values resulting from the different inconsistency measures on the various data sets.
Based on that, we conduct an experiment in which we compare different search strategies wrt.\ the SAT-based methods.
Overall, the results of our experimental analysis confirm that both the SAT-based and the ASP-based approaches perform superior to the baseline algorithms.
Nevertheless, the results also show that the ASP-based approaches altogether outperform the SAT-based ones.
% This reflects the finding from \cite{kuhlmann2022comparison} that the ASP method for the contension inconsistency measure performs stronger than the corresponding SAT method, and extends this result to the other five measures at hand.
% This extends the result from \cite{kuhlmann2022comparison} that the ASP method for the contension inconsistency measure performs stronger than the corresponding SAT method to the other five measures at hand.

In the following, we give a concise overview of the structure of this paper.
We first provide the required preliminaries on inconsistency measurement in Section \ref{sec:preliminaries}.
% In particular, we cover the fundamentals of propositional logic and inconsistency measurement, and we define the different inconsistency measures considered in this work. 
Sections \ref{sec:sat} and \ref{sec:asp} comprise detailed descriptions of our SAT-based and, respectively, ASP-based algorithms.
% The correctness proofs for each of these approaches are provided in \ref{app:proofs}.
Section \ref{sec:evaluation} encompasses an evaluation, in which we compare the SAT-based and the ASP-based approaches with each other, and additionally, we draw a comparison with the naive baseline methods. 
% In \ref{app:data}, we provide additional visualizations of our runtime results, as well as numbers on cumulative runtimes and timeouts, and histograms over the measured inconsistency values for each measure and data set. 
Section \ref{sec:conclusion} concludes this work by providing a brief summary of our results, as well as an overview of possible future work.

\section{Preliminaries}\label{sec:preliminaries}

The inconsistency measures examined in this work are all designed to be applied in propositional logic knowledge bases.
A knowledge base $\kb$ is a finite set of propositional formulas, and we define $\mathbb{K}$ as the set of all propositional knowledge bases.
Formulas are constructed by means of the usual connectives \textit{negation} ($\lnot$), \textit{disjunction} ($\lor$), and \textit{conjunction} ($\land$).
Some algorithms defined in Section $\ref{sec:sat}$ and $\ref{sec:asp}$ use the notion of \textit{subformulas}.
The set of subformulas of a formula $\varSub$ is denoted by $\mathsf{sub}(\varSub)$ and is inductively defined in the following manner.

\begin{definition} % [Subformulas of a formula]
Let $\varSub$ be a propositional formula. 
If $\varSub$ is a proposition $\varAt$, the only \emph{subformula} of $\varSub$ is the proposition itself, meaning $\mathsf{sub}(\varAt)=\{\varAt\}$. 
% The same holds for $\varSub=\top$ and $\varSub=\bot$, with  $\mathsf{sub}(\top)=\{\top\}$ and $\mathsf{sub}(\bot)=\{\bot\}$, respectively. 
If $\varSub$ is a negation $\neg \varSubSub$, the subformulas $\mathsf{sub}(\neg \varSubSub)$ are given by $\{\neg \varSubSub \}\cup \mathsf{sub}(\varSubSub)$. 
The subformulas $\mathsf{sub}(\varSubSub_1 \bowtie \varSubSub_2)$ for a formula $\varSubSub_1 \bowtie \varSubSub_2$, where $\bowtie\ \in \{\wedge,\vee\}$ is a binary operator, are $\{\varSubSub_1 \bowtie \varSubSub_2\} \cup \mathsf{sub}(\varSubSub_1) \cup \mathsf{sub}(\varSubSub_2)$.
\end{definition}
Analogously, we define the set of subformulas $\mathsf{sub}(\kb)$ of a knowledge base $\kb$:
    $$\mathsf{sub(\kb)} = \bigcup_{\varFor \in \kb} \; \mathsf{sub}(\varFor)$$
% \begin{definition} % [Subformulas of a knowledge base]
%  For a knowledge base $\kb$, we define the set of \emph{subformulas} as
%  $$\mathsf{sub(\kb)} = \bigcup_{\varFor \in \kb} \; \mathsf{sub}(\varFor).$$
% \end{definition}
Observe that we denote arbitrary formulas and subformulas as $\varSub$, and formulas that are explicitly elements of a knowledge base as $\varFor$.
If a finer granularity is required, we may use the notation $\varSubSub$ for the subformula of $\varSub$.

We denote the \textit{signature} of a propositional formula or knowledge base, i.\,e., the (propositional) atoms appearing in it, as $\atoms(\cdot)$.
% Was ist Satisfiability von Formeln / KBs? 
% In propositional logic, f
Formulas can evaluate to either \textit{true} (abbreviated as $t$) or \textit{false} (abbreviated as $f$). 
% The semantics of a propositional language $\mathcal{L}(\atoms)$ is given by interpretations.
% 
% \begin{definition} % [Propositional Interpretation] 
	An \emph{interpretation} is a function $\omega: \atoms \rightarrow \{t, f\}$ that assigns truth values to all atoms. 
	An interpretation $\omega$ satisfies an atom $\varAt \in \atoms$ if and only if $\omega(\varAt) = t$, represented by $\omega \models \varAt$. 
	For non-atomic formulas, the satisfaction relation is extended recursively according to the truth-valued functions of the connectives as usual (see \cite{vanHarmelen2008}). 
% \end{definition}

An interpretation that satisfies a formula $\varSub$ is also called a \textit{model} of $\varSub$. 
For every interpretation $\omega$, if $\omega$ is a model of a formula $\varSub_1$ if and only if $\omega$ is a model of another formula $\varSub_2$, then $\varSub_1$ and $\varSub_2$ are called \textit{logically equivalent}. 
A knowledge base is satisfied by an interpretation if all of its formulas are satisfied.
Throughout this paper we denote the set of models of a knowledge base $\kb$ by $\mathsf{Mod}(\kb)$ and the set of interpretations wrt.\ the signature $\atoms(\kb)$ by $\Omega(\atoms(\kb))$.

A formula $\varSub$ is \textit{inconsistent} if there is no interpretation that satisfies it, meaning $\mathsf{Mod}(\varSub)=\emptyset$. 
By extension, a knowledge base $\kb$ is inconsistent if $\mathsf{Mod}(\kb)=\emptyset$. 
Let $\mathbb{R}_{\geq 0}^{\infty}$ be the set containing all non-negative real numbers and $\infty$.

\begin{definition}
An \emph{inconsistency measure} is a function $\inc : \allkbs \to \realPos$ which satisfies $\inc(\kb) = 0$ iff $\kb$ is consistent, for all $\kb \in \allkbs$.
\end{definition}

Many different inconsistency measures and properties that characterize these measures have been proposed. To illustrate why there are many ways to define the severity of inconsistency, consider the following example.
\begin{example} % [Inconsistent knowledge bases]
	Let $\atoms = \{ \textit{sunny}, \textit{cloudy} \}$, where $\textit{sunny}$ represents ``it is sunny'' and $\textit{cloudy}$ represents ``it is cloudy''. The following knowledge bases are inconsistent: 
	\begin{enumerate}
	    \item $\kb_1= \{ \neg\textit{cloudy},\textit{cloudy},\textit{sunny} \}$ 
	    \item $\kb_2= \{ \neg\textit{cloudy} \vee \neg \textit{sunny}, \textit{cloudy} \vee \textit{sunny}, \textit{sunny} \leftrightarrow \textit{cloudy} \}$ 
	\end{enumerate}
\end{example}

If our aim was to decide which of these knowledge bases is more severely inconsistent, there are different aspects to consider. 
The first knowledge base has a more obvious and easily fixable conflict (remove either $\neg\textit{cloudy}$ or $\textit{cloudy}$), but it also includes a formula that is not involved in any conflict ($\textit{sunny}$), meaning it contains non-zero information even if the conflict cannot be repaired. 
The second knowledge base has a more hidden conflict (all three formulas are required to produce the inconsistency); nevertheless, since all of its formulas are involved in the conflict, there is no ``safe'' formula that does not participate in any conflict, and the whole knowledge base needs to be discarded if the conflict cannot be repaired.

% The subsequent sections define the six inconsistency measures which are considered in this paper.

\subsection{The Contension Inconsistency Measure}
\label{sec:inc-contension}
The contension inconsistency measure $\icont$ \cite{grant2011measuring} is based on Priest's three-valued logic \cite{priest1979logic}.
The latter extends the truth values of propositional logic---\textit{true} ($t$) and \textit{false} ($f$)---by a third value \textit{both} ($b$), also referred to as the inconsistent or paradoxical truth value. 
The semantics for this three-valued logic are specified in Table~\ref{tab:3VL}.
%
% \mt{Formally define three-valued interpretations with the signature.}
% \ik{Done}
Further, a \textit{three-valued interpretation} $\omega^3 : \atoms(\kb) \to \{t,f,b\}$ assigns one of the three truth values to each atom in the signature $\atoms(\kb)$ of the knowledge base $\kb$.
Further, a \textit{three-valued model} of a knowledge base is an interpretation $\omega^3$ that does not assign $f$ to any formula. 
Therefore, a formula is satisfiable if it evaluates to either $t$ or $b$ in some interpretation. 
The set of three-valued models of a knowledge base $\kb$ is denoted by $\mathsf{Mod}^3(\kb)$. 
For a given three-valued interpretation $\omega^3$, we denote the set of atoms that are assigned $b$ as the $\mathsf{Conflictbase}(\omega^3)$. 
% We can now define the \textit{contension inconsistency measure} wrt.\ a given knowledge base $\kb$.

\begin{table}
	\begin{center}
	\setlength\extrarowheight{1pt}
		%\begin{tabular}{ | c | c V{5.5} c | c | c | }
		\begin{tabular}{ | c | c || c | c | c | }
			\hline
			$x$ & $y$ & $\neg x$ & $x \vee y$ & $x \wedge y$  \\ \hline
			$t$ & $t$ & $f$ 	 & $t$ 		  & $t$  \\ 
			\rowcolor{Gray}
			$t$ & $b$ & $f$ 	 & $t$ 		  & $b$  \\ 
			$t$ & $f$ & $f$ 	 & $t$ 		  & $f$  \\ 
			\rowcolor{Gray}
			$b$ & $t$ & $b$ 	 & $t$ 		  & $b$  \\ 
			\rowcolor{Gray}
			$b$ & $b$ & $b$ 	 & $b$ 		  & $b$  \\ 
			\rowcolor{Gray}
			$b$ & $f$ & $b$ 	 & $b$ 		  & $f$  \\ 
			$f$ & $t$ & $t$ 	 & $t$ 		  & $f$  \\ 
			\rowcolor{Gray}
			$f$ & $b$ & $t$ 	 & $b$ 		  & $f$  \\ 
			$f$ & $f$ & $t$ 	 & $f$ 		  & $f$  \\ 
			\hline
		\end{tabular}
		\caption[Truth table for Priest's three-valued logic]{Truth table for Priest's three-valued logic. The rows that differ from classical propositional logic are marked in gray.}
		\label{tab:3VL}
	\end{center}
\end{table}

\begin{example} The knowledge base $\kb_3=\{ x, \neg x \}$ is unsatisfiable in propositional logic but satisfiable in Priest's three-valued logic. 
The interpretation $\omega^3_1$ on $x$ with $\omega^3_1(x)=b$ is a three-valued model for $\kb_3$ with $\mathsf{Conflictbase}(\omega^3_1) = \{x\}$. 
\end{example}

\begin{definition}[Grant \& Hunter \citeyear{grant2011measuring}] % [Contension Inconsistency Measure]
    The \emph{contension inconsistency measure} $\icont(\kb):$ $\mathbb{K} \to \mathbb{N}_0$ is defined as
	$$\icont(\kb)=\min\{|\mathsf{Conflictbase}(\omega^3)| \mid \omega^3 \in \mathsf{Mod}^3(\kb) \}.$$ 
\end{definition}
In other words, the contension inconsistency value is the minimum number of atoms in a knowledge base that need to be assigned $b$ to produce a three-valued model of the knowledge base. 
The maximum inconsistency value is $\abs{\atoms(\kb)}$, because there can only be as many $b$ assignments as there are atoms in the signature.

\begin{example}
	Let $\kb_4=\{x \wedge y, \neg y\}$. 
	The interpretation $\omega^3_2: \{x, y\} \to \{t, b, f\}$ with $\omega^3_2(x)=t$ and $\omega^3_2(y)=b$ is a three-valued model for $\kb_4$ with $\mathsf{Conflictbase}(\omega_2^3) = \{y\}$.
	There is no model that assigns $b$ to $0$ atoms, therefore $\min \{|\mathsf{Conflictbase}(\omega_2^3)|\} = |\{y\}| = 1$, and hence $\icont(\kb_4)=1$.
\end{example}

\subsection{The Forgetting-Based Inconsistency Measure}
\label{subsubsec:iforget_def}

In order to define the forgetting-based inconsistency measure, we first require the definition of \textit{atom occurrences} in formulas.
Let $\varSub$ be a propositional logic formula and $\atoms(\varSub)$ its signature. 
An atom $\varAt \in \atoms(\varSub)$ can appear multiple times in $\varSub$. 
Each of these appearances is called an \emph{occurrence} of $\varAt$. 
\begin{definition}\label{def:num-occs}
    Let $\mathsf{\#occ}(\varAt,\varSub)$ denote the \emph{number of times $\varAt$ occurs in $\varSub$}.
\end{definition}
\begin{example}\label{ex:num-occs}
    Consider $\phi_1 = x \land y \lor \lnot x \land z$ with $\atoms(\phi_1) = \{x,y,z\}$.
    We can see that $x$ appears twice in $\phi_1$, and $y$ and $z$ each appear once.
    Consequently, $\mathsf{\#occ}(x,\phi_1) = 2$ and $\mathsf{\#occ}(y,\phi_1) = \mathsf{\#occ}(z,\phi_1) = 1$.
\end{example}
We extend Definition \ref{def:num-occs} to knowledge bases such that $\mathsf{\#occ}(\varAt,\kb)$ represents the number of times $\varAt$ occurs in the formulas of the knowledge base $\kb$:
$$ \mathsf{\#occ}(\varAt, \kb) = \sum_{\varFor \in \kb} \mathsf{\#occ}(\varAt, \varFor) $$

We use superscripts to distinguish between different occurrences of the same atom. 
We assume that $\varAt^l$ refers to the $l$-th occurrence of $\varAt$ in the considered formula $\varSub$, or the considered knowledge base $\kb$, with $l \in \{1, \ldots, \mathsf{\#occ}(\varAt, \varSub)\}$, or $l \in \{1, \ldots, \mathsf{\#occ}(\varAt, \kb)\}$, respectively. 
We also refer to $l$ as the \textit{label} of $\varAt$.
% Note that distinct occurrences of the same atom, e.\,g., $\varAt^1$ and $\varAt^2$, share the same truth value for any interpretation $\omega$. 
\begin{definition}
    We define the set of all \emph{atom occurrences in a formula $\varSub$} as
    $$ \occs(\varSub) = \{ \varAt_i^l \mid \varAt_i \in \atoms(\varSub), i \in \{1, \ldots, \abs{\atoms(\varSub)}\}, l \in \{1, \ldots, \mathsf{\#occ}(\varAt_i, \varSub) \} \} $$
\end{definition}
\begin{example}
    Consider again the formula $\phi_1 = x \land y \lor \lnot x \land z$ from Example \ref{ex:num-occs}.
    We assign each atom $\varAt$ a label  $l \in \{ 1, \ldots, \mathsf{\#occ}(\varAt, \phi_1) \}$ and get 
    % $$ \phi_1^l = x^1 \land y^1 \lor \lnot x^2 \land z^1 $$
    % $\occs(\phi_1)$ contains all (labeled) atom occurrences in $\phi_1$, i.e., 
    % $$ \occs(\phi_1) = \{ x^1, x^2, y^1, z^1 \} $$
    $\phi_1^l = x^1 \land y^1 \lor \lnot x^2 \land z^1 $.
    $\occs(\phi_1)$ contains all (labeled) atom occurrences in $\phi_1$, i.\,e., 
    $ \occs(\phi_1) = \{ x^1, x^2, y^1, z^1 \} $.
\end{example}
We analogously extend the definition of $\occs(\cdot)$ to knowledge bases.
% Further, let $\occs(\kb)$ denote the set of all atom occurrences in a knowledge base $\kb$:
% $$ \occs(\kb) = \{ \varAt_i^l \mid \varAt_i \in \atoms(\kb), i \in \{1, \ldots, \abs{\atoms(\kb)}\}, l \in \{1, \ldots, \mathsf{\#occ}(\varAt_i, \kb) \} \} $$

The forgetting-based inconsistency measure $\iforget$ as defined by Besnard \citeyear{besnard2016forget} utilizes a special operation on knowledge bases (\textit{forgetting}) to measure the level of inconsistency. 
% The intuition for this measure is that the more effort needed to restore consistency for some inconsistent knowledge base $\kb \in \allkbs$, the more severe the inconsistency of $\kb$.
The idea behind this measure is similar to the one behind the contension inconsistency measure.
However, while the contension measure refers to the atoms in the signature of a knowledge base, the forgetting-based measure refers to individual atom occurrences. % in a knowledge base.
An atom occurrence can be ``forgotten'' by replacing it with $\top$ or $\bot$.
We denote the replacement of the $l$-th occurrence of atom $\varAt$ in formula $\varSub$ by $\top$ as $\varSub \mid \varAt^l \forgetting \top$ (analogously with $\bot$ instead of $\top$).
Formally, the forgetting operation is defined as follows. 

\begin{definition}[Besnard \citeyear{besnard2016forget}] % [Forgetting operation $\forgetfunc_{i,a}$ \cite{besnard2016forget}]
    Let $\varSub$ be a propositional logic formula containing $\mathsf{\#occ}(\varAt, \varSub)$ occurrences of $\varAt \in \atoms(\varSub)$. 
    The following substitution of the $l$-th occurrence of $\varAt$ is then called \emph{forgetting} $\varAt^l$ in $\varSub$: 
    \begin{align*}
        \varSub \mid \varAt^l \forgetting \top, \bot = (\varSub \mid \varAt^l \forgetting \top) \; \lor \; (\varSub \mid \varAt^l \forgetting \bot)
    \end{align*}
    For succinct notation we express the formula obtained by forgetting the occurrence $\varAt^l$ in $\varSub$ by writing $\forgetfunc_{l,\varAt}.\varSub$. 
\end{definition}
\begin{example}\label{ex:forgetting} % [Forgetting operation]
    Consider the formula $\phi_2 = \neg x^1 \land \neg y^1 \land (x^2 \lor y^2) \land z^1$.
    % \begin{align*} 
    %     \phi_2 & = \neg x^1 \land \neg y^1 \land (x^2 \lor y^2) \land z^1
    % \end{align*} 
    If the second occurrence of $x$ (i.\,e., $x^2$) is forgotten, we obtain:
\begin{equation*}
    \begin{split}
        \phi_2 \mid x^2 \forgetting \top, \bot = & (\phi_2 \mid x^2 \forgetting \top) \; \lor \; (\phi_2 \mid x^2 \forgetting \bot) \\ 
        = &
        (\neg x^1 \land \neg y^1 \land (\top \lor y^2) \land z^1) \lor
        (\neg x^1 \land \neg y^1 \land (\bot \lor y^2) \land z^1) \\ 
        = &
        \neg x^1 \land \neg y^1 \land z^1
    \end{split}
\end{equation*}
\end{example}
Observe that the original formula $\phi_2$ in Example \ref{ex:forgetting} is not satisfiable wrt.\ classical two-valued semantics, i.\,e. $\mathsf{Mod}(\phi_2) = \emptyset$. 
However, the formula obtained by forgetting $x^2$ in $\phi_2$, written as $\forgetfunc_{2,x}.\phi_2$, is satisfiable. 
In fact, any propositional logic formula $\varSub$ can be made consistent by forgetting a sufficient number of atom occurrences. By extension, consistency can be restored for any inconsistent knowledge base $\kb \in \allkbs$ if a sufficient number of atom occurrences are forgotten.

\begin{definition}[Besnard \citeyear{besnard2016forget}] % [Inconsistency Measure $\iforget$ \cite{besnard2016forget}]
Let $\land_\kb$ be the conjunction of all formulas $\varFor \in \kb$.
Based on the  forgetting operation, the measure $\iforget : \allkbs \to \natnumszero$ can be defined as
\begin{align*}
\iforget(\mathcal{K}) = \min\{n \mid \forgetfunc_{l_1,\varAt_1}.\forgetfunc_{l_2,\varAt_2}. \dots
 .\forgetfunc_{l_n,\varAt_n}.\land_\mathcal{K} \not\models \bot \}
\end{align*}
with $\varAt_i \in \atoms(\kb)$ and $l_i$ being the respective labels with $i \in \{1, \ldots, n\}$.
\end{definition}
In other words, the inconsistency value $\iforget(\kb)$ is the minimum number of atom occurrences which have to be forgotten in order to restore consistency in $\kb$. 
% Let $\occs(\kb)$ be the set of all atom occurrences in $\kb$. 
% Forgetting every atom occurrence in a knowledge base yields an empty knowledge base, which is trivially consistent. 
Forgetting every atom occurrence in a knowledge base yields a trivially consistent knowledge base. 
It is thus guaranteed that the inconsistency value $\iforget(\kb)$ will not exceed $|\occs(\kb)|$. 
% If $\kb$ is consistent, no occurrences have to be forgotten, and $\iforget(\kb)= 0$. 

\begin{example}
    Consider the knowledge base $\kb_5 = \{x \land y, x \lor y, z, \lnot x\}$.
    After assigning labels to all atom occurrences, we get $\kb_5^l = \{x^1 \land y^1, x^2 \lor y^2, z^1, \lnot x^3 \}$.
        % $$ \kb_5^l = \{x^1 \land y^1, x^2 \lor y^2, z^1, \lnot x^3 \}$$
    Clearly, there is a conflict between $x \land y$ and $\lnot x$.
    However, if we forget either the first or the third occurrence of $x$ (i.e., $x^1$ or $x^3$), $\kb_5$ becomes consistent.
    Since there is no possibility of rendering $\kb_5$ consistent by forgetting less than one atom occurrence, $\iforget(\kb_5) = 1$.
\end{example}

\subsection{The Hitting Set Inconsistency Measure}

% In order to be able to define the hitting set inconsistency measure, we first need to define the notion of a hitting set.
In general, a hitting set is a set that contains at least one element from each set of a given set of sets.
In this work, we use a context-specific form of a hitting set. %, which is defined as follows.
\begin{definition}
    A set $H \subseteq \Omega(\atoms(\kb))$ is a \emph{hitting set} of a knowledge base $\kb$ if for every $\varFor \in \kb$ there exists an $\omega \in H$ with $\omega \models \varFor$. 
\end{definition}
We use the above notion of a hitting set to define the hitting set inconsistency measure \cite{thimm2016}.
Let $\natPosInfty$ denote $\natnumszero \cup \{\infty\}$.
\begin{definition}[Thimm \citeyear{thimm2016}] % [Hitting set inconsistency measure]
    The \emph{hitting set inconsistency measure} $\ihs(\kb): \mathbb{K} \to \natPosInfty$ is defined as
	$$\ihs(\kb) = \min\{\,\abs{ H} \mid \text{H is a hitting set of } \kb \} - 1$$ 
	with $\min \{\emptyset\} = \infty$ $\forall \kb \in \mathbb{K} \backslash \{\emptyset\}$, and with $\ihs(\emptyset)=0$.
\end{definition} 
Intuitively, the hitting set inconsistency measure is based on the idea of searching for the minimum number of interpretations that are required to satisfy all formulas in a given knowledge base $\kb$, subtracted by $1$.
If $\kb$ is consistent, we only need one interpretation (i.\,e., a model of $\kb$). 
Thus, if we subtract $1$, we get $0$---the desired inconsistency value for consistent knowledge bases.
If there is a conflict in $\kb$, at least $2$ interpretations are needed to satisfy all formulas, i.\,e., the inconsistency value is $>0$.
The more interpretations we need to satisfy all formulas, the more severe the conflict is considered to be.
Note that $\ihs(\kb)=\infty$ if and only if one of the formulas in $\kb$ is contradictory. 
In this case there exists no hitting set as there exists no interpretation which can satisfy such a formula, and consequently, no number of interpretations can satisfy all formulas in $\kb$. 
If $\ihs(\kb) \neq \infty$ then the maximum inconsistency value is $|\kb|-1$. 
This value occurs when there is a model for every formula in $\kb$, but no interpretation is a model for more than one formula.

\begin{example}
	Consider again $\kb_4=\{x \wedge y, \neg y\}$. 
    Let $\omega_1$ and $\omega_2$ be interpretations % which are defined as follows
    with $\omega_1(x) = \omega_1(y) = t$, and $\omega_2(x) = t$ and $\omega_2(y) = f$.
    % \begin{align*}
    %     & \omega_1(x) = t \qquad \omega_1(y) = t \\
    %     & \omega_2(x) = t \qquad \omega_2(y) = f
    % \end{align*}
    % Clearly, $\omega_1$ is a model for $x \land y$, i.\,e., the first formula in $\kb_4$, and $\omega_2$ is a model for $\lnot y$, i.\,e., the second formula in $\kb_4$.
    Clearly, $\omega_1$ is a model for $x \land y$, and $\omega_2$ is a model for $\lnot y$.
	There is no single interpretation satisfying both formulas, making $\{\omega_1,\omega_2\}$ a minimal hitting set of $\kb_4$. 
	Therefore, $\ihs(\kb_4) = |\{\omega_1, \omega_2\}| - 1 =  2 - 1 = 1$.
	
	Consider another knowledge base $\kb_6=\{x \wedge \neg x, y, z\}$. 
	The formula $x \wedge \neg x$ is contradictory, consequently there is no hitting set of $\kb_6$, and $\ihs(\kb_6)=\infty$.
\end{example}

\subsection{Distance-Based Inconsistency Measures}
Grant \& Hunter \citeyear{grant2013distance} proposed several inconsistency measures based on calculating distances between interpretations. 
Different distance measures could be used, but we assume that the distance measure is always the Dalal distance $d(\omega,\omega')$ between two interpretations $\omega$ and $\omega'$. 
The Dalal distance, also known as Hamming distance, measures distances based on the number of differing digits. 
When talking about interpretations, we define the Dalal distance as the number of differing truth value assignments for atoms $\varAt \in \atoms$ between two interpretations $\omega, \omega'$.

\begin{definition} % [Dalal distance]
    The \emph{Dalal distance} $d_{dalal}$ between two interpretations $\omega, \omega'$ is defined as
	$$d_{dalal}(\omega, \omega')=\abs{\{\varAt \in \atoms \mid \omega(\varAt) \neq \omega'(\varAt)\}}$$
\end{definition} 

Additionally, we define the Dalal distance between a set of interpretations $I \subseteq \Omega(\atoms)$ and a single interpretation $\omega$ as the smallest distance between some interpretation in $I$ and the interpretation $\omega$. 
%follows.
% This means that  $d_{dalal}(I, \omega)$ is 

\begin{definition} % [Dalal distance for interpretations]
    The Dalal distance $d_{dalal}$ between a \emph{set of interpretations} $I$ and an \emph{individual} interpretation $\omega$ is defined as
	$$d_{dalal}(I, \omega)=\min_{\omega' \in I} (\omega',\omega), \text{ with } d_{dalal}(\emptyset, \omega) = \infty$$ 
\end{definition} 

% This means that  $d_{dalal}(I, \omega)$ is the smallest distance between some interpretation in $I$ and the interpretation $\omega$. 

We now define three distance-based inconsistency measures. 
The general idea behind all three measures is to find an interpretation that has an ``optimal'' distance to the models of the individual formulas in a given knowledge base $\kb$. 
The definition of ``optimal'' differs for each measure.

The \textit{max-distance} measure calculates the interpretation with the smallest maximum distance to the models of each formula in $\kb$. 
The inconsistency value is equal to the aforementioned smallest maximum distance.
In case that $\kb$ contains a contradictory formula $\varFor_{\bot}$, i.\,e., $\mathsf{Mod}(\varFor_\bot) = \emptyset$, the maximum distance will always be $\infty$, due to $d_{dalal}(\emptyset,\omega) = \infty$ for any $\omega \in \Omega(\atoms)$.
Hence, the minimal maximum distance is $\infty$, and $\imdalal(\kb) = \infty$.

\begin{definition}[Grant \& Hunter \citeyear{grant2013distance}] % [Max-distance inconsistency measure $\imdalal$]\hfill
    The \emph{max-distance inconsistency measure} $\imdalal(\kb): \allkbs \to \natPosInfty$ is defined as 
	$$\imdalal(\kb)=\min \{ \max_{\varFor \in \kb} d_{dalal}(\mathsf{Mod}(\varFor),\omega) \mid \omega \in \Omega(\atoms(\kb)) \}$$
\end{definition} 

The \textit{sum-distance} measure calculates an interpretation $\omega$ such that the sum of $\omega$'s distances to the models of each formula in $\kb$ is minimal. 
The inconsistency value is equal to the aforementioned smallest sum of distances.
Note that in the case that $\kb$ contains a contradictory formula $\varFor_{\bot}$, each sum will add up to $\infty$, and thus $\imdalal(\kb) = \infty$.
\begin{definition}[Grant \& Hunter \citeyear{grant2013distance}] % [Sum-distance inconsistency measure $\isdalal$]\hfill
    The \emph{sum-distance inconsistency measure} $\isdalal(\kb): \allkbs \to \natPosInfty$ is defined as 
	$$\isdalal(\kb)=\min \{ \Sigma_{\varFor \in \kb} d_{dalal}(\mathsf{Mod}(\varFor),\omega) \mid \omega \in \Omega(\atoms(\kb)) \}$$
\end{definition} 

The \textit{hit-distance} measure calculates an interpretation such that the number of distances greater than $0$ to the models of each formula is minimal. 
A different, simpler characterization of this measure \cite{grant2017}, shows that $\ihdalal$ is equal to the minimum number of formulas that need to be removed from a knowledge base $\kb$ in order to make it consistent.
\begin{definition}[Grant \& Hunter \citeyear{grant2017}] %[Hit-distance inconsistency measure $\ihdalal$]\hfill
    The \emph{hit-distance inconsistency measure} $\ihdalal(\kb): \allkbs \to \natnumszero$ is defined as 
	$$\ihdalal(\kb)=\min \{\abs{ \{ \varFor \in \kb \mid d_{dalal}(\mathsf{Mod}(\varFor),\omega) > 0 \}} \mid \omega \in \Omega(\atoms(\kb)) \}$$	
\end{definition} 

$\imdalal(\kb)$ and $\isdalal(\kb)$ take the value $\infty$ if and only if one of the formulas in $\kb$ is inconsistent individually. 
For the non-infinity case, the maximum value of $\imdalal(\kb)$ is  $\abs{\atoms(\kb)}$, because the distance is based on the number of differing atom assignments which cannot exceed the number of atoms. 
For the non-infinity case of $\isdalal(\kb)$, the maximum value is $\abs{\kb} \cdot \abs{\atoms(\kb)}$, because in the worst case, each distance $d_{dalal}(\mathsf{Mod}(\varFor),\omega)$ for all formulas $\varFor \in \kb$ is maximal.
$\ihdalal$ cannot take the value $\infty$, because the value of $\ihdalal$ is the number of distances greater than $0$ rather than a distance itself. 
The maximum value of $\ihdalal$ is therefore $|\kb|$.

\begin{example}
	Consider again $\kb_4=\{x \wedge y, \neg y\}$. 
	For simplified readability, we name the formulas: $x \wedge y =: \alpha_1$ and $\neg y =: \alpha_2$. 
	The possible interpretations for $\atoms(\kb_4)$ are $\Omega(\atoms(\kb_4))=\{\omega_0,\omega_x,\omega_y,\omega_{xy}\}$ with:
	\begin{align*}
	    & \omega_0(x)= \omega_0(y)=f \\ 
	    & \omega_x(x)=t, \omega_x(y)=f \\ 
	    & \omega_y(x)=f, \omega_y(y)=t \\ 
	    & \omega_{xy}(x) = \omega_{xy}(y)=t
	\end{align*}
	The models for the formulas in $\kb_4$ are $\mathsf{Mod}(\alpha_1)=\{\omega_{xy}\}$ and $\mathsf{Mod}(\alpha_2)=\{\omega_0,\omega_x\}$. 
	We now calculate $d_{dalal}$ for all formulas in $\kb_4$ and all interpretations in $\Omega(\atoms(\kb_4))$: 
	\begin{align*}
	    & d_{dalal}(\mathsf{Mod}(\alpha_1),\omega_0) =2 && \qquad d_{dalal}(\mathsf{Mod}(\alpha_2),\omega_0) =0 \\
	    & d_{dalal}(\mathsf{Mod}(\alpha_1),\omega_x) =1 && \qquad d_{dalal}(\mathsf{Mod}(\alpha_2),\omega_x) =0  \\
	    & d_{dalal}(\mathsf{Mod}(\alpha_1),\omega_y) =1 && \qquad d_{dalal}(\mathsf{Mod}(\alpha_2),\omega_y) =1  \\
	    & d_{dalal}(\mathsf{Mod}(\alpha_1),\omega_{xy}) =0 && \qquad d_{dalal}(\mathsf{Mod}(\alpha_2),\omega_{xy}) =1 
	\end{align*}
     The max-distance measure $\imdalal$ looks at the maximum distances per interpretation, i.\,e., $\max_{\alpha \in \kb} d_{dalal}(\mathsf{Mod}(\alpha),\omega)$. 
     Here, the set of maximal distances consists of:
     \begin{align*}
         & d_{dalal}(\mathsf{Mod}(\alpha_1),\omega_0)=2 \\
         & d_{dalal}(\mathsf{Mod}(\alpha_1),\omega_x)=1 \\
         & d_{dalal}(\mathsf{Mod}(\alpha_1),\omega_y)=1 \\
         & d_{dalal}(\mathsf{Mod}(\alpha_2),\omega_{xy})=1
     \end{align*}
     The value $\imdalal(\kb_4)$ is the minimum of those maxima, meaning  
     $$d_{dalal}(\mathsf{Mod}(\alpha_1),\omega_x)=1,$$ 
     or one of the other maxima with value $1$. 
     Therefore, $\imdalal(\kb_4)=1$. 
	
	For the sum-distance inconsistency measure $\isdalal$, we sum up the distances for each interpretation: 
	\begin{align*}
	    & \Sigma_{\alpha \in \kb_4} d_{dalal}(\mathsf{Mod}(\alpha),\omega_0)=2 \\
    	& \Sigma_{\alpha \in \kb_4} d_{dalal}(\mathsf{Mod}(\alpha),\omega_x)=1 \\
    	& \Sigma_{\alpha \in \kb_4} d_{dalal}(\mathsf{Mod}(\alpha),\omega_y)=2 \\
    	& \Sigma_{\alpha \in \kb_4} d_{dalal}(\mathsf{Mod}(\alpha),\omega_{xy})=1. 
	\end{align*}
	The value of $\isdalal(\kb_4)$ is the minimum of the above sums.
	Consequently, $\isdalal(\kb_4)=1$. 
	
	For the hit-distance inconsistency measure $\ihdalal$, we count the number of distances greater than 0 for each interpretation: 
	\begin{align*}
	    & \abs{ \{ \alpha \in \kb_4 \mid d_{dalal}(\mathsf{Mod}(\alpha),\omega_0) > 0 \}}=1 \\
    	& \abs{ \{ \alpha \in \kb_4 \mid d_{dalal}(\mathsf{Mod}(\alpha),\omega_x) > 0 \}}=1 \\
    	& \abs{ \{ \alpha \in \kb_4 \mid d_{dalal}(\mathsf{Mod}(\alpha),\omega_y) > 0 \}}=2 \\
    	& \abs{ \{ \alpha \in \kb_4 \mid d_{dalal}(\mathsf{Mod}(\alpha),\omega_{xy}) > 0 \}}=1
	\end{align*}
	The inconsistency value is the minimum of those counts: $\ihdalal(\kb_4)=1$.
\end{example}

\section{SAT-Based Algorithms for Selected Inconsistency Measures}\label{sec:sat}

By now we have established a fundamental overview of inconsistency measurement in general, and we defined the six particular inconsistency measures considered in this work.
Building on that, in the section at hand, we describe how to use SAT encodings to determine inconsistency values.
% To achieve that, we first provide some essential definitions associated with SAT solving in Section \ref{sec:satisfiability-solving} and describe a general scheme for using SAT solving combined with a search procedure in order to identify inconsistency values in Section \ref{sec:sat-scheme}.
% % Moreover, Section \ref{sec:sat-cardinality} covers the use of cardinality constraints in the context of this paper.
% Finally, in Section \ref{sec:sat-encoding-ic} through Section \ref{sec:sat-encoding-ihdalal} we describe the specific SAT encodings constructed for each of the six measures.

\subsection{Satisfiability Solving}\label{sec:satisfiability-solving}

A major problem in the realm of propositional logic is the \textit{Boolean satisfiability problem}, which is one of the most-studied problems of computer science, and which is $\np$-complete \cite{biere2009}.
%
% Die Input-KBs für IM sind nicht eingeschränkt, was die connectives angeht.
% -> Der Input für das Boolean sat problem ist jedoch eine Formel in CNF
%   -> 1 Satz: Was ist CNG
%   -> jede logische Formel kann in CNF umgewandelt werden (s.u.)
%   -> eine KB kann als Konjunktion ihrer einzelnen Formeln verstanden werden
The input of the Boolean satisfiability problem is a formula in \textit{conjunctive normal form} (CNF), which is a conjunction of clauses.
Note that we do not restrict our input knowledge bases for inconsistency measurement to CNF, however, every propositional formula can be converted into an equivalent formula in CNF (see \cite{vanHarmelen2008} for a proof). 
Naively, this conversion can be done using Boolean transformation rules, but the resulting formulas are sometimes exponentially larger than the original formula. 
Using the Tseitin method \cite{tseitin1968}, every formula can be converted into an equisatisfiable CNF formula with only a linear increase in size. 

\begin{definition} % [Boolean satisfiability problem]
	The \emph{Boolean satisfiability problem} (SAT) is the problem of deciding if there exists an interpretation that satisfies a given propositional formula. 
	\begin{itemize}
	    \item \textbf{Input:} formula $\varSub$ in CNF 
	    \item \textbf{Output:} \emph{true} iff $\mathsf{Mod}(\varSub) \neq \emptyset$, \emph{false} otherwise
	\end{itemize}
\end{definition}

A SAT solver is a program that solves SAT for a given formula.
There exist many high-performance SAT solvers (for an overview, see the results of the recurring SAT competition\footnote{\url{http://www.satcompetition.org/}}). 
A consequence of the $\np$-completeness of SAT is that SAT solvers can also be used to solve other $\np$-complete problems if we transform them into SAT problems. 
% The process of transforming a problem into a SAT problem is commonly referred to as ``finding a SAT encoding''. 
% Our algorithmic approach presented in this section is based on this idea.

% To represent $\upperi$, we need a way to encode the concept that some value is an upper limit.
In order to use SAT solving for the problem at hand, we need a way to encode the concept that some value is an upper limit.
\textit{Cardinality constraints} represent that at least, at most, or exactly some number $k$ out of a set of propositional atoms are allowed to be true. 
Using the formal definition of Abio et al.\ \citeyear{abio2013}, we define that a \textit{cardinality constraint} is a constraint of the form 
$$\varAt_1 + \ldots + \varAt_n \bowtie k$$
where $\varAt_1, \ldots, \varAt_n$ are propositional atoms, $\bowtie\ \in \{<, \leq, =, \geq, >\}$ is an operator, and $k$ is a natural number.
The meaning of the $+$ operator in this context is that for every true atom the number $1$ is added and for every false atom the number $0$ is added, thereby counting the number of true atoms. 
Constraints of the form  
$\varAt_1 + \ldots + \varAt_n \leq k$
are informally referred to as \textit{at-most-$k$ constraints}. 
Such a constraint is true if and only if\ $\leq k$ atoms out of the set $\{\varAt_1, \ldots, \varAt_n\}$ are true.
The direct approach to encode cardinality constraints is to enumerate all possible atom assignments that satisfy the constraints. 
% For example, to represent the constraint $\varAt_1 + \ldots + \varAt_n \leq k$, we add all clauses that are disjunctions of subsets $\mathbb{X}_\bot \subseteq \{ \neg \varAt_1, \ldots, \neg \varAt_n \}$ of size $k+1$ to the encoding. 
For small knowledge bases the resulting encodings can be compact; however, in general, there are $\binom{n}{k+1}$ subsets of size $k+1$, meaning we generate $\binom{n}{k+1}$ clauses. 
% This method is sometimes referred to as the ``binomial encoding'' for this reason. 
It is inefficient for larger inputs because the size of the constraint grows exponentially with $n$. 
Several more efficient methods to encode cardinality constraints more efficiently have been developed. 
In the implementations used for our experimental evaluation (Section \ref{sec:evaluation}), we use the \textit{sequential counter encoding} method \cite{sinz2005} which only generates $n \cdot k$ clauses.

In the following sections, we use the notation $\mathrm{at\_most}(k,\mathbb{X})$ where $\mathbb{X}$ is a set of atoms $\{\varAt_1, \ldots, \varAt_n \}$ to denote % that one of the encodings for cardinality constraints is used to encode the constraint $\varAt_1 + \ldots + \varAt_n \leq k$. 
the constraint $\varAt_1 + \ldots + \varAt_n \leq k$.

\subsection{Scheme for SAT-Based Algorithms for Inconsistency Measures}\label{sec:sat-scheme}
This section proposes a binary search algorithm for the computation of the inconsistency value of a given knowledge base $\kb$ wrt.\ one of the inconsistency measures considered in this work, i.\,e., the computation of $\inc(\kb)$ with $\inc \in \{\icont, \iforget, \ihs, \imdalal, \isdalal, \ihdalal\}$.
Our approach follows similar procedures that have been proposed in the literature (see, e.g., \cite{giunchiglia2006solving}).
So the objective of our algorithm is to solve the function problem $\valuei$ as formalized in \cite{thimm2019a}.
% \begin{table}[!htb]
% \centering
\begin{center}
\begin{tabular}{lll}
$\valuei$ & Input:  & $\kb \in \mathbb{K}$ \\
& Output: & value of $\mathcal{I}(\mathcal{K})$ \\
\end{tabular}
\end{center}
% \end{table}

% First we discuss the types of values that $\inc(\kb)$ can take for each of the considered measures. 
% $\icont(\kb)$ counts the minimal number of atoms that have to be assigned the truth value $b$ in Priest's three-valued logic. 
% $\iforget(\kb)$ is defined as the minimal number of atom occurrences that have to be forgotten to make $\kb$ consistent. 
% $\ihs(\kb)$ is equal to the minimal number of interpretations such that there is a model for every formula in $\kb$.
% $\imdalal$ and $\isdalal$ are distances that count differing atom assignments between certain interpretations and $\ihdalal$ counts the number of such distances greater than $0$ between certain interpretations. 
% For $\ihs, \imdalal, \isdalal$ the inconsistency value equals $\infty$ for the edge case that a formula in $\kb$ is unsatisfiable individually. 
% For all other cases and for all cases of the other measures, it trivially follows that the inconsistency values are integers, as the numbers of atom assignments, atom occurrences and interpretations can only be integers.

% Furthermore, a
All inconsistency measures considered in this paper have a clearly defined range as established in their descriptions in Section \ref{sec:preliminaries} (see also Table \ref{tab:search-ranges-sat}). 
% A summary of the ranges of all inconsistency measures considered in this paper can be found in Table \ref{tab:search-ranges-sat}. 
This clearly defined search space prompts the use of a binary search procedure. % to search in the range of possible values. 
The range is searched by the algorithm and for each iteration of the search, a call to a SAT-solver is made, to decide which half of the range needs to be searched in the following iteration. 
The problem that is solved at each iteration of the search procedure is the decision problem $\upperi$ \cite{thimm2019a}, meaning the problem of deciding whether a given value $u$ is an upper bound of the inconsistency value of a given knowledge base.
\begin{center}
\begin{tabular}{lll}
    $\upperi$ & Input:  & $\kb \in \mathbb{K} , $ \\
    & & $u \in \natnumszero $ \\
    & Output: & \emph{true} iff $\mathcal{I}(\mathcal{K}) \leq u$ \\
\end{tabular}
\end{center}

\begin{table}
    \centering
    \setlength\extrarowheight{1pt}
    \begin{tabular}{|l|c|c|}
    \hline
    \inc   & Maximum non-$\infty$ value        & $\infty$ possible? \\ 
    \hline
    \hline
    \icont    & $\abs{\atoms(\kb)}$                 & No                    \\ \hline
    \iforget  &  $\abs{\occs(\kb)}$                  & No                    \\ \hline
    \ihs      & $\abs{\kb}-1$                          & Yes                   \\ \hline
    \imdalal  & $\abs{\atoms(\kb)}$                  & Yes                   \\ \hline
    \isdalal  & $\abs{\atoms(\kb)} \cdot \abs{\kb}$ & Yes                   \\ \hline
    \ihdalal  & $\abs{\kb}$                          & No                    \\ \hline
    \end{tabular}
    \caption{Ranges of possible values of each inconsistency measure considered in this work. The minimum value is $0$ in all cases.
    % Search ranges for SAT-based algorithms. The minimum value is $0$ for all considered inconsistency measures
    }
    \label{tab:search-ranges-sat}
\end{table} 

For all inconsistency measures presented in Section \ref{sec:preliminaries}, the problem $\upperi$ is $\np$-complete as shown in \cite{thimm2019a}, and is thus reducible to SAT.
Therefore, the idea is to find SAT encodings $S: \mathbb{K} \times \mathbb{N}_0 \to \mathbb{K}$ for $\upperi$ for all inconsistency measures, that satisfiy the requirement that $(\kb, u)$ is a positive instance of $\upperiarg{\inc}$ if and only if $S(\kb,u)$ is a positive instance of SAT.
    % \begin{align*}
    % (\kb, u) \text{ is a positive instance of } \upperiarg{\inc}\\
    % \iff S(\mathcal{K}, u) \text{ is a positive instance of SAT } 
    % \end{align*}
At each search step, we compute the encoding for the current potential upper bound $u$ and make a call to a SAT solver to check if $u$ is in fact an upper limit of the inconsistency value, until the exact value is found. 
If no value is found, $\infty$ is assumed to be the inconsistency value for the measures $\ihs$, $\imdalal$, and $\isdalal$. 
With regard to the other measures ($\icont$, $\iforget$, and $\ihdalal$), the binary search procedure is guaranteed to return a finite value after $\log_2(n+1)$ steps, with $n$ being the maximally possible finite value. % \footnote{As the search space also includes the value $0$, the total size of the search space is $n+1$.}
% The search procedure returns the inconsistency value after $\log_2(n+1)$ steps, with $n$ being the maximally possible finite value\footnote{As the search space also includes the value $0$, the total size of the search space is $n+1$.}.
% (see Algorithm \ref{alg:binsearch} for a formal description) 

% Warum SAT statt MaxSAT? 
Note that the iterative SAT-based approach described in this section is essentially a naive MaxSAT approach. 
Although the literature offers many different dedicated MaxSAT solvers (see, e.\,g.,~\cite{berg2024maxsat} for a recent overview), which are based on several different optimization algorithms, we decided to rely on classical SAT solvers for this work. %, and investigate optimization strategies later. 
In particular, the focus of this work is on the presented encodings; a detailed investigation of the various optimization algorithms would exceed the scope of this paper (note, however, that we briefly touch upon the topic of optimization wrt.\ SAT in Section~\ref{sec:maxsat} and wrt.\ ASP in Section~\ref{sec:core-guided}).

\subsection{The Contension Inconsistency Measure}\label{sec:sat-encoding-ic}

% Beginning with this section, we describe how to construct the SAT encodings that are used by Algorithm \ref{alg:binsearch} for specific inconsistency measures. 
Beginning with this section, we describe how to construct the SAT encodings that are used in the binary search procedure described in the previous section for each specific inconsistency measure. 
Note that the encoding for the contension inconsistency measure covered in this section was already proposed in \cite{kuhlmann2022comparison}.
However, for the sake of completeness, we describe it once more. 

Let $\kb$ be a propositional knowledge base and $u$ an integer representing a possible upper bound for the contension inconsistency value $\icont(\kb)$. 
We present a SAT encoding for $\upperiarg{\icont}$, denoted $\satc(\kb,u)$, which is defined by the components (SC1)--(SC17). 
% The ``Signature'' part describes the signature of the encoding. 
% The ``Constraints'' part lists the formulas that make up the encoding, using the atoms from the signature. The encoding will be elaborated in the following.

Recall that $\icont(\kb)$ is defined by the minimal number of atoms in $\atoms(\kb)$ that need to be set to $b$ in order for $\kb$ to become consistent in Priest's three-valued logic. 
To encode the three-valued logic in propositional logic, we use additional variables.
First of all, for every atom $\varAt$ in the original signature $\atoms(\kb)$, we use three new atoms  $\varAt_t$, $\varAt_b$, $\varAt_f$ (SC1) to represent the three truth values $t$, $b$, $f$. % (see (SC1) in the overview). 
We need to assure that out of each triple of new atoms that represent the three possible truth values of the original atom $\varAt$, exactly one is true. 
We represent this by adding the following formula for every $\varAt \in \atoms(\kb)$ to the encoding $\satc(\kb,u)$:
\begin{align*}
    (\varAt_t \vee \varAt_f \vee \varAt_b) \wedge (\neg \varAt_t \vee \neg \varAt_f) \wedge (\neg \varAt_t \vee \neg \varAt_b) \wedge (\neg \varAt_b \vee \neg \varAt_f) \tag{SC3}
\end{align*}
To model three-valued satisfiability, we recursively represent the evaluation of formulas in three-valued logic (see the semantics described in Section \ref{sec:inc-contension}). 
To achieve this, three variables $v_{\varSub}^t,v_{\varSub}^f,v_{\varSub}^b$ (SC2) are used for every subformula $\varSub$ in every formula $\varFor \in \kb$ to represent each of the three possible valuations of $\varSub$. 
For each of those valuation atoms we add an equivalence that defines the evaluations based on the operator of the subformula. 
To represent all possible formulas, we need to encode the operators $\land$, $\lor$, and $\lnot$.
% Formulas containing the operators $\rightarrow$ and $\leftrightarrow$ can be reduced to those operators by means of the usual transformation rules, i.\,e., $\varSub_1 \rightarrow \varSub_2$ is equal to $\neg \varSub_1 \vee \varSub_2$, and $\varSub_1 \leftrightarrow \varSub_2$ is equal to $(\neg \varSub_1 \vee \varSub_2) \wedge (\neg \varSub_2 \vee \varSub_1)$ (observe that these transformations preserve the truth values in Priest's three-valued logic).

We encode conjunctions $\varSub_c = \varSubSub_{c,1} \land \varSubSub_{c,2}$ by adding the following formulas to $\satc(\kb,u)$:
\begin{align*}
    & v_{\varSub_c}^t \leftrightarrow v_{\varSubSub_{c,1}}^t \wedge v_{\varSubSub_{c,2}}^t \tag{SC4} \\
    & v_{\varSub_c}^f \leftrightarrow v_{\varSubSub_{c,1}}^f \vee v_{\varSubSub_{c,2}}^f \tag{SC5} \\
    & v_{\varSub_c}^b \leftrightarrow (\neg v_{\varSubSub_{c,1}}^t \vee \neg v_{\varSubSub_{c,2}}^t) \wedge \neg v_{\varSubSub_{c,1}}^f \wedge \neg v_{\varSubSub_{c,2}}^f \tag{SC6}
\end{align*}
Analogously, disjunctions $\varSub_d = \varSubSub_{d,1} \lor \varSubSub_{d,2}$ are encoded as follows:
% by adding the following formulas to $\satc(\kb,u)$:
\begin{align*}
    & v_{\varSub_d}^t \leftrightarrow v_{\varSubSub_{d,1}}^t \vee v_{\varSubSub_{d,2}}^t \tag{SC7} \\
    & v_{\varSub_d}^f \leftrightarrow v_{\varSubSub_{d,1}}^f \wedge v_{\varSubSub_{d,2}}^f \tag{SC8} \\
    & v_{\varSub_d}^b \leftrightarrow (\neg v_{\varSubSub_{d,1}}^f \vee \neg v_{\varSubSub_{d,2}}^f) \wedge \neg v_{\varSubSub_{d,1}}^t \wedge \neg v_{\varSubSub_{d,2}}^t \tag{SC9}
\end{align*}
To represent a negated formula $\varSub_n = \neg \varSubSub_n$, we create new variables representing the three evaluations of $\varSub_n$, meaning $v_{\varSub_n}^t$, $v_{\varSub_n}^f$ and $v_{\varSub_n}^b$, and encode the evaluation of those cases by adding the following formulas to $\satc(\kb,u)$:
\begin{align*}
    v_{\varSub_n}^t \leftrightarrow v_{\varSubSub_n}^f \tag{SC10} \\
    v_{\varSub_n}^f \leftrightarrow v_{\varSubSub_n}^t \tag{SC11} \\
    v_{\varSub_n}^b \leftrightarrow v_{\varSubSub_n}^b \tag{SC12}
\end{align*}
Thus, (SC10) encodes that the formula $\varSub_n = \neg \varSubSub_n$ evaluates to $t$ if $\varSubSub_n$ evaluates to $f$. 
In the same fashion, $\varSub_n$ evaluates to $f$ if $\varSubSub_n$ is $t$ (SC11), and $\varSub_n$ evaluates to $b$ if $\varSubSub_n$ is also $b$ (SC12).
Further, we add the following formulas for each subformula $\varSub_a$ which represents an individual atom $\varAt$:
\begin{align*}
    v_{\varSub_a}^t \leftrightarrow \varAt_t \tag{SC13} \\
    v_{\varSub_a}^f \leftrightarrow \varAt_f \tag{SC14} \\
    v_{\varSub_a}^b \leftrightarrow \varAt_b \tag{SC15}
\end{align*}

We also add a formula to $\satc(\kb,u)$ that represents when a formula $\varFor \in \kb$ becomes true. 
This is the case when the subformula that contains the whole formula evaluates to $t$ or $b$:
\begin{align*}
v_\varFor^t \vee v_\varFor^b \tag{SC16}
\end{align*}

Finally, we add a cardinality constraint representing that at most $u$ of the $b$-atoms are allowed to be true. 
Let $\atoms_b$ be the set of  $b$-atoms of $\kb$, i.\,e., 
$ \atoms^\kb_b = \{ \varAt_b \mid \varAt \in \atoms(\kb) \} $.
We add the following cardinality constraint % which is encoded as described in Section \ref{sec:sat-scheme} 
to $\satc(\kb,u)$:
\begin{align}
\mathrm{at\_most}(u,\atoms^\kb_b) \tag{SC17}
\end{align}

$\satc(\kb,u)$ is finally comprised of the signature and the formulas presented in (SC1)--(SC17) (see Encoding \ref{tab:sat-contension} in Appendix~\ref{app:proof-sat-contension} for an overview).
The following result establishes that this encoding faithfully implements $\icont$.
The proof of the theorem below is provided in Appendix \ref{app:proof-sat-contension}.
\begin{theorem} \label{thm:sat-c}
For a given value $u$, the encoding $\satc(\kb,u)$ is satisfiable if and only if $\icont(\kb) \leq u$.
\end{theorem}

% The following example illustrates the construction of $\satc(\kb,u)$.

\begin{example}\label{ex:sat-contension}
    We illustrate the construction of $\satc(\kb,u)$ for the following knowledge base:
    $$\mathcal{K}_7 = \left\{\underbrace{\overbrace{x}^{\phi_{1,1}} \land \overbrace{y}^{\phi_{1,2}}}_{\alpha_1}, \quad \underbrace{\overbrace{x}^{\phi_{2,1}} \lor \overbrace{y}^{\phi_{2,2}}}_{\alpha_2}, \quad \underbrace{\lnot \overbrace{x}^{\phi_{3}}}_{\alpha_3} \right\}$$
    In this example, we aim to construct $\satc(\kb_7,1)$, i.\,e., a SAT encoding which returns true if and only if $1$ is an upper bound of the contension inconsistency value of $\kb_7$.
    
    As $\atoms(\kb_7) = \{x,y\}$, we require $2 \cdot 3 = 6$ new atoms (SC1):
        $$x_t, x_b, x_f, \qquad y_t, y_b, y_f$$
    We also have a total of $8$ (sub)formulas, so we need $8 \cdot 3 = 24$ corresponding atoms (SC2):
        \begin{gather*}
            v^t_{\alpha_1}, v^b_{\alpha_1}, v^f_{\alpha_1}, \qquad  v^t_{\alpha_2}, v^b_{\alpha_2}, v^f_{\alpha_2}, \qquad  v^t_{\alpha_3}, v^b_{\alpha_3}, v^f_{\alpha_3}, \\
            v^t_{\phi_{1,1}}, v^b_{\phi_{1,1}}, v^f_{\phi_{1,1}}, \qquad v^t_{\phi_{1,2}}, v^b_{\phi_{1,2}}, v^f_{\phi_{1,2}},\\
            v^t_{\phi_{2,1}}, v^b_{\phi_{2,1}}, v^f_{\phi_{2,1}}, \qquad v^t_{\phi_{2,2}}, v^b_{\phi_{2,2}}, v^f_{\phi_{2,2}},\\
            v^t_{\phi_3}, v^b_{\phi_3}, v^f_{\phi_3}
        \end{gather*}
    We add the following constraints to the encoding $\satc(\kb_7,u)$ via (SC3):
    \begin{gather*}
	(x_t \vee x_f \vee x_b) \wedge (\neg x_t \vee \neg x_f) \wedge (\neg x_t \vee \neg x_b) \wedge (\neg x_b \vee \neg x_f) \\
	(y_t \vee y_f \vee y_b) \wedge (\neg y_t \vee \neg y_f) \wedge (\neg y_t \vee \neg y_b) \wedge (\neg y_b \vee \neg y_f)	    
	\end{gather*}
	We encode the first formula $\alpha_1 = \phi_{1,1} \land \phi_{1,2}$, which is a conjunction, by adding the following formulas corresponding to (SC4)--(SC6):
    \begin{align*}
        & v_{\alpha_1}^t \leftrightarrow v_{\phi_{1,1}}^t \wedge v_{\phi_{1,2}}^t \\
        & v_{\alpha_1}^f \leftrightarrow v_{\phi_{1,1}}^f \vee v_{\phi_{1,2}}^f \\
        & v_{\alpha_1}^b \leftrightarrow (\neg v_{\phi_{1,1}}^t \vee \neg v_{\phi_{1,2}}^t) \wedge \neg v_{\phi_{1,1}}^f \wedge \neg v_{\phi_{1,2}}^f
    \end{align*}
    We encode the second formula $\alpha_2 = \phi_{2,1} \lor \phi_{2,2}$, which is a disjunction, by adding the formulas corresponding to (SC7)--(SC9):
    \begin{align*}
        & v_{\alpha_2}^t \leftrightarrow v_{\phi_{2,1}}^t \vee v_{\phi_{2,2}}^t \\
        & v_{\alpha_2}^f \leftrightarrow v_{\phi_{2,1}}^f \wedge v_{\phi_{2,2}}^f  \\
        & v_{\alpha_2}^b \leftrightarrow (\neg v_{\phi_{2,1}}^f \vee \neg v_{\phi_{2,2}}^f) \wedge \neg v_{\phi_{2,1}}^t \wedge \neg v_{\phi_{2,2}}^t 
    \end{align*}
    The last formula ($\alpha_3 = \lnot \phi_3$), which is a negation, is modeled by adding the following formulas via (SC10)--(SC12):
    \begin{align*}
        v_{\alpha_3}^t \leftrightarrow v_{\phi_3}^f \\
        v_{\alpha_3}^f \leftrightarrow v_{\phi_3}^t \\
        v_{\alpha_3}^b \leftrightarrow v_{\phi_3}^b 
    \end{align*}
    We add formulas to encode that $\phi_{1,1}$, $\phi_{1,2}$, $\phi_{2,1}$, $\phi_{2,2}$, and $\phi_3$ are subformulas consisting of individual atoms, following (SC13)--(SC15):
    \begin{align*}
        &\phi^t_{1,1} \leftrightarrow x_t \qquad && \phi^t_{1,2} \leftrightarrow y_t \qquad & \phi^t_{2,1} \leftrightarrow x_t \qquad && \phi^t_{2,2} \leftrightarrow y_t \qquad & \phi^t_3 \leftrightarrow x_t \\
        &\phi^f_{1,1} \leftrightarrow x_f \qquad && \phi^f_{1,2} \leftrightarrow y_f \qquad & \phi^f_{2,1} \leftrightarrow x_f \qquad && \phi^f_{2,2} \leftrightarrow y_f \qquad & \phi^f_3 \leftrightarrow x_f \\
        &\phi^b_{1,1} \leftrightarrow x_b \qquad && \phi^b_{1,2} \leftrightarrow y_b \qquad & \phi^b_{2,1} \leftrightarrow x_b \qquad && \phi^b_{2,2} \leftrightarrow y_b \qquad & \phi^b_3 \leftrightarrow x_b 
    \end{align*}
    For each formula in $\kb$ we additionally add a constraint wrt.\ (SC16):
    \begin{gather*}
        v_{\alpha_1}^t \vee v_{\alpha_1}^b \\ 
        v_{\alpha_2}^t \vee v_{\alpha_2}^b \\
        v_{\alpha_3}^t \vee v_{\alpha_3}^b 
    \end{gather*}
    Finally, via (SC17), we add the cardinality constraint $\mathrm{at\_most}(1,\{x_b, y_b\})$.
\end{example}

\subsection{The Forgetting-Based Inconsistency Measure}\label{sec:sat-encoding-if}

Let $\kb$ again be a propositional knowledge base and $u$ an integer representing a candidate for an upper bound of $\iforget(\kb)$. 
% Encoding \ref{tab:sat-fb} summarizes the proposed SAT encoding $\satf(\kb,u)$ of $\upperiarg{\iforget}$.
Recall that $\iforget(\kb)$ is defined as the minimal number of atom occurrences which need to be forgotten in order to resolve all inconsistencies in $\kb$. 
% In other words, $\upperiarg{\iforget}$ for the instance $(\kb, u)$ decides if $\kb$ can reach a consistent state, if at most $u$ atom occurrences are allowed to be forgotten.
% We construct an encoding $\satf(\kb,u)$ which is able to simulate the effects of forgetting atom occurrences and which restricts the number of occurrences that can be forgotten to $u$.
We construct an encoding $\satf(\kb,u)$ that represents $\upperiarg{\iforget}$ for the instance $(\kb, u)$.

To begin with, we label the given knowledge base $\kb$ to retrieve $\occs(\kb)$.
% Every labeled atom occurrence $\varAt^l \in \occs(\kb)$ is then added to the signature of our encoding $\satf(\kb,u)$ (SF1).
% Moreover, f
For each labeled atom occurrence $\varAt^l \in \occs(\kb)$ we also add two atoms $t_{\varAt,l}$ and $f_{\varAt,l}$ to the signature (SF1).
Further, for each formula $\varFor \in \kb$ a corresponding constraint $\varFor'$ is created and added to $\satf(\kb,u)$.
$\varFor'$ is an extension of $\varFor$, gained by substituting every atom occurrence $\varAt^l$ in $\varFor$ with the subformula $(t_{\varAt,l} \lor \varAt) \land \neg f_{\varAt,l}$. 
Assigning truth values to the new atoms $t_{\varAt,l}$, $f_{\varAt,l}$ models the forgetting operation on $\varAt^l$ in $\varFor$. 
More precisely, setting $t_{\varAt,l}$ to true represents forgetting $\varAt^l$ and replacing it with $\top$, while setting $f_{\varAt,l}$ to true represents forgetting $\varAt^l$ and replacing it with $\bot$.
% This is made clear by taking a closer look at the behavior of the subformula $\varAt^l$ was substituted with:
We define $\varSub_{\varAt,l}$ to be the subformula $\varAt^l$ is substituted with:
    \begin{align*}
        \varSub_{\varAt,l} := (t_{\varAt,l} \lor \varAt) \land \neg f_{\varAt,l} \tag{SF2}
    \end{align*}
\begin{itemize}
    \item If $t_{\varAt,l}$ is set to true and $f_{\varAt,l}$ is set to false, then $\varSub_{\varAt,l}$ becomes true, regardless of the truth value assigned to $\varAt$.
    This represents the case in which $\varAt^l$ has been forgotten and was replaced by $\top$.
    \item If $t_{\varAt,l}$ is set to false and $f_{\varAt,l}$ is set to true, then $\varSub_{\varAt,l}$ becomes false, regardless of the truth value assigned to $\varAt$.
    This represents the case in which $\varAt^l$ has been forgotten and was replaced by $\bot$.
    \item If both $t_{\varAt,l}$ and $f_{\varAt,l}$ are set to false, then the truth value of $\varSub_{\varAt,l}$ is equal to the truth value assigned to $\varAt$.
    This represents the case in which $\varAt^l$ has not been forgotten.
\end{itemize}
The case in which both $t_{\varAt,l}$ and $f_{\varAt,l}$ are set to true holds no meaning in the chosen representation, as this would imply that $\varAt^l$ had somehow been replaced by $\top$ and $\bot$ simultaneously. 
This invalid state needs to be avoided, which motivates the addition of constraints that prevent $t_{\varAt,l}$ and $f_{\varAt,l}$ from both being set to true at the same time:
    \begin{align*}
        \neg t_{\varAt,l} \lor \neg f_{\varAt,l}  \tag{SF3}
    \end{align*}

Finally, an at-most-$u$ constraint is added to limit the number of occurrences that can be forgotten. 
Let $\atoms_{\forget}$ denote the set of all atoms introduced by substitutions of atom occurrences in formulas $\varFor \in \kb$, i.\,e., the set of all $t_{\varAt,l}, f_{\varAt,l}$ wrt.\ $\varAt^l \in \occs(\kb)$. 
The constraint 
    \begin{align*}
        \mathrm{at\_most}(u, \atoms_{\forget}) \tag{SF4}
    \end{align*}
is added to $\satf(\kb, u)$ in order to keep the number of forgotten occurrences from exceeding $u$.
In total, $\satf(\kb, u)$ is comprised of (SF1)--(SF4) (see Encoding \ref{tab:sat-fb} for a complete overview).
The proof for the following theorem, along with Encoding~\ref{tab:sat-fb}, is provided in Appendix \ref{app:proof-sat-fb}
\begin{theorem} \label{thm:sat-fb}
    For a given value $u$, the encoding $\satf(\kb,u)$ is satisfiable if and only if $\iforget(\kb) \leq u$.
\end{theorem} 

\begin{example}\label{ex:sat-fb}
    Consider again the knowledge base $\kb_7$ from Example \ref{ex:sat-contension}:
    $$\mathcal{K}_7 = \left\{\underbrace{\overbrace{x}^{\phi_{1,1}} \land \overbrace{y}^{\phi_{1,2}}}_{\alpha_1}, \quad \underbrace{\overbrace{x}^{\phi_{2,1}} \lor \overbrace{y}^{\phi_{2,2}}}_{\alpha_2}, \quad \underbrace{\lnot \overbrace{x}^{\phi_{3}}}_{\alpha_3} \right\}$$
    In this example, we aim to construct $\satf(\kb_7,2)$, i.\,e., a SAT encoding which returns true if and only if $2$ is an upper bound of the forgetting-based inconsistency value of $\kb_7$.
    
    % To begin with, we require $|\occs(\kb_7)| = 5$ atoms which model the labeled atom occurrences (SF1):
    % $$ x^1, y^1, \qquad x^2, y^2, \qquad x^3 $$
    % Moreover, 
    To begin with, we need $|\occs(\kb_7)| \cdot 2 = 5 \cdot 2 = 10$ atoms that represent the forgetting operation wrt.\ each atom occurrence (SF1):
    $$ t_{x,1}, f_{x,1}, t_{y,1}, f_{y,1}, \qquad t_{x,2}, f_{x,2}, t_{y,2}, f_{y,2}, \qquad t_{x,3}, f_{x,3} $$ 
    We substitute each occurrence in every formula in $\kb_7$ as described by (SF2). 
    This produces the following $3$ formulas, which are all added to $\satf(\kb_7,2)$:
    \begin{align*}
        & ((t_{x,1} \lor x) \land \neg f_{x,1}) \land ((t_{y,1} \lor y) \land \neg f_{y,1}) \\
        & ((t_{x,2} \lor x) \land \neg f_{x,2}) \lor ((t_{y,2} \lor y) \land \neg f_{y,2}) \\
        & \lnot ((t_{x,3} \lor x) \land \neg f_{x,3})
    \end{align*}
    Next, we apply (SF3):
    {\allowdisplaybreaks
    \begin{align*}
        \lnot t_{x,1} \lor \lnot f_{x,1} & \qquad \lnot t_{y,1} \lor \lnot f_{y,1} \\
        \lnot t_{x,2} \lor \lnot f_{x,2} & \qquad \lnot t_{y,2} \lor \lnot f_{y,2} \\
        \lnot t_{x,3} \lor \lnot f_{x,3} &
    \end{align*}
    }
    Finally, we add the at-most-$u$ constraint (SF4) to $\satf(\kb_7,2)$:
    $$ \mathrm{at\_most}(2,\{ t_{x,1}, f_{x,1}, t_{y,1}, f_{y,1}, t_{x,2}, f_{x,2}, t_{y,2}, f_{y,2}, t_{x,3}, f_{x,3} \}) $$
\end{example}

\subsection{The Hitting Set Inconsistency Measure}\label{sec:sat-encoding-ih}
% Let $\kb$ be a propositional knowledge base and let $u$ be an integer representing a candidate for an upper bound for $\ihs(\kb)$. 
% Encoding \ref{tab:sat-hs} presents a short overview of the proposed SAT encoding $\sath(\kb,u)$ of $\upperiarg{\ihs}$. 
% The encoding will be described in more detail in the following.

In the following, we propose a SAT encoding, denoted $\sath(\kb,u)$, for the problem $\ihs(\kb)$, with $\kb$ being a knowledge base and $u$ an integer representing a candidate for an upper bound.
We reiterate that $\ihs(\kb)$ is the minimum cardinality of a set of interpretations such that there is at least one model for every formula in the knowledge base $\kb$, subtracted by $1$. 
This is equivalent to finding the minimal number of blocks of a partitioning of $\kb$ such that each block is satisfiable (again, subtracted by $1$). 
% We construct the SAT encoding $\sath(\kb,u)$ based on the idea of finding such a partitioning.
% We construct the SAT encoding $\sath(\kb,u)$ based on this idea.

The formulas in $\kb$ are partitioned into $u$ blocks, and every formula belongs to at least one block.
For every formula $\varFor \in \kb$, we create copies $\varFor_{i}$ with $i=1, \ldots, u$. 
The meaning of formula $\varFor_{i}$ is ``the $i$-th copy of formula $\varFor$''. 
Then, we create $\abs{\mathcal{K}} \cdot u$ new atoms $p_{\varFor,i}$ (SH2).
The meaning of the atom $p_{\varFor,i}$ is ``formula $\varFor$ is a member of block $i$''. 
For every $i$, all atoms $\varAt$ that appear in a formula $\varFor_{i}$ are replaced with new atoms $\varAt_i$ (SH1), meaning each block has its own copy of $\atoms(\kb)$, allowing us to check for satisfiability block-wise. 

For every block, the set of formulas inside it is satisfiable.
We add $\abs{\kb} \cdot u$ constraints of the form
\begin{align*}
    p_{\varFor,i} \rightarrow \varFor_{i}. \tag{SH3}
\end{align*}
% The constraints represented by (SH3)
This models that if a formula $\varFor$ belongs to block $i$, then the corresponding formula $\varFor_{i}$ must be satisfied (within this block).
Every formula must be sorted into at least one block, meaning for all $\varFor \in \kb$, at least one $p_{\varFor,i}$ is true. 
A simple encoding for the ``at-least-one'' constraint is to add the following clauses for every $\varFor \in \kb$:
\begin{align*}
    \bigvee_{1 \leq i \leq u}p_{\varFor,i} \tag{SH4}
\end{align*}

When a value for the minimal number of satisfiable partition blocks is found, we subtract $1$ to generate the value $\ihs(\kb)$. If no value is found in the range, the result is $\infty$. % $\ihs(\kb) = \infty$.

The following result establishes that the above encoding $\sath(\kb,u)$, defined by (SH1)--(SH4) (for an overview, see Encoding \ref{tab:sat-hs} in Appendix~\ref{app:proof-sat-hs}), faithfully implements $\ihs$.
The proof of the theorem below is given in Appendix \ref{app:proof-sat-hs} as well.
\begin{theorem} \label{thm:sat-hs}
For a given value $u$, the encoding $\sath(\kb,u)$ is satisfiable if and only if $\ihs(\kb) \leq u - 1$.
$\sath(\kb,u)$ is unsatisfiable for all $u=1,\ldots,|\kb|$ if and only if $\ihs(\kb) = \infty$.
\end{theorem}

% The following example illustrates the construction of $\sath(\kb,u)$.

\begin{example}
    Consider again the knowledge base $\kb_7$ from Example \ref{ex:sat-contension}:
    $$\mathcal{K}_7 = \left\{\underbrace{\overbrace{x}^{\phi_{1,1}} \land \overbrace{y}^{\phi_{1,2}}}_{\alpha_1}, \quad \underbrace{\overbrace{x}^{\phi_{2,1}} \lor \overbrace{y}^{\phi_{2,2}}}_{\alpha_2}, \quad \underbrace{\lnot \overbrace{x}^{\phi_{3}}}_{\alpha_3} \right\}$$
    In this example, we aim to construct $\sath(\kb_7,2)$, i.\,e., a SAT encoding which returns true if and only if $3-1 = 2$ is an upper bound of the hitting set inconsistency value of $\kb_7$.
    
    First, we create variables for each atom $x_i \in \atoms(\kb_7) = \{x,y\}$ with $i \in \{1,2\}$ (SH1):
    $$ x_1, x_2, \qquad y_1, y_2 $$
    In addition, for each formula in $\kb_7$ we create the following variables (SH2):
    $$ p_{\alpha_1, 1}, p_{\alpha_1, 2}, \qquad p_{\alpha_2, 1}, p_{\alpha_2, 2}, \qquad p_{\alpha_3, 1}, p_{\alpha_3, 2}, $$
    Now we create indexed copies of the formulas in $\kb_7$:
    \begin{align*}
        & x_1 \land y_1,  && \qquad x_1 \lor y_1, && \qquad \lnot x_1, \\
        & x_2 \land y_2,  && \qquad x_2 \lor y_2, && \qquad \lnot x_2 
    \end{align*}
    Following (SH3), we add the subsequent formulas to $\sath(\kb_7,2)$:
    \begin{align*}
        & p_{\alpha_1,1} \rightarrow (x_1 \land y_1), && \qquad p_{\alpha_2,1} \rightarrow (x_1 \lor y_1), && \qquad p_{\alpha_3,1} \rightarrow (\lnot x_1), \\
        & p_{\alpha_1,2} \rightarrow (x_2 \land y_2), && \qquad p_{\alpha_2,2} \rightarrow (x_2 \lor y_2), && \qquad p_{\alpha_3,2} \rightarrow (\lnot x_2) 
    \end{align*}
    At last, we add the following constraints via (SH4):
    \begin{align*}
        p_{\alpha_1,1} \lor p_{\alpha_1,2}, \qquad p_{\alpha_2,1} \lor p_{\alpha_2,2}, \qquad p_{\alpha_3,1} \lor p_{\alpha_3,2}
    \end{align*}
\end{example}

\subsection{The Max-Distance Inconsistency Measure}\label{sec:sat-encoding-imdalal}

Let, once again, $\kb$ be a propositional knowledge base, and let $u$ be an integer representing a potential upper bound for $\imdalal(\kb)$. 
% As in the previous sections, see Encoding \ref{tab:sat-dalal-max} for a concise overview on the proposed SAT encoding $\satmax(\kb,u)$ for the problem $\upperiarg{\imdalal}$. 
% The construction of the encoding will be elaborated in the following.
In the following, we construct a SAT encoding $\satmax(\kb,u)$ for the problem $\upperiarg{\imdalal}$.

As a reminder, the distance-based inconsistency measures look for an interpretation with an ``optimal'' distance. %; the definition of ``optimal'' depends on the measure. 
Regarding $\imdalal$, the ``optimal'' interpretation has the minimal maximum distance to the models of all formulas in $\kb$.
In case one or more formulas in $\kb$ are unsatisfiable, i.\,e., they possess no models, $\imdalal(\kb)= \infty$.

First we need a representation of the ``optimal'' interpretation that will be computed by the measure. 
For each atom $\varAt \in \atoms(\kb)$, we add a new atom $\varAt_o$ that represents this atom's value assignment in the ``optimal interpretation'' (SDM1). %in Encoding \ref{tab:sat-dalal-max}). 
Let $\atoms_O$ be the set of atoms of the optimal interpretation, i.\,e., $\atoms_O = \{X_o \mid X\in \atoms(\kb)\}$.

With the exception of the case that $\imdalal(\kb)=\infty$ (i.\,e., the case in which at least one of the formulas is contradictory), all formulas in $\kb$ must be satisfied by some interpretation. 
For every formula $\varFor \in \kb$, we create a clone $\varFor_i$ with $i\in \{1, \ldots, \abs{\mathcal{K}}\}$. 
For every $i$, all atoms $\varAt$ that appear in a formula $\varFor_{i}$ are replaced with new atoms $\varAt_i$ (SDM2). 
We add the modified clones $\varFor_i$ to the encoding (SDM4). 
This models that the formulas are all satisfiable individually and allows us to put their models and the optimal interpretation in relation in the next step.

For each formula, the distance between at least one model of the formula and the optimal interpretation consists of at most $u$ differing atom assignments. 
We now aim to model that the minimum distance between some model of each formula and the optimal world is at most $u$. 
This can also be formulated as ``for each formula, there is some model that can be converted into the optimal world by inverting at most $u$ value assignments'' (excluding the $\infty$ case). 
To represent this, we create new atoms $\mathsf{inv}_{\varAt,i}$ with $i\in \{1,\ldots,\abs{\kb}\}$ for all $\varAt_i \in \atoms(\satmax(\kb,u)) \setminus \atoms_O$ that model the inverted value assignments (SDM3). 
Then, we add the following formulas for every $\varAt_i$, using the optimal interpretation atoms $\varAt_o \in \atoms_O$:
\begin{align*}
    \varAt_i & \rightarrow \varAt_o \lor \mathsf{inv}_{\varAt,i} \tag{SDM5} \\
    \neg \varAt_i & \rightarrow \lnot \varAt_o \lor \mathsf{inv}_{\varAt,i} \tag{SDM6}
\end{align*}
The above formulas represent that for every atom $\varAt_i$ of a model of a formula, the corresponding atom in the optimal interpretation has the same value unless it is inverted by setting $\mathsf{inv}_{\varAt,i}$ to true. 
We then represent the upper bound % for distances 
by adding at-most-$u$ constraints for all $\mathsf{inv}_{\varAt,i}$. % for $i\in \{1,\ldots,\abs{\mathcal{K}}\}$. 
Let $\mathsf{INV}_i$ denote the set containing all $\mathsf{inv}_{\varAt,i}$ for a fixed $i$:
\begin{align*}
    \mathrm{at\_most}(u, \mathsf{INV}_i ) \tag{SDM7}
\end{align*}

In total, the SAT encoding $\satmax(\kb,u)$ comprises (SDM1)--(SDM7). % (see Encoding \ref{tab:sat-dalal-max} for an overview).
The following result establishes that the above encoding faithfully implements $\imdalal$.
\begin{theorem}\label{thm:sat-maxdalal}
    For a given value $u$, the encoding $\satmax(\kb,u)$ is satisfiable if and only if $\imdalal(\kb) \leq u$.
    $\satmax(\kb,u)$ is unsatisfiable for all $u \in \{ 0, \ldots, \abs{\atoms(\kb)}\}$ if and only if $\imdalal(\kb) = \infty$.
\end{theorem}
The proof of the above theorem, as well as Encoding~\ref{tab:sat-dalal-max}, which contains a complete overview of $\satmax(\kb,u)$, are provided in Appendix \ref{app:proof-sat-max-dist}.

% The following example demonstrates how to construct a SAT encoding $\satmax(\kb,u)$.

\begin{example}\label{ex:sat-max-dist}
    Consider again the knowledge base $\kb_7$ from Example \ref{ex:sat-contension}:
    $$\mathcal{K}_7 = \left\{\underbrace{\overbrace{x}^{\phi_{1,1}} \land \overbrace{y}^{\phi_{1,2}}}_{\alpha_1}, \quad \underbrace{\overbrace{x}^{\phi_{2,1}} \lor \overbrace{y}^{\phi_{2,2}}}_{\alpha_2}, \quad \underbrace{\lnot \overbrace{x}^{\phi_{3}}}_{\alpha_3} \right\}$$
    In this example, we aim to construct $\satmax(\kb_7,1)$, i.\,e., a SAT encoding which returns true if and only if $1$ is an upper bound of the max-distance inconsistency value of $\kb_7$.
    
    % At first, we create the following variables for each atom in $\atoms(\kb_7) = \{x,y\}$ (SDM1):
    % $$ x_o, y_o $$
    At first, we create the variables $ x_o, y_o $ for the atoms in $\atoms(\kb_7) = \{x,y\}$ (SDM1).
    We create additional variables for each atom in $\atoms(\kb_7)$ wrt.\ indices $i \in \{1, \ldots, |\kb_7|\}$ with $|\kb_7| = 3$ via (SDM2) and (SDM3):
    \begin{align*}
        & x_1, x_2, x_3, && \qquad y_1, y_2, y_3 \\
        & \mathsf{inv}_{x,1}, \mathsf{inv}_{x,2}, \mathsf{inv}_{x,3}, && \qquad \mathsf{inv}_{y,1}, \mathsf{inv}_{y,2}, \mathsf{inv}_{y,3}
    \end{align*}
    We add indexed copies of the formulas in $\kb_7$ to $\satmax(\kb_7,1)$ (SDM4):
    % \begin{align*}
    %     & x_1 \land y_1,  && \qquad x_1 \lor y_1, && \qquad \lnot x_1, \\
    %     & x_2 \land y_2,  && \qquad x_2 \lor y_2, && \qquad \lnot x_2, \\
    %     & x_3 \land y_3,  && \qquad x_3 \lor y_3, && \qquad \lnot x_3 
    % \end{align*}
    \begin{align*}
        & x_1 \land y_1,  && \qquad x_2 \lor y_2, && \qquad \lnot x_3
    \end{align*}
    % The next step consists of adding the following constraints conforming to (SDM5) and (SDM6) to $\satmax(\kb_7,1)$:
    We now add the following constraints conforming to (SDM5) and (SDM6) to $\satmax(\kb_7,1)$:
    {\allowdisplaybreaks
    \begin{align*}
        & x_1 \rightarrow x_o \lor \mathsf{inv}_{x,1}, && \qquad \lnot x_1 \rightarrow \lnot x_o \lor \mathsf{inv}_{x,1}, \\
        & y_1 \rightarrow y_o \lor \mathsf{inv}_{y,1}, && \qquad \lnot y_1 \rightarrow \lnot y_o \lor \mathsf{inv}_{y,1}, \\
        & x_2 \rightarrow x_o \lor \mathsf{inv}_{x,2}, && \qquad \lnot x_2 \rightarrow \lnot x_o \lor \mathsf{inv}_{x,2}, \\
        & y_2 \rightarrow y_o \lor \mathsf{inv}_{y,2}, && \qquad \lnot y_2 \rightarrow \lnot y_o \lor \mathsf{inv}_{y,2}, \\
        & x_3 \rightarrow x_o \lor \mathsf{inv}_{x,3}, && \qquad \lnot x_3 \rightarrow \lnot x_o \lor \mathsf{inv}_{x,3}, \\
        & y_3 \rightarrow y_o \lor \mathsf{inv}_{y,3}, && \qquad \lnot y_3 \rightarrow \lnot y_o \lor \mathsf{inv}_{y,3} 
    \end{align*}
    }
    Lastly, the following at-most-$u$ constraints are added:
    % \begin{align*}
    %     & \mathrm{at\_most}(1,\{\mathsf{inv}_{x,1},\mathsf{inv}_{y,1}\}) \\
	   %  & \mathrm{at\_most}(1,\{\mathsf{inv}_{x,2},\mathsf{inv}_{y,2}\}) \\
	   %  & \mathrm{at\_most}(1,\{\mathsf{inv}_{x,3},\mathsf{inv}_{y,3}\}) 
    % \end{align*}
    \begin{align*}
        \mathrm{at\_most}(1,\{\mathsf{inv}_{x,1},\mathsf{inv}_{y,1}\}), \quad
	    \mathrm{at\_most}(1,\{\mathsf{inv}_{x,2},\mathsf{inv}_{y,2}\}), \quad
	    \mathrm{at\_most}(1,\{\mathsf{inv}_{x,3},\mathsf{inv}_{y,3}\}) 
    \end{align*}
\end{example}

\subsection{The Sum-Distance Inconsistency Measure}\label{sec:sat-encoding-isdalal}

The definition of the sum-distance measure $\isdalal$ is quite similar to the definition of the previously addressed max-distance measure $\imdalal$.
The only difference is the notion of optimality---while $\imdalal$ seeks the interpretation with the smallest maximum distance to the models of the formulas in a given knowledge base $\kb$, $\isdalal$ seeks an interpretation such that the sum of its distances to the models of the formulas in $\kb$ is minimal. 

Let once again $u$ be an integer representing a candidate for an upper bound of $\isdalal(\kb)$. 
% Encoding \ref{tab:sat-dalal-sum} presents an overview of the proposed SAT encoding $\satsum(\kb,u)$ for the problem $\upperiarg{\isdalal}$. 
% The construction of the encoding will be described in the following.
In the following, we construct a SAT encoding $\satsum(\kb,u)$ for the problem $\upperiarg{\isdalal}$. 

The constraints (SDS1)--(SDS6) of the encoding $\satsum(\kb, u)$ are identical to constraints (SDM1)--(SDM6) of $\satmax(\kb, u)$. 
Only the use of the at-most-$u$ constraints is changed---we represent the upper bound for the sum of the distances by adding a single at-most-$u$ constraint for all $\mathsf{inv}_{\varAt,i}$ with $i\in \{1,\ldots,\abs{\mathcal{K}}\}$.
Let $\mathsf{INV}$ be the set containing all $\mathsf{inv}_{\varAt,i}$: % wrt.\ all $i$:
    \begin{align*}
        \mathrm{at\_most}(u, \textsf{INV}) \tag{SDS7}
    \end{align*}

Altogether, $\satmax(\kb, u)$ is comprised of (SDS1)--(SDS7) (for an overview, see Encoding~\ref{tab:sat-dalal-sum} in Appendix~\ref{app:proof-sat-sum-dist}).
The following result establishes that the above encoding faithfully implements $\isdalal$.
The corresponding proof can likewise be found in Appendix \ref{app:proof-sat-sum-dist}.
\begin{theorem}\label{thm:sat-sumdalal}
    For a given value $u$, the encoding $\satsum(\kb,u)$ is satisfiable if and only if $\isdalal(\kb) \leq u$.
    $\satsum(\kb,u)$ is unsatisfiable for all $u \in \{0, \ldots, \abs{\atoms(\kb)} \cdot \abs{\kb} \}$ if and only if $\isdalal(\kb) = \infty$.
\end{theorem}

% The construction of $\satsum(\kb,u)$ is briefly described in the following example.

\begin{example}
    Consider again the knowledge base $\kb_7$ from Example \ref{ex:sat-contension}:
    $$\mathcal{K}_7 = \left\{\underbrace{\overbrace{x}^{\phi_{1,1}} \land \overbrace{y}^{\phi_{1,2}}}_{\alpha_1}, \quad \underbrace{\overbrace{x}^{\phi_{2,1}} \lor \overbrace{y}^{\phi_{2,2}}}_{\alpha_2}, \quad \underbrace{\lnot \overbrace{x}^{\phi_{3}}}_{\alpha_3} \right\}$$
    In this example, we aim to construct $\satsum(\kb_7,2)$, i.\,e., a SAT encoding which returns true if and only if $2$ is an upper bound of the sum-distance inconsistency value of $\kb_7$.
    
    The variables we require (defined by (SDS1)--(SDS3)) are the same as those in Example \ref{ex:sat-max-dist}. % which are defined by (SDM1)--(SDM3).
    Furthermore, the constraints we need to add to $\satsum(\kb_7,2)$ via (SDS4)--(SDS6) correspond exactly to those defined by (SDM4)--(SDM6) in Example \ref{ex:sat-max-dist}.
    % The only difference between the encoding in the latter example and the one at hand is the at-most-$u$ constraint (SDS7).
    However, we use a different at-most-$u$ constraint following (SDS7):
    $$ \mathrm{at\_most}(2, \{ \mathsf{inv}_{x,1},\mathsf{inv}_{y,1},\mathsf{inv}_{x,2},\mathsf{inv}_{y,2},\mathsf{inv}_{x,3},\mathsf{inv}_{y,3} \}) $$
\end{example}

\subsection{The Hit-Distance Inconsistency Measure}\label{sec:sat-encoding-ihdalal}
Let $\kb$ once again be a knowledge base and $u$ an integer representing a possible upper bound of $\ihdalal(\kb)$. 
% Encoding \ref{tab:sat-dalal-hit} presents an overview of the proposed SAT encoding $\sathit(\kb,u)$ for $\upperiarg{\ihdalal}$.
In the following, we construct a SAT encoding $\sathit(\kb,u)$ for $\upperiarg{\ihdalal}$.

In contrast to the other distance-based measures, the value of $\ihdalal(\kb)$ is not a distance value, but a count of distances greater than $0$. 
In other words, $\ihdalal(\kb)$ is the maximum number of formulas in $\kb$ that can be satisfied by one interpretation, or the minimum number of formulas that need to be removed in order to make $\kb$ consistent.
To represent this, we add one new variable $\mathsf{hit}_\varFor$ for every formula $\varFor \in \kb$ (SDH1). 
% Let $\mathsf{HIT_{\kb}}$ be the set of all new variables $\mathsf{hit}_\varFor$. 
Let $\mathsf{HIT_{\kb}} = \{\mathsf{hit}_\varFor \mid \varFor \in \kb\}$. 
For each $\varFor \in \kb$, we also add a disjunction that uses one of the new variables:
\begin{align*}
    \varFor \vee \mathsf{hit}_\varFor \tag{SDH3}
\end{align*}
This formula represents that if $\varFor$ is not true, then the new variable $\mathsf{hit}_\varFor$ is set to true. 
Note that all atoms of the input knowledge base $\kb$ still appear in the encoding as part of $\varFor$ (SDH2).
We then represent the upper bound for the number of models with distances $>0$ by adding one at-most-$u$ constraint that restricts the number of $\mathsf{hit}_\varFor \in \mathsf{HIT_{\kb}}$ that are allowed to be set to true:
\begin{align*}
    \mathrm{at\_most}(u, \mathsf{HIT_{\kb}}) \tag{SDH4}
\end{align*}

In total, $\sathit(\kb,u)$ is defined by (SDH1)--(SDH4). % (see Encoding \ref{tab:sat-dalal-hit} for an overview).
% The following result establishes that the above encoding faithfully implements the hit-distance measure.
% The proof of the theorem below is given in Appendix \ref{app:proof-sat-hit-dist}.
Appendix \ref{app:proof-sat-hit-dist} provides an overview of $\sathit(\kb,u)$ (i.\,e., Encoding~\ref{tab:sat-dalal-hit}), as well as the correctness proof of the theorem below.
\begin{theorem}\label{thm:sat-hitdalal}
For a given value $u$, the encoding $\sathit(\kb,u)$ is satisfiable if and only if $\ihdalal(\kb) \leq u$.
\end{theorem}

% The following example shows the construction of $\sathit(\kb,u)$.
\begin{example}
    Consider again the knowledge base $\kb_7$ from Example \ref{ex:sat-contension}:
    $$\mathcal{K}_7 = \left\{\underbrace{\overbrace{x}^{\phi_{1,1}} \land \overbrace{y}^{\phi_{1,2}}}_{\alpha_1}, \quad \underbrace{\overbrace{x}^{\phi_{2,1}} \lor \overbrace{y}^{\phi_{2,2}}}_{\alpha_2}, \quad \underbrace{\lnot \overbrace{x}^{\phi_{3}}}_{\alpha_3} \right\}$$
    In this example, we aim to construct $\sathit(\kb_7,1)$, i.\,e., a SAT encoding which returns true if and only if $1$ is an upper bound of the hit-distance inconsistency value of $\kb_7$.
    
    % First, we define the following variables wrt.\ the formulas in $\kb_7$ (SDH1):
    % $$ \mathsf{hit}_{\alpha_1}, \mathsf{hit}_{\alpha_2}, \mathsf{hit}_{\alpha_3} $$
    First, following (SDH1), we define the variables $ \mathsf{hit}_{\alpha_1}, \mathsf{hit}_{\alpha_2}, \mathsf{hit}_{\alpha_3} $ wrt.\ the formulas in $\kb_7$.
    % Moreover, we require the variables corresponding to $\atoms(\kb_7)$ (SDH2):
    % $$ x,y $$
    Moreover, we require the variables $\atoms(\kb_7) = \{x,y\}$ (SDH2).
    Following (SDH3), we add the subsequent constraints to $\sathit(\kb_7,1)$:
    % \begin{align*}
    %     & (x \land y) \lor \mathsf{hit}_{\alpha_1} \\
    %     & (x \lor y) \lor \mathsf{hit}_{\alpha_2} \\
    %     & (\lnot x) \lor \mathsf{hit}_{\alpha_3} \\
    % \end{align*}
    \begin{align*}
        (x \land y) \lor \mathsf{hit}_{\alpha_1}, \qquad
        (x \lor y) \lor \mathsf{hit}_{\alpha_2}, \qquad
        (\lnot x) \lor \mathsf{hit}_{\alpha_3} 
    \end{align*}
    % At last, we add the at-most-$u$ constraint (SDH4):
    % $$ \mathrm{at\_most}(1,\{\mathsf{hit}_{\alpha_1}, \mathsf{hit}_{\alpha_2}, \mathsf{hit}_{\alpha_3} \}) $$
    At last, we add $ \mathrm{at\_most}(1,\{\mathsf{hit}_{\alpha_1}, \mathsf{hit}_{\alpha_2}, \mathsf{hit}_{\alpha_3} \}) $ via (SDH4).
\end{example}

\section{ASP-Based Algorithms for Selected Inconsistency Measures}\label{sec:asp}

After the introduction of a general SAT-based approach for computing inconsistency degrees, as well as specific encodings for the problem $\upperi$ wrt.\ all six inconsistency measures addressed in this work, we now introduce answer set programming as an alternative method to compute inconsistency values.
% The subsequent sections provide an overview of answer set programming in general (Section \ref{sec:asp-preliminaries}), followed by detailed descriptions of encodings of all six inconsistency measures considered in this work (Section \ref{sec:asp-algo-c} through Section \ref{sec:asp-algo-dhit}). 

\subsection{Answer Set Programming}\label{sec:asp-preliminaries}

%%%% Hier ggf. die ASP-Preliminaries aus der langen Version des JELIA-Papers einfügen => ist für die Kurzversion sowieso rausgeflogen
% -> Müsste vermutlich ein wenig ergänzt werden (um Aggregates und Optimize-Statements), aber wäre schon mal eine gute Grundlage?
% -> Auch Grounding besser erklären 
% -> Beispiele erstmal weglassen 
% => sollten dann am Ende max. 1.5 Seiten sein (eher nur eine)

\textit{Answer Set Programming} (ASP) \cite{lifschitz2019answer,gebser2012answer,lifschitz2008answer,BrewkaET11} is a declarative problem solving approach targeted at difficult search problems.
Thus, rather than modeling instructions on how to solve a problem, a representation of the problem itself is modeled.
The goal is to represent a problem in a logical format (a \textit{logic program}) such that the models of this representation describe solutions of the original problem. 
These models are called \textit{answer sets}.

In short, a logic program is a finite set of \textit{rules} of the form
    % $$ h_1 \mid \ldots \mid h_m \texttt{ :- } b_{m+1}, \ldots, b_n, \texttt{not}\, b_{n+1}, \ldots, \texttt{not}\, b_k. $$ 
    \begin{align*}\label{eq:asp-rule}
    r = a_0 \texttt{ :- } a_1, \ldots, a_n, \texttt{ not } a_{n+1}, \ldots, \texttt{ not } a_m.
    \end{align*} 
with each $a_i$ ($0\leq i \leq n \leq m$) being atoms, and ``\texttt{not}'' denoting the default negation in the sense of Reiter \citeyear{reiter1980logic}.
% An atom is an expression p(t1, ..., tk), where p is a predicate symbol of arity k ≥ 0
% and t1, ..., tk are terms – either object constants or variables.
In ASP, an atom has the form $p(\termASP_1, \ldots, \termASP_n)$, with $p$ being a predicate symbol, and $\termASP_1, \ldots, \termASP_n$ being terms.
Terms are either constants, variables, arithmetic terms (i.\,e., $-\termASP_1$ or $\termASP_1 \circ \termASP_2$ with $\circ \in \{+,-,*,/\}$ wrt.\ some terms $\termASP_1,\termASP_2$), or functional terms (i.\,e., $\varphi(\termASP_1, \ldots, \termASP_m)$ with $\varphi$ being a functor, $\termASP_1, \ldots, \termASP_m$ being terms, and $m > 0$)~\cite{calimeri2020asp}.
Moreover, we express the arity $n$ of a predicate or function $\varphi$ as $\varphi/n$, and an ASP literal $\ell$ is either an atom or a default-negated atom.

The \textit{head} of a rule $r$ (as shown above) is $\head(r) = \{a_0\}$
and the \textit{body} is $\body(r) = \{a_{1}, \ldots, a_n, \texttt{not} \, a_{n+1}, \ldots, \texttt{not} \, a_m\}$.
% We also refer to the elements in $\body(r)$ as (body) literals.
If $\body(r) = \emptyset$, then $r$ is a \textit{fact}, and
if $\head(r) = \emptyset$, then $r$ is a \textit{constraint}.
We further divide the body elements of $r$ into $\body^+(r) = \{a_{1}, \ldots, a_n\}$ and $\body^-(r) = \{a_{n+1}, \ldots, a_m\}$.

An atom/rule/program is \textit{ground} if it does not contain any variables.\footnote{Note that, following the Clingo syntax, all variable names we use start with an uppercase letter, and all constant names start with a lowercase letter. 
We also make use of \textit{anonymous variables}. Those are variables that do not recur within a rule, and that are denoted by ``\texttt{\_}''.}
Let $X$ be a set of ground atoms.
We define $X$ to be a \textit{model} of a ground logic program $P$ if for all $r \in P$, $\head(r) \in X$ whenever $\body^+(r) \subseteq X$ and $\body^-(r) \cap X = \emptyset$.
The \textit{reduct}~\cite{GelfondL88} of $P$ wrt.~$X$ is defined as 
$$P^X = \{\head(r) \texttt{ :- } \body^+(r) \mid \body^-(r) \cap X = \emptyset \text{ with } r \in P \}$$
%%% EDIT 09.12.2024: diese Formulierung wurde bereits korrigiert
Moreover, $X$ is an \textit{answer set} of $P$ if it is a subset-minimal model of $P^X$.
%%%%%%%%%% END JELIA PART

% hier was zu choice rules? 
% -> vielleicht auch erst nach der Einführung von Cardinality Constraints
%   -> genau genommen verwenden wir ja Cardinality Constraints im Kopf mancher Regeln (keine Choice Rules?) und Cardinality Constraints sind ja sowieso schon (mehr oder weniger^^) formell definiert 
% -> genauer gesagt: 
%   -> cardinality constraints können auch in heads vorkommen
%   -> dadurch wird die Regel zur (extended) choice rule
%       -> erklären, was choice rule ist; formelle Def.
%   -> dann aber anmerken, dass choice rules zu einer Menge normaler Regeln "reduziert" werden können 
%       -> "can be compiled into normal rules"
%       -> Clingo kann aber auch intern damit umgehen

Another language concept we make use of is the \textit{conditional literal}, which is of the form $\ell_0 : \ell_1, \ldots, \ell_n$, with $\ell_0, \ldots, \ell_n$ being literals.
The idea behind a conditional literal is to regulate the instantiation of $\ell_0$ by means of $\ell_1, \ldots, \ell_n$.
In other words, we can view a conditional literal as the set $\{\ell_0 \mid \ell_1, \ldots, \ell_n\}$.
% We further use \textit{cardinality constraints} and \textit{optimization statements}.
Moreover, we use \textit{cardinality constraints}.
A cardinality constraint is of the form $l\{c_1; \ldots; c_m\}u$, with $c_1, \ldots, c_m$ being conditional literals, $l$ constituting an optional lower bound, and $u$ an optional upper bound.
% Intuitively, this can be read as ``at least $l$, and at most $u$ of the atoms in $\{a_1, \ldots, a_n\}$ must be included in the answer set''.
Intuitively, this can be read as ``at least $l$, and at most $u$ of the literals specified by $c_1, \ldots, c_m$ must be satisfied''.
%
% Further language concepts we make use of are \textit{cardinality constraints} and \textit{optimization statements}.
% A cardinality constraint is of the form $l\{a_1; \ldots; a_n\}u$, with $l$ constituting a lower bound, and $u$ an upper bound.
% Intuitively, this can be read as ``at least $l$, and at most $u$ of the atoms in $\{a_1, \ldots, a_n\}$ must be included in the answer set''.
% Aggregates are built-in functions.
% We utilize the \texttt{\#count} aggregate, which allows for counting the number of ground instances.
% -> aggregates gar nicht direkt definieren, sondern nur bei der ersten Verwendung im Text kurz intuitiv erklären?
%   -> wie auch für Pooling und Intervals von R3 vorgeschlagen 
%As for optimization statements, we only need a specific form thereof in this work: 
%    $$\texttt{\#minimize\{}a_1, \ldots, a_n \texttt{\}}. $$

It is possible to use cardinality constraints not only in rule bodies, but also in heads.
A rule with a cardinality constraint as the head is referred to as an \textit{(extended) choice rule}. 
Formally, a choice rule has the form 
    $$r_c = l\{ a_1, \ldots, a_m \} u \texttt{ :- } a_{m+1}, \ldots, a_n, \texttt{ not } a_{n+1}, \ldots, \texttt{ not } a_o $$
with $0 \leq m \leq n \leq o$, $a_1, \ldots, a_o$ being atoms, $0 \leq l \leq u$, and $l,u$ being (optional) lower and upper bounds, respectively.
The intuition behind such a rule is that any subset of the head atoms (which complies with the upper and lower bound, if specified) can be included in the answer set.
Note that choice rules can be transformed into sets of normal rules~\cite{gebser2012answer}.

Furthermore, we use \textit{aggregates} and \textit{optimization statements}.
The former are used to reason about minima, maxima, sums, and counts over sets of literals.
Let an \textit{aggregate element} $g$ be defined as $g = \termASP_1, \ldots, \termASP_m : \ell_1, \ldots, \ell_n$
% \begin{align*}
%     g = t_1, \ldots, t_m : \ell_1, \ldots, \ell_n
% \end{align*}
% - ggf. s und t statt n und m 
with $\termASP_1, \ldots, \termASP_m$ being terms, and $\ell_1, \ldots, \ell_n$ being literals. 
An \textit{aggregate} is then defined as $\mathtt{\#agg} \{ g_1, \ldots, g_n \} \prec \termASP$, with $\mathtt{\#agg} \in \{ \mathtt{\#count}, \mathtt{\#min}, \mathtt{\#max}, \mathtt{\#sum} \}$ being an aggregate function name, $g_1, \ldots, g_n$ being aggregate elements, $\prec\ \in \{<, >, \leq, \geq, =, \neq\}$ being an aggregate relation, and $\termASP$ being a term~\cite{calimeri2020asp}.
Optimization statements serve the purpose of expressing cost functions that are subject to minimization or maximization.
In this work, we only use a specific form of \textit{minimize statements}, which are of the form
$ \texttt{\#minimize\{}\termASP_1, \ldots, \termASP_m : \ell_1, \ldots, \ell_n \texttt{\}}.$,
% Such a \textit{minimize statement} instruct the ASP solver to include only a minimal number of the atoms $a_1, \ldots, a_n$ in any answer set.
with $\termASP_1, \ldots, \termASP_m$ being terms and $\ell_1, \ldots, \ell_n$ being literals.
We refer to a set that complies with the minimization 
% (i.\,e., an answer set that contains a minimal number of $a_1, \ldots, a_n$) 
as an \textit{optimal} answer set.

We further make use of the interval (``..'') and pooling (``;'') operators to abbreviate notation.
% As a means of syntactic sugar, we can use interval (“..”) and pooling (“;”) operators to abbreviate notation. 
Intervals let us create multiple instances of a predicate determined by an interval of numerical values, and pooling lets us create multiple instances of a predicate by separating elements by ``;''.
For more details on ASP we refer the reader to the ASP-Core-2 standard~\cite{calimeri2020asp}.

\subsection{The Contension Inconsistency Measure}\label{sec:asp-algo-c}

% With regard to the contension inconsistency measure $\icont$, we can construct an extended logic program $\pc(\kb)$ to compute $\icont(\mathcal{K})$ wrt.\ a knowledge base $\mathcal{K}$ as described in the following.
Given a knowledge base $\kb$, we construct an extended logic program $\pc(\kb)$ to determine the contension inconsistency value $\icont(\kb)$ as described in the following.
Note that, like the SAT-based approach for $\icont$, the ASP-based method was already proposed in \cite{kuhlmann2022comparison}, but is covered again for the sake of completeness.

% Recall that the contension measure looks for the minimal number of atoms that need to be assigned truth value $b$ in order to render the given knowledge base consistent in Priest's three-valued logic.
At first, we declare some facts describing the overall composition of the knowledge base, i.\,e., we describe which formulas are elements of the knowledge base, which atoms are in the knowledge base's signature, and whether a (sub)formula is a conjunction, disjunction, negation, or atom.
We define for every formula $\varFor \in \kb$ 
\begin{align*}
    \texttt{kbMember(}\varFor \texttt{).} \tag{AC1} 
\end{align*}
and for every atom $\varAt \in \mathsf{At}(\kb)$
\begin{align*}
    \texttt{atom(} \varAt \texttt{).} \tag{AC2}
\end{align*}
Note that the $\varFor$ in $\kb$ and the $\varFor$ in $\texttt{kbMember/1}$ are formally not the same---in $\kb$ it is a propositional formula, while in $\texttt{kbMember/1}$ it is a string representation of that formula.
More precisely, a formula $\varFor$ is ``translated'' to a uniquely defined ASP constant, i.\,e., a string starting with a lowercase letter.\footnote{For example, a given formula $\alpha_1$ could be represented as \texttt{alpha\_1}.} 
% (e.\,g., a given formula $\alpha_1$ could be represented as \texttt{alpha\_1}). 
% \footnote{For example, a given formula $\alpha_1$ could be represented as \texttt{alpha\_1}.}
The same applies to the $\varAt$ in $\atoms(\kb)$ and the $\varAt$ in $\texttt{atom/1}$.
In the following, whenever a propositional logic symbol appears in an ASP rule, it is to be interpreted as a uniquely defined ASP constant.

Every conjunction $\varSub_c = \varSubSub_{c,1} \land \varSubSub_{c,2}$ is encoded as 
\begin{align*}
    \texttt{conjunction(} \varSub_c \texttt{,} \varSubSub_{c,1} \texttt{,} \varSubSub_{c,2} \texttt{).} \tag{AC3} 
\end{align*}
In the same fashion, we define disjunctions $\varSub_d = \varSubSub_{d,1} \lor \varSubSub_{d,2}$ as
\begin{align*}
    \texttt{disjunction(} \varSub_d \texttt{,} \varSubSub_{d,1} \texttt{,} \varSubSub_{d,2} \texttt{).} \tag{AC4}
\end{align*}
For each negated formula, i.\,e., for each $\varSub_n = \lnot \varSubSub_n$, we define
\begin{align*}
  \texttt{negation(} \varSub_n \texttt{,} \varSubSub_n \texttt{).} \tag{AC5}  
\end{align*}
Further, we need to encode subformulas which consist of individual atoms.
Hence, for each formula $\varSub_a$ which is equal to an atom $\varAt$ we define
\begin{align*}
   \texttt{formulaIsAtom(} \varSub_a \texttt{,} \varAt \texttt{).} \tag{AC6} 
\end{align*}
Additionally, we need to declare the truth values of Priest's three-valued logic ($t$, $f$, $b$) as facts:
% Note that we make use of a pooling operator, which lets us create multiple instances of a predicate by separating elements with ``\texttt{;}'':
\begin{align*}
   \texttt{tv(} t \texttt{;} f \texttt{;} b \texttt{).} \tag{AC7} 
\end{align*}

To encode how the contension inconsistency measure actually works, we first need to ensure that an atom is not assigned multiple truth values at the same time, i.\,e.,\ we need to ensure that each atom gets a distinct evaluation.
To this end, we introduce the predicate \texttt{truthValue/2} 
% which represents that some element \texttt{X} gets assigned a truth value \texttt{T}.
% We 
and define that an atom has exactly one truth value by making use of a cardinality constraint:
\begin{align*}
   \texttt{1\{truthValue(A,T) : tv(T)\}1 :- atom(A).} \tag{AC8} 
\end{align*}

In order to encode the evaluation of formulas, we need to model the role of the connectives $\land$, $\lor$, and $\lnot$ in Priest's three-valued logic (see Table \ref{tab:3VL}).
For a conjunction to be true in Priest's three-valued logic, both of its conjuncts have to be true: 
% \begin{align*}
%     & \texttt{truthValue(F,} t \texttt{) :-} \\
%     & \qquad \qquad \texttt{conjunction(F,G,H),} \\
%     & \qquad \qquad \texttt{truthValue(G,} t \texttt{),} \\
%     & \qquad \qquad \texttt{truthValue(H,} t \texttt{).} \tag{AC9} 
% \end{align*}
\begin{align*}
    \texttt{truthValue(F,} t \texttt{) :- }&\texttt{conjunction(F,G,H),} \\
    & \texttt{truthValue(G,} t \texttt{),} \\
    & \texttt{truthValue(H,} t \texttt{).} \tag{AC9} 
\end{align*}
For a conjunction to be false, it is sufficient if at least one of its conjuncts is false:
% \begin{align*}
%     & \texttt{truthValue(F,} f \texttt{) :-} \\
%     & \qquad \qquad \texttt{conjunction(F,G,H),} \\
%     & \qquad \qquad \texttt{1\{truthValue(G,} f \texttt{), truthValue(H,} f \texttt{)\}.} \tag{AC10}
% \end{align*}
{\allowdisplaybreaks
\begin{align*}
    \texttt{truthValue(F,} f \texttt{) :- }&\texttt{conjunction(F,G,H),} \\
    & \texttt{1\{truthValue(G,} f \texttt{); truthValue(H,} f \texttt{)\}.} \tag{AC10}
\end{align*}
}
A conjunction is $b$ in three-valued logic if it is neither $t$ nor $f$:
% \begin{align*}
%     & \texttt{truthValue(F,} b \texttt{) :- } \\
%     & \qquad \qquad \texttt{conjunction(F,\_,\_),}  \\
%     & \qquad \qquad \texttt{not truthValue(F,} t \texttt{),} \\
%     & \qquad \qquad \texttt{not truthValue(F,} f \texttt{).} \tag{AC11}
% \end{align*}
\begin{align*}
    \texttt{truthValue(F,} b \texttt{) :- } & \texttt{conjunction(F,\_,\_),}\\
    & \texttt{not truthValue(F,} t \texttt{),} \\
    & \texttt{not truthValue(F,} f \texttt{).} \tag{AC11}
\end{align*}
Analogously, we define that a disjunction is false if both its disjuncts are false, and true if at least one of its disjuncts is true. 
Again, if a disjunction is neither $t$ nor $f$, it is $b$.
% \begin{align*}
%     &\texttt{truthValue(F,} f \texttt{) :-}  \\
%         & \qquad \qquad \texttt{disjunction(F,G,H),} \\
%         & \qquad \qquad \texttt{truthValue(G,} f \texttt{),} \\
%         & \qquad \qquad \texttt{truthValue(H,} f \texttt{).} \tag{AC12}
% \end{align*}
{\allowdisplaybreaks
\begin{align*}
    \texttt{truthValue(F,} f \texttt{) :- } & \texttt{disjunction(F,G,H),}  \\
        & \texttt{truthValue(G,} f \texttt{),} \\
        & \texttt{truthValue(H,} f \texttt{).} \tag{AC12} \\
% \end{align*}
% \begin{align*}
%     &\texttt{truthValue(F,} t \texttt{) :-}  \\
%         & \qquad \qquad \texttt{disjunction(F,G,H),} \\
%         & \qquad \qquad \texttt{1\{truthValue(G,} t \texttt{), truthValue(H,} t \texttt{)\}.} \tag{AC13}
% \end{align*}
% \begin{align*}
    \texttt{truthValue(F,} t \texttt{) :- } & \texttt{disjunction(F,G,H),} \\
        & \texttt{1\{truthValue(G,} t \texttt{); truthValue(H,} t \texttt{)\}.} \tag{AC13} \\
% \end{align*}
% \begin{align*}
%     & \texttt{truthValue(F,} b \texttt{) :-} \\
%         & \qquad \qquad \texttt{disjunction(F,\_,\_),}  \\
%         & \qquad \qquad \texttt{not truthValue(F,} t \texttt{),} \\
%         & \qquad \qquad \texttt{not truthValue(F,} f \texttt{).} \tag{AC14}
% \end{align*}
% \begin{align*}
    \texttt{truthValue(F,} b \texttt{) :- } & \texttt{disjunction(F,\_,\_),} \\
        & \texttt{not truthValue(F,} t \texttt{),} \\
        & \texttt{not truthValue(F,} f \texttt{).} \tag{AC14} 
\end{align*}
}
A negation is $t$ in three-valued logic if its base formula is $f$ and vice versa.
A negation is $b$ if its base formula is $b$ as well.
Thus, we can formulate the following rules to model three-valued negation in ASP:
\begin{align*}
    \texttt{truthValue(F,} t \texttt{) :- }& \texttt{negation(F,G), truthValue(G,} f \texttt{).} \tag{AC15} \\
    \texttt{truthValue(F,} f \texttt{) :- }& \texttt{negation(F,G), truthValue(G,} t \texttt{).} \tag{AC16} \\
    \texttt{truthValue(F,} b \texttt{) :- }& \texttt{negation(F,G), truthValue(G,} b \texttt{).} \tag{AC17}
\end{align*}
Moreover, if a (sub)formula only consists of a single atom, it must have the same truth value:
     % \begin{align*}
     %     & \texttt{truthValue(F,T) :-}  \\
     %     & \qquad \qquad \texttt{formulaIsAtom(F,G),} \\
     %     & \qquad \qquad \texttt{truthValue(G,T),} \\
     %     & \qquad \qquad \texttt{tv(T).} \tag{AC18}
     % \end{align*}
     \begin{align*}
         \texttt{truthValue(F,T) :- } & \texttt{formulaIsAtom(F,G),} \\
         & \texttt{truthValue(G,T),} \\
         & \texttt{tv(T).} \tag{AC18}
     \end{align*}

Further, all formulas $\varFor \in \kb$ need to be either $t$ or $b$ wrt.\ Priest's three-valued logic.
Thus, no formula is allowed to be evaluated to $f$.
We accomplish this by introducing the following integrity constraint:
    \begin{align*}
        \texttt{:- truthValue(F,} f \texttt{), kbMember(F).}
        \tag{AC19}
    \end{align*}

At last, to actually compute $\icont(\kb)$, we require the minimal number of atoms in $\kb$ evaluated to $b$.
This is achieved by means of a minimize statement:
    \begin{align*}
       \texttt{\#minimize\{1,A: truthValue(A,} b \texttt{), atom(A)\}}. 
       \tag{AC20}
    \end{align*}
    
Now $\pc(\kb)$ is the union of all rules defined in (AC1)--(AC20) (for an overview, see Encoding \ref{tab:asp-contension} in Appendix~\ref{app:proof-asp-contension}).
Further, let $\omega^3_M$ be the three-valued interpretation represented by an answer set $M$ of $\pc(\kb)$. 
To be precise, if $M$ is an answer set of $\pc(\kb)$ then the three-valued interpretation $\omega^3_M$ defined as
\begin{align*}
\omega^3_M(\varAt) = \begin{cases} 
t & \quad \texttt{truthValue(}\varAt \texttt{,} t \texttt{)} \in M\\
f & \quad \texttt{truthValue(}\varAt \texttt{,} f \texttt{)} \in M\\
b & \quad \texttt{truthValue(}\varAt \texttt{,} b \texttt{)} \in M \end{cases}
\end{align*}
is a model of $\kb$ with $\varAt \in \mathsf{At}(\kb)$. 
The proofs regarding the well-definedness of the previous definition as well as the following theorem are provided in Appendix \ref{app:proof-asp-contension}.
\begin{theorem}\label{thm:asp-c}
    Let $M_o$ be an optimal answer set of $\pc(\kb)$. 
    Then $|(\omega^3_{M_o})^{-1}(b)|=\icont(\kb)$.\footnote{For any function $\varphi:X\mapsto Y$ and $y \in Y$ we define $\varphi^{-1}(y) = \{x\in X \mid \varFor(x) = y\}$}
\end{theorem}

% In the following, we present an example of how to construct $\pc$ wrt.\ a given knowledge base $\kb$.

\begin{example}\label{ex:asp-contension}
    Consider again the knowledge base $\kb_7$ from Example \ref{ex:sat-contension}:
    $$\mathcal{K}_7 = \left\{\underbrace{\overbrace{x}^{\phi_{1,1}} \land \overbrace{y}^{\phi_{1,2}}}_{\alpha_1}, \quad \underbrace{\overbrace{x}^{\phi_{2,1}} \lor \overbrace{y}^{\phi_{2,2}}}_{\alpha_2}, \quad \underbrace{\lnot \overbrace{x}^{\phi_{3}}}_{\alpha_3} \right\}$$
    We aim to construct $\pc(\kb_7)$, i.\,e., an ASP encoding which allows us to retrieve $\icont(\kb_7)$.
    First, we represent the formulas and atoms in $\kb_7$ by applying (AC1) and (AC2):
    \begin{align*}
        \mathtt{kbMember(}\alpha_1\mathtt{).}  & \qquad \mathtt{atom(}x\mathtt{).} \\
        \mathtt{kbMember(}\alpha_2\mathtt{).}  & \qquad \mathtt{atom(}y\mathtt{).} \\
        \mathtt{kbMember(}\alpha_3\mathtt{).}  &
    \end{align*}
    % \begin{align*}
    %     \mathtt{kbMember(}\alpha_1\mathtt{).}  \qquad \mathtt{kbMember(}\alpha_2\mathtt{).} \qquad \mathtt{kbMember(}\alpha_3\mathtt{).} \\
    %     \mathtt{atom(}x\mathtt{).} \qquad \mathtt{atom(}y\mathtt{).} 
    % \end{align*}
    % Next, we represent the only conjunction present in $\kb_7$, i.\,e., $\alpha_1$, via (AC3).
    % % We represent it by using (AC3):
    % \begin{align*}
    %     &\mathtt{conjunction(}\alpha_1, \phi_{1,1}, \phi_{1,2} \mathtt{).} 
    % \end{align*}
    % In the same fashion, we apply (AC4) to $\alpha_2$, the only disjunction occurring in $\kb_7$.
    % \begin{align*}
    %     &\mathtt{disjunction(}\alpha_2, \phi_{2,1}, \phi_{2,2}  \mathtt{).} 
    % \end{align*}
    % Now, the only formula that is not represented in ASP yet is $\alpha_3$.
    % Hence, we apply (AC5).
    % $$ \mathtt{negation(} \alpha_3, \phi_3 \mathtt{).} $$
    We use (AC3) to represent the only conjunction in $\kb_7$ ($\alpha_1$) as $\mathtt{conjunction(}\alpha_1, \phi_{1,1}, \phi_{1,2} \mathtt{)} $.
    In the same fashion, we apply (AC4) to $\alpha_2$, the only disjunction occurring in $\kb_7$, and get
    $\mathtt{disjunction(}\alpha_2, \phi_{2,1}, \phi_{2,2}  \mathtt{)}$.
    Now, the only formula that is not represented in ASP yet is $\alpha_3$, so we represent it as
    $\mathtt{negation(} \alpha_3, \phi_3 \mathtt{)}$ via (AC5).
    To represent the subformulas which consist of individual atoms, we use (AC6):
    \begin{align*}
        & \mathtt{formulaIsAtom}(\phi_{1,1}, x). \qquad \mathtt{formulaIsAtom}(\phi_{1,2}, y). \\
        & \mathtt{formulaIsAtom}(\phi_{2,1}, x). \qquad \mathtt{formulaIsAtom}(\phi_{2,2}, y). \\
        & \mathtt{formulaIsAtom}(\phi_{3}, x).
    \end{align*}
    The remainder of the logic program, i.\,e., (AC7)--(AC20), is static.
\end{example}

\subsection{The Forgetting-Based Inconsistency Measure}\label{sec:asp-algo-fb}

The forgetting-based inconsistency measure $\iforget$ is determined by the number of atom occurrences that need to be ``forgotten'' in order to make the knowledge base $\kb$ consistent.
% An extended logic program $\pf(\kb)$ which computes $\iforget(\kb)$ can be constructed as described below.
In the following , we construct an extended logic program $\pf(\kb)$ which computes $\iforget(\kb)$.
Just as described in the preceding section, we begin with the definition of facts.
Again, every formula $\varFor \in \kb$ is represented as $\texttt{kbMember(}\varFor \texttt{)}$ (AF1).
Conjunctions (AF3), disjunctions (AF4), and negations (AF5) are also represented in exactly the same form as introduced in Section \ref{sec:asp-algo-c}.
% (see also (AF3--AF5) in Encoding \ref{tab:asp-fb}).
On the other hand, formulas which consist of individual atoms have to be handled differently.
Since the forgetting-based inconsistency measure does not only consider each atom, but each atom \textit{occurrence}, we have to represent this in our ASP encoding. 
Therefore, we define the predicate \texttt{formulaIsAtomOcc/3} which contains the formula $\varSub$, the atom $\varAt$ it consists of, and the atom's label $l$:
\begin{align*}
    \texttt{formulaIsAtomOcc(}\varSub \texttt{,} \varAt \texttt{,} l \texttt{).} \tag{AF2}
\end{align*}
The representation of an atom occurrence can be extracted from the above rule as follows:
\begin{align*}
    \texttt{atomOcc(A,L) :- formulaIsAtomOcc(\_,A,L).} \tag{AF7}
\end{align*}
Moreover, we can gain the representation of an atom by applying the following rule:
\begin{align*}
  \texttt{atom(A) :- atomOcc(A,\_).} \tag{AF8}  
\end{align*}
We model the truth values in the usual manner:
% Furthermore, we need to differentiate between (sub)formulas, which can only be either $t$ or $f$, and atom occurrences, which are either assigned a truth value ($t$ or $f$), or can be forgotten,
% i.\,e., it is either replaced by $\top$ ($\forget_\top$) or by $\bot$ ($\forget_\bot$).
% In addition, 
% An atom occurrence can be forgotten in two different ways---either by being replaced by $\top$ ($\forget_\top$) or by $\bot$ ($\forget_\bot$).
% Consequently, we define two different predicates.
\begin{align*}
    & \texttt{tv(} t \texttt{;} f \texttt{).} \tag{AF6} 
% \end{align*}
% and
% \begin{align*}
    % & \texttt{atv(} t \texttt{;} f \texttt{;} \forget_\top \texttt{;} \forget_\bot \texttt{).} \tag{AF7}
\end{align*}
Next, we include a rule that ensures unique atom evaluation.
With regard to $\iforget$, this means that an atom as a whole must be evaluated to either $t$ or $f$, even though individual occurrences of that atom might be replaced by $\top$ or $\bot$:
% Thus, we define.
\begin{align*}
    \texttt{1\{truthValue(A,T) : tv(T)\}1 :- atom(A).} \tag{AF9}
\end{align*}

The connective encodings simply model propositional entailment.
Thus, the evaluation of a conjunction $\varSub_c = \varSubSub_{c,1} \land \varSubSub_{c,2}$ is modeled as follows:
% \begin{align*}
%     & \texttt{truthValue(F,} t \texttt{) :- }  \\
%         & \qquad \qquad \texttt{conjunction(F,G,H),} \\
%         & \qquad \qquad \texttt{truthValue(G,} t \texttt{),} \\
%         & \qquad \qquad \texttt{truthValue(H,} t \texttt{).} \tag{AF11} \\
%     & \texttt{truthValue(F,} f \texttt{) :-} \\
%         & \qquad \qquad \texttt{conjunction(F,\_,\_), }  \\
%         & \qquad \qquad \texttt{not truthValue(F,} t \texttt{).} \tag{AF12}
% \end{align*}
\begin{align*}
    \texttt{truthValue(F,} t \texttt{) :- } & \texttt{conjunction(F,G,H),} \\
        & \texttt{truthValue(G,} t \texttt{),} \\
        & \texttt{truthValue(H,} t \texttt{).} \tag{AF10} \\
    \texttt{truthValue(F,} f \texttt{) :- } & \texttt{conjunction(F,\_,\_), }\\
        & \texttt{not truthValue(F,} t \texttt{).} \tag{AF11}
\end{align*}
In the same fashion, we define the evaluation of a disjunction $\varSub_d = \varSubSub_{d,1} \lor \varSubSub_{d,2}$:
% \begin{align*}
%     & \texttt{truthValue(F,} f \texttt{) :- }  \\
% 	    & \qquad \qquad \texttt{disjunction(F,G,H),} \\
% 	    & \qquad \qquad \texttt{truthValue(G,} f \texttt{),} \\
% 	    & \qquad \qquad \texttt{truthValue(H,} f \texttt{).} \tag{AF13}  \\
%     & \texttt{truthValue(F,} t \texttt{) :- } \\
%         & \qquad \qquad\texttt{disjunction(F,\_,\_), }  \\
%         & \qquad \qquad \texttt{not truthValue(F,} f \texttt{).} \tag{AF14}
% \end{align*}
\begin{align*}
    \texttt{truthValue(F,} f \texttt{) :- } & \texttt{disjunction(F,G,H),} \\
	    & \texttt{truthValue(G,} f \texttt{),} \\
	    & \texttt{truthValue(H,} f \texttt{).} \tag{AF12}  \\
    \texttt{truthValue(F,} t \texttt{) :- } & \texttt{disjunction(F,\_,\_), }\\
        & \texttt{not truthValue(F,} f \texttt{).} \tag{AF13}
\end{align*}
At last we define the evaluation of a negation $\varSub_n = \lnot \varSubSub_n$:
% \begin{align*}
%     & \texttt{truthValue(F,} t \texttt{) :- } \\
%         & \qquad \qquad \texttt{negation(F,G),} \\
%         & \qquad \qquad \texttt{truthValue(G,} f \texttt{).} \tag{AF15} \\
%     & \texttt{truthValue(F,} f \texttt{) :- } \\
%         & \qquad \qquad \texttt{negation(F,G),} \\
%         & \qquad \qquad \texttt{truthValue(G,} t \texttt{).} \tag{AF16} 
% \end{align*}
{\allowdisplaybreaks
\begin{align*}
    \texttt{truthValue(F,} t \texttt{) :- } & \texttt{negation(F,G),}\\
        & \texttt{truthValue(G,} f \texttt{).} \tag{AF14} \\
    \texttt{truthValue(F,} f \texttt{) :- } & \texttt{negation(F,G),} \\
        & \texttt{truthValue(G,} t \texttt{).} \tag{AF15} 
\end{align*}
}

% If a formula $\varSub_a$ consists of a single atom occurrence $\varAt^l$, we need to consider that the forgetting operation might have been applied.
% If $\varAt^l$ is either $t$ or has been replaced by $\top$, $\varSub_a$ is evaluated to $t$.
% Thus, if $\varAt^l$ is either $f$ or has been replaced by $\bot$, $\varSub_a$ evaluates to $f$.
% We express this by means of the following rules:
% \begin{align*}
%     &\texttt{truthValue(F,}t \texttt{) :- } \\
%         & \qquad \qquad \texttt{formulaIsAtomOcc(F,A,L),} \\ 
%         & \qquad \qquad \texttt{atomTruthValue(A,L,} t \texttt{).} \tag{AF17} \\
%     &\texttt{truthValue(F,}t \texttt{) :- } \\
%         & \qquad \qquad \texttt{formulaIsAtomOcc(F,A,L),} \\ 
%         & \qquad \qquad \texttt{atomTruthValue(A,L,} \forget_\top \texttt{).} \tag{AF18} \\
%     &\texttt{truthValue(F,}f \texttt{) :- } \\
%         & \qquad \qquad \texttt{formulaIsAtomOcc(F,A,L),} \\ 
%         & \qquad \qquad \texttt{atomTruthValue(A,L,} f \texttt{).} \tag{AF19} \\
%     &\texttt{truthValue(F,}f \texttt{) :- } \\
%         & \qquad \qquad \texttt{formulaIsAtomOcc(F,A,L),} \\ 
%         & \qquad \qquad \texttt{atomTruthValue(A,L,} \forget_\bot \texttt{).} \tag{AF20}
% \end{align*}

If a formula $\varSub_a$ consists of a single atom occurrence $\varAt^l$, we need to consider that the forgetting operation might be applied.
To model this, we first guess for each atom occurrence whether it is being forgotten:
\begin{align*}
    \texttt{\{atomOccForgotten(A,L)\} :- atomOcc(A,L).} \tag{AF16}
\end{align*}
If $\varAt^l$ is not forgotten, it needs to evaluate to the truth value of $\varAt$ itself:
% \begin{align*}
%     & \texttt{truthValue(F,T) :- } \\
%     & \qquad \qquad \texttt{formulaIsAtomOcc(F,A,L),} \\
%     & \qquad \qquad \texttt{truthValue(A,T),} \\
%     & \qquad \qquad \texttt{not atomOccForgotten(A,L).} \tag{AF18}
% \end{align*}
\begin{align*}
    \texttt{truthValue(F,T) :- } & \texttt{formulaIsAtomOcc(F,A,L),} \\
    & \texttt{truthValue(A,T),} \\
    & \texttt{not atomOccForgotten(A,L).} \tag{AF17}
\end{align*}
If, on the other hand, $\varAt^l$ needs to be forgotten, it means that the (sub)formula consisting of $\varAt^l$ evaluates to the opposite truth value:
% die Formel darf quasi den gegenteiligen truth value haben 
{\allowdisplaybreaks
\begin{align*}
    \texttt{truthValue(F,}t \texttt{) :- } & \texttt{formulaIsAtomOcc(F,A,L),} \\
        & \texttt{truthValue(A,}f\texttt{),} \\
        & \texttt{atomOccForgotten(A,L).} \tag{AF18} \\
    \texttt{truthValue(F,}f \texttt{) :- } & \texttt{formulaIsAtomOcc(F,A,L),} \\
    & \texttt{truthValue(A,}t\texttt{),} \\
    & \texttt{atomOccForgotten(A,L).} \tag{AF19}
\end{align*}
}

All formulas $\varFor \in \kb$ must evaluate to $t$ after the forgetting operation is applied.
Hence, we add the following integrity constraint which corresponds exactly to the one we used for the contension inconsistency measure (see (AC19) in Section \ref{sec:asp-algo-c}): % in Encoding \ref{tab:asp-contension}): 
\begin{align*}
     \texttt{:- truthValue(F,} f \texttt{), kbMember(F).} \tag{AF20}
\end{align*}
 		
Lastly, we need to minimize the number of atom occurrences which are forgotten:
% To achieve this, we define the predicate \texttt{atomOccForgotten/2} which refers to an atom occurrence and which is included in the answer set if said atom occurrence is forgotten:
% \begin{align*}
%     \texttt{atomOccForgotten(A,L) :- atomTruthValue(A,L,} \forget_\top \texttt{).} \tag{AF26} \\
%     \texttt{atomOccForgotten(A,L) :- atomTruthValue(A,L,} \forget_\bot \texttt{).} \tag{AF27}
% \end{align*}
% With the aid of this predicate, we can now minimize the number of forgotten atom occurrences:
\begin{align*}
    \texttt{\#minimize\{1,A,L : atomOccForgotten(A,L)\}.} \tag{AF21}
\end{align*}

% The union of all rules defined above (summarized in (AF1)--(AF22) in Encoding \ref{tab:asp-fb}) constitute the extended logic program $\pf(\kb)$.
The union of the rules (AF1)--(AF21) defined above (see Encoding \ref{tab:asp-fb} in Appendix~\ref{app:proof-asp-fb} for an overview) constitute the extended logic program $\pf(\kb)$.
% We denote the set of atom occurrences that are replaced by $\top$ wrt.\ a knowledge base $\kb$ as $T_M$, and the set of atom occurrences that are replaced by $\bot$ as $F_M$.
With $M$ being an answer set of $\pf(\kb)$, we denote the set of atom occurrences that are forgotten as 
% With $M$ being an answer set of $\pf(\kb)$ we define
    % \begin{align*}
    %     T_M & = \{\varAt^l \mid \varAt \in \atoms(\kb), \texttt{atomTruthValue(} \varAt \texttt{,} l \texttt{,} \forget_\top \texttt{)} \in M\},\\
    %     F_M & = \{\varAt^l \mid \varAt \in \atoms(\kb), \texttt{atomTruthValue(} \varAt \texttt{,} l \texttt{,} \forget_\bot \texttt{)} \in M\}.
    % \end{align*}
    \begin{align*}
        F_M & = \{\varAt^l \mid \varAt \in \atoms(\kb), \texttt{atomOccForgotten(} \varAt \texttt{,} l \texttt{)}  \in M\}
    \end{align*}
\begin{theorem}\label{thm:asp-fb}
    Let $M_o$ be an optimal answer set of $\pf(\kb)$.
    % Then $|T_{M_o}| + |F_{M_o}| = \iforget(\kb)$.
    Then $|F_{M_o}| = \iforget(\kb)$.
\end{theorem}
The proof of the above theorem is provided in Appendix \ref{app:proof-asp-fb}.

\begin{example}\label{ex:asp-fb}
    Consider again the knowledge base $\kb_7$ from Example \ref{ex:sat-contension}:
    $$\mathcal{K}_7 = \left\{\underbrace{\overbrace{x}^{\phi_{1,1}} \land \overbrace{y}^{\phi_{1,2}}}_{\alpha_1}, \quad \underbrace{\overbrace{x}^{\phi_{2,1}} \lor \overbrace{y}^{\phi_{2,2}}}_{\alpha_2}, \quad \underbrace{\lnot \overbrace{x}^{\phi_{3}}}_{\alpha_3} \right\}$$
    We aim to construct $\pf(\kb_7)$, i.\,e., an ASP encoding which allows us to retrieve $\iforget(\kb_7)$.
    %
    % We label each atom occurrence and get the following labeled knowledge base $\kb_7^l$:
    % $$ \kb_7^l = \{ x^1 \land y^1, x^2 \lor y^2, \lnot x^3 \} $$
    To begin with, the formulas $\alpha_1$, $\alpha_2$, and $\alpha_3$ are represented using \emph{\texttt{kbMember/1}} (AF1), exactly as in Example \ref{ex:asp-contension}.
    Likewise, the representation of the conjunction (AF3), disjunction (AF4), and negation (AF5) is the same as in Example \ref{ex:asp-contension}.
    We label each atom occurrence to get $ \kb_7^l = \{ x^1 \land y^1, x^2 \lor y^2, \lnot x^3 \} $, and
    represent each atom occurrence via (AF2):
    % Each atom occurrence is represented via (AF2).
    \begin{align*}
        & \mathtt{formulaIsAtomOcc}(\phi_{1,1}, x, 1). \qquad \mathtt{formulaIsAtomOcc}(\phi_{1,2}, y, 1). \\
        & \mathtt{formulaIsAtomOcc}(\phi_{2,1}, x, 2). \qquad \mathtt{formulaIsAtomOcc}(\phi_{2,2}, y, 2). \\
        & \mathtt{formulaIsAtomOcc}(\phi_{3}, x, 3). 
    \end{align*}
    The remainder of the logic program, i.e., (AF6)--(AF21), is static.
\end{example}

\subsection{The Hitting Set Inconsistency Measure}\label{sec:asp-algo-hs}

The hitting set inconsistency measure $\ihs(\kb)$ is defined by the size of the minimal hitting set wrt.\ a knowledge base $\kb$, subtracted by $1$.
The maximal size of such a hitting set is determined by the number of formulas in $\kb$.
Further, $\omega_i$ refers to the $i$-th interpretation out of the $|\kb|$ possible interpretations we need to consider, assuming that the interpretations have an arbitrary, but fixed order.
We construct an extended logic program $\ph(\kb)$ which computes $\ihs(\kb)$ as follows.

We begin by representing each element $\varFor \in \kb$ (AH1) as well as each conjunction (AH4), disjunction (AH5), and negation (AH6) in the same manner described in the two preceding sections. % (see also (AH1) and (AH4)--(AH6) in Encoding \ref{tab:asp-hs}).
Atoms are represented as \texttt{atom/1} (AH2), and formulas consisting of individual atoms are represented as \texttt{formulaIsAtom/2} (AH7). %, as introduced in Section \ref{sec:asp-algo-c}.
The two classical truth values are encoded in the same fashion as shown in Section \ref{sec:asp-algo-fb} (AH9).
In addition, the hitting set inconsistency measure requires the use of interpretations. 
Therefore, we add $|\kb|$ interpretations as follows:
\begin{align*}
   \texttt{interpretation(1..}|\kb|\texttt{).} \tag{AH3}
\end{align*}
% Note that “..” is an interval operator indicating that \texttt{interpretation/1} gets instantiated with each number between $1$ and $|\kb|$.

Again, we need to model that each atom takes exactly one (classical) truth value. 
However, as opposed to the two previously discussed measures, wrt.\ $\ihs$ we need to take into account the formulas' interpretations.
Consequently, each atom requires a unique evaluation wrt.\ each interpretation.
We model this by introducing the predicate \texttt{truthValueInt/3} which represents the truth value of an atom wrt.\ a specific interpretation:
\begin{align*}
    \texttt{1\{truthValueInt(A,I,T) : tv(T)\}1 :- } & \texttt{atom(A),} \\
    & \texttt{interpretation(I).} \tag{AH10}
\end{align*}

The connective encodings for each (sub)formula follow classical propositional entailment.
Hence, the rules which model conjunction, disjunction, and negation wrt.\ $\ihs$ are essentially the same as those wrt.\ $\iforget$, but with an additional reference to an interpretation:
% \begingroup
% \allowdisplaybreaks
% \begin{align*}
%     & \texttt{truthValueInt(F,I,} t \texttt{) :- } \\
%         & \qquad \qquad \texttt{conjunction(F,G,H),} \\
%         & \qquad \qquad \texttt{interpretation(I),} \\
%         & \qquad \qquad \texttt{truthValueInt(G,I,} t \texttt{),}\\
%         & \qquad \qquad \texttt{truthValueInt(H,I,} t \texttt{).} \tag{AH11} \\
%     & \texttt{truthValueInt(F,I,} f \texttt{) :- } \\
%         & \qquad \qquad \texttt{conjunction(F,\_,\_),} \\
%         & \qquad \qquad \texttt{interpretation(I),} \\
%         & \qquad \qquad \texttt{not truthValueInt(F,I,} t \texttt{).} \tag{AH12} \\
%     & \texttt{truthValueInt(F,I,} f \texttt{) :- } \\
%         & \qquad \qquad \texttt{disjunction(F,G,H),} \\
%         & \qquad \qquad \texttt{interpretation(I),} \\
%         & \qquad \qquad \texttt{truthValueInt(G,I,} f \texttt{),} \\
%         & \qquad \qquad \texttt{truthValueInt(H,I,} f \texttt{).} \tag{AH13} \\
%     & \texttt{truthValueInt(F,I,} t \texttt{) :- } \\
%         & \qquad \qquad \texttt{disjunction(F,\_,\_),} \\
%         & \qquad \qquad \texttt{interpretation(I),} \\
%         & \qquad \qquad \texttt{not truthValueInt(F,I,} f \texttt{).} \tag{AH14} \\
%     & \texttt{truthValueInt(F,I,} t \texttt{) :- } \\
%         & \qquad \qquad \texttt{negation(F,G),} \\
%         & \qquad \qquad \texttt{truthValueInt(G,I,}f \texttt{).} \tag{AH15} \\
%     & \texttt{truthValueInt(F,I,} f \texttt{) :- } \\
%         & \qquad \qquad \texttt{negation(F,G),} \\
%         & \qquad \qquad \texttt{truthValueInt(G,I,}t \texttt{).} \tag{AH16}
% \end{align*}
% \endgroup
\begingroup
\allowdisplaybreaks
\begin{align*}
    \texttt{truthValueInt(F,I,} t \texttt{) :- } & \texttt{conjunction(F,G,H),} \\
        & \texttt{interpretation(I),} \\
        & \texttt{truthValueInt(G,I,} t \texttt{),}\\
        & \texttt{truthValueInt(H,I,} t \texttt{).} \tag{AH11} \\
    \texttt{truthValueInt(F,I,} f \texttt{) :- } & \texttt{conjunction(F,\_,\_),} \\
        & \texttt{interpretation(I),} \\
        & \texttt{not truthValueInt(F,I,} t \texttt{).} \tag{AH12} \\
    \texttt{truthValueInt(F,I,} f \texttt{) :- } & \texttt{disjunction(F,G,H),} \\
        & \texttt{interpretation(I),} \\
        & \texttt{truthValueInt(G,I,} f \texttt{),} \\
        & \texttt{truthValueInt(H,I,} f \texttt{).} \tag{AH13} \\
    \texttt{truthValueInt(F,I,} t \texttt{) :- } & \texttt{disjunction(F,\_,\_),} \\
        & \texttt{interpretation(I),} \\
        & \texttt{not truthValueInt(F,I,} f \texttt{).} \tag{AH14} \\
    \texttt{truthValueInt(F,I,} t \texttt{) :- } & \texttt{negation(F,G),} \\
        & \texttt{truthValueInt(G,I,}f \texttt{).} \tag{AH15} \\
    \texttt{truthValueInt(F,I,} f \texttt{) :- } & \texttt{negation(F,G),} \\
        & \texttt{truthValueInt(G,I,}t \texttt{).} \tag{AH16}
\end{align*}
\endgroup
If a formula is composed of an individual atom, it needs to be assigned the same truth value as the atom itself wrt.\ an interpretation:
% \begin{align*}
%     & \texttt{truthValueInt(F,I,T) :- } \\
%         & \qquad \qquad \texttt{formulaIsAtom(F,G),} \\
%         & \qquad \qquad \texttt{truthValueInt(G,I,T),} \\
%         & \qquad \qquad \texttt{interpretation(I),} \\
%         & \qquad \qquad \texttt{tv(T).} \tag{AH17}
% \end{align*}
\begin{align*}
    \texttt{truthValueInt(F,I,T) :- } & \texttt{formulaIsAtom(F,G),} \\
        & \texttt{truthValueInt(G,I,T),} \\
        & \texttt{interpretation(I),} \\
        & \texttt{tv(T).} \tag{AH17}
\end{align*}

In order to meet the definition of the hitting set measure, each formula $\varFor \in \kb$ must evaluate to $t$ wrt.\ at least one interpretation included in the hitting set.
We model this by using the predicate \texttt{truthValue/2} as follows:
%, which similarly used with $\icont$ (Section \ref{sec:asp-algo-c}) and $\iforget$ (Section \ref{sec:asp-algo-fb}).
% \begin{align*}
%     & \texttt{truthValue(F,} t \texttt{) :- } \\
%         & \qquad \qquad \texttt{truthValueInt(F,I,} t \texttt{),} \\
%         & \qquad \qquad \texttt{kbMember(F),} \\
%         & \qquad \qquad \texttt{interpretation(I),} \\
%         & \qquad \qquad \texttt{interpretationActive(I).} \tag{AH18}
% \end{align*}
\begin{align*}
    \texttt{truthValue(F,} t \texttt{) :- } & \texttt{truthValueInt(F,I,} t \texttt{),} \\
        & \texttt{kbMember(F),} \\
        & \texttt{interpretation(I),} \\
        & \texttt{interpretationActive(I).} \tag{AH18}
\end{align*}
The predicate \texttt{interpretationActive/1} serves the purpose of marking which interpretations are included in the final hitting set. 
Later on, we minimize the number of those ``active interpretations''. % included in the hitting set.
However, to ensure that at least one (and at most $|\kb|$) interpretations are included in the hitting set, we require the following cardinality constraint:
\begin{align*}
    \texttt{1\{interpretationActive(I) : interpretation(I)\}}|\kb| \texttt{.} \tag{AH8}
\end{align*}
To avoid symmetries introduced by (AH8), we add another integrity constraint:
% To ensure that ...
% -> es sollen immer die Interpretationen 1...x aktiv sein, nicht irgendwelche beliebigen 
\begin{align*}
    \texttt{:- interpretationActive(I), 1\,<\,I, not interpretationActive(I-1).} \tag{AH20}
\end{align*}

If a formula $\varFor \in \kb$ is not $t$ wrt.\ any interpretation, it is $f$. 
In other words, if \texttt{truthValue(}$\varFor$\texttt{,}$t$\texttt{)} is not included in the answer set, then $\varFor$ is false with regard to all interpretations:
\begin{align*}
    \texttt{truthValue(F,}f \texttt{) :- kbMember(F), not truthValue(F,}t \texttt{).} \tag{AH19}
\end{align*}
However, since every formula in $\kb$ must be satisfied by at least one interpretation, we need to include the following integrity constraint:
\begin{align*}
    \texttt{:- truthValue(F,}f \texttt{),kbMember(F).} \tag{AH21}
\end{align*}

At last, we minimize the number of interpretations that are required to be ``active'' in order to satisfy all formulas in the given knowledge base:
\begin{align*}
    \texttt{\#minimize\{1,I : interpretationActive(I)\}.} \tag{AH22}
\end{align*}

We define $\ph(\kb)$ to be the extended logic program specified by the union of all rules defined in (AH1)--(AH22) (see Encoding \ref{tab:asp-hs} in Appendix~\ref{app:proof-asp-hs} for a complete list of all rules).
% Further, let $\ph'(\kb)$ be equivalent to $\ph(\kb)$, but without the minimize statement (AH21).
Let $M$ be an answer set of $\ph(\kb)$.
% We define the set of all instances of \texttt{interpretationActive(}$\omega_i$\texttt{)} included in $M$ as
% $$ A(M) = \{ \texttt{interpretationActive(}\omega_i\texttt{)} \in M \} $$
% with $i \in \{1, \ldots, |\kb|\}$.
% In other words, $A(M)$ describes the set of ASP representations of those interpretations $\omega_i$ that are included in the answer set.
We define the set of interpretations $\omega_i$ represented in $M$ as 
$ \Omega(M) = \{ \omega_i \mid \texttt{interpretationActive(}\omega_i^{\mathrm{asp}}\texttt{)} \in M \} $, with $\omega_i^{\mathrm{asp}}$ being an ASP representation of $\omega_i$.
The proof of the following theorem can be found in Appendix \ref{app:proof-asp-hs}.
\begin{theorem}\label{thm:asp-hs}
    Let $M_o$ be an optimal answer set of $\ph(\kb)$.
    Then $|\Omega(M_o)| -1 = \ihs(\kb)$.
    If no answer set of $\ph(\kb)$ exists, $\ihs(\kb) = \infty$.
\end{theorem}
% The proof of the theorem above can be found in Appendix \ref{app:proof-asp-hs}.

\begin{example} \label{ex:asp-hs}
    Consider again the knowledge base $\kb_7$ from Example \ref{ex:sat-contension}:
    $$\mathcal{K}_7 = \left\{\underbrace{\overbrace{x}^{\phi_{1,1}} \land \overbrace{y}^{\phi_{1,2}}}_{\alpha_1}, \quad \underbrace{\overbrace{x}^{\phi_{2,1}} \lor \overbrace{y}^{\phi_{2,2}}}_{\alpha_2}, \quad \underbrace{\lnot \overbrace{x}^{\phi_{3}}}_{\alpha_3} \right\}$$
    Our aim is to construct $\ph(\kb_7)$, i.\,e., an ASP encoding which allows us to retrieve the $\ihs(\kb_7)$.
    The formulas in $\kb_7$ are represented the same way as in Example \ref{ex:asp-contension} and \ref{ex:asp-fb}, using \emph{\texttt{kbMember/1}} (AH1).
    Atoms are represented like in Example \ref{ex:asp-contension} as well (by using \emph{\texttt{atom/1}}).
    Likewise, subformulas consisting of individual atoms (AH7) are modeled as in Example \ref{ex:asp-contension}.
    Moreover, the conjunction (AH4), disjunction (AH5), and negation (AH6) are also defined as before.
    However, we now need to define $|\kb_7| = 3$ interpretations via~(AH3):
    $$ \mathtt{interpretation(1 .. 3).} $$
    Also, we ensure that at least one and at most $|\kb_7| = 3$ interpretations are included in the hitting set (AH8):
    $$ \mathtt{1\{ interpretationActive(X) : interpretation(X)\}3.} $$
    The remainder of the logic program, i.e., (AH9)--(AH22), is static.
\end{example}

\subsection{The Max-Distance Inconsistency Measure}\label{sec:asp-algo-dmax}

% The three distance-based inconsistency measures covered in this paper all aim to find an interpretation with an ``optimal'' distance to the models of the formulas contained in a given knowledge base $\kb$.
% At first we examine the max-distance inconsistency measure, which 
Recall that, given a knowledge base $\kb$, the max-distance inconsistency measure $\imdalal$ aims to find an interpretation that has a minimal maximum distance to the models of the formulas $\varFor \in \kb$.
We construct an extended logic program $\pmax (\kb)$ that calculates $\imdalal (\kb)$ as described in the following.

Atoms are represented as \texttt{atom/1}, as described previously in the context of $\icont$, $\iforget$, and $\ihs$ (see also (ADM2) in Encoding \ref{tab:asp-dmax}).
Conjunctions, disjunctions, negations, and formulas consisting of single atoms are handled exactly as presented in the previous section (see (ADM4)--(ADM7) and (ADM11)--(ADM17) in Encoding \ref{tab:asp-dmax}). 
% In the context of $\imdalal$, we (ADM4)--(ADM7) correspond exactly to (AH4)--(AH7), and (ADM11)--(ADM17) correspond to (AH11)--(AH17).
% Section \ref{sec:asp-algo-hs}, which covers the ASP encoding of $\ihs$ (see (ADM4)--(ADM7), and (ADM11)--(ADM17)).
Truth values are again defined by \texttt{tv/1} (ADM9), and unique atom evaluation is also expressed as for $\ihs$ (ADM10).
Moreover, we use the predicate \texttt{interpretation/1} %, which was introduced for $\ihs$ as well, 
to represent $|\kb| + 1$ interpretations:
\begin{align*}
    \texttt{interpretation(0..}|\kb| \texttt{).} \tag{ADM3}
\end{align*}
As opposed to the encoding of $\ihs$, which only requires $|\kb|$ interpretations, the encoding of $\imdalal$ requires $|\kb|+1$ interpretations. 
% This is due to the fact that the maximum finite value of $\ihs(\kb)$ is limited by the number of formulas in $\kb$.
% $\imdalal$, on the other hand, aims to find an interpretation with the smallest maximum distance to the models of the formulas in $\kb$. 
This is due to the fact that for $\imdalal$ 
% Consequently, 
we first need to provide $|\kb|$ interpretations in order to ensure that each formula is satisfied by at least one interpretation (i.\,e., that we have at least one model for each formula).
However, the interpretation with the smallest maximum distance to the formulas' models is not necessarily a model of one of the formulas itself. 
Thus, we need to provide one additional interpretation and end up with a total of $|\kb|+1$ interpretations.

As stated previously, each $\varFor \in \kb$ must be satisfied by at least one interpretation.
To achieve this, we demand that the $i$-th formula in $\kb$ must be satisfied by the $i$-th interpretation.
In order to model this in ASP, we first add a fact for every $\varFor \in \kb$ which additionally incorporates an index $i$:
\begin{align*}
    \texttt{kbMember(} \varFor \texttt{,} i \texttt{).} \tag{ADM1}
\end{align*}
% Now we can formulate a rule which connects the truth value \texttt{T} of a formula \texttt{F} under an interpretation \texttt{I} (modeled by \texttt{truthValueInt/3}) to the formula index \texttt{L}, given that the indexed formula is actually an element of $\kb$ (represented by \texttt{kbMember/2}).
% \begin{align*}
%     & \texttt{truthValueInt(F,L,I,T) :- } \\
%         & \qquad \qquad \texttt{kbMember(F,L),} \\
%         & \qquad \qquad \texttt{interpretation(I),} \\
%         & \qquad \qquad \texttt{tv(T),} \\
%         & \qquad \qquad \texttt{truthValueInt(F,I,T).} \tag{ADM18}
% \end{align*}
Further, we need to add an integrity constraint which ultimately ensures that the formula with index $i$ cannot be set to $f$ (i.e., must evaluate to $t$) under the $i$-th interpretation:
% \begin{align*}
%     & \texttt{ :- } \\
%     & \qquad \qquad \texttt{truthValueInt(F,L,I,}f \texttt{),} \\
%     & \qquad \qquad \texttt{kbMember(F,L),} \\
%     & \qquad \qquad \texttt{interpretation(I),} \\
%     & \qquad \qquad \texttt{L == I.} \tag{ADM19}
% \end{align*}
\begin{align*}
    \texttt{ :- } & \texttt{truthValueInt(F,I,}f \texttt{),} \\
        & \texttt{kbMember(F,I).} \tag{ADM18} 
\end{align*}

We know that the $i$-th formula in $\kb$ must be satisfied by the $i$-th interpretation.
Thus, each interpretation $\omega_i$ with $i \in \{0, \ldots, |\kb|-1\}$ is a model of at least one formula in $\kb$.
We are now looking for an interpretation $\omega_{|\kb|}$ with the smallest maximum distance to the models of each formula. 
% diff/1 ist hier eine Art ID und spiegelt nicht direkt wider, ob ein bestimmtes Atom "unterschiedlich" ist (schließlich wird immer für jedes Atom ein diff hergeleitet)
% maximal kann I_d^max |At(K)| betragen, da im worst case jedes Atom untersch. ist
%   -> dies wird folgendermaßen modelliert: 
In the worst case, $\imdalal(\kb) = |\atoms(\kb)|$. %, i.\,e., every atom in $\omega_{|\kb|}$ has a different truth value than . 
We represent each possible value as follows.
\begin{align*}
    \texttt{diff(1..X) :- } & \texttt{X = \#count\{A: atom(A)\}.} \tag{ADM19}
\end{align*}
% Nun können wir *pro Interpretation I* den Wert d berechnen 
% -> das #count-Aggregate berechnet, für wie viele Atome der Wahrheitswert bzgl. I und |K| unterschiedlich ist
%   -> es wird also die Dalal-Distanz zw. I und |K| berechnet 
% Um d(X) für jedes diff(X) herzuleiten, müsste die Dalal-Distanz jeden Wert zw. 1 und |\kb| genau ein mal annehmen 
%   -> in diesem Fall wäre der IM-Wert auch maximal, da für jedes verschiedene X 1 addiert wird 
% Note that \texttt{\#count} is a so-called \textit{aggregate}, a built-in function that is applied to sets.
% Specifically, the \texttt{\#count} aggregate, as the name suggests, allows for counting the number of elements in a set for which certain conditions hold. 
By means of the \texttt{\#count} aggregate above, we count the number of (ground) \texttt{atom/1} instances.
We can now calculate the distance between $\omega_{|\kb|}$ and the models of each formula using another \texttt{\#count} aggregate:
\begin{align*}
    \texttt{d(X) :- } & \texttt{diff(X),} \\
        & \texttt{interpretation(I),} \\
        & \texttt{X <= \#count\{A: atom(A), } \\
        & \qquad \texttt{truthValueInt(A,I,T),} \\
        & \qquad \texttt{not truthValueInt(A,}|\kb|\texttt{,T)\}.} \tag{ADM8}
\end{align*}
% Indem man zählt, wie viele verschiedene Werte X annehmen kann und diesen Wert minimiert, hat man auch gleichzeitig den maximalen Wert erhalten 
% -> diff ist wichtig, damit kein Wert doppelt gezählt wird 
Essentially, \texttt{d/1} checks how many different values the distance between $\omega_{|\kb|}$ and each model corresponding to a formula can take. 
% Essentially, \texttt{d/1} represents the different distances between $\omega_{|\kb|}$ and the models of each formula (i.\,e., each $\omega_i$, $0 \leq i < |\kb|$). 
Since \texttt{diff/1} ensures that \texttt{d/1} is only derived once for a given value, we can count how many different values (i.\,e., different distances) are calculated.
By minimizing this value, we indirectly get the minimal maximum distance: % (with \texttt{diff/1} ensuring that \texttt{d/1} is only derived once for a given value):
\begin{align*}
    % \texttt{\#minimize\{1,X : d(X)\}.} \tag{ADM20}
    \texttt{\#minimize\{1,X : d(X)\}.} \tag{ADM20}
\end{align*}

Let $\pmax$ be the extended logic program specified by the union of rules (ADM1)--(ADM20) (see Encoding \ref{tab:asp-dmax} in Appendix~\ref{app:proof-asp-max-dist} for a complete overview).
Further, let $D^{\max}_M = \{n \mid \texttt{d(}n\texttt{)} \in M, n\in \mathbb{N}\}$ wrt.\ an answer set $M$.
The proof of the theorem below is given in Appendix~\ref{app:proof-asp-max-dist} as well.

% \begin{theorem}\label{thm:asp-maxdalal}
%     Let $M_o$ be an optimal answer set of $\pmax(\kb)$.
%     Then $D^{\max}_{M_o} = \imdalal(\kb)$ with \texttt{dMax(}$D^{\max}_{M_o}$\texttt{)} $\in M_o$.
%     If no answer set of $\pmax(\kb)$ exists, $\imdalal(\kb) = \infty$.
% \end{theorem}
\begin{theorem}\label{thm:asp-maxdalal}
    Let $M_o$ be an optimal answer set of $\pmax(\kb)$.
    Then $|D^{\max}_{M_o}| = \imdalal(\kb)$.
    % Then $\max\{D^{\max}_{M_o}\} = \imdalal(\kb)$.
    If no answer set of $\pmax(\kb)$ exists, $\imdalal(\kb) = \infty$.
\end{theorem}

\begin{example} \label{ex:asp-d-max}
    Consider again the knowledge base $\kb_7$ from Example \ref{ex:sat-contension}:
    $$\mathcal{K}_7 = \left\{\underbrace{\overbrace{x}^{\phi_{1,1}} \land \overbrace{y}^{\phi_{1,2}}}_{\alpha_1}, \quad \underbrace{\overbrace{x}^{\phi_{2,1}} \lor \overbrace{y}^{\phi_{2,2}}}_{\alpha_2}, \quad \underbrace{\lnot \overbrace{x}^{\phi_{3}}}_{\alpha_3} \right\}$$
    In this example, we aim to construct $\pmax(\kb_7)$, i.\,e., an ASP encoding which allows us to retrieve $\imdalal(\kb_7)$.
    Following (ADM1), we assign each formula in $\kb_7$ an index $i \in \{0, \ldots, |\kb_7|-1\}$, i.\,e., $i \in \{0, \ldots, 2 \}$:
        $$ \mathtt{kbMember(}\alpha_1 \mathtt{, 0).} \qquad \mathtt{kbMember(}\alpha_2 \mathtt{, 1).} \qquad \mathtt{kbMember(}\alpha_3 \mathtt{, 2).} $$
    Moreover, we need to define $|\kb_7|+1$ (i.e., $4$) interpretations (ADM3):
        $$ \mathtt{interpretation(0..3).} $$
    Atoms are represented using \emph{\texttt{atom/1}} (ADM2), as we saw in Example \ref{ex:asp-contension} and \ref{ex:asp-hs}.
    The same applies to the conjunction (ADM4), disjunction (ADM5), and negation (ADM6), as well as all (sub)formulas consisting of individual atoms (ADM7).
    
    % We define maximum distance between the interpretation with index $|\kb_7| = 3$ and the models of $\kb_7$ (ADM8):
    % % \begin{align*}
    % %     & \texttt{dMax(X) :- } \\
    % %     & \qquad \qquad \texttt{X = \#max\{Y : d(I,3,Y), interpretation(I)\}, } \\
    % %     & \qquad \qquad \texttt{X >= 0.} 
    % % \end{align*}
    % \begin{align*}
    %     \texttt{dMax(X) :- } & \texttt{X = \#max\{Y : d(I,3,Y), interpretation(I)\}, } \\
    %     & \texttt{X >= 0.} 
    % \end{align*}
    
    The remainder of the program, i.\,e., (ADM8)--(ADM20), is static.
\end{example}

\subsection{The Sum-Distance Inconsistency Measure}\label{sec:asp-algo-dsum}

% The sum-distance inconsistency measure $\isdalal$ is almost identical to the max-distance inconsistency measure $\imdalal$.
% The only difference between the two measures is the definition of the ``optimal'' distance.
The difference between the sum-distance inconsistency measure $\isdalal$ and the max-distance inconsistency measure $\imdalal$ is the definition of the ``optimal'' distance.
While $\imdalal$ aims to find the smallest maximum distance, $\isdalal$ seeks to find the smallest sum of all distances between an interpretation $\omega_{|\kb|}$ and the models of all $\varFor \in \kb$ (i.\,e., $\omega_i$ with $0 \leq i < |\kb|$).
% Consequently, the ASP encodings of the two measures are almost identical as well.
% We simply replace \texttt{dMax/1} (as introduced in the previous Section \ref{sec:asp-algo-dmax}) by the predicate \texttt{dSum/1}, which we define as follows:
% % \begin{align*}
% %     & \texttt{dSum(X) :- }\\
% %         & \qquad \qquad \texttt{X = \#sum\{Y,I : d(I,} |\kb| \texttt{,Y), interpretation(I)\}, } \\
% %         & \qquad \qquad \texttt{X >= 0.} \tag{ADS8}
% % \end{align*}
% \begin{align*}
%     \texttt{dSum(X) :- } & \texttt{X = \#sum\{Y,I : d(I,} |\kb| \texttt{,Y), interpretation(I)\}, } \\
%         & \texttt{X >= 0.} \tag{ADS8}
% \end{align*}
% Besides, \texttt{dMax(X)} in the final minimize statement (ADM22) is replaced by \texttt{dSum(X)}, yielding
% \begin{align*}
%     \texttt{\#minimize\{X : dSum(X)\}.} \tag{ADS21}
% \end{align*}
As opposed to $\imdalal$, we do not need to model the different distances explicitly---we can model the summation directly by means of the following minimize statement:
\begin{align*}
    \texttt{\#minimize\{1,A,I: }&\texttt{atom(A),} \\
    & \texttt{truthValueInt(A,I,T),} \\
    & \texttt{not truthValueInt(A,}|\kb|\texttt{,T)\}.} \tag{ADS8}
\end{align*}
Intuitively, we consider all models $\omega_i$ and sum up the number of differing atom valuations wrt.\ $\omega_{|\kb|}$.
(Each differing atom valuation per $\omega_i$ is counted as $1$.)
Due to the minimization, we get the minimal sum of distances.

% TODO: 
% - anmerken, welche Regeln aus maxdalal nicht mehr benötigt werden?

Let $\psum$ be the extended logic program specified by the union of rules (ADS1)--(ADS18) (see Encoding \ref{tab:asp-dsum} in Appendix~\ref{app:proof-asp-sum-dist} for a complete overview).
Further, let $M$ be an answer set of $\psum$, and let $\theta \in \{t, f\}$.
We define 
$$D^\Sigma_M = \{ \mathrm{diff}_M(\varAt,i) \mid \texttt{truthValueInt(}\varAt, i, \theta\texttt{)} \in M, \texttt{truthValueInt(}\varAt, |\kb|, \theta\texttt{)} \notin M \}.$$
% with $\theta \in \{t, f\}$.
The proof of the following theorem is provided in Appendix \ref{app:proof-asp-sum-dist}.

% \begin{theorem}\label{thm:asp-sumdalal}
%     Let $M_o$ be an optimal answer set of $\psum(\kb)$.
%     Then $D^{\Sigma}_{M_o} = \isdalal(\kb)$ with \texttt{dSum(}$D^{\Sigma}_{M_o}$\texttt{)} $\in M_o$.
%     If no answer set of $\psum(\kb)$ exists, $\isdalal(\kb) = \infty$.
% \end{theorem}
\begin{theorem}\label{thm:asp-sumdalal}
    Let $M_o$ be an optimal answer set of $\psum(\kb)$.
    Then $|D^{\Sigma}_{M_o}| = \isdalal(\kb)$.
    If no answer set of $\psum(\kb)$ exists, $\isdalal(\kb) = \infty$.
\end{theorem}

\begin{example} \label{ex:asp-d-sum}
    Consider again the knowledge base $\kb_7$ from Example \ref{ex:sat-contension}:
    $$\mathcal{K}_7 = \left\{\underbrace{\overbrace{x}^{\phi_{1,1}} \land \overbrace{y}^{\phi_{1,2}}}_{\alpha_1}, \quad \underbrace{\overbrace{x}^{\phi_{2,1}} \lor \overbrace{y}^{\phi_{2,2}}}_{\alpha_2}, \quad \underbrace{\lnot \overbrace{x}^{\phi_{3}}}_{\alpha_3} \right\}$$
    In this example, we aim to construct $\psum(\kb_7)$, i.\,e., an ASP encoding which allows us to retrieve the sum-distance inconsistency value of $\kb_7$.
    The formulas in $\kb_7$, the atoms in $\atoms(\kb_7)$, the interpretations, the conjunction, the disjunction, the negation, and the subformulas consisting of individual atoms are encoded as described in (ADS1)--(ADS7) (see Encoding \ref{tab:asp-dsum}), which corresponds exactly to (ADM1)--(ADM7) wrt.\ $\imdalal$ (see Example \ref{ex:asp-d-max}).
    % Consequently, the resulting rules are the same as in Example \ref{ex:asp-d-max}.
    %
%     Following (ADS8), we define the sum of distances between the interpretation with index $|\kb_7|=3$ and the models of $\kb_7$:
% %     \begin{align*}
% %     & \texttt{dSum(X) :- } \\
% %         & \qquad \qquad \texttt{X = \#sum\{Y,I : d(I,3,Y), interpretation(I)\}, } \\
% %         & \qquad \qquad \texttt{X >= 0.} 
% % \end{align*}
%     \begin{align*}
%         \texttt{dSum(X) :- } & \texttt{X = \#sum\{Y,I : d(I,3,Y), interpretation(I)\}, } \\
%             & \texttt{X >= 0.} 
%     \end{align*}
    %
    Since $|\kb| = 3$, the minimize statement (ADM8) is expressed as follows.
    \begin{align*}
        \emph{\texttt{\#minimize\{1,A,I: }}&\mathtt{atom(A),} \\
        & \mathtt{truthValueInt(A,I,T),} \\
        & \mathtt{not truthValueInt(A,3,T)\}.} 
    \end{align*}
    
    The remaining rules, i.e., (ADS9)--(ADS18), are static.
\end{example}

\subsection{The Hit-Distance Inconsistency Measure}\label{sec:asp-algo-dhit}

% The hit-distance inconsistency measure $\ihdalal(\kb)$ wrt.\ a knowledge base $\kb$ aims to find the smallest number of distances between an interpretation and the models of each $\varFor \in \kb$ which are greater than $0$.
% In other words, $\ihdalal$ is equal to the minimal number of formulas that need to be removed from the knowledge base in order to render it consistent.
As for the SAT-based approach, in order to model the hit-distance inconsistency measure $\ihdalal$ in ASP, we use the characterization that $\ihdalal(\kb)$ is equal to the minimal number of formulas that need to be removed from the given knowledge base $\kb$ in order to render it consistent.
Utilizing the latter characterization, we construct an extended logic program $\phit$ which calculates $\ihdalal(\kb)$ as described below.

To start with, atoms, conjunctions, disjunctions, negations, and formulas consisting of single atoms are represented exactly as presented before wrt.\ $\imdalal$ and $\isdalal$ (see also (ADH2)--(ADH6) in Encoding \ref{tab:asp-dhit}).
Likewise, we define the two truth values via \texttt{tv/1} (ADH7).
Moreover, we need to ensure once again that each atom is assigned a distinct truth value.
This is realized in the same manner as for $\icont$ and $\iforget$ (ADH8).

As we do not need to take into account multiple interpretations at the same time, the assignment of truth values as well as the functionality of the connectives does not need to include a representation of the concept of interpretations. % as well. 
% Thus, we simply need to encode classical propositional entailment.
Thus, conjunction, disjunction, and negation are therefore modeled in the same manner as presented for $\iforget$ (ADH9)--(ADH14).
If a formula consists of an individual atom, its evaluation is modeled in the same way as presented for $\icont$ (ADH15).

As opposed to the encodings of the other two distance-based measures, elements of the given knowledge base $\kb$ do not require an index in the encoding of $\ihdalal$.
Thus, instead of representing $\varFor \in \kb$ in combination with an index $i$ as \texttt{kbMember(}$\varFor$\texttt{,}$i$\texttt{)}, we simply use \texttt{kbMember(}$\varFor$\texttt{)} (ADH1), like we saw previously in the encodings of $\icont$, $\iforget$, and $\ihs$.

% Because of the characterization of $\ihdalal$ as the minimal number of formulas that need to be removed from $\kb$ to make it consistent, we do not need to include the calculation of the Dalal distance in $\phit$.
% Instead, 
The objective is to minimize the number of formulas in $\kb$ which are evaluated to $f$.
To achieve this, we first define a rule that extracts the truth values of those formulas which are elements of $\kb$.
% \begin{align*}
%     & \texttt{truthValueKbMember(F,T) :- } \\
%         & \qquad \qquad \texttt{kbMember(F),} \\
%         & \qquad \qquad \texttt{tv(T),} \\
%         & \qquad \qquad \texttt{truthValueInt(F,T).} \tag{ADH16}
% \end{align*}
\begin{align*}
    \texttt{truthValueKbMember(F,T) :- } & \texttt{kbMember(F),} \\
        & \texttt{tv(T),} \\
        & \texttt{truthValue(F,T).} \tag{ADH16}
\end{align*}
% Now we can find the minimal number of formulas which are evaluated to $f$---i.e., the minimal number of formulas which need to be removed in order to make $\kb$ consistent---by using the following minimize statement:
Now we can minimize the number of formulas which are evaluated to $f$---i.e., we minimize the number of formulas which need to be removed in order to make $\kb$ consistent.
\begin{align*}
    & \texttt{\#minimize\{1,F : truthValueKbMember(F,}f\texttt{)\}.} \tag{ADH17} \\
\end{align*}

Let $\phit$ be the extended logic program specified by the union of rules (ADH1)--(ADH17) (see Encoding \ref{tab:asp-dhit} in Appendix~\ref{app:proof-asp-hit-dist} for a complete overview).
Further, with $M$ being an answer set, we define
    $ K_M = \{ \varFor^\mathrm{asp} \mid \varFor \in \kb, \texttt{truthValueKbMember(}\varFor^\mathrm{asp},f \texttt{)} \in M \} $
with $\varFor^\mathrm{asp}$ being an ASP representation of $\varFor$.
The proof of the theorem below can also be found in Appendix \ref{app:proof-asp-hit-dist}.

\begin{theorem}\label{thm:asp-hitdalal}
    Let $M_o$ be an optimal answer set of $\phit(\kb)$.
    Then $|K_{M_o}| = \ihdalal(\kb)$.
\end{theorem}

\begin{example} \label{ex:asp-d-hit}
    Consider again the knowledge base $\kb_7$ from Example \ref{ex:sat-contension}:
    $$\mathcal{K}_7 = \left\{\underbrace{\overbrace{x}^{\phi_{1,1}} \land \overbrace{y}^{\phi_{1,2}}}_{\alpha_1}, \quad \underbrace{\overbrace{x}^{\phi_{2,1}} \lor \overbrace{y}^{\phi_{2,2}}}_{\alpha_2}, \quad \underbrace{\lnot \overbrace{x}^{\phi_{3}}}_{\alpha_3} \right\}$$
    In this example, we aim to construct $\phit(\kb_7)$, i.\,e., an ASP encoding which allows us to retrieve $\ihdalal(\kb_7)$.
    At first, we use \emph{\texttt{kbMember/1}} (ADH1) to represent the formulas $\alpha_1$, $\alpha_2$, and $\alpha_3$ in the same fashion as wrt.\ $\icont$, $\iforget$, and $\ihs$ (see Example \ref{ex:asp-contension}).
    Further, the atoms, the conjunction, the disjunction, the negation, and the subformulas consisting of individual atoms (see (ADH2)--(ADH6)) are represented exactly as wrt.\ $\icont$, $\ihs$, $\imdalal$ and $\isdalal$ (see again Example \ref{ex:asp-contension}).
    The remaining rules, i.e., (ADS7)--(ADS17), are static.
\end{example}

\section{Evaluation}\label{sec:evaluation}

The aim of our experimental evaluation is to compare the proposed SAT-based and ASP-based approaches to each other wrt.\ all six inconsistency measures considered in this work.
We additionally compare our methods to naive baseline algorithms (see Section \ref{sec:baseline-approaches}) in terms of runtime.
Although we can expect both the SAT-based and the ASP-based approaches to be superior to the baseline methods (regarding ASP this has already been demonstrated for $\icont$, $\iforget$, and $\ihs$ \cite{kuhlmann2021algorithms}, and regarding SAT this has been shown for $\icont$ in \cite{kuhlmann2022comparison}), we still draw this comparison in order to concretely quantify this assumption.
Besides, the baseline algorithms are, to the best of our knowledge, the only existing implementations for the inconsistency measures in question. 
However, as the result of comparing the SAT and ASP methods is far less predictable (both SAT and ASP are established formalisms for dealing with problems on the first level of the polynomial hierarchy), we examine the two approaches more closely.
% In addition to comparing the overall runtimes of the different approaches, we consider how the runtimes are composed.
% For instance, we inspect how much time the respective SAT/ASP solvers require, or the time it takes to compute the encodings. 
% Moreover, we analyze some SAT- and ASP-specific aspects.
% Regarding SAT, we investigate the impact of the choice of search algorithm by comparing the binary search procedure introduced in Section \ref{sec:sat-scheme} with a linear one.
% With regard to ASP, we review the comparison of the previous ASP-based approaches for $\icont$ \cite{kuhlmann2020algorithm,kuhlmann2021algorithms} and the latest one (which is also used in this work) drawn in \cite{kuhlmann2022comparison}.

% To begin with, we outline our experimental setup, including the data sets we compiled, the hardware we used, and a few implementation details in Section \ref{sec:experimental-setup}.
% In Section \ref{sec:baseline-approaches}, we describe the naive baseline algorithms.
% We proceed by presenting our results (Section \ref{sec:results}), and subsequently discussing our findings (Section \ref{sec:discussion}).

\subsection{Experimental Setup}\label{sec:experimental-setup}

% Due to the unavailability of real-world data sets for inconsistency measurement, we compiled a series of different benchmark data sets, either by artificially generating knowledge bases or by translating benchmark data from other fields. 
Due to the unavailability of standard benchmarking sets for inconsistency measurement, we compiled a series of different data sets, either by artificially generating knowledge bases or by translating benchmark data from other fields. 
For all these data sets the goal was to have benchmark instances that are (mostly) inconsistent and feature a structure that can be expected in real-world applications. 
We compiled five data sets in total that we briefly describe in the following.
Note that the SRS data set and the ML data set were already used in \cite{kuhlmann2022comparison}; the remaining three data sets are novel.\footnote{Download (all data sets): \url{https://fernuni-hagen.sciebo.de/s/sML2faFBiCib1nm}}
\begin{description}
\item[SRS] Data set SRS consists of knowledge bases randomly generated using the \textit{SyntacticRandomSampler} provided by \emph{TweetyProject}\footnote{\url{http://tweetyproject.org/api/1.19/org/tweetyproject/logics/pl/util/SyntacticRandomSampler.html}}. 
This generator randomly generates a single propositional formula $\varSub$ as follows. 
Given input parameters $p_d,p_c,p_n\in [0,1]$ with $p_d+p_c+p_n\leq 1$, it is randomly chosen whether $\varSub$ is a disjunction $\varSub=\varSubSub_1\vee \varSubSub_2$ (with probability $p_d$), a conjunction $\varSub=\varSubSub_1\wedge \varSubSub_2$ (with probability $p_c$), a negation $\varSub=\neg \varSubSub_1$ (with probability $p_n$), or a proposition $\varSub=\varAt$ (with probability $1-p_d-p_c-p_n$; the exact proposition is chosen uniformly at random from some given set of propositions). 
If $\varSub$ is not a proposition then this process is repeated for the subformula $\varSubSub_1$ (and additionally $\varSubSub_2$ in the first two cases), where the probabilities for the first three cases are multiplied by some discount parameter $d\in(0,1)$ (in order to ensure that the process will, in practice, terminate at some point).

We created a total of $1800$ knowledge bases using the \textit{SyntacticRandomSampler}.
In order to include instances of varying complexity, we used different parameter settings to create nine sets which each consist of $200$ knowledge bases.
To be specific, we used different signature sizes and different numbers of formulas per knowledge base. 
The parameters $p_d$, $p_c$, and $p_n$ were always set to $0.3$.
A more detailed overview of the composition of the SRS data set is given in Table \ref{tab:RAN-dataset}.

\begin{table}
    \centering
    \setlength\extrarowheight{1pt}
    \begin{tabular}{|C{.13\textwidth} | C{.13\textwidth} | R{.13\textwidth} | R{.13\textwidth} |R{.13\textwidth} | R{.13\textwidth} |}
    \hline
    Signature size & Formulas / KB & Mean sig.\ size / formula & Max.\ sig.\ size / formula & Mean \#atom occ.\ / formula & Max \#atom occ.\ / formula  \\
    \hline
    \hline
    $3$  & $5$--$15$ 	& $1.69$ & $3$ & $2.23$ & $6$ \\ \hline
    $5$  & $15$--$25$ 	& $2.32$ & $5$ & $3.11$ & $11$ \\ \hline
    $10$ & $15$--$25$	& $2.70$ & $8$ & $3.13$ & $10$ \\ \hline
    $15$ & $15$--$25$ 	& $2.80$ & $9$ & $3.11$ & $11$ \\ \hline
    $15$ & $25$--$50$ 	& $2.81$ & $9$ & $3.11$ & $11$ \\ \hline
    $20$ & $25$--$50$ 	& $2.88$ & $11$ & $3.11$ & $11$ \\ \hline
    $25$ & $25$--$50$   & $2.91$ & $12$ & $3.09$ & $13$ \\ \hline
    $25$ & $50$--$100$  & $2.92$ & $10$ & $3.11$ & $11$ \\ \hline
    $30$ & $50$--$100$  & $2.94$ & $11$ & $3.10$ & $13$ \\
    \hline
    \end{tabular}
    \caption{Overview of the sets of knowledge bases making up data set SRS. Each row represents a set of $200$ knowledge bases. The columns describe (from left to right) the signature size, the number of formulas per knowledge base, the mean (resp.\ maximum) signature size per formula, and the mean (resp.\ maximum) number of atom occurrences per formula.
    }
    \label{tab:RAN-dataset}
\end{table}

\item[ML] The ML data set consists of knowledge bases containing formulas learnt from machine learning data \cite{Thimm:2020a}. 
More precisely, the data set ``Animals with attributes''\footnote{\url{http://attributes.kyb.tuebingen.mpg.de}} describes $50$ animals, e.\,g. ox, mouse, or dolphin, using $85$ binary attributes such as ``swims'', ``black'', and ``arctic''. 
We used the Apriori algorithm \cite{Agrawal:1994} to mine association rules from this data set for a given minimal confidence value $c$ and minimal support value $s$. 
All these rules were then interpreted as propositional logic implications. 
Finally, we selected one animal at random and added all its attributes as facts (thus making the knowledge base inconsistent as even rules with low confidence values were interpreted as strict implications). 
We set 
\begin{align*}
    c\in\{0.6, 0.65, 0.70, 0.75, 0.8, 0.85, 0.90, 0.95\}, \\
    s\in \{0.6, 0.65, 0.70, 0.75, 0.8, 0.85, 0.90, 0.95\},
\end{align*}
and allowed maximally $4$ literals per rule. 
The final data set contains $1920$ instances.

\item[ARG] Data set ARG consists of a total of $326$ knowledge bases extracted from benchmark data of the International Competition on Computational Models of Argumentation 2019 (ICCMA'19)\footnote{\url{http://argumentationcompetition.org/2019/}}. 
An abstract argumentation framework \cite{Dung:1995} is a directed graph $F=(A,R)$ where $A$ is a set of arguments, and $R$ models a conflict relation between such arguments. 
A computational task here is to find a \emph{stable extension}, i.\,e., a set $E\subseteq A$ with $(a,b)\notin R$ for all $a,b\in E$ and $(a,c)\in R$ for all $c\in A\setminus E$ and some $a\in E$. 
For each instance from ICCMA'19, we encoded the instance and the problem of finding such a stable extension via the approach from \cite{Besnard:2014} and, additionally, added constraints to ensure that 20\% of randomly selected arguments have to be contained in $E$. 
Note that the latter constraints usually make the knowledge base inconsistent.
    
\item[SC] Data set SC consists of the 100 smallest (in terms of file size) instances of the benchmark data set from the main track of the SAT competition 2020.\footnote{\url{https://satcompetition.github.io/2020/downloads.html}}

\item[LP] Data set LP is generated from benchmark data for answer set programming. More specifically, we used the problems \textit{15-Puzzle} (11 instances), \textit{Hamiltonian Cycle} (20 instances), and \textit{Labyrinth} (14 instances).
    % \begin{itemize}
    %     \item 15-Puzzle (11 instances)
    %     \item Hamiltonian Cycle (20 instances)
    %     \item Labyrinth (14 instances)
    % \end{itemize}
available from the Asparagus website\footnote{\url{https://asparagus.cs.uni-potsdam.de}}. 
Each problem specification and a single instance was first grounded and then converted to a propositional logic knowledge base using Clark completion \cite{Clark:1978}, i.\,e., a procedure that, basically, translates rules $\varLit_0\ \texttt{:-}\ \varLit_1,\dots ., \varLit_0\ \texttt{:-}\ \varLit_n.$ with identical head to a single formula $\varLit_0 \leftrightarrow \bigwedge \varLit_1\vee\ldots\vee \bigwedge \varLit_n$ (where each $\varLit_i$, $i\in \{1, \ldots, n\}$, may be a set of literals). 
Both classical negation and default negation in the rules are converted to logical negation in the resulting propositional formula.\footnote{Note that this procedure does not maintain the semantics of the original logic program. However, since we are not interested in computing solutions to the original problems but only require benchmarks that feature a similar structure as application scenarios, we believe that this does not pose an issue.} 
\end{description}

\begin{table}
\setlength\extrarowheight{1pt}
    \centering
    \begin{tabular}{|L{.08\textwidth} | R{.08\textwidth} | R{.08\textwidth} | R{.08\textwidth} | R{.08\textwidth} | R{.08\textwidth} | R{.08\textwidth} | R{.08\textwidth} |}
        \hline
        \multirow{2}{*}{Name} & \multirow{2}{*}{\#KBs} & \multicolumn{2}{c|}{\#Formulas} & \multicolumn{2}{c|}{Sig. size} & \multicolumn{2}{c|}{\#Con./formula} \\
        \cline{3-8}
          & & Mean & Std. & Mean & Std. & Mean & Std. \\ 
         \hline
         \hline
         SRS & $1800$ & $36$ & $24$ & $16$ & $9$ & $3.1$ & $2.0$ \\
         \hline
         ML & $1920$ & $7506$ & $17984$ & $76$ & $0$ & $5.5$ & $1.4$ \\
         \hline
         ARG & $326$ & $989$ & $2133$ & $827$ & $1781$ & $198.3$ & $455.1$ \\
         \hline
         SC & $100$ & $14254$ & $12788$ & $2559$ & $2616$ & $4.5$ & $3.9$ \\
         \hline
         LP & $45$ & $9150$ & $4683$ & $10097$ & $4834$ & $11.0$ & $905.1$ \\
         \hline
    \end{tabular}
    \caption{Overview of some properties of the data sets used in the evaluation, which include (from left to right) the total number of knowledge bases, the number of formulas per knowledge base (mean and standard deviation), the signature size per KB (mean and standard deviation), and the number of connectives per formula (mean and standard deviation).}
    \label{tab:datasets-overview}
\end{table}

An overview of some basic statistics (including the number of instances, the mean number of formulas, the mean signature size, and the mean number of connectives per formula) of all five data sets is provided in Table \ref{tab:datasets-overview}.

Both the SAT-based and the ASP-based approaches are implemented in C++.
The SAT solver we use is CaDiCal sc2021\footnote{\url{https://github.com/arminbiere/cadical}} \cite{Biere-SAT-2020}, and the ASP solver we use is Clingo 5.5.1\footnote{\url{https://potassco.org/clingo/}} \cite{gebser2019multi}.
For the computation of cardinality constraints in SAT we use sequential counter encoding \cite{sinz2005}, and for transforming formulas to CNF we use Tseitin's method \cite{tseitin1968}.

Due to the significant amount of time required to complete the evaluation, two different servers were used to facilitate the evaluation of multiple data sets in parallel.
The experiments regarding the SRS data set were run on a computer with $128$\,GB RAM and an Intel Xeon E5-2690 CPU which has a basic clock frequency of $2.90$\,GHz.
The other data sets (ML, ARG, SC, and LP) were evaluated on server instances with $32$\,GB RAM and an Intel Xeon Platinum 8260M CPU with a basic clock frequency of $2.40$\,GHz.

% To visualize our results, we make use of \textit{cactus plots} and \textit{scatter plots}.
% A cactus plot is comprised of time measurements ordered from low to high.
% Hence, we measure the runtime of an algorithm wrt.\ each knowledge base of a given data set, and subsequently order the measurements from low to high and plot them.
% If an inconsistency value wrt.\ a given knowledge base could not be computed within the given time limit (i.e., if it timed out) or if a memory error occurred, the corresponding time measurement is not plotted.
% The specific type of scatter plots we use depicts a comparison of the runtimes required by two different approaches with regard to each instance of a data set (that could be solved by at least one of the two approaches). 
% To be precise, $x$- and $y$-axis present the runtime required by the two approaches, respectively. 
% For each instance that was solved by both approaches in question, we mark the required runtime accordingly.
% If an instance was in fact only solved by one of the two approaches, the mark is set on the timeout line of the other.

\subsection{Baseline Approaches}\label{sec:baseline-approaches}

We will compare our two families of approaches against existing baseline implementations of the respective measures from \textit{TweetyProject}. 
These are quite simple computational approaches mainly relying on brute force methods which are written in Java. 
We briefly explain these baseline approaches as follows:
\begin{itemize}
    \item \textit{The contension inconsistency measure.}\footnote{\url{http://tweetyproject.org/r/?r=base_contension}} On input $\kb$ this approach first converts $\kb$ into conjunctive normal form $\kb'$ (exploiting the fact that the contension inconsistency value does not change when applying syntactic transformations). 
    Then a SAT solver is used to check whether $\kb'$ is satisfiable. 
    If it is, the inconsistency value is $0$ and the algorithm terminates. 
    Otherwise, for some atom $\varAt$ appearing in $\kb'$, all clauses containing $\varAt$ or $\neg \varAt$ are removed (this is equivalent to setting the truth value of $\varAt$ to $b$). 
    If the resulting knowledge base is satisfiable, the inconsistency value $1$ is returned. 
    Otherwise, every other atom appearing in $\kb'$ is tested in the same way. 
    If none of the resulting knowledge bases is satisfiable, the procedure is continued with all pairs of atoms appearing in $\kb'$, and so forth.
    
    \item \textit{The forgetting-based inconsistency measure.}\footnote{\url{http://tweetyproject.org/r/?r=base_forget}} This approach works similar as the approach for the contension inconsistency measure. 
    First, satisfiability of the input $\kb$ is tested using transformation into conjunctive normal form and application of a SAT solver. 
    Then, all possible substitutions of forgetting a single atom occurrence are applied on the input and it is checked whether the resulting knowledge base is consistent (again by using a SAT solver). Then, all pairs of possible substitutions are tested, and so on.
    
    \item \textit{The hitting set inconsistency measure.}\footnote{\url{http://tweetyproject.org/r/?r=base_hs}} This approach first checks whether there is a single interpretation that satisfies the input $\kb$ (by exhaustive enumeration) and returns the inconsistency value $0$ in that case. Then, we check whether there is any pair of interpretations (by exhaustively enumerating all possible combinations) such that every formula is satisfied by at least one of them. If such a pair can be found, the inconsistency value $1$ is returned. Otherwise, the process is repeated with three interpretations, and so on.
    
    \item \textit{The max-distance inconsistency measure.}\footnote{\url{http://tweetyproject.org/r/?r=base_dmax}} This approach simply iterates over all possible interpretations and calculates the distances to each formula (i.\,e., the minimal distance to all models of the formula). The maximum value found is returned.
    
    \item \textit{The sum-distance inconsistency measure.}\footnote{\url{http://tweetyproject.org/r/?r=base_dsum}} This approach is analogous to the max-distance inconsistency measure, but sums are calculated rather than maxima.
    
    \item \textit{The hit-distance inconsistency measure.}\footnote{\url{http://tweetyproject.org/r/?r=base_dhit}} This approach is analogous to the max-distance inconsistency measure, but only the number of unsatisfied formulas is checked.
\end{itemize}
It is apparent that theses approaches solve the corresponding problems in quite a naive fashion, but they are still able to produce results for small (toy) examples.

\subsection{Results}\label{sec:results}

In this section, we present the results of our experimental analysis.
To begin with, we compare the overall runtime of the SAT-based, ASP-based, and baseline approaches.
% wrt.\ each inconsistency measure considered in this work on all five data sets (Section \ref{sec:overall-runtimes}).
% Afterwards, we analyze more specific aspects.
% Section \ref{sec:runtime-composition} covers some experiments regarding the runtime composition of the SAT- and ASP-based methods.
% Further, in Section \ref{sec:linear-search} we perform an experiment in which we compare a linear search procedure to the standard binary search used for the SAT approaches. 
% Following up on this, in Section \ref{sec:maxsat} we perform an experiment in which we adapt an exemplary SAT encoding to a MaxSAT encoding, and compare this approach to the previously introduced ones.
% In Section \ref{sec:previous-asp}, we review previous versions of ASP encodings for $\icont$ from the literature, and compare them to the current version.
Afterwards, we analyze more specific aspects, namely the runtime composition of the SAT- and ASP-based methods, the choice of search strategy used in the SAT-based approaches, the use of MaxSAT, and a comparison of previous versions of ASP encodings for $\icont$ with the current version.

\subsubsection{Overall Runtime Comparison}\label{sec:overall-runtimes}

% Our first goal of the evaluation is to obtain a general overview of the runtimes of the different approaches.
Our first goal of the evaluation is to obtain an overview of the runtimes of the different approaches.
% To be precise, we investigate the performance of the individual approaches for the different measures on the one hand, and the different data sets on the other hand.
% Hence, 
To get a most thorough overview, we consider all three approaches for each of the six measures, and use all five previously described data sets (see Section \ref{sec:experimental-setup}).

A timeout was set to $600$ seconds (i.\,e., $10$ minutes) for instances from the SRS, ML, and ARG data sets.
For instances from the SC and LP data sets, which are overall more challenging compared to those instances from the previously mentioned data sets, we increased the timeout to $5000$ seconds (i.\,e., $83.\overline{3}$ minutes).
This corresponds to the time limit used in the SAT competition 2020, from which the instances in the SC data set were selected.

% We include the cumulative runtime as well as the number of timeouts of each approach wrt.\ each measure and each data set in Table \ref{tab:cumulative-runtimes}.
% This data refers to the experiments on the overall runtime comparison (Section \ref{sec:overall-runtimes}).

\paragraph{SRS} 

% The first data set we considered is the SRS data set, which corresponds to the union of data sets A and B in \cite{kuhlmann2021algorithms}, or the SRS data set in \cite{kuhlmann2022comparison}.
First we consider the SRS data set, which has been used previously in \cite{kuhlmann2021algorithms,kuhlmann2022comparison}.
% Figure \ref{fig:overall-runtime-all-RAN} displays six cactus plots\footnote{A cactus plot is comprised of time measurements ordered from low to high.
% Hence, we measure the runtime of an algorithm wrt.\ each knowledge base of a given data set, and subsequently order the measurements from low to high and plot them.
% If an inconsistency value wrt.\ a given knowledge base could not be computed within the given time limit (i.e., if it timed out) or if a memory error occurred, the corresponding time measurement is not plotted.}, each corresponding to a comparison between all three approaches wrt.\ one of the six inconsistency measures. 
Table \ref{tab:cumulative-runtimes} provides the number of solved instances as well as the cumulative runtime of each approach wrt.\ each measure and each data set.
(For a more detailed overview, see also the plots in \ref{app:cactus} and \ref{app:scatter-plots}.)
% A first noticeable observation is that, overall, all plots paint a quite similar picture---the ASP-based approaches perform best, while the naive approaches perform, as expected, poorest. 
Looking at the SRS part of Table \ref{tab:cumulative-runtimes}, a first noticeable observation is that, overall, the ASP-based approaches perform best, while the naive approaches perform, as expected, poorest. 
More precisely, the ASP-based approaches only time out in $175$ out of $10{,}800$ cases ($1.62$\,\%). % (for details such as the exact numbers of timeouts per inconsistency measure, see Table \ref{tab:cumulative-runtimes}).
To set this into perspective: the SAT approaches time out in $1552$ cases ($14.37$\,\%), and the naive ones in $5703$ cases ($52.81$\,\%).
The only exception to this pattern is the sum-distance inconsistency measure ($\isdalal$)---here, the naive approach outperforms the SAT-based one. 
This might be due to the fact that the search space for $\isdalal$ is quite large compared to most other measures\footnote{The only other measure that could possibly result in a higher value wrt.\ a given knowledge base $\kb$ is $\iforget$. However, since $\iforget$ depends on the number of atom occurrences in $\kb$, and $\isdalal$ depends on the number of formulas in $\kb$ as well as the signature size $\atoms(\kb)$, a general assessment of which measure has a larger search space is not possible.} 
% (the maximum non-$\infty$ value is $|\atoms(\kb)| \times |\kb|$ (see Table \ref{tab:search-ranges-sat})). 
(see Table \ref{tab:search-ranges-sat}). 
% Hence, since the search procedure consists of $\log(|\atoms(\kb)| \times |\kb| + 1)$ steps, we also need to calculate a new cardinality constraint exactly as many times.
As the maximum non-$\infty$ value for $\isdalal$ is $|\atoms(\kb)| \times |\kb|$, the search procedure consists of $\log(|\atoms(\kb)| \times |\kb| + 1)$ steps, and we also need to calculate a new cardinality constraint exactly as many times.
This results in the SAT method for $\isdalal$ being slower than those for the other two distance-based measures, while the respective naive methods are about equally fast for all three measures.

% Although the cactus plots in Figure \ref{fig:overall-runtime-all-RAN} give a general, ``global'' overview of how the different approaches perform in relation to each other, we do not know how they relate to each other in individual cases.
% % For instance, a naive approach could actually be faster than the corresponding ASP approach in certain cases. 
% Therefore, we additionally compare each pair of approaches wrt.\ their runtimes regarding each individual instance in the data set. % in \ref{app:data}. 
% % Figures \ref{fig:scatter-RAN-asp-sat}, \ref{fig:scatter-RAN-asp-naive}, and \ref{fig:scatter-RAN-sat-naive} visualize the comparisons between each pair of approaches wrt.\ all six inconsistency measures.

% From the aforementioned comparisons we learn that all ASP-based approaches outperform the SAT-based ones in almost all cases (see Figure \ref{fig:scatter-RAN-asp-sat}).
From concrete comparisons between each pair of approaches (see Appendix \ref{app:scatter-plots} for some visualizations) we learn that all ASP-based approaches outperform the SAT-based ones in almost all cases (see Figure \ref{fig:scatter-RAN-asp-sat}).
The only exception worth mentioning are a few knowledge bases for which the SAT approach for $\ihs$ was faster.
% In some cases, the ASP approach even timed out, while the SAT approach was still able to yield a result. 
With regard to the comparison of the ASP-based and naive methods we observe a similar pattern---again, the ASP methods are faster than the naive ones in most cases (see Figure \ref{fig:scatter-RAN-asp-naive}).
There are again a few instances for which $\ihs$ is computed faster by the naive method.
%, however, the more striking issue is that the naive method computes $\isdalal$  faster in quite a few cases. 
% There are even instances that could not be solved at all by the ASP approach within the time limit, but could be solved by the naive approach.
% Nevertheless, for the majority of instances, $\isdalal$ could be computed faster by the ASP-based method, and there are more cases for which the latter yielded a result while the naive one timed out, than vice versa.
% Overall, these results reinforce the perception conveyed by Figure \ref{fig:overall-runtime-all-RAN} that ASP performs strongest across this data set. 
Overall, these results reinforce the perception conveyed by Table \ref{tab:cumulative-runtimes} that ASP performs strongest across this data set. 

The comparison between the SAT-based and the naive approaches shows that the naive approach to solving $\isdalal$ is faster than its SAT-based counterpart in the majority of cases (see Figure \ref{fig:scatter-RAN-sat-naive}), and
in addition, there are numerous instances for which the SAT-based method could not return a result within the time limit, while the naive method could.
% This reflects the previously discussed results regarding Figure \ref{fig:overall-runtime-all-RAN-sum-dalal}.
Wrt.\ all other measures (with the exception of $\icont$), there exist some instances for which the naive approach was faster as well.
% However, Figure \ref{fig:scatter-RAN-sat-naive} illustrates that in total, the SAT-based approaches perform superior to the naive ones. 
However, in total, the SAT-based approaches perform superior to the naive ones. 
% For instance, there are multiple instances wrt.\ all measures except $\isdalal$, for which the SAT-based method returned a result, while the naive one timed out.

\begin{table}
    \setlength\extrarowheight{1pt}
    \centering
    \begin{tabular}{|l|l|rr|rr|rr|}
        \hline
         & & \multicolumn{2}{c|}{Naive} & \multicolumn{2}{c|}{SAT} & \multicolumn{2}{c|}{ASP} \\
        \hline
        \hline
         & & \#solved & CRT (s) & \#solved & CRT (s) & \#solved & CRT (s) \\
        \hline
        \multirow{6}{*}{\rotatebox[origin=c]{90}{SRS (1800)}} & $\icont$ &  $1162$ & $49237$ & $\bf 1800$ & $793$ & $\bf 1800$ & $48$ \\ 
         & $\iforget$ & $271$ & $22230$ & $1580$ & $103172$ & $\bf 1667$ & $16482$ \\ 
         & $\ihs$ & $647$ & $16176$ & $\bf 1792$ & $57644$ & $\bf 1792$ & $1612$ \\ 
         & $\imdalal$ & $1006$ & $27156$ & $1618$ & $97589$ & $\bf 1800$ & $3919$ \\ 
         & $\isdalal$ & $1006$ & $27188$ & $658$ & $67065$ & $\bf 1766$ & $11785$ \\ % $\bf 1473$ & $46596$ \\ 
         & $\ihdalal$ & $1005$ & $26557$ & $\bf 1800$ & $4585$ & $\bf 1800$ & $63$ \\
         \hline
         \multirow{6}{*}{\rotatebox[origin=c]{90}{ML (1920)}} & $\icont$ & $1134$ & $19075$ & $1763$ & $99260$ & $\bf 1920$ & $9013$ \\ 
         & $\iforget$ & $555$ & $3373$ & $523$ & $33994$ & $\bf 1467$ & $67220$ \\ 
         & $\ihs$ & $0$ & $-$ & $704$ & $76745$ & $\bf 1163$ & $58203$ \\ 
         & $\imdalal$ & $0$ & $-$ & $0$ & $-$ & $\bf 733$ & $89206$ \\ 
         & $\isdalal$ & $0$ & $-$ & $0$ & $-$ & $\bf 651$ & $18818$ \\  % $0$ & $-$ \\ 
         & $\ihdalal$ & $0$ & $-$ & $770$ & $38730$ & $\bf 1920$ & $7803$ \\ 
         \hline
         \multirow{6}{*}{\rotatebox[origin=c]{90}{ARG (326)}} & $\icont$ & $103$ & $3045$ & $196$ & $10252$  & $\bf 246$ & $9071$ \\ 
         & $\iforget$ & $70$ & $1175$ & $68$ & $6395$ & $\bf 143$ & $1908$ \\ 
         & $\ihs$ & $47$ & $1412$ & $157$ & $8539$ & $\bf 218$ & $12975$ \\ 
         & $\imdalal$ & $23$ & $762$ & $85$ & $6222$ & $\bf 145$ & $8408$ \\ 
         & $\isdalal$ & $23$ & $773$ & $30$ & $3694$ & $\bf 133$ & $2644$ \\ % $\bf 66$ & $3653$ \\ 
         & $\ihdalal$ & $23$ & $772$ & $\bf 164$ & $3340$ & $157$ & $1382$ \\ 
         \hline
         \multirow{6}{*}{\rotatebox[origin=c]{90}{SC (100)}} & $\icont$ & $11$ & $15481$ & $\bf 13$ & $21244$ & $7$ & $8619$ \\ 
         & $\iforget$ & $\bf 11$ & $11859$ & $1$ & $111$ & $4$ & $285$ \\ 
         & $\ihs$ & $0$ & $-$ & $2$ & $153$ & $\bf 3$ & $1461$ \\ 
         & $\imdalal$ & $0$ & $-$ & $1$ & $313$ & $\bf 2$ & $1172$ \\ 
         & $\isdalal$ & $0$ & $-$ & $0$ & $-$ & $\bf 3$ & $322$ \\  % $\bf 1$ & $2594$ \\ 
         & $\ihdalal$ & $0$ & $-$ & $3$ & $2513$ & $\bf 5$ & $2347$ \\ 
         \hline
         \multirow{6}{*}{\rotatebox[origin=c]{90}{LP (45)}} & $\icont$ & $0$ & $-$ & $0$ & $-$ & $\bf 45$ & $15869$ \\ 
         & $\iforget$ & $0$ & $-$ & $0$ & $-$ & $\bf 11$ & $36315$ \\ 
         & $\ihs$ & $0$ & $-$ & $0$ & $-$ & $0$ & $-$ \\ 
         & $\imdalal$ & $0$ & $-$ & $0$ & $-$ & $0$ & $-$ \\ 
         & $\isdalal$ & $0$ & $-$ & $0$ & $-$ & $0$ & $-$ \\ 
         & $\ihdalal$ & $0$ & $-$ & $0$ & $-$ & $\bf 31$ & $7983$ \\ 
         \hline
    \end{tabular}
    \caption{Overview of the number of solved instances and cumulative runtime (CRT) per data set, measure, and approach. The total number of instances contained in each data set is provided in braces after the data set name. Timeout: $600$ seconds ($10$ minutes) for SRS, ML, and ARG, and $5000$ seconds for SC and LP.}
    \label{tab:cumulative-runtimes}
\end{table}

\paragraph{ML}

% A first glance at Figure \ref{fig:overall-runtime-all-ML}, which displays the cactus plots regarding all six measures wrt.\ the ML data set, immediately reveals that the latter appears to be more challenging than the SRS data set (see Figure \ref{fig:overall-runtime-all-RAN}).
A first glance at the ML part of Table \ref{tab:cumulative-runtimes} immediately reveals that this data set appears to be more challenging than the SRS data set, as there are significantly more timeouts across the three approaches wrt.\ all six inconsistency measures. % (see Figure \ref{fig:overall-runtime-all-RAN}).
% To begin with, there are significantly more timeouts across the three approaches wrt.\ all six inconsistency measures. 
To be precise, the ASP approaches time out in $3666$ out of $11{,}520$ cases ($31.82$\,\%), the SAT approaches time out in $7760$ cases ($67.36$\,\%), and the naive approaches in $9831$ cases ($85.34$\,\%). 
% (see Table \ref{tab:cumulative-runtimes} for more detailed information regarding the numbers of timeouts as well as cumulative runtimes).
Thus, overall, the relationships between the approaches are again similar to those for the SRS data set---the ASP approaches achieve the strongest results, the naive ones, as expected, the weakest. 

Particularly noticeable are the results regarding the three distance-based measures. % (see Figure \ref{fig:overall-runtime-all-ML-maxdalal}--\ref{fig:overall-runtime-all-ML-hitdalal}).
The naive methods did not solve a single instance for any one of the three measures within the time limit.
The SAT approaches did not yield any results for $\imdalal$ and $\isdalal$ either; only for $\ihdalal$, a total of $770$ instances could be solved.
The ASP approach was the only one that could actually solve a number of instances for both $\imdalal$ and $\isdalal$.
% Although the ASP approaches could return a number of inconsistency values wrt.\ $\imdalal$ and $\ihdalal$, the approach for $\isdalal$ timed out in all cases, just like the SAT-based and naive versions.

Another noticeable observation is that the naive approach for $\iforget$ slightly outperforms its SAT-based equivalent---which was clearly not the case wrt.\ the SRS data set. 
This is most likely due to the fact that the SRS data set contains significantly fewer instances that are consistent, i.e., have an inconsistency degree of $0$, than the ML data set ($19$ vs.\ $548$ instances; see Figures \ref{fig:histos-RAN} and \ref{fig:histos-ML} for more information). 
Since the naive approach first checks whether the given knowledge base is consistent (by means of a SAT solver), consistent instances can be solved rather fast, in particular when compared to the SAT-based approach, where $\log(|\occs(\kb)|+1)$ calls to a SAT solver are required.
Besides, the SAT approach requires the computation of a SAT encoding for $\upperiforget$ (including the computation of an at-most-$k$ constraint in each step), which leads to an additional overhead.
The same applies to the naive approach for $\icont$. 
Furthermore, the naive approaches for $\icont$ and $\iforget$ might also yield a result rather quickly if the corresponding inconsistency value is not $0$, but still very low, as they both search for the correct value from low to high in a linear manner.
Deeper analysis revealed that the naive approach for $\icont$ was faster than the corresponding SAT-based approach in a total of $895$ cases.
In $507$ out of these $895$ cases, the contension inconsistency degree of the respective instances was in fact $0$.
Moreover, the inconsistency degree was $1$ in $340$ cases, and $2$ in $48$ cases.
Whenever the inconsistency degree was $>2$, the naive approach for $\icont$ was slower than the SAT-based approach.
Regarding $\iforget$, the naive approach was also faster than the SAT-based one in $507$ cases in which the inconsistency value is $0$.
Further, there are $19$ instances $\kb$ with $\iforget(\kb)=1$ which could be solved faster by the naive approach; there are no instances with $\iforget(\kb)>1$ where this is the case.

Wrt.\ ASP, the picture looks a bit different: the naive method was faster than the ASP-based method in only $65$ cases for $\icont$, and $57$ cases for $\iforget$.
In $45$ of the cases regarding $\icont$, the inconsistency value is $0$, and in the remaining $20$ cases it is $1$.
With regard to $\iforget$, in $56$ out of the $57$ cases in which the naive approach was faster, the corresponding inconsistency value is $0$, and in the single remaining case it is $1$.

A comparison between the runtimes of the ASP-based and the SAT-based approaches (see Figure \ref{fig:scatter-ML-asp-sat}) shows that the former are faster wrt.\ all instances.
When comparing the ASP approaches with the naive ones (see also Figure \ref{fig:scatter-ML-asp-naive}), we can see that there is a number of instances which could be solved by the naive method for $\iforget$, while the ASP-based one resulted in a timeout.
Yet overall, the ASP methods perform significantly superior compared to the naive ones.
% Firstly, wrt.\ $\icont$ and $\iforget$, the majority of instances can be solved faster by the respective ASP-based methods.
% Secondly, the naive method for $\ihs$ does not outperform the ASP-based (and SAT-based) versions at all.
% Besides, it should be considered that the naive approaches for $\imdalal$ and $\ihdalal$ timed out in all cases, while the ASP-based approaches could solve $733$ instances for $\imdalal$ and all $1920$ instances for $\ihdalal$.
% In particular, the naive methods for $\ihs$, $\imdalal$, $\isdalal$, and $\ihdalal$ could not solve a single instance. 
The comparison between the SAT approaches and the naive ones looks similar, however the naive approaches overall perform a bit stronger than when compared to the ASP-based versions.
However, although the SAT-based methods are outperformed by the naive ones in a total of $1465$ cases, the former still solve $1799$ instances more than the latter (across all inconsistency measures).

\paragraph{ARG}

% We move on to data set ARG.
% Figure \ref{fig:overall-runtime-all-ARG} displays the overall runtimes of all six measures wrt.\ all three approaches, respectively, on this data set.
From the ARG part of Table \ref{tab:cumulative-runtimes} we can observe that all three approaches wrt.\ all six inconsistency measures can solve some instances, but none can solve all instances from the data set.
In total, the ASP-based approaches time out in $914$ out of $1956$ cases ($46.73$\,\%), the SAT-based approaches time out in $1256$ cases ($64.21$\,\%), and the naive ones in $1667$ cases ($85.22$\,\%).
% (further information on cumulative runtimes, as well as exact numbers of timeouts, are found in Table \ref{tab:cumulative-runtimes}).
Once again, the ASP approaches overall clearly perform strongest, and the naive ones weakest.
However, similarly to what we observed for the ML data set, the naive approach for $\iforget$ performs a bit stronger than the SAT-based one.
Another exception to the overall pattern is that the SAT-based and the ASP-based approaches for $\ihdalal$ solve a quite similar number of instances; in fact, the ASP-based approach produces a few more timeouts than the SAT-based one. 
This has not been the case for any measure with regard to the SRS or ML data set.
However, it should also be mentioned that the ASP approach is clearly faster in the majority of cases (see also the scatter plot in Figure \ref{fig:scatter-ARG-asp-sat-hitdalal}).
Further, wrt.\ $\isdalal$, the naive method and the SAT-based method perform similarly. 
Nevertheless, although the SAT-based approach is a bit slower on average, the naive one still produces $7$ more timeouts.

The results of a comparison regarding the individual runtimes per knowledge base between the ASP-based, SAT-based, and naive approaches underline the proposition that the ASP-based approaches perform superior to the SAT-based and naive ones on the ARG data set.
The comparison to the SAT-based approaches (see Figure \ref{fig:scatter-ARG-asp-sat}) shows that ASP is faster in all cases, except for a few instances wrt.\ $\icont$ and $\ihdalal$.
Compared to the naive approaches (see Figure \ref{fig:scatter-ARG-asp-naive}), the ASP-based ones were faster in all cases, except for one instance in which the naive version was slightly faster wrt.\ $\iforget$. 

The comparison between the SAT-based and the naive approaches (see Figure \ref{fig:scatter-ARG-sat-naive}) shows that the naive ones outperformed the SAT-based ones in multiple cases.
More precisely, wrt.\ all inconsistency measures (except $\ihdalal$) there are instances for which the naive approach was faster.
In particular, there are some instances for which the SAT-based versions of $\iforget$ or $\isdalal$ time out, but their naive counterparts do not.
At large, the SAT approaches still perform superior to the naive ones. 
This becomes particularly clear on inspection of the number of instances that the SAT-based methods can solve within the time limit, but the naive versions cannot.

% \begin{figure}
%     \begin{subfigure}{.5\textwidth}
%         \includegraphics[width=.98\textwidth]{img/cactus/cactus_ARG_contension.pdf}
%         \caption{Contension inconsistency measure ($\icont$)}
%     \end{subfigure}%
%     \begin{subfigure}{.5\textwidth}
%         \includegraphics[width=.98\textwidth]{img/cactus/cactus_ARG_fb.pdf}
%         \caption{Forgetting-based inconsistency measure ($\iforget$)}
%     \end{subfigure}\\[2ex]
    
%     \begin{subfigure}{.5\textwidth}
%         \includegraphics[width=.98\textwidth]{img/cactus/cactus_ARG_hs.pdf}
%         \caption{Hitting Set inconsistency measure ($\ihs$)}
%     \end{subfigure}%
%     \begin{subfigure}{.5\textwidth}
%         \includegraphics[width=.98\textwidth]{img/cactus/cactus_ARG_max-dalal.pdf}
%         \caption{Max-distance inconsistency measure ($\imdalal$)}
%     \end{subfigure}\\[2ex]
    
%     \begin{subfigure}{.5\textwidth}
%         \includegraphics[width=.98\textwidth]{img/cactus/cactus_ARG_sum-dalal.pdf}
%         \caption{Sum-distance inconsistency measure ($\isdalal$)}
%     \end{subfigure}%
%     \begin{subfigure}{.5\textwidth}
%         \includegraphics[width=.98\textwidth]{img/cactus/cactus_ARG_hit-dalal.pdf}
%         \caption{Hit-distance inconsistency measure ($\ihdalal$)}
%     \end{subfigure}
%     \caption{Runtime comparison of the ASP-based, SAT-based, and naive approaches on the ARG data set. Timeout: $10$ minutes.}
%     \label{fig:overall-runtime-all-ARG}
% \end{figure}

\paragraph{SC}

To begin with, we would like to recall the fact that we increased the timeout for the data sets SC and LP to $5000$\,seconds ($83.\overline{3}$\,minutes). %, as mentioned in Section \ref{sec:overall-runtimes}.
However, even with this increased timeout, the data set still poses a great challenge for all three approaches wrt.\ all six inconsistency measures.
On the other hand, this is not entirely surprising---in a SAT competition, the problem at hand is to decide whether a given knowledge base is satisfiable or not (which is \np-complete), while the problem we are dealing with in this work is to determine an inconsistency value (which is in $\mathsf{FP}^{\np [\log n]}$ wrt.\ $\ihs$ and $\imdalal$, and proven to be $\mathsf{FP}^{\np [\log n]}$-complete wrt.\ $\icont$, $\iforget$, $\isdalal$, and $\ihdalal$ \cite{thimm2019a}).

In summary, all three approaches time out in most cases---the ASP approaches in $576$ out of $600$ cases, the SAT approaches in $580$, and the naive approaches in $578$ cases (see Table \ref{tab:cumulative-runtimes} for more details).
% It is striking that the naive approaches exhibit the same number of timeouts as the ASP-based approaches.
% However, i
% It should also be noted that only the naive methods for $\icont$ and $\iforget$ yield any results whatsoever.
It should be noted that wrt.\ the naive methods, only the those for $\icont$ and $\iforget$ could solve any instances whatsoever.
% With regard to all other measures (i.e., $\ihs$, $\imdalal$, $\isdalal$, and $\ihdalal$), the naive methods time out in all cases.
The reason for this behavior is most likely the previously mentioned fact that the naive approaches for both $\icont$ and $\iforget$ use a SAT solver to check whether the given knowledge base is consistent as a first step, which allows for retrieving the inconsistency value fast in case it is $0$.
If the inconsistency value is $1$, both approaches can still yield a result rather fast if the ``correct'' atom (in the case of $\icont$) or atom occurrence (in the case of $\iforget$) is removed at an early stage.
Our analysis (see Figure \ref{fig:histo-SAT}) shows that, in fact, all instances from the SC data set that could be solved have either inconsistency value $0$ or $1$ (wrt.\ all measures).\footnote{This is a quite reasonable result, as the knowledge bases in this data set were designed for a SAT competition. Hence, it is supposed to be challenging to decide whether they are unsatisfiable at all.}
% Moreover, it makes sense that the knowledge bases in this data set do not contain a large number of conflicts---after all, they were designed for a SAT competition.
% Hence, it is supposed to be challenging to decide whether they are satisfiable or not. 
Furthermore, all knowledge bases in the data set at hand are already provided in CNF---hence, as opposed to all other data sets considered in this work, the formulas do not require any additional transformation steps for the SAT approaches.

\paragraph{LP}

At last we inspect the runtime results regarding the LP data set.
% As in the preceding paragraph, the timeout is set to $5000$\,s.
Again, the timeout is set to $5000$\,s.
% The results show that only the ASP-based approaches for $\icont$, $\iforget$, and $\ihdalal$ were able to solve any instances at all.
% Hence, wrt.\ $\ihs$, $\imdalal$, and $\isdalal$, not a single instance could be solved by either one of the applied methods.
The results (see Table \ref{tab:cumulative-runtimes}) show that neither one of the SAT-based or naive approaches could solve any instance for any of the six measures within the time limit.
As for the ASP-based approaches, some instances could be solved wrt.\ $\icont$, $\iforget$, and $\ihdalal$, but wrt.\ $\ihs$, $\imdalal$, and $\isdalal$, again, not a single inconsistency value could be determined within the time limit.
% From Figure \ref{fig:overall-runtime-all-ASP}\footnote{We omitted the (empty) cactus plots for $\ihs$, $\imdalal$, and $\isdalal$. Information on cumulative runtimes and numbers of timeouts is provided in Table \ref{tab:cumulative-runtimes}.} we can observe that the ASP-based method for $\icont$ could, in fact, solve all $45$ instances, while the one for $\iforget$ could solve $11$ instances, but timed out $34$ times.
To be exact, we can observe that the ASP-based method for $\icont$ could, in fact, solve all $45$ instances, while the one for $\iforget$ could solve $11$ instances, but timed out $34$ times.
The ASP method for $\ihdalal$ solved $31$ instances, and consequently timed out in $14$ cases.
A closer look at the runtimes (see also Figure \ref{fig:overall-runtime-all-ASP}) reveals that, although the ASP-based approaches for $\icont$, $\iforget$, and $\ihdalal$ could solve a number of instances, they did require a rather large amount of time in a lot of cases, in particular wrt.\ $\iforget$.
More precisely, the ASP-based method for $\icont$ still solved $16$ instances between $10$ and $100$\,s, however, the remaining $29$ instances took between $100$ and $1000$\,s, with $12$ of them exceeding the $500$\,s mark.
With regard to the ASP method for $\iforget$, all instances (except one) even require $>1700$\,s.
These comparatively long runtimes might be an indicator of why none of the SAT-based or naive approaches could solve even a single instance.
% The comparison between the runtimes of the ASP-based approaches wrt.\ each individual instance of the other four data sets with the runtimes of the SAT-based and naive approaches revealed that in the vast majority of cases, the former are faster than the latter (see the corresponding scatter plots in Appendix \ref{app:scatter-plots}). 

% \begin{figure}
%     \begin{subfigure}{.5\textwidth}
%         \includegraphics[width=\textwidth]{img/cactus/cactus_ASP_contension.pdf}
%         \caption{Contension inconsistency measure ($\icont$)}
%     \end{subfigure}%
%     \begin{subfigure}{.5\textwidth}
%         \includegraphics[width=\textwidth]{img/cactus/cactus_ASP_fb.pdf}
%         \caption{Forgetting-based inconsistency measure ($\iforget$)}
%     \end{subfigure}\\[2ex]
    
%     \centering
%     \begin{subfigure}{.5\textwidth}
%         \includegraphics[width=\textwidth]{img/cactus/cactus_ASP_hit-dalal.pdf}
%         \caption{Hit-distance inconsistency measure ($\ihdalal$)}
%     \end{subfigure}
%     \caption{Runtime comparison of the ASP-based, SAT-based, and naive approaches on the LP data set for the contension and forgetting-based inconsistency measures. Timeout: $10$ minutes.}
%     \label{fig:overall-runtime-all-ASP}
% \end{figure}

\subsubsection{Runtime Composition}\label{sec:runtime-composition}

The experiments illustrated in the preceding section revealed that the ASP-based approaches overall outperform their SAT-based and naive counterparts.
Whereas this is not surprising wrt.\ the naive methods, it is not immediately clear why the SAT methods are outperformed as well.
The objective of this section is to explore the reasons for this in more detail by investigating how the runtimes of both the SAT and the ASP methods are composed on average.

Because both the ASP and the SAT approaches solved a significant number of instances from the SRS data set, and, respectively, the ARG data set, wrt.\ all six inconsistency measures, we selected those two data sets for further analysis. 
More specifically, we examined the average runtime \textit{composition}, meaning we defined a number of categories representing the main elements of the approaches and measured how much of the overall runtime is spent on each of them.
The categories are ``encoding generation'', ``solving'', and for the SAT approaches additionally ``transformation to CNF''.
Moreover, we defined the category ``other'' for any remaining tasks, such as loading the given knowledge base.
Note that ``solving'' does not only include the plain solving time, but also the solver initialization for both ASP and SAT methods. 

Furthermore, we compare the \textit{average} runtimes of the SAT and ASP approaches wrt.\ each measure and regarded data set. 
For the SAT methods, this includes all required search steps, i.e., multiple encoding generation steps (this particularly affects cardinality constraints), multiple solver calls, and multiple transformations to CNF.

To begin with, we consider the results wrt.\ the SRS data set. 
% With regard to both approaches, the average runtimes differ greatly, depending on the inconsistency measure.
In total, the ASP approaches take between $0.03$\,s ($\icont$ and $\ihdalal$) and % $31.63$\,s ($\isdalal$), and the SAT approaches between $0.44$\,s ($\icont$) and $101.92$\,s ($\isdalal$) on average.
$9.89$\,s ($\iforget$), and the SAT approaches between $0.44$\,s ($\icont$) and $101.92$\,s ($\isdalal$) on average.
Figure \ref{fig:runtime-composition-SRS-contension} visualizes the runtime composition of the ASP-based and the SAT-based approach for $\icont$.
We see that the encoding generation covers the largest fraction of the overall average runtime of the SAT-based approach. 
% Wrt.\ the ASP-based approach on the other hand, encoding generation merely takes $0.0008$\,s (and is thus not even visible in the figure).
Wrt.\ the ASP-based approach on the other hand, encoding generation merely takes a very small fraction of a second (and is thus not even visible in the figure).
Solving makes up the second-largest share of the overall runtime of the SAT-based method; the other two categories (i.e., the transformation to CNF, and ``other'') only have a relatively minor impact.
Regarding the ASP-based method, both solving and ``other'' contribute significantly to the overall runtime. 
A probable reason for the relatively large proportion of ``other'' is the overall very short average runtime of the ASP-based approach.

\begin{figure}
    \centering
    \includegraphics[width=.95\textwidth]{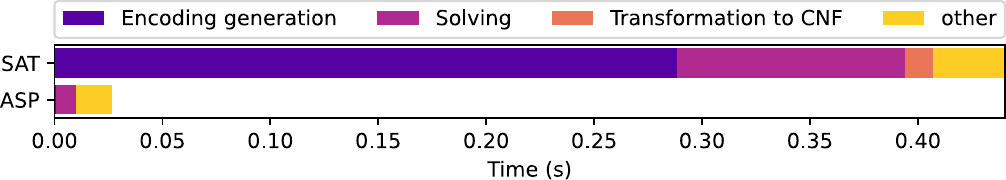}
    \caption{Overview of the average runtime composition of the ASP-based and SAT-based approaches for $\icont$ wrt. the SRS data set.}
    \label{fig:runtime-composition-SRS-contension}
\end{figure}

The previously described findings do not, however, reflect the results regarding the remaining five inconsistency measures on the SRS data set.
More precisely, wrt.\ the SAT-based approaches for the latter measures (i.e., $\iforget$, $\ihs$, $\imdalal$, $\isdalal$, and $\ihdalal$), the encoding part does not make up the largest fraction of the overall average runtime---in fact, it makes up a significantly smaller proportion than solving or CNF transformation in all cases (see also Figure \ref{fig:runtime-composition-SRS}).
The transformation to CNF, which represents the smallest runtime component for $\icont$, plays a more prominent role for the other measures.
In the case of $\ihs$, it even represents the largest fraction. 
Regarding the forgetting-based, as well as all three distance-based measures, the largest runtime component is comprised by solving.
Figure \ref{fig:runtime-composition-SRS-fb} shows the runtime composition with regard to $\iforget$, which represents the previously described pattern that can also be applied to $\imdalal$, $\isdalal$, $\ihdalal$, and (under the restriction that the solving-to-CNF transformation ratio does not quite fit) also to $\ihs$.
Moreover, this figure shows that the only relevant runtime component of the ASP-based approach for $\iforget$ is the solving part.
The same applies to the four remaining measures ($\ihs$, $\imdalal$, $\isdalal$, $\ihdalal$) as well.

As mentioned in the beginning of this section, the total runtime of both ASP and SAT approaches differs from measure to measure. % (for an overview, see Figure \ref{fig:runtime-composition-SRS}).
Moreover, the ratio of the average ASP runtime to the average SAT runtime differs wrt.\ the inconsistency measures as well.
For example, wrt.\ $\ihdalal$, ASP is roughly $73$ times as fast as SAT ($0.03$\,s vs.\ $2.55$\,s), % however, wrt.\ $\isdalal$, ASP is only about $3$ times as fast ($31.63$\,s vs.\ $101.92$\,s).
however, wrt.\ $\iforget$, ASP is only about $7$ times as fast ($9.89$\,s vs.\ $65.30$\,s).

\begin{figure}
    \centering
    \includegraphics[width=.95\textwidth]{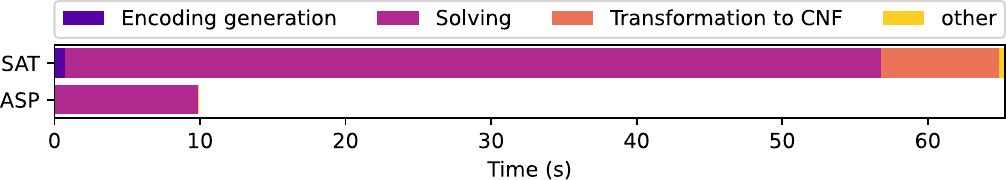}
    \caption{Overview of the average runtime composition of the ASP-based and SAT-based approaches for $\iforget$ wrt. the SRS data set.}
    \label{fig:runtime-composition-SRS-fb}
\end{figure}

Some observations we made wrt.\ the SRS data set also apply to the ARG data set: 
\begin{itemize}
    \item For the ASP-based approaches, the solving part is essentially the only relevant component regarding the overall runtime.
    \item For the SAT-based approaches, solving overall makes up the most relevant part of the runtime composition.
    \item Wrt.\ the SAT approach for $\icont$, the encoding generation part takes up a larger share of the overall runtime than wrt.\ the SAT approaches for the other five measures.
    \item Wrt.\ the SAT approach for $\ihs$, the transformation to CNF takes up a larger fraction of the overall runtime than wrt.\ the SAT approaches for the other measures.
\end{itemize}

On the other hand, the runtime composition wrt.\ the ARG data set also provides some new insights (see Figure \ref{fig:runtime-composition-ARG} for an overview of all six measures).
To begin with, the ratio between the average ASP runtime and the average SAT runtime is smaller than wrt.\ the SRS data set.
More precisely, an ASP approach is at most roughly $7$ times as fast as its corresponding SAT version ($\iforget$).
Moreover, in the case of $\ihs$, the SAT-based method is on average even a bit faster than the ASP-based one. % (see Figure \ref{fig:runtime-composition-ARG-hs}).
However, this observation has to be taken with a grain of salt---the number of instances the ASP-based method could solve is higher than the number of instances the SAT-based approach could solve, and when considering individual knowledge bases, the ASP method was faster in each case (see Figure \ref{fig:scatter-ARG-asp-sat-hs}). 
% and a number of those instances that could be solved by the ASP-based approach but not by the SAT-based one took a rather long time to be solved (i.e., close to the timeout) (see also the respective scatter plot depicted in Figure \ref{fig:scatter-ARG-asp-sat-hs}).
% and a number of those instances that could exclusively be solved by the ASP-based approach took a rather long time to be solved (i.e., close to the timeout) (see also the respective scatter plot depicted in Figure \ref{fig:scatter-ARG-asp-sat-hs}).

% \begin{figure}
%     \centering
%     \includegraphics[width=\textwidth]{img/runtime_composition/runtime_composition_ARG_hs.pdf}
%     \caption{Overview of the average runtime composition of the ASP-based and SAT-based approaches for $\ihs$ wrt. the ARG data set.}
%     \label{fig:runtime-composition-ARG-hs}
% \end{figure}

\subsubsection{Linear Search in SAT Approaches}\label{sec:linear-search}

In addition to the runtimes of the different approaches considered in this work we logged the inconsistency values wrt.\ the different measures and data sets (in Appendix \ref{app:inc-values}, we include histograms over all inconsistency values that could be retrieved for each data set and measure). % (i.e., we only excluded values that could not be computed by either one of the three respective methods within the given time limit). 
From the results we can see that the range of inconsistency values varies greatly, in particular when considering the size of the search space.
For instance, the contension inconsistency values range from $0$ to $25$ on the SRS data set, with a maximum possible value of\ $|\atoms(\kb)| = 30$ (for the largest fraction of instances).
In contrast, the max-distance inconsistency measure has the same search space as $\icont$, %(i.e., the possible values lie between $0$ and $30$), 
but the values we actually measured in our experiments are only ranging from $0$ to $2$, with roughly $1000$ instances resulting in the value $1$.
Another example is the sum-distance measure---here, the values are ranging from $0$ to $42$. %, and are thus similar to the contension values. 
However, the search space is much larger, with a maximum possible value of $|\atoms(\kb)| \cdot |\kb| = 30 \cdot 100 = 3000$ (again, for the largest instances).
These observations raise the question of whether the binary search approach used in the SAT-based methods is the most practical solution in all cases.
We therefore conduct an experiment to compare the previously used SAT approach which includes a binary search procedure to an adapted version which uses linear search instead.

\begin{table}
    \setlength\extrarowheight{1pt}
    \centering
    \begin{tabular}{|l|rr|rr|}
        \hline
         & \multicolumn{2}{c|}{Binary search} & \multicolumn{2}{c|}{Linear search} \\
        \hline
        \hline
         & \#solved & CRT (s) & \#solved & CRT (s) \\
        \hline
         $\icont$ & $\bf 1800$ & $793$ & $\bf 1800$ & $1670$ \\ 
         $\imdalal$ & $\bf 1618$ & $97589$ & $1585$ & $43506$ \\
         $\isdalal$ & $658$ & $67065$ & $\bf 1068$ & $77712$ \\ 
         \hline
    \end{tabular}
    \caption{Comparison between binary search and linear search in the SAT approach. Displayed are the number of solved instances and the cumulative runtime (CRT) per measure on the SRS data set (1800 knowledge bases). Timeout: $10$ minutes.}
    \label{tab:binary-vs-linear}
\end{table}

In our experiment, we focus on the three previously mentioned inconsistency measures ($\icont$, $\imdalal$, and $\isdalal$) wrt.\ the SRS data set.
% Figure \ref{fig:cactus-binary-vs-linear-search} visualizes the results, additionally including the previous runtime measurements for the corresponding ASP-based and naive methods for reference. 
Table \ref{tab:binary-vs-linear} displays the results (see also Figures \ref{fig:cactus-binary-vs-linear-search}, \ref{fig:scatter-SAT-binary-vs-linear-1}, and \ref{fig:scatter-SAT-binary-vs-linear-2} for further visualizations).
We can see that with regard to $\icont$, the linear search variant of the SAT approach is on average a bit slower than the binary search variant.
Nevertheless, the two methods perform quite similarly.  % (see also Figures \ref{fig:scatter-SAT-binary-vs-linear-1} and \ref{fig:scatter-SAT-binary-vs-linear-2} for a more detailed overview). 
Concerning $\imdalal$, the two approaches perform again similarly. 
However, despite being a bit faster on average, the linear search version times out in $33$ cases more than its binary search counterpart.
Regarding $\isdalal$, the results look quite different than those described before---the linear search variant performs clearly superior, exhibiting $410$ fewer timeouts than the binary search version.
Thus, although binary search is in theory a more efficient search strategy than linear search, it is not necessarily the superior option in practice.

% Hence, overall, we can observe that the selection of the most suitable search strategy for a practical application depends on the given data, as well as the choice of inconsistency measure.

% \begin{figure}
%     \begin{subfigure}{.5\textwidth}
%         \includegraphics[width=\textwidth]{img/linear_vs_binary/cactus_RAN_contension_binary-vs-linear.pdf}
%         \caption{Contension inconsistency measure ($\icont$)}
%     \end{subfigure}%
%     \begin{subfigure}{.5\textwidth}
%         \includegraphics[width=\textwidth]{img/linear_vs_binary/cactus_RAN_maxdalal_binary-vs-linear.pdf}
%         \caption{Max-distance inconsistency measure ($\imdalal$)}
%     \end{subfigure}
    
%     \centering
%     \begin{subfigure}{.5\textwidth}
%         \includegraphics[width=\textwidth]{img/linear_vs_binary/cactus_RAN_sumdalal_binary-vs-linear.pdf}
%         \caption{Sum-distance inconsistency measure ($\isdalal$)}
%     \end{subfigure}
%     \caption{Runtime comparison of the SAT-based approaches based on linear search (for $\icont$, $\imdalal$, and $\isdalal$) and the corresponding binary search versions, as well as the ASP-based versions, on the SRS data set. Timeout: $600$ seconds.}
%     \label{fig:cactus-binary-vs-linear-search}
% \end{figure}

\subsubsection{MaxSAT Approaches}\label{sec:maxsat}

% In the previous section, we examined the impact of different search strategies used in some SAT approaches on the overall runtime.
Another aspect we aim to investigate regarding SAT is the use of a \textit{MaxSAT} solver.
The \textit{maximum satisfiability problem} (MaxSAT) is the problem of finding an assignment of truth values that satisfies a maximum number of clauses (for an overview on this topic, see, e.\,g., \cite{bacchus2021maximum,li2021maxsat}).
Hence, our proposed technique of using iterative SAT checks is essentially a naive method of solving a MaxSAT problem.
In the following, we use the contension measure as an example to show how our SAT encodings can be modified to become MaxSAT encodings.

A (partial) MaxSAT encoding consists of \textit{hard clauses} and \textit{soft clauses}.
While the hard clauses must be satisfied, the soft clauses do not.
Nevertheless, the goal is to satisfy as many soft clauses as possible.
Recall the SAT encoding for $\icont$ being determined by (SC1)--(SC17) in Section \ref{sec:sat-encoding-ic}.
The signature, defined by (SC1)--(SC2), remains the same in the MaxSAT case, and
% Moreover, 
we use the constraints defined via (SC3)--(SC16) as hard clauses. % in the MaxSAT encoding.
Only (SC17), the at-most-$u$ constraint, %($\text{at\_most}(u, \{\varAt_b \mid \varAt \in \atoms(\kb)\}$), 
is not required anymore.
Instead, we define a soft clause $\lnot \varAt_b$ for each $\varAt \in \atoms(\kb)$.
The intuition behind this is that as many atoms as possible are supposed to be \textit{not} set to $b$.
Note that the SAT-based approaches for the other inconsistency measures can be adapted to MaxSAT in a similar fashion.

We perform an experiment to compare the MaxSAT approach with the previously discussed approaches for $\icont$.
To achieve this, we use the SRS data set, and as a MaxSAT solver, we use \textit{EvalMaxSAT}\footnote{\url{https://github.com/FlorentAvellaneda/EvalMaxSAT}} \cite{avellaneda2020short}.
% The runtime results are presented in Figure \ref{fig:maxsat-cactus} (in addition, the corresponding scatter plots are found in Figure \ref{fig:scatter-maxsat} in Appendix \ref{app:scatter-plots}).
The runtime results are presented in Table \ref{tab:maxsat} (see also Figures \ref{fig:maxsat-cactus} and \ref{fig:scatter-maxsat}).
We can observe that the MaxSAT approach is indeed overall faster than the iterative SAT approach.
However, it is, in total, still a bit inferior to the ASP-based method.
On the other hand, it should be noted that the runtimes for both the MaxSAT and the ASP approach are quite short ($<1$\,s for each instance). 
Altogether, the results of this experiments show that the use of MaxSAT has potential and is worth being investigated % further (e.\,g., for other inconsistency measures and wrt.\ other data sets) 
in future work.

\begin{table}
    \setlength\extrarowheight{1pt}
    \centering
    \begin{tabular}{|rr|rr|rr|}
        \hline
         \multicolumn{2}{|c|}{SAT binary search} & \multicolumn{2}{c|}{SAT linear search} & \multicolumn{2}{c|}{MaxSAT} \\
        \hline
        \hline
        \#solved & CRT (s) & \#solved & CRT (s) & \#solved & CRT (s) \\
        \hline
        $1800$ & $793$ & $1800$ & $1670$ & $1800$ & $\bf 221$\\ 
         \hline
    \end{tabular}
    \caption{Comparison between the two variants of the iterative SAT approach for $\icont$ and their MaxSAT counterpart. Displayed are the number of solved instances and the cumulative runtime (CRT) per measure on the SRS data set (1800 knowledge bases). Timeout: $10$ minutes.}
    \label{tab:maxsat}
\end{table}

% \begin{figure}
%     \centering
%     \includegraphics[width=.5\textwidth]{img/maxsat/cactus_maxsat.pdf}
%     \caption{Runtime comparison of the MaxSAT approach for $\icont$ and the previously introduced approaches on the SRS data set. Timeout: 10 minutes.}
%     \label{fig:maxsat-cactus}
% \end{figure}

\subsubsection{Core-Guided Optimization in ASP}\label{sec:core-guided}

% With regard to ASP, it is also possible to apply different optimization strategies. 
Modern ASP systems, such as Clingo, allow for the use of different optimization strategies.
% -> Standard sollte Branch-and-Bound sein 
% - hier vielleicht direkt anmerken, dass wir sonst ja die Standard-Einstellungen benutzen und default-mäßig ein model-guided approach eingestellt ist 
%   -> (diese Aussage noch mal genau nachprüfen!)
% - Es gibt jedoch noch einige weitere Lösungsstrategien, wie z. B. core-guided 
%   -> (hier ggf. noch weitere nennen?) 
Clingo's default option (which is used in all experiments described above) is a \textit{branch-and-bound} approach~\cite{gebser2012conflict}.
% - erst wird ein stable model gefunden (sofern dies möglich ist), dann wird versucht, eine bessere Lösung zu finden 
%   -> mithilfe von #sum? 
In this approach, we first search for an initial solution, i.\,e., an initial stable model, and (in case a solution exists) obtain its objective value.
The next step consists of a loop in which we try to find a solution with a strictly ``better'' objective value---i.\,e., in the case of minimization, a strictly lower value, and in the case of maximization, a strictly higher value.
This is achieved by adding constraints which ensure that any new solution must have a lower (resp.\ higher) objective value than the solution derived before.
If no ``better'' solution can be found, we know that the previous solution corresponds to an optimum~\cite{kaminski2017tutorial}. 
% a constraint is added to the problem specification requiring that a solution must have a strictly better objective value than the one just obtained

Although the branch-and-bound optimization approach led to promising results in our previous experiments, Clingo has further optimization techniques to offer, which we have yet to explore.
%   -> kurz erklären, was das bedeutet / Intuition geben 
% Es können jedoch auch andere Strategien verwendet werden.
% -> dann der folgende Absatz mit "As an example, we perform ... "
%   -> dann dann erklären, was core-guided ist
As an example, we perform an experiment in which we set Clingo's optimization strategy to ``\texttt{usc}''~\cite{gebser2015progress}, % which implements the \textsc{OLL} algorithm~\cite{andres2012unsatisfiability}, thereby employing a \textit{core-guided} approach. % ---a technique from the area of MaxSAT solving.
employs a \textit{core-guided} approach~\cite{andres2012unsatisfiability}.
% - kurz erklären, was diese Strategie macht, bzw. wann sie Anwendung findet
%   -> lt. Kommentar vom Reviewer funktioniert diese Strategie besonders bei niedrigen Werten gut
% - bzw. noch besser: sowohl die default-Strategie *und* usc kurz erklären (bzw. die Intuition davon)
% \ik{core-guided erklären und erklären, wann das sinnvoll sein könnte}
The latter emerged in the area of MaxSAT solving~\cite{fu2006solving,marques2008towards} and is based on the following procedure.
We first check if all soft clauses can be satisfied, i.\,e., we implicitly introduce a lower bound of $0$ soft clauses that are allowed to be falsified.
If this is not the case, we allow for at most one soft clause to be falsified (i.\,e., we increase the lower bound to $1$), and check for satisfiability again.
This process is repeated until we find a satisfiable solution. 
In order to guide this process, at each step an \textit{unsatisfiable core} (an unsatisfiable set of clauses) is extracted and only soft clauses from that a core are allowed to be falsified~\cite{fu2006solving}.
% ---hence the name \textit{core-guided optimization}~\cite{fu2006solving}.
This approach has been transferred to ASP~\cite{andres2012unsatisfiability}, with the literals of ASP optimization statements being interpreted as the soft clauses of a MaxSAT problem.
Since, in the case of minimization, core-guided optimization is useful when the target values are close to $0$, we selected inconsistency measures and data sets for which this applies (see the histograms in Appendix~\ref{app:inc-values}).
To be precise, we performed an experiment with the ASP approach for $\ihs$ on the ML data set; here, the inconsistency values tend to be very low (either $0$ or $1$; see Figure~\ref{fig:histos-ML}), while the search space (i.\,e., the number of formulas) is still quite large (the mean number of formulas is $7506$; see Table~\ref{tab:datasets-overview}). 
In addition, we considered the ARG data set in combination with $\ihs$ and $\imdalal$, as the inconsistency values are likewise always $0$ or $1$.
% - ARG hat aber im Durchschnitt nur 989 Formeln (bzgl. Suchraum von hs)
%   -> Suchraum von max-dist: (|At(K)|) im Durchschnitt 827 Atome 
% - Wieviel % mehr Instanzen konnten gelöst werden?
%   -> ML/hs: 0,046721311
%   -> ARG/hs: 0,013574661
%   -> ARG/max-dist: 0,176136364

% \begin{table}
%     \setlength\extrarowheight{1pt}
%     \centering
%     \begin{tabular}{|rr|rr|}
%         \hline
%          \multicolumn{2}{|c|}{ASP core-guided} & \multicolumn{2}{c|}{ASP} \\
%         \hline
%         \hline
%          \#solved & CRT (s) & \#solved & CRT (s) \\
%         \hline
%         % ML + hs:
%          $\bf 1220$ & $51441$ & $1163$ & $58203$ \\ 
%          \hline
%     \end{tabular}
%     \caption{Comparison between a core-guided optimization strategy for the ASP approach and the default optimization strategy (branch-and-bound).
%     Displayed are the number of solved instances and the cumulative runtime (CRT) for $\ihs$ on the ML data set (1920 knowledge bases).
%     Timeout: $10$ minutes.}
%     \label{tab:asp-usc}
% \end{table}

\begin{table}
    \setlength\extrarowheight{1pt}
    \centering
    \begin{tabular}{|l|l|rr|rr|}
        \hline
        & & \multicolumn{2}{c|}{ASP core-guided} & \multicolumn{2}{c|}{ASP} \\
        \hline
        \hline
        & & \#solved & CRT (s) & \#solved & CRT (s) \\
        \hline
        % ML + hs:
        ML & $\ihs$ & $\bf 1220$ & $51441$ & $1163$ & $58203$ \\ 
        ARG & $\ihs$ & $\bf 221$ & $11146$ & $218$ & $12975$ \\ 
        ARG & $\imdalal$ & $\bf 176$ & $7266$ & $145$ & $8408$ \\ 
         \hline
    \end{tabular}
    \caption{Comparison between a core-guided optimization strategy for the ASP approach and the default optimization strategy (branch-and-bound).
    Displayed are the number of solved instances and the cumulative runtime (CRT) for $\ihs$ on the ML data set (1920 knowledge bases), as well as $\ihs$ and $\imdalal$ on the ARG data set (326 knowledge bases).
    Timeout: $10$~minutes.}
    \label{tab:asp-usc}
\end{table}

% - dann Tabelle zeigen und kurz auf Ergebnisse eingehen 
%   -> auf Tabelle referenzieren
%   -> im Anhang dann auch Cactus Plots (+ Scatter Plots?) 
% - es konnten mehr Instanzen gelöst werden bei gleichzeitig geringerer Laufzeit 
The results in Table~\ref{tab:asp-usc} show that the core-guided optimization strategy can indeed lead to improved runtimes 
(see also Figure~\ref{fig:asp-usc-cactus} in Appendix~\ref{app:cactus} for the cactus plots and Figure~\ref{fig:scatter-asp-usc-bb} in Appendix~\ref{app:scatter-plots} for the corresponding scatter plots).
The core-guided variant solves more instances while also maintaining a lower cumulative runtime in all three cases.
More specifically, the core-guided approach solves between $1.36\,\%$ (ARG/$\ihs$) and $17.61\,\%$ (ARG/$\imdalal$) more instances.
% More specifically, the core-guided variant solves $57$ instances more than the default branch-and-bound variant, while also maintaining a lower cumulative runtime.

\subsubsection{Previous ASP Approaches}\label{sec:previous-asp}

There already exist two ASP-based approaches for $\icont$, and one for $\iforget$, and $\ihs$, respectively, in the literature.
To be precise, in \cite{kuhlmann2020algorithm}, the authors propose a method similar to our overall SAT approach, which uses ASP encodings for the problem $\upperic$ in order to find $\valueic$ via binary search.
A revised version of this approach, which calculates $\valueic$ directly within ASP by means of a minimize statement, is introduced in \cite{kuhlmann2021algorithms}.
Note that in the latter version, only propositional language concepts are used, which leads to a program that is already ground. % (i.e., the program is essentially grounded manually, instead of by a grounder).
The authors also propose ASP encodings for $\iforget$ and $\ihs$ in the same manner.
In contrast, the ASP approach presented in \cite{kuhlmann2022comparison} (which is also used in the work at hand) makes use of first-order predicates and variables, which enables an automated, and internally optimized, grounding procedure.
% The same applies to the ASP approaches for the other five inconsistency measures.
In order to determine whether this actually has a positive effect on the runtime, we compare the different versions with each other using the contension measure as an example.
% Note that we apply exactly those Java implementations which were used in the two corresponding papers \cite{kuhlmann2020algorithm,kuhlmann2021algorithms}.
Note that we apply exactly those implementations which were used in the two corresponding papers \cite{kuhlmann2020algorithm,kuhlmann2021algorithms}.

\begin{table}
    \setlength\extrarowheight{1pt}
    \centering
    \begin{tabular}{|rr|rr|rr|}
        \hline
        \multicolumn{2}{|c|}{ASP binary search} & \multicolumn{2}{c|}{ASP minimize v1} & \multicolumn{2}{c|}{ASP} \\
        \hline
        \hline
        \#solved & CRT (s) & \#solved & CRT (s) & \#solved & CRT (s) \\
        \hline
        $1200$ & $59618$ & $\bf 1800$ & $479$ & $\bf 1800$ & $48$ \\
        \hline
    \end{tabular}
    \caption{Runtime comparison of the different versions of the ASP approach for $\icont$ on the SRS data set. ``ASP binary search'' refers to the version from \cite{kuhlmann2020algorithm}, ``ASP minimize v1'' to the version from \cite{kuhlmann2021algorithms}, and ``ASP'' to the current one.
    Displayed are the number of solved instances and the cumulative runtime (CRT). Timeout: $10$ minutes.}
    \label{tab:asp-comparison}
\end{table}

% The results of this experiment, which are illustrated in Figure \ref{fig:asp-cactus}, show that the new version of the method in fact outperforms its predecessors.
The results of this experiment, presented in Table \ref{tab:asp-comparison} (see Figure \ref{fig:asp-cactus} for an additional visualization), confirm that the newest version of the ASP-based method in fact outperforms its predecessors.
The first ASP approach \cite{kuhlmann2020algorithm}, which is based on a binary search procedure, clearly performs the poorest, and hits the timeout of 10 minutes in precisely $600$ cases.
% These cases correspond exactly to the three most challenging parts of the SRS data set, i.e., those with the largest signature sizes in combination with the largest numbers of formulas per knowledge base (see the bottom three rows of Table \ref{tab:RAN-dataset}).
The second version of the approach \cite{kuhlmann2021algorithms} solves all $1800$ instances within the time limit, nevertheless it performs on average roughly $10$ times slower than the current version ($0.266$ vs. $0.027$ seconds).
Although the new version might have an advantage by being implemented in C++, both rely on the same ASP solver (Clingo 5.5.1). 
In fact, the solving time itself is around $3$ times shorter wrt. the new ASP version compared to the previous one ($0.009$ vs. $0.028$ seconds on average).

% \begin{figure}
%     \centering
%     \includegraphics[width=.5\textwidth]{img/cactus/asp_comparison.pdf}
%     \caption{Runtime comparison of the different versions of the ASP approach on part of the SRS data set. ``ASP binary search'' refers to the version from \cite{kuhlmann2020algorithm}, ``ASP minimize v1'' to the version from \cite{kuhlmann2021algorithms}, and ``ASP'' to the new one. Timeout: 10 minutes.}
%     \label{fig:asp-cactus}
% \end{figure}

\subsection{Discussion}\label{sec:discussion}

In the experiments presented above, we investigated three approaches for the calculation of six different inconsistency measures on five different data sets.
Overall, we focused our analysis on the comparison of runtimes; however, we took additional aspects into account, such as the distribution of measured inconsistency values with regard to the different measures and data sets. 
As a first step, we compared the runtimes of all three approaches per inconsistency measure and data set.
Our results confirmed that both the SAT-based and the ASP-based approaches perform, as expected, superior to the naive baseline algorithms.
Moreover, the results showed that altogether, the ASP methods are faster than the SAT variants.
Only in very isolated cases can the SAT-based or naive approaches retrieve inconsistency values faster than the ASP-based ones or result in fewer timeouts.
% For instance, on data set ARG, the SAT method for the hit-distance measure is on average slower than its ASP counterpart, but it can solve a few instances more.
Furthermore, the ASP-based methods are not only faster than the other two approaches on average, but also wrt.\ the vast majority of individual instances.
% E.\,g., regarding the ML data set, there exists not a single instance for which a SAT-based method outperformed the corresponding ASP-based one wrt.\ either one of the inconsistency measures considered in this work (for further details, see the scatter plots in Appendix \ref{app:scatter-plots}).

Two factors that presumably play a role in the performance differences between the SAT approaches and the ASP approaches are the following.
First, in each SAT-based approach, we need to generate a new at-most-$k$ constraint in each iteration (note that this has been addressed in Section~\ref{sec:runtime-composition} as well).
Second, the SAT solver does not reuse learned clauses from previous calls, i.\,e., each instance has to be constructed from scratch. 
%   -> sollte in Future Work vielleicht mal verbessert werden?^^ 
The ASP encodings, on the other hand, only have to be constructed and solved once for each knowledge base. %, and, consequently, there is also only one solving process involved.

We could further observe that the naive baseline methods for the contension and the forgetting-based measure outperformed the ASP-based and SAT-based approaches in some cases. % (note that wrt.\ SAT, this happened significantly more often than wrt.\ ASP). 
The most likely reason for this is the fact that the particular instances for which this effect was observed had a very low inconsistency value---in fact, most of these instances were consistent. 
The baseline methods for both the contension measure and the forgetting-based measure include a satisfiability check (via a SAT solver) as a first step. 
Hence, consistent knowledge bases can be identified rather quickly. 
Very small inconsistency values might also be retrieved rather fast by the two naive methods, as they both search for the correct value in a linear manner, starting from the lowest possible value. % (i.e., $1$, if the initial SAT check determines that the knowledge base is not consistent).
% Another factor that can further decrease the runtime of the two methods is whether or not a given knowledge base is already in CNF.
% If a knowledge base is not in CNF, additional transformation steps are required in order to guarantee a valid input for the SAT solver.
Another factor that can further decrease the runtime of the two methods is if the given knowledge base is already in CNF, so that no additional transformation steps are required in order to guarantee a valid input for the SAT solver.
% In our evaluation, only the SC data set consists of knowledge bases that are in CNF; all other data sets contain arbitrary formulas.
% Moreover, all solved instances from this data set are either consistent or possess inconsistency value $1$.
% Although the SC data set proved to be overall rather challenging (only a small fraction of instances could be solved at all), the naive methods for the contension and the forgetting-based measure performed relatively well. 
% To be specific, the two aforementioned methods could both solve $11$ instances, while the ASP-based ones could only solve $7$ and $4$, respectively (see Table \ref{tab:cumulative-runtimes} for details regarding numbers of timeouts).

A general observation from the evaluation results is that none of the three approaches for either one of the considered inconsistency measures could solve all instances from all data sets.
Hence, even the ASP approaches, which altogether performed superior, clearly reached their limits.
% For example, no instance of the ML data set could be solved for the sum-distance inconsistency measure.
% Another, previously mentioned, example is that only very few instances from the SC data set could be solved---even with an increased timeout of $5000$\,s. 
In the following, we aim to identify factors that make the computation of inconsistency values for a given knowledge base ``difficult'' (either for individual approaches or in general).
\begin{itemize}
    \item A rather obvious factor is the size of the given knowledge base in terms of the number of formulas, number of atoms in the signature, and number of connectives per formula (see Table \ref{tab:datasets-overview} for an overview regarding the data sets used in our experiments).
    A larger knowledge base consequently leads to larger encodings, as well as a larger range of possible inconsistency values, which in turn leads to a more difficult problem that the SAT or ASP solver has to solve, as well as to more search steps wrt.\ the SAT-based approaches.
    The SAT/ASP solver having to deal with a more difficult problem is directly connected to an increase in runtime, as the results in Section \ref{sec:runtime-composition} indicated that the solving process overall makes up the largest fraction of the runtimes of both the SAT-based and ASP-based methods.
    
    \item 
    % Although all six measures considered in this paper lie on the same level of the polynomial hierarchy, they are not equally ``expressive'' \cite{Thimm:2016a}, i.\,e., the ranges of their possible values vary greatly.
    Although all six measures considered in this paper lie on the same level of the polynomial hierarchy \cite{Thimm:2016a}, the ranges of their possible values vary greatly.
    A larger search space particularly influences the SAT-based approaches. 
    This is reflected, e.g., in the fact that for the sum-distance measure, fewer instances could be solved by the SAT-based approach than for any of the other measures.
    % For instance, we could observe that across all five data sets, the sum-distance inconsistency measure leads to the highest number of timeouts wrt.\ both the ASP-based and SAT-based methods.
    % This is no coincidence, as the search space of the sum-distance measure is always at least as large as (and in practice usually much larger than) that of the contension, the hitting set, the max-distance, and the hit-distance inconsistency measure.
    
    \item The range of the resulting inconsistency values may also be a relevant factor to take into account.
    If the inconsistency values are close to $0$, a SAT check as a preprocessing step could prove useful---we observed this effect with the naive approaches for the contension measure and the forgetting-based measure, in particular if the given knowledge base is already in CNF.
    % This is also important for the choice of search strategy regarding the SAT-based approaches (see Section \ref{sec:linear-search}).
    % As mentioned before, if the inconsistency values are close to $0$, it might even make sense to use a naive approach for the contension measure or the forgetting-based measure, in particular if the given knowledge base is already in CNF. 
    Moreover, the range of the inconsistency values is also relevant for the choice of search strategy regarding the SAT-based approaches (see Section \ref{sec:linear-search}).
\end{itemize}

Although the previously discussed points hint at which measures might be more suitable in practice in terms of runtime, the choice of an appropriate inconsistency measure for a specific practical application still depends on other factors as well.
For instance, in some scenarios it could be useful to look for conflicting formulas (which could be done by using the hitting set measure or the hit-distance measure), while in other scenarios it could be of greater interest to look for the atoms involved in a conflict (which could be done by using the contension or forgetting-based measure). 
Moreover, the granularity of the measured inconsistency could be of interest (e.\,g., the forgetting-based measure has a finer granularity than the contension measure).
Yet another aspect that should be taken into consideration when dealing with a practical application is the usefulness of the inconsistency measure for resolving the conflicts in the given knowledge base.

\section{Conclusion}\label{sec:conclusion}

In the course of this work, we proposed a SAT-based and an ASP-based approach for each of six different inconsistency measures (i.\,e., the contension, the forgetting-based, the hitting set, the max-distance, the sum-distance, and the hit-distance inconsistency measure).
With the SAT-based approaches, we encode the problem of whether a given value is an upper bound of the inconsistency degree, and retrieve the actual inconsistency value by means of iterative SAT solver calls in a (binary) search procedure.
In ASP, we can encode the problem of finding an inconsistency value directly by utilizing optimization statements.
% More precisely, we encode the problem of finding an inconsistency value in such a way that the optimal answer set provides the solution.

In an extensive experimental evaluation we compared the ASP-based to the SAT-based approaches, focusing on runtime.
To achieve this, we used a total of five different data sets, two of which were used in preceding works \cite{kuhlmann2021algorithms,kuhlmann2022comparison}, while the remaining three are novel. 
All data sets are publicly available for use in future work.
The results of our experiments first of all demonstrated that both the SAT-based and the ASP-based methods overall clearly outperform the naive implementations we used as a baseline---which, to the best of our knowledge, are the only previously existing implementations for the inconsistency measures considered.
Moreover, the ASP approaches performed superior to the SAT approaches (with only very few exceptions). 

Nonetheless, the results also showed that even the ASP methods are by far not able to solve all instances from all data sets wrt.\ all considered inconsistency measures.
Some data sets turned out to be generally challenging; an example of this is the SC data set, which contains benchmark data from the SAT competition 2020. 
A reason for this lies in the sheer size of the individual knowledge bases, i.\,e., the number of formulas contained in a knowledge base, its signature size, and the number of connectives per formula.
A large knowledge base requires a large SAT/ASP encoding (in terms of variables and clauses/rules), which in turn leads to an increased solving time and therefore an increased runtime---our results also showed that the solving process altogether makes up the largest fraction of the runtime composition of both the SAT and the ASP approaches.
% Further, with regard to the SAT methods, a large knowledge base is usually associated with a large search space and hence, an increased number of SAT solver calls.
% On a different note, our results revealed that, although all six inconsistency measures considered in this work are on the same level of the polynomial hierarchy, some measures are faster to compute than others.
% An explanation for this is that the different measures have differently sized search spaces and, in addition, the SAT and ASP encodings require different numbers of variables and clauses/rules, which affects the solving time.

In order to use approaches to inconsistency measurement in a practical application in a sensible fashion, an analysis of the given data is vital.
Information about signature sizes or numbers of formulas per knowledge base can be an indicator for which inconsistency measure will likely be a reasonable choice in terms of runtime. 
Moreover, if there is any information available about what degrees of inconsistency can roughly be expected, an adjusted search strategy can be used.
Such information can further be used to develop preprocessing techniques.
For instance, if there is a certain probability that the knowledge bases at hand might be consistent, an initial satisfiability check by means of a SAT solver (as we saw with the naive implementations of the contension and the forgetting-based measure) could prove useful.
Certainly there are many other options regarding the topic of preprocessing; however, this is subject to future work.

Another issue to contemplate in terms of future work is a more in-depth examination of MaxSAT approaches.
We have already observed in our experiments that a MaxSAT approach for the contension inconsistency measure performs superior to its iterative SAT counterpart.
% This should be investigated further, e.\,g., with regard to different inconsistency measures, or different data sets.
Moreover, there are numerous problem solving paradigms that have not been explored yet in the field of inconsistency measurement.
Examples include \textit{Integer Linear Programming}, \textit{Quantified Boolean Formulas}, and \textit{Counterexample-Guided Abstraction Refinement} approaches.
The choice of problem solving paradigm depends on the complexity of the measure at hand---which is related to yet another possible topic of future work: the investigation of algorithms for inconistency measures beyond the first level of the polynomial hierarchy.
Furthermore, we merely covered inconsistency measures designed for propositional knowledge bases.
The development of approaches for computing inconsistencies in other domains, such as different logics, data bases, or business process models, constitutes another open problem.
In addition, our evaluation showed that all approaches introduced in this work hit their limit at some point during the experiments.
To counteract this effect, i.\,e., to develop more scalable approaches, we already mentioned the possibility of applying preprocessing methods.
A different strategy to improve scalability could be to apply approximation techniques. % \\
% \\\
% \vspace*{0.4cm}
% \noindent
% {\bf Acknowledgements} 
% The research reported here was supported by the Deutsche Forschungsgemeinschaft under grant 506604007. \\
% \\\
% % \vspace*{0.4cm}
% % \noindent
% {\bf Competing Interests} 
% The authors declare none.

\acks{The research reported here was supported by the Deutsche Forschungsgemeinschaft under grant 506604007.
}

\appendix

\section{Proofs} \label{app:proofs}

This appendix contains the correctness proofs for the SAT-based approaches introduced in Section \ref{sec:sat}, and the ASP-based approaches introduced in Section \ref{sec:asp}.

%%% SAT
\subsection{SAT-based Algorithms}\label{app:proofs-sat}

Each of the following sections comprises a correctness proof corresponding to the SAT encoding for one of the six inconsistency measures considered in this work.

\subsubsection{The Contension Inconsistency Measure}\label{app:proof-sat-contension}

{\setcounter{table}{0}
\renewcommand\tablename{Encoding}
%\begin{table}
\begin{small}
\setlength\extrarowheight{2pt}
% \begin{longtable}{| >{\columncolor{gray!20}}L{.05\textwidth}  >{\columncolor{gray!40}}L{.496\textwidth} >{\columncolor{gray!40}}L{.34\textwidth} |}
\begin{longtable}{| >{\columncolor{gray!20}}L{.05\textwidth} >{\columncolor{gray!40}}L{.76\textwidth} >{\columncolor{gray!40}}L{.08\textwidth} |}
    % \centering
    % \small
    %\setlength\extrarowheight{2pt}
    % %\begin{tabular}{| >{\columncolor{gray!20}}L{.05\textwidth}  >{\columncolor{gray!40}}L{.496\textwidth} >{\columncolor{gray!40}}L{.34\textwidth} |}
     \hline 
        \rowcolor{gray!20} \multicolumn{3}{|l|}{\textbf{Signature}} \\
        \rowcolor{gray!20} \multicolumn{3}{|l|}{For every atom $\varAt \in \atoms(\kb)$:} \\
        & Create atoms $\varAt_t$, $\varAt_f$, $\varAt_b$ & (SC1) \\
        
        \rowcolor{gray!20} \multicolumn{3}{|l|}{For every formula $\varSub \in \mathsf{sub}(\kb)$:} \\
        & Create atoms $v^t_{\varSub}$, $v^f_{\varSub}$, $v^b_{\varSub}$ & (SC2) \\
        
        \hline
        
        \rowcolor{gray!20} \multicolumn{3}{|l|}{\textbf{Constraints}} \\
        \rowcolor{gray!20} \multicolumn{3}{|l|}{For every atom $\varAt \in \atoms(\kb)$:} \\
        & $(\varAt_t \vee \varAt_f \vee \varAt_b) \wedge (\neg \varAt_t \vee \neg \varAt_f) \wedge (\neg \varAt_t \vee \neg \varAt_b) \wedge (\neg \varAt_b \vee \neg \varAt_f)$ & (SC3)  \\
        
        % conjunction:
        \rowcolor{gray!20} \multicolumn{3}{|l|}{For every conjunction $\varSub_c = \varSubSub_{c,1} \land \varSubSub_{c,2}$ appearing in some formula:} \\
        & $v_{\varSub_c}^t \leftrightarrow v_{\varSubSub_{c,1}}^t \wedge v_{\varSubSub_{c,2}}^t$ & (SC4) \\
        & $v_{\varSub_c}^f \leftrightarrow v_{\varSubSub_{c,1}}^f \vee v_{\varSubSub_{c,2}}^f$ & (SC5) \\
        & $v_{\varSub_c}^b \leftrightarrow (\neg v_{\varSubSub_{c,1}}^t \vee \neg v_{\varSubSub_{c,2}}^t) \wedge \neg v_{\varSubSub_{c,1}}^f \wedge \neg v_{\varSubSub_{c,2}}^f$ & (SC6) \\
        
        % disjunction:
        \rowcolor{gray!20} \multicolumn{3}{|l|}{For every disjunction $\varSub_d = \varSubSub_{d,1} \lor \varSubSub_{d,2}$ appearing in some formula:} \\
        & $v_{\varSub_d}^t \leftrightarrow v_{\varSubSub_{d,1}}^t \vee v_{\varSubSub_{d,2}}^t$ & (SC7) \\
        & $v_{\varSub_d}^f \leftrightarrow v_{\varSubSub_{d,1}}^f \wedge v_{\varSubSub_{d,2}}^f$ & (SC8) \\
        & $v_{\varSub_d}^b \leftrightarrow (\neg v_{\varSubSub_{d,1}}^f \vee \neg v_{\varSubSub_{d,2}}^f) \wedge \neg v_{\varSubSub_{d,1}}^t \wedge \neg v_{\varSubSub_{d,2}}^t$  & (SC9) \\
        
        % negation:
        \rowcolor{gray!20} \multicolumn{3}{|l|}{For each negation $\varSub_n = \lnot \varSubSub_n$ appearing in some formula:} \\
        & $v^t_{\varSub_n} \leftrightarrow v^f_{\varSubSub_n}$ & (SC10) \\
        & $v^f_{\varSub_n} \leftrightarrow v^t_{\varSubSub_n}$ & (SC11) \\
        & $v^b_{\varSub_n} \leftrightarrow v^b_{\varSubSub_n}$ & (SC12) \\
        
        % if a formula consists of a single atom:
        \rowcolor{gray!20} \multicolumn{3}{|l|}{For every formula $\varSub_a$ which consists of a single atom $\varAt$:} \\
        & $v_{\varSub_a}^t \leftrightarrow \varAt_t$ & (SC13) \\
        & $v_{\varSub_a}^f \leftrightarrow \varAt_f$ & (SC14) \\
        & $v_{\varSub_a}^b \leftrightarrow \varAt_b$  & (SC15) \\
        
        \rowcolor{gray!20} \multicolumn{3}{|l|}{For every formula $\varFor \in \kb$:} \\
        & $v_\varFor^t \lor v_\varFor^b$ & (SC16) \\
        
        % Cardinality constraint:
        \rowcolor{gray!20} \multicolumn{3}{|l|}{Let $\atoms^\kb_b = \{\varAt_b \mid \varAt \in \atoms(\kb) \}$. Cardinality constraint:} \\
        & $\mathrm{at\_most}(u,\atoms^\kb_b)$ & (SC17) \\
        
        \hline
    %\end{tabular}
    \caption{Overview of the SAT encoding $\satc(\kb,u)$ of the contension inconsistency measure $\icont$.
    $\kb$ is a given knowledge base and $\atoms(\kb)$ its signature. 
    $u$ is a candidate for an upper limit of $\icont(\kb)$. 
    Let $\mathsf{sub}(\kb)$ be the set of subformulas of $\kb$ as defined in Section \ref{sec:satisfiability-solving}. 
    % All constraints are converted into CNF using Tseitin's method.
    The signature size of the encoding is $\abs{\atoms(\kb)} \cdot 3 + \abs{\mathsf{sub}(\kb)} \cdot 3 $.
    % The signature size of the encoding is $\abs{\atoms(\kb)} \cdot \abs{\mathsf{sub}(\kb)} \cdot 3 $.
    }
    \label{tab:sat-contension}
%\end{table}
\end{longtable}
\end{small}
}

{
\renewcommand{\thetheorem}{\ref{thm:sat-c}}
\begin{theorem}
For a given value $u$, the encoding $\satc(\kb,u)$ is satisfiable if and only if $\icont(\kb) \leq u$.
\begin{proof}
    Assume $\satc(\kb,u)$ is satisfiable and let $\omega$ be a model of $\satc(\kb,u)$. Define the three-valued interpretation $\omega^3$ via
    \begin{align*}
        \omega^3(\varAt) & = \left\{\begin{array}{ll}
                            t & \text{if~} \omega(\varAt_t)=t\\
                            f & \text{if~} \omega(\varAt_f)=t\\
                            b & \text{if~} \omega(\varAt_b)=t
                            \end{array}\right.
    \end{align*}
    Due to (SC3) (see Encoding~\ref{tab:sat-contension}) it is clear that $\omega^3$ is well-defined. We now show, via structural induction, that $\omega^3(\varSub)=\varTV$ iff $\omega(v^\varTV_\varSub)=t$ for all $\varSub \in \mathsf{sub}(\kb)$:
    \begin{itemize}
        \item 
            $\varSub_a=\varAt$ for $\varAt\in\atoms$: this follows already from the definition of $\omega^3$ above and (SC13)--(SC15).
        \item 
            $\varSub=\neg \varSubSub$: by induction hypothesis, $\omega^3(\varSubSub)=\varTV$ iff $\omega(v^\varTV_\varSubSub)=t$. 
            This implies that exactly one of $v^t_\varSubSub$, $v^f_\varSubSub$, $v^b_\varSubSub$ is true. 
            Due to (SC10)--(SC12), it follows that exactly one of $v^t_{\neg\varSubSub}$, $v^f_{\neg\varSubSub}$, $v^b_{\neg\varSubSub}$ is true. 
            Observe that (SC10)--(SC12) models exactly the semantics of negation in three-valued logic. 
            It follows $\omega^3(\neg\varSubSub)=\varTV$ iff $\omega(v^\varTV_{\neg\varSubSub})=t$.
        \item 
            $\varSub=\varSubSub_1\wedge \varSubSub_2$: by induction hypothesis, $\omega^3(\varSubSub_1)=\varTV$ iff $\omega(v^\varTV_{\varSubSub_1})=t$ and $\omega^3(\varSubSub_2)=\varTV$ iff $\omega(v^\varTV_{\varSubSub_2})=t$. 
            This implies that exactly one of $v^t_{\varSubSub_1}$, $v^f_{\varSubSub_1}$, $v^b_{\varSubSub_1}$ and exactly one of $v^t_{\varSubSub_2}$, $v^f_{\varSubSub_2}$, $v^b_{\varSubSub_2}$ is true. 
            Due to (SC4)--(SC6), it follows that exactly one of $v^t_{\varSubSub_1\wedge \varSubSub_2}$, $v^f_{\varSubSub_1\wedge \varSubSub_2}$, $v^b_{\varSubSub_1\wedge \varSubSub_2}$ is true. 
            Observe that (SC4)--(SC6) models exactly the semantics of conjunction in three-valued logic: (SC4) models the case that both $\varSubSub_1$ and $\varSubSub_2$ are \emph{true} (then $\varSubSub_1\wedge\varSubSub_2$ is also \emph{true}), (SC5) models the case that one of $\varSubSub_1$ and $\varSubSub_2$ is \emph{false} (then $\varSubSub_1\wedge\varSubSub_2$ is also \emph{false}), and (SC6) covers all remaining cases (then $\varSubSub_1\wedge\varSubSub_2$ is  \emph{both}). 
        \item 
            $\varSub=\varSubSub_1\vee \varSubSub_2$: by induction hypothesis, $\omega^3(\varSubSub_1)=\varTV$ iff $\omega(v^\varTV_{\varSubSub_1})=t$ and $\omega^3(\varSubSub_2)=\varTV$ iff $\omega(v^\varTV_{\varSubSub_2})=t$. 
            This implies that exactly one of $v^t_{\varSubSub_1}$, $v^f_{\varSubSub_1}$, $v^b_{\varSubSub_1}$ and exactly one of $v^t_{\varSubSub_2}$, $v^f_{\varSubSub_2}$, $v^b_{\varSubSub_2}$ is true. 
            Due to (SC7)--(SC9), it follows that exactly one of $v^t_{\varSubSub_1\vee \varSubSub_2}$, $v^f_{\varSubSub_1\vee \varSubSub_2}$, $v^b_{\varSubSub_1\vee \varSubSub_2}$ is true. 
            Observe that (SC7)--(SC9) models exactly the semantics of disjunction in three-valued logic: (SC7) models the case that at least one of $\varSubSub_1$ and $\varSubSub_2$ is \emph{true} (then $\varSubSub_1\vee\varSubSub_2$ is also \emph{true}), (SC8) models the case that both $\varSubSub_1$ and $\varSubSub_2$ are \emph{false} (then $\varSubSub_1\vee\varSubSub_2$ is also \emph{false}), and (SC9) covers all remaining cases (then $\varSubSub_1\vee\varSubSub_2$ is \emph{both}).
    \end{itemize}
    So it follows that $\omega^3(\varSub)=\varTV$ iff $\omega(v^\varTV_\varSub)=t$ for all $\varSub\in\mathsf{sub}(\kb)$. Note that this includes all $\varSub\in\kb$ as well. 
    Due to (SC16), for all $\varSub\in\kb$ we have $\omega^3(\varSub)\in\{t,b\}$ and therefore $\omega^3\in\textsf{Mod}^3(\kb)$. 
    Due to (SC17), we also have $|\textsf{Conflictbase}(\omega^3)|\leq u$ and it follows $\icont(\kb) \leq u$.
    
    For the other direction, assume $\icont(\kb) \leq u$. Then there exists a three-valued interpretation $\omega^3$ with $\omega^3\in\textsf{Mod}^3(\kb)$ and $|\textsf{Conflictbase}(\omega^3)|\leq u$. Define $\omega$ on the signature of $\satc(\kb,u)$ via
    \begin{align*}
        \omega(x_t) & =  \left\{\begin{array}{ll}
                            t & \text{if~} \omega^3(\varAt)=t\\
                            f & \text{otherwise}
                            \end{array}\right.\\
        \omega(x_f) & =  \left\{\begin{array}{ll}
                            t & \text{if~} \omega^3(\varAt)=f\\
                            f & \text{otherwise}
                            \end{array}\right.\\
        \omega(x_b) & =  \left\{\begin{array}{ll}
                            t & \text{if~} \omega^3(\varAt)=b\\
                            f & \text{otherwise}
                            \end{array}\right.\\
        \omega(v^t_\varSub) & =  \left\{\begin{array}{ll}
                            t & \text{if~} \omega^3(\varSub)=t\\
                            f & \text{otherwise}
                            \end{array}\right.\\
        \omega(v^f_\varSub) & =  \left\{\begin{array}{ll}
                            t & \text{if~} \omega^3(\varSub)=f\\
                            f & \text{otherwise}
                            \end{array}\right.\\
        \omega(v^b_\varSub) & =  \left\{\begin{array}{ll}
                            t & \text{if~} \omega^3(\varSub)=b\\
                            f & \text{otherwise}
                            \end{array}\right.
    \end{align*}
    for all $\varAt \in \atoms(\kb)$ and $\varSub \in \mathsf{sub}(\kb)$. Using the same argumentation as above, it can be shown that $\omega$ satisfies $\satc(\kb,u)$.
\end{proof}
\end{theorem} 
\addtocounter{theorem}{-1}
}

\subsubsection{The Forgetting-Based Inconsistency Measure}\label{app:proof-sat-fb}

{
\renewcommand\tablename{Encoding}
\begin{small}
\setlength\extrarowheight{2pt}
\begin{longtable}{| >{\columncolor{gray!20}}L{.05\textwidth} >{\columncolor{gray!40}}L{.76\textwidth} >{\columncolor{gray!40}}L{.08\textwidth} |}
     \hline 
        \rowcolor{gray!20} \multicolumn{3}{|l|}{\textbf{Signature}} \\
        \rowcolor{gray!20} \multicolumn{3}{|l|}{We make use of the original signature $\atoms(\kb)$.} \\
        \rowcolor{gray!20} \multicolumn{3}{|l|}{For every occurrence \par $\varAt^l \in \occs(\kb)$:} \\
        % & Create atom $\varAt^l$ & (SF1) \\
        & Create atoms $t_{\varAt,l}$, $f_{\varAt,l}$ & (SF1) \\
        
        \hline
        
        \rowcolor{gray!20} \multicolumn{3}{|l|}{\textbf{Constraints}} \\
        \rowcolor{gray!20} \multicolumn{3}{|l|}{For every formula $\varFor \in \kb$, replace every atom occurrence $\varAt^l$ with:} \\
        & $(t_{\varAt,l} \lor \varAt) \land \neg f_{\varAt,l}$ & (SF2)  \\
        
        \rowcolor{gray!20} \multicolumn{3}{|l|}{Let $\atoms_\forget$ be the set of all atoms defined by (SF3).} \\
        \rowcolor{gray!20} \multicolumn{3}{|l|}{For each pair $t_{\varAt,l}, f_{\varAt,l} \in \atoms_{\forget}$:} \\
        & $\neg t_{\varAt,l} \lor \neg f_{\varAt,l}$ & (SF4) \\
        
        \rowcolor{gray!20} \multicolumn{3}{|l|}{Cardinality constraint:} \\
        & $\mathrm{at\_most}(u,\atoms_\forget)$ & (SF4) \\
        
        \hline
    \caption{Overview of the SAT encoding $\satf(\kb,u)$ of the forgetting-based inconsistency measure $\iforget$.
    $\kb$ is a knowledge base and $\atoms(\kb)$ its signature. 
    $u$ is a candidate for an upper limit for $\iforget(\kb)$. 
    Let $\occs(\kb)$ be the set of all atom occurrences in $\kb$ as defined in Section \ref{sec:satisfiability-solving}. 
    The signature size of the encoding is $|\atoms(\kb)| + 2 \cdot |\occs(\kb)|$.
    }
    \label{tab:sat-fb}
\end{longtable}
\end{small}
}

Let $\kb$ be an arbitrary knowledge base and $\satf(\kb, u)$ the corresponding SAT encoding defined via (SF1)--(SF4).
Assume that $\satf(\kb,u)$ is satisfiable, and let $\omega$ be a model of $\satf(\kb,u)$.
We define
    \begin{align*}
        T_{\satf(\kb,u)}(\omega) = & \{ t_{\varAt,l} \mid \omega(t_{\varAt,l}) = t \} \\
        F_{\satf(\kb,u)}(\omega) = & \{ f_{\varAt,l} \mid \omega(f_{\varAt,l}) = t \}
    \end{align*}
Due to (SF3) (i.\,e., $ \lnot t_{\varAt,l} \lor \lnot f_{\varAt,l}$), $T_{\satf(\kb,u)}(\omega) \cap F_{\satf(\kb,u)}(\omega) = \emptyset$ for every model $\omega$ of $\satf(\kb,u)$.

\begin{lemma}\label{lem:sat-forget}
    If $\omega$ is a model of $\satf(\kb,u)$ then $( \bigwedge \mathcal{K})[\varAt^{1}_1 \rightarrow \top;\ldots; \varAt^{n}_p \rightarrow \top ; Y^{1}_1 \rightarrow \bot;\ldots; Y^{m}_q \rightarrow \bot] $ is consistent with $T_{\satf(\kb,u)}(\omega) =\{t_{\varAt_1,1},\ldots,t_{\varAt_p,n}\}$ and $F_{\satf(\kb,u)}(\omega) = \{f_{Y_1,1},\ldots,f_{Y_q,m}\}$.
    \begin{proof}
        To begin with, each $\varAt^l$ represents exactly the $l$-th occurrence of an atom $\varAt$. % (SF1).
        Moreover, each $\varAt^l$ gets a distinct pair of atoms $t_{\varAt,l}$ and $f_{\varAt,l}$ (SF1).
        Thus, $t_{\varAt,l}$ and $f_{\varAt,l}$ represent exactly the $l$-th occurrence of $\varAt$.
    
        Following (SF2), each atom occurrence $\varAt^l$ is replaced by the term 
        $(t_{\varAt, l} \lor \varAt ) \land \lnot f_{\varAt, l}$.
        If $\omega(t_{\varAt, l}) = t$ (and $\omega(f_{\varAt, l}) = f$ due to (SF3)), the term becomes true, regardless of whether $\omega(\varAt) = t$ or $\omega(X) = f$.
        Hence, this corresponds exactly to replacing the $l$-th occurrence of $\varAt$ in the original formula by $\top$.
        On the other hand, if $\omega(f_{\varAt, l}) = t$ (and $\omega(t_{\varAt, l}) = f$ due to (SF3)), the term becomes false, regardless of the truth value of $\varAt$.
        Consequently, this corresponds to replacing the $l$-th occurrence of $\varAt$ in the original formula by $\bot$.
        If $\omega(t_{\varAt, l}) = f$ and $\omega(f_{\varAt, l}) = f$, the right part of the conjunction is true, and the left part is true if $\omega(\varAt) = t$ and false if $\omega(\varAt) = f$.
        Thus, the term evaluates to the truth value of $\varAt^l$, i.\,e., the truth value of the $l$-th occurrence of $\varAt$ in the original formula.
    \end{proof}
\end{lemma}
{
\renewcommand{\thetheorem}{\ref{thm:sat-fb}}
\begin{theorem}
    For a given value $u$, the encoding $\satf(\kb,u)$ is satisfiable if and only if $\iforget(\kb) \leq u$.
    \begin{proof}
        From Lemma \ref{lem:sat-forget} we know that $T_{\satf(\kb,u)}(\omega)$ and $F_{\satf(\kb,u)}(\omega)$ contain exactly those $t_{\varAt,l}$ and $f_{\varAt,l}$ which correspond to those occurrences of $\varAt$ that are being forgotten (i.\,e., replaced by $\top$, and $\bot$, respectively).
        Moreover, the cardinality constraint (SF4) ensures that $|T_{\satf(\kb,u)}(\omega) \cup F_{\satf(\kb,u)}(\omega)| \leq u$.
        Thus, if more than $u$ atom occurrences must be forgotten in order for $\kb$ to be satisfiable, then $\satf(\kb,u)$ is unsatisfiable.
        Otherwise, i.\,e., if a minimum of $\leq u$ atom occurrences must be forgotten, $\satf(\kb,u)$ is satisfiable.
    \end{proof}
\end{theorem}
\addtocounter{theorem}{-1}
}

\newpage

\subsubsection{The Hitting-Set Inconsistency Measure}\label{app:proof-sat-hs}

{
\renewcommand\tablename{Encoding}
\begin{small}
\setlength\extrarowheight{2pt}
\begin{longtable}{| >{\columncolor{gray!20}}L{.05\textwidth} >{\columncolor{gray!40}}L{.76\textwidth} >{\columncolor{gray!40}}L{.08\textwidth} |}
     \hline 
        \rowcolor{gray!20} \multicolumn{3}{|l|}{\textbf{Signature}} \\
        \rowcolor{gray!20} \multicolumn{3}{|l|}{For every atom $\varAt \in \atoms(\kb)$:} \\
        & Create variables $\varAt_i$ with $i\in \{1,\ldots , u\}$ & (SH1) \\
        
        \rowcolor{gray!20} \multicolumn{3}{|l|}{For every formula $\varFor \in \kb$:} \\
        & Create variables $p_{\varFor,i}$ with $i\in \{1,\ldots , u\}$ & (SH2) \\
        
        \hline
        
        \rowcolor{gray!20} \multicolumn{3}{|l|}{\textbf{Constraints}} \\
        
        \rowcolor{gray!20} \multicolumn{3}{|l|}{For every formula $\varFor \in \kb$ and $i\in \{1,\ldots , u\}$, create $\varFor_{i}$, which is a copy of $\varFor$} \\
        \rowcolor{gray!20} \multicolumn{3}{|l|}{where each instance of atom $\varAt$ is replaced by $\varAt_i$, and add:} \\

        & $p_{\varFor,i} \rightarrow \varFor_i$ & (SH3)  \\
        
        \rowcolor{gray!20} \multicolumn{3}{|l|}{For every formula $\varFor \in \kb$:} \\
        & $\bigvee_{1\leq i \leq u} p_{\varFor, i}$ & (SH4) \\[1ex]
        
        \hline
    \caption{Overview of the SAT encoding $\sath(\kb, u)$ of the hitting set inconsistency measure $\ihs$.
    $\kb$ is a knowledge base and $\atoms(\kb)$ its signature. 
    $u$ is a candidate for an upper limit for $\ihs(\kb)$. 
    The signature size of the encoding is $u \cdot \abs{\atoms(\kb)} + u \cdot \abs{\kb}$.}
    \label{tab:sat-hs}
\end{longtable}
\end{small}
}

Let $\kb$ be an arbitrary knowledge base and $\sath(\kb, u)$ the corresponding SAT encoding defined via (SH1)--(SH4).

\begin{lemma}\label{lem:sat-hs-1}
    If $\kb$ contains at least one contradictory formula, $\sath(\kb, u)$ is unsatisfiable for all $u \in \{1, \ldots, |\kb|\}$.
    \begin{proof}
        Let $u \in \{1, \ldots, |\kb|\}$ be arbitrary, and let some $\varFor^\bot \in \kb$ be contradictory. 
        Then there exists no interpretation which could satisfy a copy $\varFor^\bot_i$ ($i \in \{1, \ldots, u\}$) of $\varFor^\bot$. 
        Hence, in order to satisfy the constraints defined by (SH3) (i.\,e., $p_{\varFor^\bot,i} \rightarrow \varFor^\bot_i$), each $p_{\varFor^\bot,i}$ would need to evaluate to false.
        However, due to (SH4), at least one $p_{\varFor^\bot,i}$ must be true. 
        Thus, $\sath(\kb, u)$ is unsatisfiable, regardless of the value of $u$.
    \end{proof}
\end{lemma}

Assume that $\kb$ does not contain any contradictory formulas and that $\satf(\kb,u)$ is satisfiable, and let $\omega$ be a model of $\satf(\kb,u)$.
We define
    $$ I_{\sath(\kb,u)}(\omega) = \{i_a \mid \exists \omega(p_{\varFor,i_a}) = t \} $$
with $\varFor \in \kb$, and $i_a \in \{1, \ldots, u\}$.

\begin{lemma}\label{lem:sat-hs-2}
    $I_{\sath(\kb,u)}(\omega)$ corresponds to a hitting set of $\kb$.
    \begin{proof}
        For each formula $\varFor \in \kb$ we have at least one $p_{\varFor,i_a}$ that evaluates to true, because of (SH4).
        Consequently, wrt.\ that $p_{\varFor,i_a}$, $\varFor_{i_a}$ must be true as well (SH3).
        Hence, at least one copy $\varFor_{i_a}$ of each $\varFor \in \kb$ is satisfied.

        Further, each $i \in \{1, \ldots, u\}$ corresponds to an interpretation, since each $A_i$ uses the $i$-th copy of $\atoms(\kb)$.
        Moreover, each $i_a \in I_{\sath(\kb,u)}$ corresponds to an ``active'' interpretation, i.\,e., one that is actually used to satisfy one or more formulas.
        Thus, $I_{\sath(\kb,u)}(\omega)$ corresponds to a hitting set of $\kb$.
    \end{proof}
\end{lemma}

{
\renewcommand{\thetheorem}{\ref{thm:sat-hs}}
\begin{theorem}
    For a given value $u$, the encoding $\sath(\kb,u)$ is satisfiable if and only if $\ihs(\kb) \leq u - 1$.
    $\sath(\kb,u)$ is unsatisfiable for all $u=1,\ldots,|\kb|$ if and only if $\ihs(\kb) = \infty$.
    \begin{proof}
        Lemma \ref{lem:sat-hs-1} shows that $\sath(\kb,u)$ is unsatisfiable for all $u \in \{1, \ldots, |\kb|\}$ if $\kb$ contains a contradictory formula, i.\,e., if $\ihs(\kb) = \infty$.
        Lemma \ref{lem:sat-hs-2} shows that, if $\kb$ does not contain any contradictory formula, $I_{\sath(\kb,u)}(\omega) = \{i_a \mid \exists \omega(p_{\varFor,i_a}) = t \}$ corresponds to a hitting set of $\kb$.
        Since the maximum cardinality of $I_{\sath(\kb,u)}(\omega)$ is restricted by $u$, $\sath(\kb,u)$ is unsatisfiable if more than $u$ copies of each formula are required, i.\,e., if more than $u$ different interpretations are required to satisfy each formula in $\kb$.
        Otherwise, $\sath(\kb,u)$ is satisfiable.
        
        Thus, if a hitting set of size $\leq u$ exists, $\sath(\kb,u)$ is satisfiable.
        Otherwise, i.\,e., if no hitting set of size $\leq u$ exists, $\sath(\kb,u)$ is unsatisfiable.
        Since we have to subtract $1$ from the cardinality of the hitting set in order to get $\ihs(\kb)$, it follows that if $\ihs(\kb) \leq u-1$, then $\sath(\kb,u)$ is satisfiable, and if $\ihs(\kb) > u-1$, then $\sath(\kb,u)$ is unsatisfiable.
    \end{proof}
\end{theorem} 
\addtocounter{theorem}{-1}
}

\subsubsection{The Max-Distance Inconsistency Measure}\label{app:proof-sat-max-dist}

{
\renewcommand\tablename{Encoding}
\begin{small}
\setlength\extrarowheight{2pt}
\begin{longtable}{| >{\columncolor{gray!20}}L{.05\textwidth} >{\columncolor{gray!40}}L{.73\textwidth} >{\columncolor{gray!40}}L{.11\textwidth} |}
     \hline 
        \rowcolor{gray!20} \multicolumn{3}{|l|}{\textbf{Signature}} \\
        \rowcolor{gray!20} \multicolumn{3}{|l|}{For every atom $\varAt \in \atoms(\kb)$:} \\
        & Create variable $\varAt_o$ & (SDM1) \\
        
        \rowcolor{gray!20} \multicolumn{3}{|l|}{For every atom $\varAt \in \atoms(\kb)$ and $i \in \{1, \ldots, |\kb|\}$:} \\
        & Create variables $\varAt_i$ & (SDM2) \\
        
        & Create variables $\textsf{inv}_{\varAt,i}$ & (SDM3) \\
        
        \hline
        
        \rowcolor{gray!20} \multicolumn{3}{|l|}{\textbf{Constraints}} \\
        
        \rowcolor{gray!20} \multicolumn{3}{|l|}{For every formula $\varFor \in \kb$ and $i\in \{1,\ldots , |\kb|\}$:} \\
        & Create $\varFor_{i}$, which is a copy of $\varFor$ where each instance of atom $\varAt$ is replaced by $\varAt_i$ & (SDM4) \\
        
        \rowcolor{gray!20} \multicolumn{3}{|l|}{For every $\varAt_i$ with $i\in \{1,\ldots , |\kb|\}$:} \\
        & $\varAt_i \rightarrow \varAt_o \lor \mathsf{inv}_{\varAt,i}$ & (SDM5) \\
        & $\lnot \varAt_i \rightarrow \lnot \varAt_o \lor \mathsf{inv}_{\varAt,i}$ & (SDM6) \\
        
        \rowcolor{gray!20} \multicolumn{3}{|l|}{Let $\textsf{INV}_i$ be the set containing all $\textsf{inv}_{\varAt,i}$ for a fixed $i$.} \\
        \rowcolor{gray!20} \multicolumn{3}{|l|}{For every $i \in \{1, \ldots,|\kb|\}$:} \\
        & $\mathrm{at\_most}(u, \textsf{INV}_i )$ & (SDM7) \\
        
        \hline
    \caption{Overview of the SAT encoding $\satmax(\kb, u)$ of the max-distance inconsistency measure $\imdalal$.
    $\kb$ is a knowledge base and $\atoms(\kb)$ its signature. 
    $u$ is a candidate for an upper limit for $\imdalal(\kb)$. 
    The signature size of the encoding is $\abs{\atoms(\kb)} + 2 \cdot \abs{\kb} \cdot \abs{\atoms(\kb)}$.}
    \label{tab:sat-dalal-max}
\end{longtable}
\end{small}
}

Let $\kb$ be an arbitrary knowledge base and $\satmax(\kb, u)$ the corresponding SAT encoding defined via (SDM1)--(SDM7).

\begin{lemma}\label{lem:sat-max-dist-1}
    If $\kb$ contains at least one contradictory formula, $\satmax(\kb, u)$ is unsatisfiable for all $u \in \{0, \ldots, |\atoms(\kb)|\}$.
    \begin{proof}
        Let some $\varFor^\bot \in \kb$ be contradictory. 
        Following (SDM4), we add indexed copies of all formulas $\varFor \in \kb$ to $\satmax(\kb,u)$.
        Thus, we also have an indexed copy $\varFor^\bot_i$ of $\varFor^\bot$.
        Since adding an index $i$ to each atom in $\varFor$ does not resolve the conflict within the formula, $\varFor^\bot_i$ is still contradictory.
        Therefore, $\satmax(\kb,u)$ contains an unsatisfiable formula and is consequently itself unsatisfiable, regardless of the value of $u$.
    \end{proof}
\end{lemma}

Assume that $\kb$ does not contain any contradictory formulas and that $\satmax(\kb,u)$ is satisfiable, and let $\omega$ be a model of $\satmax(\kb,u)$.
For each $i \in \{1, \ldots, |\kb|\}$ we create variables $X_i$ (SDM2) and variables $\mathsf{inv}_{\varAt,i}$ (SDM3) for each $\varAt \in \atoms(\kb)$.
We define
    $$ I_{\max}(\omega,i) = \{\mathsf{inv}_{\varAt,i} \mid \omega(\mathsf{inv}_{\varAt,i}) = t \} $$

Recall that we represent an ``optimal'' interpretation by introducing a variable $\varAt_o$ for each $\varAt \in \atoms(\kb)$ (SDM1).
The truth values assigned to each $\varAt_o$ then correspond to the ``optimal'' interpretation.
\begin{lemma}\label{lem:sat-max-dist-2}
    If $\omega(\varAt_i) \neq \omega(\varAt_o)$ then $\mathsf{inv}_{\varAt, i} \in I_{\max}(\omega,i)$.
    \begin{proof}
        $\varAt_i$ must be either true or false.
        If $\omega(\varAt_i) = t$, then the left part of the implication $\varAt_i \rightarrow \varAt_o \lor \mathsf{inv}_{\varAt,i}$ (SDM5) is true. 
        Thus, in order to satisfy the implication, the right part must be true as well.
        If $\varAt_o$ is false, i.\,e., $\omega(\varAt_i) \neq \omega(\varAt_o)$, then $\mathsf{inv}_{\varAt,i}$ must be true. 
        Consequently, $\mathsf{inv}_{\varAt, i} \in I_{\max}(\omega,i)$.
        Moreover, the implication $\lnot \varAt_i \rightarrow \lnot \varAt_o \lor \mathsf{inv}_{\varAt,i}$, defined via (SDM6), evaluates to true as well, because the left part ($\lnot \varAt_i$) is false.

        If $\omega(\varAt_i) = f$, the left part of the implication defined by (SDM6), i.\,e., $\lnot \varAt_i \rightarrow \lnot \varAt_o \lor \mathsf{inv}_{\varAt,i}$, is true.
        Thus, to satisfy this implication, $\lnot \varAt_o \lor \mathsf{inv}_{\varAt,i}$ must evaluate to true.
        If $\omega(\varAt_o) = t$, meaning if $\omega(\varAt_i) \neq \omega(\varAt_o)$, then $\lnot \varAt_o$ is false, and $\mathsf{inv}_{\varAt,i}$ must be true. 
        Hence, $\mathsf{inv}_{\varAt, i} \in I_{\max}(\omega,i)$.
        In addition, the implication $\varAt_i \rightarrow \varAt_o \lor \mathsf{inv}_{\varAt,i}$, defined by (SDM5), is also satisfied, because the left part ($\varAt_i$) is false.
    \end{proof}
\end{lemma}

{
\renewcommand{\thetheorem}{\ref{thm:sat-maxdalal}}
\begin{theorem}
    For a given value $u$, the encoding $\satmax(\kb, u)$ is satisfiable if and only if $\imdalal(\kb) \leq u$.
    $\satmax(\kb,u)$ is unsatisfiable for all $u \in \{ 0, \ldots, \abs{\atoms(\kb)}\}$ if and only if $\imdalal(\kb) = \infty$.
    \begin{proof}
        %To begin with, 
        Lemma \ref{lem:sat-max-dist-1} shows that $\satmax(\kb,u)$ is unsatisfiable for all $u \in \{1, \ldots, |\atoms(\kb)|\}$ if $\kb$ contains a contradictory formula, i.\,e., if $\imdalal(\kb) = \infty$.

        Let us now assume that $\kb$ does not contain any contradictory formulas, i.\,e., $\imdalal(\kb) \neq \infty$.
        From Lemma \ref{lem:sat-max-dist-2} we know that for each $i\in |\kb|$, if $\omega(\varAt_i) \neq \omega(\varAt_o)$, then the corresponding $\mathsf{inv}_{\varAt,i}$ must be included in $I_{\max}(\omega, i)$ (i.\,e., $\omega(\mathsf{inv}_{\varAt,i}) = t$).
        This means that if a formula $\varFor \in \kb$ with index $i$ requires an interpretation that differs from the ``optimal'' interpretation, then for each $\varAt_i \in \atoms(\varFor_i)$ with $\omega(\varAt_i) \neq \omega(\varAt_o)$, $\mathsf{inv}_{\varAt,i} \in I_{\max}(\omega, i)$.
        Further, (SDM7) restricts $|I_{\max}(\omega, i)|$ to $u$ (for each $i$).
        If $\imdalal(\kb) \leq u$, there exists a solution for which $|I_{\max}(\omega, i)| \leq u$ for each $i\in \{1, \ldots, |\kb|\}$.
        
        Hence, the constraints defined by (SDM7) are satisfied, and $\satmax(\kb,u)$ is satisfiable.
        If $\imdalal(\kb) > u$, then for at least one $i$ we have $|I_{\max}(\omega, i)| > u$.
        Consequently, at least one constraint defined by (SDM7) cannot be satisfied, which makes $\satmax(\kb,u)$ unsatisfiable.
    \end{proof}
\end{theorem} 
\addtocounter{theorem}{-1}
}
% \vspace{.2cm}
\newpage

\subsubsection{The Sum-Distance Inconsistency Measure}\label{app:proof-sat-sum-dist}

{
\renewcommand\tablename{Encoding}
\begin{small}
\setlength\extrarowheight{2pt}
\begin{longtable}{| >{\columncolor{gray!20}}L{.05\textwidth} >{\columncolor{gray!40}}L{.73\textwidth} >{\columncolor{gray!40}}L{.11\textwidth} |}
     \hline 
        \rowcolor{gray!20} \multicolumn{3}{|l|}{\textbf{Signature}} \\
        \rowcolor{gray!20} \multicolumn{3}{|l|}{For every atom $\varAt \in \atoms(\kb)$:} \\
        & Create variable $\varAt_o$ & (SDS1) \\
        
        \rowcolor{gray!20} \multicolumn{3}{|l|}{For every atom $\varAt \in \atoms(\kb)$ and $i \in \{1, \ldots, |\kb|\}$:} \\
        & Create variables $\varAt_i$ & (SDS2) \\
        
        & Create variables $\textsf{inv}_{\varAt,i}$ & (SDS3) \\
        
        \hline
        
        \rowcolor{gray!20} \multicolumn{3}{|l|}{\textbf{Constraints}} \\
        
        \rowcolor{gray!20} \multicolumn{3}{|l|}{For every formula $\varFor \in \kb$ and $i\in \{1,\ldots , |\kb|\}$:} \\
        & Create $\varFor_{i}$, which is a copy of $\varFor$ where each instance of atom $\varAt$ is replaced by $\varAt_i$ & (SDS4) \\
        
        \rowcolor{gray!20} \multicolumn{3}{|l|}{For every $\varAt_i$ with $i\in \{1,\ldots , |\kb|\}$:} \\
        & $\varAt_i \rightarrow \varAt_o \vee \textsf{inv}_{\varAt,i}$ & (SDS5) \\
        & $\lnot \varAt_i \rightarrow \neg \varAt_o \vee \textsf{inv}_{\varAt,i} $ & (SDS6) \\
        
        \rowcolor{gray!20} \multicolumn{3}{|l|}{Let $\textsf{INV}$ be the set containing all $\textsf{inv}_{\varAt,i}$ for all $i$. Cardinality constraint:} \\
        & $\mathrm{at\_most}(u, \textsf{INV})$ & (SDS7) \\
        
        \hline
    \caption{Overview of the SAT encoding $\satsum(\kb, u)$ of the sum-distance inconsistency measure $\isdalal$.
    $\kb$ is a knowledge base and $\atoms(\kb)$ its signature. 
    $u$ is a candidate for an upper limit for $\isdalal(\kb)$. 
    The signature size of the encoding is $\abs{\atoms(\kb)} + 2 \cdot \abs{\kb} \cdot \abs{\atoms(\kb)}$.}
    \label{tab:sat-dalal-sum}
\end{longtable}
\end{small}
}

Let $\kb$ be an arbitrary knowledge base and $\satsum(\kb, u)$ the corresponding SAT encoding defined via (SDS1)--(SDS7).

\begin{lemma}\label{lem:sat-sum-dist-1}
    If $\kb$ contains at least one contradictory formula, $\satsum(\kb, u)$ is unsatisfiable for all $u \in \{0, \ldots, |\atoms(\kb)| \cdot |\kb|\}$.
    \begin{proof}
        Analogous to the proof of Lemma \ref{lem:sat-max-dist-1}.
    \end{proof}
\end{lemma}

Assume that $\kb$ does not contain any contradictory formulas and that $\satsum(\kb,u)$ is satisfiable, and let $\omega$ be a model of $\satsum(\kb,u)$.
For each $i \in \{1, \ldots, |\kb|\}$ we create variables $X_i$ (SDS2) and variables $\mathsf{inv}_{\varAt,i}$ (SDS3) for each $\varAt \in \atoms(\kb)$.
We define
    $$ I_{\Sigma}(\omega) = \{\mathsf{inv}_{\varAt,i} \mid \omega(\mathsf{inv}_{\varAt,i}) = t \} $$
    
Recall that we represent an ``optimal'' interpretation by introducing a variable $\varAt_o$ for each $\varAt \in \atoms(\kb)$ (SDS1).
The truth values assigned to each $\varAt_o$ then correspond to the ``optimal'' interpretation.

\begin{lemma}\label{lem:sat-sum-dist-2}
    If $\omega(\varAt_i) \neq \omega(\varAt_o)$ then $\mathsf{inv}_{\varAt, i} \in I_{\Sigma}(\omega)$.
    \begin{proof}
        Analogous to the proof of Lemma \ref{lem:sat-max-dist-2}, since (SDM1)--(SDM6) is equivalent to (SDS1)--(SDS6).
    \end{proof}
\end{lemma}

{
\renewcommand{\thetheorem}{\ref{thm:sat-sumdalal}}
\begin{theorem}
    For a given value $u$, the encoding $\satsum(\kb, u)$ is satisfiable if and only if $\isdalal(\kb) \leq u$.
    $\satsum(\kb,u)$ is unsatisfiable for all $u \in \{0, \ldots, \abs{\atoms(\kb)} \cdot \abs{\kb} \}$ if and only if $\isdalal(\kb) = \infty$.
    \begin{proof}
        Lemma \ref{lem:sat-sum-dist-1} shows that $\satsum(\kb,u)$ is unsatisfiable for all $u \in \{0, \ldots, |\atoms(\kb)| \cdot |\kb|\}$ if $\kb$ contains a contradictory formula, i.\,e., if $\isdalal(\kb) = \infty$.

        Let us now assume that $\kb$ does not contain any contradictory formulas, i.\,e., $\isdalal(\kb) \neq \infty$.
        From Lemma \ref{lem:sat-sum-dist-2} we know that for any $i\in |\kb|$, if $\omega(\varAt_i) \neq \omega(\varAt_o)$, then the corresponding $\mathsf{inv}_{\varAt,i}$ must be included in $I_{\Sigma}(\omega)$ (i.\,e., $\omega(\mathsf{inv}_{\varAt,i}) = t$).
        This means that if a formula $\varFor \in \kb$ with index $i$ requires an interpretation that differs from the ``optimal'' interpretation, then for each $\varAt_i \in \atoms(\varFor_i)$ with $\omega(\varAt_i) \neq \omega(\varAt_o)$, $\mathsf{inv}_{\varAt,i} \in I_{\Sigma}(\omega)$.
        Further, (SDS7) restricts $|I_{\Sigma}(\omega)|$ to $u$.
        If $\isdalal(\kb) \leq u$, there exists a solution for which $|I_{\Sigma}(\omega)| \leq u$.
        Hence, the constraint defined by (SDS7) is satisfied, and $\satsum(\kb,u)$ is satisfiable.
        If $\isdalal(\kb) > u$, then we have $|I_{\Sigma}(\omega)| > u$.
        Consequently, the constraint defined by (SDS7) cannot be satisfied, which makes $\satsum(\kb,u)$ unsatisfiable.
    \end{proof}
\end{theorem} 
\addtocounter{theorem}{-1}
}

\subsubsection{The Hit-Distance Inconsistency Measure}\label{app:proof-sat-hit-dist}

{
\renewcommand\tablename{Encoding}
\begin{small}
\setlength\extrarowheight{2pt}
\begin{longtable}{| >{\columncolor{gray!20}}L{.05\textwidth} >{\columncolor{gray!40}}L{.73\textwidth} >{\columncolor{gray!40}}L{.11\textwidth} |}
     \hline 
     
        \rowcolor{gray!20} \multicolumn{3}{|l|}{\textbf{Signature}} \\
        \rowcolor{gray!20} \multicolumn{3}{|l|}{For every formula $\varFor \in \kb$:} \\
        & Create variable $\textsf{hit}_\varFor$ & (SDH1) \\
        
        \rowcolor{gray!20} \multicolumn{3}{|l|}{For every atom $\varAt \in \atoms(\kb)$:} \\
        & Create variable $\varAt$ & (SDH2) \\
        
        \hline
        
        \rowcolor{gray!20} \multicolumn{3}{|l|}{\textbf{Constraints}} \\
        
        \rowcolor{gray!20} \multicolumn{3}{|l|}{For every formula $\varFor \in \kb$ :} \\
        & $\varFor \lor \textsf{hit}_\varFor$ & (SDH3) \\
        
        \rowcolor{gray!20} \multicolumn{3}{|l|}{Let $\textsf{HIT}_\kb$ be the set containing all variables added in (SDH1).} \\
        \rowcolor{gray!20} \multicolumn{3}{|l|}{Cardinality constraint:} \\
        & $\mathrm{at\_most}(u, \textsf{HIT}_\kb )$ & (SDH4) \\
        
        \hline
    \caption{Overview of the SAT encoding $\sathit(\kb, u)$ of the hit-distance inconsistency measure $\ihdalal$.
    $\kb$ is a knowledge base and $\atoms(\kb)$ its signature. 
    $u$ is a candidate for an upper limit for $\ihdalal(\kb)$. 
    The signature size of the encoding is $\abs{\atoms(\kb)} + \abs{\kb}$.}
    \label{tab:sat-dalal-hit}
\end{longtable}
\end{small}
}

Let $\kb$ be an arbitrary knowledge base and $\sathit(\kb, u)$ the corresponding SAT encoding defined via (SDH1)--(SDH4).
Assume that $\sathit(\kb,u)$ is satisfiable, and let $\omega$ be a model of $\sathit(\kb,u)$.
We define 
    $$ H_\mathsf{hit}(\omega) = \{ \mathsf{hit}_\varFor \mid \omega(\mathsf{hit}_\varFor) = t \}. $$

\begin{lemma}\label{lem:hit-dist-sat}
    Each $\mathsf{hit}_\varFor \in H_{\mathsf{hit}}(\omega)$ corresponds to a formula $\varFor \in \kb$ being removed from $\kb$.
    \begin{proof}
        Due to (SDH3), every formula $\varFor \in \kb$ is extended by ``$\lor\ \mathsf{hit}_\varFor$''. 
        If $\omega(\mathsf{hit}_\varFor) = t$, the term $\varFor \lor \mathsf{hit}_\varFor$ will always be true, regardless of the truth value of $\varFor$ under $\omega$.
        Hence, this is equivalent to the removal of $\varFor$ from $\kb$.
    \end{proof}
\end{lemma}

{
\renewcommand{\thetheorem}{\ref{thm:sat-hitdalal}}
\begin{theorem}
    For a given value $u$, the encoding $\sathit(\kb, u)$ is satisfiable if and only if $\ihdalal(\kb) \leq u$.
    \begin{proof}
        The cardinality constraint (SDH4) restricts the number of atoms $\mathsf{hit}_\varFor$ being allowed to be true, to $u$.
        If the minimal number of formulas $\varFor \in \kb$ that need to be removed in order to make $\kb$ consistent is $\leq u$ (i.\,e., if $\ihdalal(\kb) \leq u$) then $\leq u$ atoms $\mathsf{hit}_\varFor$ can be set to true (i.\,e., $|H_\mathsf{hit}(\omega)| \leq u$).
        From Lemma \ref{lem:hit-dist-sat} we know that those atoms correspond to the formulas that are removed.
        Thus, there exists a solution in which exactly those atoms $\mathsf{hit}_\varFor$ are set to true which correspond to the formulas $\varFor \in \kb$ that need to be removed to make $\kb$ consistent.
        Consequently, the remaining formulas in $\kb$ are satisfiable, and $\sathit(\kb, u)$ is satisfiable.
        If more than $u$ formulas need to be removed (i.\,e., if $\ihdalal(\kb) > u$), (SDH4) still restricts the number of $\mathsf{hit}_\varFor$ atoms being set to true to $u$. 
        Hence, there exists no model for $\{\bigcup \varFor \in \kb \mid \omega(\mathsf{hit}_\varFor) = f\}$, and $\sathit(\kb, u)$ is unsatisfiable.
    \end{proof}
\end{theorem} 
\addtocounter{theorem}{-1}
}

%%% ASP
\subsection{ASP-based Algorithms}\label{app:proofs-asp}

The following sections are each comprised of a correctness proof corresponding to the ASP encoding for one of the six inconsistency measures considered in this work.

\subsubsection{The Contension Inconsistency Measure}\label{app:proof-asp-contension}

{
\renewcommand\tablename{Encoding}
\begin{small}
\setlength\extrarowheight{2pt}
\begin{longtable}{| >{\columncolor{gray!20}}L{.05\textwidth} >{\columncolor{gray!40}}L{.76\textwidth} >{\columncolor{gray!40}}R{.08\textwidth} |}
     \hline 
        \rowcolor{gray!20} \multicolumn{3}{|l|}{For every $\varFor \in \kb$: } \\
        & $\texttt{kbMember(} \varFor \texttt{).}$ & (AC1) \\
        \rowcolor{gray!20} \multicolumn{3}{|l|}{For every $\varAt \in \atoms(\kb)$: } \\
        & $\texttt{atom(} \varAt \texttt{).}$ & (AC2) \\
        \rowcolor{gray!20} \multicolumn{3}{|l|}{For every conjunction $\varSub_c = \varSubSub_{c,1} \land \varSubSub_{c,2}$ appearing in some formula: } \\
        & $\texttt{conjunction(} \varSub_c \texttt{,} \varSubSub_{c,1} \texttt{,} \varSubSub_{c,2} \texttt{).}$ & (AC3) \\
        \rowcolor{gray!20} \multicolumn{3}{|l|}{For every disjunction $\varSub_d = \varSubSub_{d,1} \lor \varSubSub_{d,2}$ appearing in some formula: } \\
        & $\texttt{disjunction(} \varSub_d \texttt{,} \varSubSub_{d,1} \texttt{,} \varSubSub_{d,2} \texttt{).}$ & (AC4) \\
        \rowcolor{gray!20} \multicolumn{3}{|l|}{For each negation $\varSub_n = \lnot \varSubSub_n$ appearing in some formula: } \\
        & $\texttt{negation(} \varSub_n \texttt{,} \varSubSub_n \texttt{).}$ & (AC5) \\
        \rowcolor{gray!20} \multicolumn{3}{|l|}{For every formula $\varSub_a$ which consists of a single atom $\varAt$: } \\
        & $\texttt{formulaIsAtom(} \varSub_a \texttt{,} \varAt \texttt{).}$ & (AC6) \\
        
        \hline
        \rowcolor{gray!20} \multicolumn{3}{|l|}{As the static part, we define:} \\
        & $\texttt{tv(} t \texttt{;} f \texttt{;} b \texttt{).}$ & (AC7) \\[1ex]
        & \texttt{1\{truthValue(A,T)\,:\,tv(T)\}1\,:-} &  \\
            & \qquad \qquad \texttt{atom(A).} & (AC8) \\[1ex]
        % conjunction rules:
        & \texttt{truthValue(F,}$t$\texttt{)\,:-} & \\
            & \qquad \qquad \texttt{conjunction(F,G,H),} & \\
            & \qquad \qquad \texttt{truthValue(G,}$t$\texttt{),} & \\
            & \qquad \qquad \texttt{truthValue(H,}$t$\texttt{).} & (AC9) \\[1ex]
        &\texttt{truthValue(F,}$f$\texttt{)\,:-} & \\
            & \qquad \qquad \texttt{conjunction(F,G,H),} & \\
            & \qquad \qquad \texttt{1\{truthValue(G,}$f$\texttt{), truthValue(H,}$f$\texttt{)\}.} & (AC10) \\[1ex]
        & \texttt{truthValue(F,}$b$\texttt{)\,:-} & \\
            & \qquad \qquad \texttt{conjunction(F,\_,\_),} & \\
            & \qquad \qquad \texttt{not truthValue(F,}$t$\texttt{),} & \\
            & \qquad \qquad \texttt{not truthValue(F,}$f$\texttt{).} & (AC11) \\[1ex]
         
        % disjunction rules:   
        & \texttt{truthValue(F,}$f$\texttt{)\,:-} & \\
            & \qquad \qquad \texttt{disjunction(F,G,H),} & \\
            & \qquad \qquad \texttt{truthValue(G,}$f$\texttt{),} & \\
            & \qquad \qquad \texttt{truthValue(H,}$f$\texttt{),} & (AC12) \\[1ex]
        &\texttt{truthValue(F,}$t$\texttt{)\,:-} & \\
            & \qquad \qquad \texttt{disjunction(F,G,H),} & \\
            & \qquad \qquad \texttt{1\{truthValue(G,}$t$\texttt{), truthValue(H,}$t$\texttt{)\}.} & (AC13) \\[1ex]
        & \texttt{truthValue(F,}$b$\texttt{)\,:-} & \\
            & \qquad \qquad \texttt{disjunction(F,\_,\_),} & \\
            & \qquad \qquad \texttt{not truthValue(F,}$t$\texttt{),} & \\
            & \qquad \qquad \texttt{not truthValue(F,}$f$\texttt{).} & (AC14) \\[1ex]
            
        % negation rules:
        & \texttt{truthValue(F,}$t$\texttt{)\,:-} & \\
            & \qquad \qquad \texttt{negation(F,G),} & \\
            & \qquad \qquad \texttt{truthValue(G,}$f$\texttt{).} & (AC15) \\[1ex]
        & \texttt{truthValue(F,}$f$\texttt{)\,:-} & \\
            & \qquad \qquad \texttt{negation(F,G),} & \\
            & \qquad \qquad \texttt{truthValue(G,}$t$\texttt{).} & (AC16) \\[1ex]
        & \texttt{truthValue(F,}$b$\texttt{)\,:-} & \\
            & \qquad \qquad \texttt{negation(F,G),} & \\
            & \qquad \qquad \texttt{truthValue(G,}$b$\texttt{).} & (AC17) \\[1ex]

        % formula consisting of a single atom:
        & \texttt{truthValue(F,T)\,:-} & \\
            & \qquad \qquad \texttt{formulaIsAtom(F,G),} & \\
            & \qquad \qquad \texttt{truthValue(G,T),} & \\
            & \qquad \qquad \texttt{tv(T).} & (AC18) \\[1ex]
        
        % integrity constraint:
        & \texttt{:-} & \\
            & \qquad \qquad \texttt{truthValue(F,}$f$\texttt{),} & \\
            & \qquad \qquad \texttt{kbMember(F).} & (AC19) \\[1ex]
            
        % minimize statement:
        &  \texttt{\#minimize\{1,A\,:\,truthValue(A,}$b$\texttt{), atom(A)\}.} & (AC20) \\
        \hline
    \caption{Overview of the ASP encoding $\pc(\kb)$ of the contension inconsistency measure $\icont$. $\kb$ is the given knowledge base.}
    \label{tab:asp-contension}
\end{longtable}
\end{small}
}

Consider an arbitrary knowledge base $\kb$ as well as the set of rules listed in Encoding \ref{tab:asp-contension} in Section \ref{sec:asp-algo-c}, which make up the extended logic program $\pc (\kb)$. 
Let $\pc'(\kb)$ denote the extended logic program $\pc(\kb)$ \emph{without} the minimize statement (AC20).
\begin{lemma}\label{lem:asp-c-1}
% \mt{better write ``...then the interpretation w defined as ... is a model of K''}
If $M$ is an answer set of $\pc'(\kb)$ then the three-valued interpretation $\omega^3_M$ defined as
\begin{align*}
\omega^3_M(\varAt) = \begin{cases} 
t & \quad \mathtt{truthValue(}\varAt , t \mathtt{)} \in M\\
f & \quad \mathtt{truthValue(}\varAt , f \mathtt{)} \in M\\
b & \quad \mathtt{truthValue(}\varAt , b \mathtt{)} \in M \end{cases}
\end{align*}
is a model of $\kb$ with $\varAt \in \mathsf{At}(\kb)$. 
% and $\emph{\texttt{atom(}}x\emph{\texttt{)}}\in M$
%  \mt{are the latter two conditions not equivalent (at least they should be)? Then drop the latter}
\end{lemma}

\begin{proof}
$\omega^3_M$ is well-defined: 
As (AC8) states, each answer set of $\pc'(\kb)$ will hold exactly one of $\texttt{truthValue(}\varAt \texttt{,} t \texttt{)}$, $\texttt{truthValue(}\varAt \texttt{,} f \texttt{)}$, $\texttt{truthValue(}\varAt \texttt{,} b \texttt{)}$ for each atom $\varAt \in \atoms(\kb)$:
\begin{align*}
    1\texttt{\{truthValue(}\varAt \texttt{,T) : tv(T)\}}1 \texttt{ :- atom(}\varAt\texttt{).}
\end{align*}
% \mt{the following observation has no need in a proof}
% An equivalent, but more explicit notation of the above rule is:
% \begin{align*}
%     1 \texttt{\{}&\texttt{truthValue(}x\texttt{,}t\texttt{),} \\
%     & \texttt{truthValue(}x\texttt{,}f\texttt{),} \\
%     & \texttt{truthValue(}x\texttt{,}b\texttt{)\}}1 \texttt{ :- atom(}x\texttt{).}
% \end{align*}

% \begin{align*}
% 	& e_{x_T} \leftarrow \mathtt{not}\, e_{x_B}, \mathtt{not}\, e_{x_F}.,\\
% 	& e_{x_B} \leftarrow \mathtt{not}\, e_{x_T}, \mathtt{not}\, e_{x_F}.,\\
% 	& e_{x_F} \leftarrow \mathtt{not}\, e_{x_B}, \mathtt{not}\, e_{x_T}.
% \end{align*}

Via structural induction, we show that for all \mbox{(sub)}formulas $\varSub$ in $\kb$, $\omega^3_M(\varSub) = \varTV $ (with $\varTV \in \{t,f,b\}$) if and only if $\texttt{truthValue(}\varSub \texttt{,} \varTV \texttt{)} \in M$.

\textbf{Induction Hypothesis (I.H.):}  $\omega^3_M(\varSub) = \varTV $ if and only if $ \texttt{truthValue(}\varSub \texttt{,} \varTV \texttt{)} \in M$ is true for all subformulas of the considered formula.
\begin{enumerate}
    \item Let $\varSub = \varAt$ for $\varAt\in \mathsf{At}(\kb)$.
        \begin{itemize}
            \item ``$\Rightarrow$'': 
                $$\omega^3_M(\varSub) = \varTV \Rightarrow \omega^3_M(\varAt) = \varTV \defArrow \texttt{truthValue(} \varAt \texttt{,} \varTV \texttt{)}\in M$$ 
            \item ``$\Leftarrow$'': 
                $$\texttt{truthValue(} \varAt \texttt{,} \varTV \texttt{)} \in M \defArrow \omega^3_M(\varAt) = \varTV \Rightarrow \omega^3_M(\varSub) = \varTV$$
        \end{itemize}
        
    \item Let $\varSub = \varSubSub_1 \land \varSubSub_2$. % \mt{a conjunction can have more than two conjuncts}
        \begin{enumerate}
            \item $\varTV = t$:
                \begin{itemize}
                    \item ``$\Rightarrow$'': 
                    \begin{align*}
                         \omega^3_M(\varSubSub_1 \land \varSubSub_2) = \varTV  
                      &  \xRightarrow{\varTV=t}{} 
                        \left\{
                        \begin{array}{l}
                             \omega^3_M(\varSubSub_1) = t \\
                             \omega^3_M(\varSubSub_2) = t
                        \end{array} \right\} \\
                      &  \xRightarrow{\text{I.H.}}{} 
                        \left\{\begin{array}{l}
                        \texttt{truthValue(} \varSubSub_1 \texttt{,} t \texttt{)} \in M\\
                        \texttt{truthValue(} \varSubSub_2 \texttt{,} t \texttt{)} \in M
                        \end{array}\right\} \\
                        & \xRightarrow{\text{(AC9)}}{} 
                        \texttt{truthValue(} \varSubSub_1 \land \varSubSub_2 \texttt{,} t \texttt{)} \in M  
                        % \\ \mt{\text{which rule is responsible?}}
                    \end{align*}
                    
                    \item ``$\Leftarrow$'':
                        \begin{align*}
                            \texttt{truthValue(} \varSubSub_1 \land \varSubSub_2 \texttt{,} t \texttt{)} \in M 
                          &  \Rightarrow 
                            \left\{\begin{array}{l}
                            \texttt{truthValue(} \varSubSub_1 \texttt{,} t \texttt{)} \in M\\
                            \texttt{truthValue(} \varSubSub_2 \texttt{,} t \texttt{)} \in M
                            \end{array}\right\} \\
                            & \Rightarrow 
                            \left\{
                            \begin{array}{l}
                                 \omega^3_M(\varSubSub_1) = t \\
                                 \omega^3_M(\varSubSub_2) = t
                            \end{array} 
                            \right\} \\
                            & \Rightarrow
                            \omega^3_M(\varSubSub_1 \land \varSubSub_2) = t
                        \end{align*}
                \end{itemize}  
                
            \item $\varTV = b$: 
                \begin{itemize}
                    \item ``$\Rightarrow$'': % \mt{add a bit more text such as ''let us first assume that w(psi1)=t and w(psi2)=b, ...'', also add rules responsible for derivations everywhere}
                    \begin{align*}
                         \omega^3_M(\varSubSub_1 \land \varSubSub_2) = \varTV
                        & \xRightarrow{\varTV=b}{}   \left\{ \begin{array}{l}
                         \omega^3_M(\varSubSub_1) = t \\
                         \omega^3_M(\varSubSub_2) = b
                        \end{array} \right\} \\
                        & \xRightarrow{\text{I.H.}}{} 
                        \left\{\begin{array}{l}
                        \texttt{truthValue(} \varSubSub_1 \texttt{,} t \texttt{)}  \in M\\
                        \texttt{truthValue(} \varSubSub_2 \texttt{,} b \texttt{)}  \in M
                        \end{array}\right\} \\
                        & \xRightarrow{\text{(AC11)}}{} 
                        \texttt{truthValue(} \varSubSub_1 \land \varSubSub_2 \texttt{,} b \texttt{)}  \in M
                        \end{align*}
                         Analogous for the case where $\omega^3_M(\varSubSub_1) = b$ and $\omega^3_M(\varSubSub_2) = t$. 
                         \begin{align*}
                              \text{or } & \xRightarrow{\varTV=b}{}   \left\{ \begin{array}{l}
                         \omega^3_M(\varSubSub_1) = b \\
                         \omega^3_M(\varSubSub_2) = b
                        \end{array} \right\} \\
                        & \xRightarrow{\text{I.H.}}{} 
                        \left\{\begin{array}{l} 
                                \texttt{truthValue(} \varSubSub_1 \texttt{,} b \texttt{)} \in M\\
                                \texttt{truthValue(} \varSubSub_2 \texttt{,} b \texttt{)} \in M
                                \end{array}\right\} \\
                                & \xRightarrow{\text{(AC11)}}{} 
                                \texttt{truthValue(} \varSubSub_1 \land \varSubSub_2 \texttt{,} b \texttt{)} \in M
                    \end{align*}
                    \item ``$\Leftarrow$'':
                    \begin{align*}
                        \texttt{truthValue(} \varSubSub_1 \land \varSubSub_2 \texttt{,} b \texttt{)} \in M 
                        & \Rightarrow  \left\{\begin{array}{l}
                            \texttt{truthValue(} \varSubSub_1 \texttt{,} t \texttt{)} \in M\\
                            \texttt{truthValue(} \varSubSub_2 \texttt{,} b \texttt{)} \in M
                            \end{array}\right\} \\
                            & \Rightarrow 
                            \left\{ \begin{array}{l}
                             \omega^3_M(\varSubSub_1) = t \\
                             \omega^3_M(\varSubSub_2) = b
                            \end{array} \right\} \\
                            & \Rightarrow 
                            \omega^3_M(\varSubSub_1 \land \varSubSub_2) = b
                    \end{align*}
                    Analogous for the case where $\omega^3_M(\varSubSub_1) = b$ and $\omega^3_M(\varSubSub_2) = t$.
                    \begin{align*}
                        \text{or } & \Rightarrow \left\{\begin{array}{l}
                            \texttt{truthValue(} \varSubSub_1 \texttt{,} b \texttt{)} \in M\\
                            \texttt{truthValue(} \varSubSub_2 \texttt{,} b \texttt{)} \in M
                            \end{array}\right\} \\
                            & \Rightarrow 
                            \left\{ \begin{array}{l}
                             \omega^3_M(\varSubSub_1) = b \\
                             \omega^3_M(\varSubSub_2) = b
                            \end{array} \right\}  \\
                            & \Rightarrow 
                            \omega^3_M(\varSubSub_1 \land \varSubSub_2) = b
                    \end{align*}
                \end{itemize}
            \item $\varTV = f$: 
                \begin{itemize}
                    \item ``$\Rightarrow$'':
                    \begin{align*}
                         \omega^3_M(\varSubSub_1 \land \varSubSub_2) = \varTV 
                        & \xRightarrow{\varTV=f}{}  \left\{ \begin{array}{l}
                         \omega^3_M(\varSubSub_1) = t \\
                         \omega^3_M(\varSubSub_2) = f
                        \end{array} \right\} \\
                        & \xRightarrow{\text{I.H.}}{} 
                        \left\{\begin{array}{l}
                            \texttt{truthValue(} \varSubSub_1 \texttt{,} t \texttt{)} \in M\\
                            \texttt{truthValue(} \varSubSub_f \texttt{,} f \texttt{)} \in M
                            \end{array}\right\} \\
                            & \xRightarrow{\text{(AC10)}}{} 
                            \texttt{truthValue(} \varSubSub_1 \land \varSubSub_2 \texttt{,} f \texttt{)} \in M
                    \end{align*}
                    Analogous for the case where $\omega^3_M(\varSubSub_1) = f$ and $\omega^3_M(\varSubSub_2) = t$.
                    \begin{align*}
                        \text{or } & \xRightarrow{\varTV=f}{} \left\{ \begin{array}{l}
                         \omega^3_M(\varSubSub_1) = b \\
                         \omega^3_M(\varSubSub_2) = f
                        \end{array} \right\} \\
                        & \xRightarrow{\text{I.H.}}{} 
                        \left\{\begin{array}{l}
                            \texttt{truthValue(} \varSubSub_1 \texttt{,} b \texttt{)} \in M\\
                            \texttt{truthValue(} \varSubSub_2 \texttt{,} f \texttt{)} \in M
                            \end{array}\right\} \\
                            & \xRightarrow{\text{(AC10)}}{} 
                            \texttt{truthValue(} \varSubSub_1 \land \varSubSub_2 \texttt{,} f \texttt{)} \in M
                    \end{align*}
                    Analogous for the case where $\omega^3_M(\varSubSub_1) = f$ and $\omega^3_M(\varSubSub_2) = b$.
                    \begin{align*}
                         \text{or } & \xRightarrow{\varTV=f}{}   \left\{ \begin{array}{l}
                             \omega^3_M(\varSubSub_1) = f \\
                             \omega^3_M(\varSubSub_2) = f
                        \end{array} \right\} \\
                        & \xRightarrow{\text{I.H.}}{} 
                        \left\{\begin{array}{l}
                            \texttt{truthValue(} \varSubSub_1 \texttt{,} f \texttt{)} \in M\\
                            \texttt{truthValue(} \varSubSub_2 \texttt{,} f \texttt{)} \in M
                            \end{array}\right\} \\
                            & \xRightarrow{\text{(AC10)}}{} 
                            \texttt{truthValue(} \varSubSub_1 \land \varSubSub_2 \texttt{,} f \texttt{)} \in M
                    \end{align*}
                    
                    \item ``$\Leftarrow$'':
                    \begin{align*}
                        \texttt{truthValue(} \varSubSub_1 \land \varSubSub_2 \texttt{,} f \texttt{)} \in M
                        & \Rightarrow \left\{\begin{array}{l}
                            \texttt{truthValue(} \varSubSub_1 \texttt{,} t \texttt{)} \in M\\
                            \texttt{truthValue(} \varSubSub_2 \texttt{,} f \texttt{)} \in M
                            \end{array}\right\} \\
                            & \Rightarrow  
                            \left\{ \begin{array}{l}
                             \omega^3_M(\varSubSub_1) = t \\
                             \omega^3_M(\varSubSub_2) = f
                            \end{array} \right\} \\
                            & \Rightarrow 
                            \omega^3_M(\varSubSub_1 \land \varSubSub_2) = f 
                    \end{align*}
                    Analogous for the case where $\omega^3_M(\varSubSub_1) = f$ and $\omega^3_M(\varSubSub_2) = t$.
                    \begin{align*}
                        \text{or } & \Rightarrow \left\{\begin{array}{l}
                            \texttt{truthValue(} \varSubSub_1 \texttt{,} b \texttt{)} \in M\\
                            \texttt{truthValue(} \varSubSub_2 \texttt{,} f \texttt{)} \in M
                            \end{array}\right\} \\
                            & \Rightarrow  
                            \left\{ \begin{array}{l}
                             \omega^3_M(\varSubSub_1) = b \\
                             \omega^3_M(\varSubSub_2) = f
                            \end{array} \right\} \\
                            & \Rightarrow 
                            \omega^3_M(\varSubSub_1 \land \varSubSub_2) = f 
                    \end{align*}
                    Analogous for the case where $\omega^3_M(\varSubSub_1) = f$ and $\omega^3_M(\varSubSub_2) = b$.
                    \begin{align*}
                         \text{or } & \Rightarrow \left\{\begin{array}{l}
                            \texttt{truthValue(} \varSubSub_1 \texttt{,} f \texttt{)} \in M\\
                            \texttt{truthValue(} \varSubSub_2 \texttt{,} f \texttt{)} \in M
                            \end{array}\right\} \\
                            & \Rightarrow 
                            \left\{ \begin{array}{l}
                             \omega^3_M(\varSubSub_1) = f \\
                             \omega^3_M(\varSubSub_2) = f
                             \end{array} \right\} \\
                            & \Rightarrow 
                            \omega^3_M(\varSubSub_1 \land \varSubSub_2) = f
                    \end{align*}
                \end{itemize}
        \end{enumerate}
        
    \item Let $\varSub = \varSubSub_1 \lor \varSubSub_2$.
        \begin{enumerate}
            \item $\varTV = t$:
                \begin{itemize}
                    \item ``$\Rightarrow$'': 
                    \begin{align*}
                        \omega^3_M(\varSubSub_1 \lor \varSubSub_2) = \varTV
                        & \xRightarrow{\varTV=t}{}   \left\{ \begin{array}{l}
                         \omega^3_M(\varSubSub_1) = t \\
                         \omega^3_M(\varSubSub_2) = f
                        \end{array} \right\} \\
                        & \xRightarrow{\text{I.H.}}{} 
                        \left\{\begin{array}{l}
                            \texttt{truthValue(} \varSubSub_1 \texttt{,} t \texttt{)} \in M\\
                            \texttt{truthValue(} \varSubSub_2 \texttt{,} f \texttt{)} \in M
                            \end{array}\right\} \\
                            & \xRightarrow{\text{(AC13)}}{} 
                            \texttt{truthValue(} \varSubSub_1 \lor \varSubSub_2 \texttt{,} t \texttt{)} \in M
                    \end{align*}
                    Analogous for the case where $\omega^3_M(\varSubSub_1) = f$ and $\omega^3_M(\varSubSub_2) = t$.
                    \begin{align*}
                        \text{or } & \xRightarrow{\varTV=t}{} \left\{ \begin{array}{l}
                         \omega^3_M(\varSubSub_1) = t \\
                         \omega^3_M(\varSubSub_2) = b
                        \end{array} \right\} \\
                        & \xRightarrow{\text{I.H.}}{}
                        \left\{\begin{array}{l}
                            \texttt{truthValue(} \varSubSub_1 \texttt{,} t \texttt{)} \in M\\
                            \texttt{truthValue(} \varSubSub_2 \texttt{,} b \texttt{)} \in M
                            \end{array}\right\} \\
                            & \xRightarrow{\text{(AC13)}}{} 
                            \texttt{truthValue(} \varSubSub_1 \lor \varSubSub_2 \texttt{,} t \texttt{)} \in M
                    \end{align*}
                    Analogous for the case where $\omega^3_M(\varSubSub_1) = b$ and $\omega^3_M(\varSubSub_2) = t$. 
                    \begin{align*}
                        \text{or } & \xRightarrow{\varTV=t}{} \left\{ \begin{array}{l}
                         \omega^3_M(\varSubSub_1) = t \\
                         \omega^3_M(\varSubSub_2) = t
                        \end{array} \right\} \\
                        & \xRightarrow{\text{I.H.}}{} 
                        \left\{\begin{array}{l}
                            \texttt{truthValue(} \varSubSub_1 \texttt{,} t \texttt{)} \in M\\
                            \texttt{truthValue(} \varSubSub_2 \texttt{,} t \texttt{)} \in M
                            \end{array}\right\} \\
                            & \xRightarrow{\text{(AC13)}}{} 
                            \texttt{truthValue(} \varSubSub_1 \lor \varSubSub_2 \texttt{,} t \texttt{)} \in M
                    \end{align*}
                    
                    \item ``$\Leftarrow$'':
                    \begin{align*}
                        \texttt{truthValue(} \varSubSub_1 \lor \varSubSub_2 \texttt{,} t \texttt{)} \in M 
                        & \Rightarrow  \left\{\begin{array}{l}
                        \texttt{truthValue(} \varSubSub_1 \texttt{,} t \texttt{)} \in M\\
                        \texttt{truthValue(} \varSubSub_2 \texttt{,} f \texttt{)} \in M
                        \end{array}\right\} \\
                        & \Rightarrow
                        \left\{ \begin{array}{l}
                         \omega^3_M(\varSubSub_1) = t \\
                         \omega^3_M(\varSubSub_2) = f
                        \end{array} \right\} \\
                        & \Rightarrow 
                        \omega^3_M(\varSubSub_1 \lor \varSubSub_2) = t
                    \end{align*}
                    Analogous for the case where $\omega^3_M(\varSubSub_1) = f$ and $\omega^3_M(\varSubSub_2) = t$.
                    \begin{align*}
                        \text{or } & \Rightarrow \left\{\begin{array}{l}
                        \texttt{truthValue(} \varSubSub_1 \texttt{,} t \texttt{)} \in M\\
                        \texttt{truthValue(} \varSubSub_2 \texttt{,} b \texttt{)} \in M
                        \end{array}\right\} \\
                        & \Rightarrow
                        \left\{ \begin{array}{l}
                         \omega^3_M(\varSubSub_1) = t \\
                         \omega^3_M(\varSubSub_2) = b
                        \end{array} \right\} \\
                        & \Rightarrow 
                        \omega^3_M(\varSubSub_1 \lor \varSubSub_2) = t
                    \end{align*}
                    Analogous for the case where $\omega^3_M(\varSubSub_1) = b$ and $\omega^3_M(\varSubSub_2) = t$.
                    \begin{align*}
                        \text{or } & \Rightarrow \left\{\begin{array}{l}
                        \texttt{truthValue(} \varSubSub_1 \texttt{,} t \texttt{)} \in M\\
                        \texttt{truthValue(} \varSubSub_2 \texttt{,} t \texttt{)} \in M
                        \end{array}\right\} \\
                        & \Rightarrow 
                        \left\{ \begin{array}{l}
                         \omega^3_M(\varSubSub_1) = t \\
                         \omega^3_M(\varSubSub_2) = t
                        \end{array} \right\} \\
                        & \Rightarrow 
                        \omega^3_M(\varSubSub_1 \lor \varSubSub_2) = t
                    \end{align*}
                \end{itemize}
            \item $\varTV = b$:
                \begin{itemize}
                    \item ``$\Rightarrow$'':
                    \begin{align*}
                        \omega^3_M(\varSubSub_1 \lor \varSubSub_2) = \varTV 
                        & \xRightarrow{\varTV=b}{}  \left\{ \begin{array}{l}
                         \omega^3_M(\varSubSub_1) = b \\
                         \omega^3_M(\varSubSub_2) = f
                        \end{array} \right\} \\
                        & \xRightarrow{\text{I.H.}}{}
                        \left\{\begin{array}{l}
                        \texttt{truthValue(} \varSubSub_1 \texttt{,} b \texttt{)} \in M\\
                        \texttt{truthValue(} \varSubSub_2 \texttt{,} f \texttt{)} \in M
                        \end{array}\right\} \\
                        & \xRightarrow{\text{(AC14)}}{} 
                        \texttt{truthValue(} \varSubSub_1 \lor \varSubSub_2 \texttt{,} b \texttt{)} \in M
                    \end{align*}
                    Analogous for the case where $\omega^3_M(\varSubSub_1) = f$ and $\omega^3_M(\varSubSub_2) = b$. 
                    \begin{align*}
                        \text{or } & \xRightarrow{\varTV=b}{} \left\{ \begin{array}{l}
                         \omega^3_M(\varSubSub_1) = b \\
                         \omega^3_M(\varSubSub_2) = b
                        \end{array} \right\} \\
                        & \xRightarrow{\text{I.H.}}{} 
                        \left\{\begin{array}{l}
                        \texttt{truthValue(} \varSubSub_1 \texttt{,} b \texttt{)} \in M\\
                        \texttt{truthValue(} \varSubSub_2 \texttt{,} b \texttt{)} \in M
                        \end{array}\right\} \\
                        & \xRightarrow{\text{(AC14)}}{} 
                        \texttt{truthValue(} \varSubSub_1 \lor \varSubSub_2 \texttt{,} b \texttt{)} \in M
                    \end{align*}
                    
                    \item ``$\Leftarrow$'': 
                    \begin{align*}
                        \texttt{truthValue(} \varSubSub_1 \lor \varSubSub_2 \texttt{,} b \texttt{)} \in M 
                        & \Rightarrow \left\{\begin{array}{l}
                        \texttt{truthValue(} \varSubSub_1 \texttt{,} b \texttt{)} \in M\\
                        \texttt{truthValue(} \varSubSub_2 \texttt{,} f \texttt{)} \in M
                        \end{array}\right\} \\
                        & \Rightarrow 
                        \left\{ \begin{array}{l}
                         \omega^3_M(\varSubSub_1) = b \\
                         \omega^3_M(\varSubSub_2) = f
                        \end{array} \right\} \\
                        & \Rightarrow 
                        \omega^3_M(\varSubSub_1 \lor \varSubSub_2) = b
                    \end{align*}
                    Analogous for the case where $\omega^3_M(\varSubSub_1) = f$ and $\omega^3_M(\varSubSub_2) = b$.
                    \begin{align*}
                        \text{or } & \Rightarrow \left\{\begin{array}{l}
                        \texttt{truthValue(} \varSubSub_1 \texttt{,} b \texttt{)} \in M\\
                        \texttt{truthValue(} \varSubSub_2 \texttt{,} b \texttt{)} \in M
                        \end{array}\right\} \\
                        & \Rightarrow 
                        \left\{ \begin{array}{l}
                         \omega^3_M(\varSubSub_1) = b \\
                         \omega^3_M(\varSubSub_2) = b
                        \end{array} \right\} \\
                        & \Rightarrow 
                        \omega^3_M(\varSubSub_1 \lor \varSubSub_2) = b
                    \end{align*}
                \end{itemize}
                
            \item $\varTV = f$:
                \begin{itemize}
                    \item ``$\Rightarrow$'': 
                    \begin{align*}
                        \omega^3_M(\varSubSub_1 \lor \varSubSub_2) = \varTV
                        & \xRightarrow{\varTV = f}{} 
                        \left\{
                        \begin{array}{l}
                             \omega^3_M(\varSubSub_1) = f \\
                             \omega^3_M(\varSubSub_2) = f
                        \end{array} \right\} \\
                        & \xRightarrow{\text{I.H.}}{} 
                        \left\{\begin{array}{l}
                        \texttt{truthValue(} \varSubSub_1 \texttt{,} f \texttt{)} \in M\\
                        \texttt{truthValue(} \varSubSub_2 \texttt{,} f \texttt{)} \in M
                        \end{array}\right\} \\
                        & \xRightarrow{\text{(AC12)}}{} 
                        \texttt{truthValue(} \varSubSub_1 \lor \varSubSub_2 \texttt{,} f \texttt{)} \in M
                    \end{align*}
                    
                    \item ``$\Leftarrow$'':
                    \begin{align*}
                        \texttt{truthValue(} \varSubSub_1 \lor \varSubSub_2 \texttt{,} f \texttt{)} \in M 
                        & \Rightarrow 
                        \left\{\begin{array}{l}
                        \texttt{truthValue(} \varSubSub_1 \texttt{,} f \texttt{)} \in M\\
                        \texttt{truthValue(} \varSubSub_2 \texttt{,} f \texttt{)} \in M
                        \end{array}\right\} \\
                        & \Rightarrow 
                        \left\{
                        \begin{array}{l}
                             \omega^3_M(\varSubSub_1) = f \\
                             \omega^3_M(\varSubSub_2) = f
                        \end{array}
                        \right\} \\
                        & \Rightarrow 
                        \omega^3_M(\varSubSub_1 \lor \varSubSub_2) = f
                    \end{align*}
                \end{itemize}
        \end{enumerate}
        
        \item Let $\varSub = \lnot \varSubSub$.
            \begin{enumerate}
                \item $\varTV = t$:
                    \begin{itemize}
                        \item ``$\Rightarrow$'':
                        \begin{align*}
                             \omega^3_M(\lnot \varSubSub) = \varTV
                            & \xRightarrow{\varTV=t}{}  \omega^3_M( \lnot \varSubSub) = t \\
                            & \Rightarrow   \omega^3_M(\varSubSub) = f  \\
                            & \xRightarrow{\text{I.H.}}{}  \texttt{truthValue(} \varSubSub \texttt{,} f \texttt{)} \in M  \\
                            & \xRightarrow{(\text{AC15})}  \texttt{truthValue(} \lnot \varSubSub \texttt{,} t \texttt{)} \in M
                        \end{align*}
                        
                        \item ``$\Leftarrow$'':
                        \begin{align*}
                             \texttt{truthValue(} \lnot \varSubSub \texttt{,} t \texttt{)} \in M 
                            & \Rightarrow  \texttt{truthValue(} \varSubSub \texttt{,} f \texttt{)} \in M \\
                            & \Rightarrow  \omega^3_M(\varSubSub) = f \\
                            & \Rightarrow  \omega^3_M(\lnot \varSubSub) = t
                        \end{align*}
                    \end{itemize}
                \item $\varTV = b$:
                    \begin{itemize}
                        \item ``$\Rightarrow$'': 
                        \begin{align*}
                             \omega^3_M(\lnot \varSubSub) = \varTV
                            & \xRightarrow{\varTV=b}{} \omega^3_M(\lnot \varSubSub) = b \\
                            & \Rightarrow  \omega^3_M( \varSubSub) = b \\
                            & \xRightarrow{\text{I.H.}}{}  \texttt{truthValue(} \varSubSub \texttt{,} b \texttt{)} \in M \\
                            & \xRightarrow{\text{(AC17)}}  \texttt{truthValue(} \lnot \varSubSub \texttt{,} b \texttt{)} \in M
                        \end{align*}
                        \item ``$\Leftarrow$'': 
                        \begin{align*}
                             \texttt{truthValue(} \lnot \varSubSub \texttt{,} b \texttt{)} \in M 
                            & \Rightarrow  \texttt{truthValue(} \varSubSub \texttt{,} b \texttt{)} \in M \\
                            & \Rightarrow  \omega^3_M( \varSubSub) = b \\
                            & \Rightarrow  \omega^3_M(\lnot \varSubSub) = b
                        \end{align*}
                    \end{itemize}
                \item $\varTV = f$:
                    \begin{itemize}
                        \item ``$\Rightarrow$'': 
                        \begin{align*}
                             \omega^3_M(\lnot \varSubSub) = \varTV 
                            & \xRightarrow{\varTV=f}{}  \Rightarrow  \omega^3_M(\lnot \varSubSub) = f \\
                            & \Rightarrow \omega^3_M(\varSubSub) = t \\
                            & \xRightarrow{\text{I.H.}}{}  \texttt{truthValue(}  \varSubSub \texttt{,} t \texttt{)} \in M \\
                            & \xRightarrow{\text{(AC16)}}  \texttt{truthValue(} \lnot \varSubSub \texttt{,} f \texttt{)} \in M
                        \end{align*}
                        \item ``$\Leftarrow$'':
                        \begin{align*}
                             \texttt{truthValue(} \lnot \varSubSub \texttt{,} f \texttt{)} \in M 
                            & \Rightarrow  \texttt{truthValue(}  \varSubSub \texttt{,} t \texttt{)} \in M \\
                            & \Rightarrow  \omega^3_M(\varSubSub) = t \\
                            & \Rightarrow  \omega^3_M(\lnot \varSubSub) = f
                        \end{align*}
                    \end{itemize}
            \end{enumerate}
\end{enumerate}
It has been shown that $\omega^3_M(\varSub) = \varTV$ if and only if $\texttt{truthValue(} \varSub \texttt{,} \varTV \texttt{)} \in M$.
The integrity constraint (AC23) ensures that $\texttt{truthValue(} \varFor \texttt{,} f \texttt{)}$ is never true wrt.\ each formula $\varFor \in \kb$.
In other words, $\texttt{truthValue(} \varFor \texttt{,} f \texttt{)} \not\in M$ for all $\varFor\in \kb$.
Consequently, $\omega^3_M(\varFor) \neq f$ for all $\varFor\in \kb$.
Thus, $\omega^3_M(\varFor)$ evaluates to either $t$ or $b$ in all cases.
This corresponds to the definition of a model in Priest's three-valued logic.
Hence, $\omega^3_M \in \mathsf{Models}(\kb)$.
\end{proof}

\begin{lemma}\label{lem:asp-c-2}
If $\omega^3\in \mathsf{Models}(\kb)$ then $\pc'(\kb)$ has an answer set $M$ such that $\omega^3_M = \omega^3$.
\end{lemma}

\begin{proof}
Analogous to the proof of Lemma~\ref{lem:asp-c-1} in reverse.
\end{proof}

{
\renewcommand{\thetheorem}{\ref{thm:asp-c}}
\begin{theorem}
    Let $M_o$ be an optimal answer set of $\pc(\kb)$. 
    Then $|(\omega^3_{M_o})^{-1}(b)|=\icont(\kb)$.
    \begin{proof}
        From Lemma \ref{lem:asp-c-1} and Lemma \ref{lem:asp-c-2} it follows that the answer sets of $\pc'(\kb)$ correspond exactly to the 3-valued models of $\kb$.
        The minimize statement included in $\pc(\kb)$ (AC20) then ensures that only a minimal number of instances $\texttt{truthValue(} \varAt \texttt{,} b \texttt{)}$ are included in the answer set. 
        This corresponds to the minimal number of atoms $\varAt \in \atoms(\kb)$ being evaluated to $b$ which is exactly the definition of $\icont(\kb)$.
        Note that, following (AC2), for all $\varAt\in \atoms(\kb)$, $\texttt{atom(}\varAt\texttt{)}$ must be included in any answer set of $\pc(\kb)$.
    \end{proof}
\end{theorem} 
\addtocounter{theorem}{-1}
}

\subsubsection{The Forgetting-Based Inconsistency Measure}\label{app:proof-asp-fb}

{
\renewcommand\tablename{Encoding}
\begin{small}
\setlength\extrarowheight{2pt}
\begin{longtable}{| >{\columncolor{gray!20}}L{.05\textwidth} >{\columncolor{gray!40}}L{.76\textwidth} >{\columncolor{gray!40}}R{.08\textwidth} |}
     \hline 
        \rowcolor{gray!20} \multicolumn{3}{|l|}{For every $\varFor \in \kb$: } \\
        & $\texttt{kbMember(} \varFor \texttt{).}$ & (AF1) \\
        \rowcolor{gray!20} \multicolumn{3}{|l|}{For every formula $\varSub_a$ which consists of a single atom occurrence $\varAt^l$: } \\
        & $\texttt{formulaIsAtomOcc(} \varSub \texttt{,} \varAt \texttt{,} l \texttt{).}$ & (AF2) \\
        
        \rowcolor{gray!20} \multicolumn{3}{|l|}{For every conjunction $\varSub_c = \varSubSub_{c,1} \land \varSubSub_{c,2}$ appearing in some formula: } \\
            & $\texttt{conjunction(} \varSub_c \texttt{,} \varSubSub_{c,1} \texttt{,} \varSubSub_{c,2} \texttt{).} $ & (AF3) \\
        \rowcolor{gray!20} \multicolumn{3}{|l|}{For every disjunction $\varSub_d = \varSubSub_{d,1} \lor \varSubSub_{d,2}$ appearing in some formula: } \\
            & $\texttt{disjunction(} \varSub_d \texttt{,} \varSubSub_{d,1} \texttt{,} \varSubSub_{d,2} \texttt{).}$ & (AF4) \\
        \rowcolor{gray!20} \multicolumn{3}{|l|}{For each negation $\varSub_n = \lnot \varSubSub_n$ appearing in some formula: } \\
        & $\texttt{negation(} \varSub_n \texttt{,} \varSubSub_n \texttt{).}$ & (AF5) \\
        
        \hline
        \rowcolor{gray!20} \multicolumn{3}{|l|}{For the static part, we define:} \\
        & $\texttt{tv(} t \texttt{;} f \texttt{).}$ & (AF6) \\[1ex]
        % & $\texttt{atv(} t \texttt{;} f \texttt{;} \forget_\top \texttt{;} \forget_\bot \texttt{).}$ & (AF7) \\[1ex]
        
        & \texttt{atomOcc(A,L)\,:-} & \\
            & \qquad \qquad \texttt{formulaIsAtomOcc(\_,A,L).} & (AF7) \\[1ex]
        & \texttt{atom(A)\,:-} & \\
            & \qquad \qquad \texttt{atomOcc(A,\_).} & (AF8) \\[1ex]
        
        % Unique atom evaluation:
        & \texttt{1\{truthValue(A,T)\,:\,tv(T)\}1\,:-} & \\
            & \qquad \qquad \texttt{atom(A).} & (AF9) \\[1ex]
        
        % conjunction rules:
        & \texttt{truthValue(F,}$t$\texttt{)\,:-} & \\
            & \qquad \qquad \texttt{conjunction(F,G,H),} & \\
            & \qquad \qquad \texttt{truthValue(G,}$t$\texttt{),} & \\
            & \qquad \qquad \texttt{truthValue(H,}$t$\texttt{).} & (AF10) \\[1ex]
        & \texttt{truthValue(F,}$f$\texttt{)\,:-} & \\
            & \qquad \qquad \texttt{conjunction(F,\_,\_),} & \\
            & \qquad \qquad \texttt{not truthValue(F,}$t$\texttt{).} & (AF11) \\[1ex]
         
        % disjunction rules:   
        & \texttt{truthValue(F,}$f$\texttt{)\,:-} & \\
            & \qquad \qquad \texttt{disjunction(F,G,H),} & \\
            & \qquad \qquad \texttt{truthValue(G,}$f$\texttt{),} & \\
            & \qquad \qquad \texttt{truthValue(H,}$f$\texttt{).} & (AF12) \\[1ex]
        & \texttt{truthValue(F,}$t$\texttt{)\,:-} & \\
            & \qquad \qquad \texttt{disjunction(F,\_,\_),} & \\
            & \qquad \qquad \texttt{not truthValue(F,}$f$\texttt{).} & (AF13) \\[1ex]
            
        % negation rules:
        & \texttt{truthValue(F,}$t$\texttt{)\,:-} & \\
            & \qquad \qquad \texttt{negation(F,G),} & \\
            & \qquad \qquad \texttt{truthValue(G,}$f$\texttt{).} & (AF14) \\[1ex]
        & \texttt{truthValue(F,}$f$\texttt{)\,:-} &  \\
            & \qquad \qquad \texttt{negation(F,G),} & \\
            & \qquad \qquad \texttt{truthValue(G,}$t$\texttt{).} & (AF15) \\[1ex]

        % guess whether an atom occurrence has been forgotten:
        & \texttt{\{atomOccForgotten(A,L)\}\,:-} & \\
            & \qquad \qquad \texttt{atomOcc(A,L).} & (AF16) \\[1ex]

        % handling forgetting:
        & \texttt{truthValue(F,T) :- } & \\
            & \qquad \qquad \texttt{formulaIsAtomOcc(F,A,L),} & \\
            & \qquad \qquad \texttt{truthValue(A,T),} & \\
            & \qquad \qquad \texttt{not atomOccForgotten(A,L).} & (AF17) \\[1ex]
        & \texttt{truthValue(F,}$t$\texttt{) :- } & \\
            & \qquad \qquad \texttt{formulaIsAtomOcc(F,A,L),} & \\
            & \qquad \qquad \texttt{truthValue(A,}$f$\texttt{),} & \\
            & \qquad \qquad \texttt{atomOccForgotten(A,L).} & (AF18) \\[1ex]
        & \texttt{truthValue(F,}$f$\texttt{) :- } & \\
            & \qquad \qquad \texttt{formulaIsAtomOcc(F,A,L),} & \\
            & \qquad \qquad \texttt{truthValue(A,}$t$\texttt{),} & \\
            & \qquad \qquad \texttt{atomOccForgotten(A,L).} & (AF19) \\[1ex]

        & \texttt{:-} & \\
            & \qquad \qquad \texttt{truthValue(F,}$f$\texttt{),} & \\
            & \qquad \qquad \texttt{kbMember(F).} & (AF20) \\[1ex]
            
        % minimization:
        % & \texttt{atomOccForgotten(A,L)\,:-} & \\
        %     & \qquad \qquad \texttt{atomTruthValue(A,L,}$\forget_\top$\texttt{).} & (AF26) \\[1ex]
        % & \texttt{atomOccForgotten(A,L)\,:-} & \\
        %     & \qquad \qquad \texttt{atomTruthValue(A,L,}$\forget_\bot$\texttt{).} & (AF27) \\[1ex]
        &  \texttt{\#minimize\{1,A,L\,:\,atomOccForgotten(A,L)\}}. & (AF21) \\
        \hline
    \caption{Overview of the ASP encoding $\pf(\kb)$ of the forgetting-based inconsistency measure $\iforget$.
    $\kb$ is the given knowledge base.}
    \label{tab:asp-fb}
\end{longtable}
\end{small}
}

Let $\kb$ be an arbitrary knowledge base and $\pf(\kb)$ the extended logic program consisting of rules (AF1)--(AF21) included Encoding \ref{tab:asp-fb}. % in Section \ref{sec:asp-algo-fb}.
Let $\pf'(\kb)$ denote the extended logic program $\pf(\kb)$ \emph{without} the minimize statement (AF21).

For an answer set $M$ we define:
\begin{align*}
    % T_M & = \{\varAt^l \mid \varAt \in \atoms(\kb), \texttt{atomTruthValue(} \varAt \texttt{,} l \texttt{,} \forget_\top \texttt{)} \in M\},\\
    F_M & = \{\varAt^l \mid \varAt \in \atoms(\kb), \texttt{atomOccForgotten(} \varAt \texttt{,} l \texttt{)} \in M\}
\end{align*}
% It is clear that for each answer set $M$ of $\pf'(\kb)$, it holds that $T_M\cap F_M=\emptyset$, because $\texttt{atomTruthValue(} \varAt \texttt{,} l \texttt{,} \forget_\top \texttt{)}$ cannot be included in an answer set if $\texttt{atomTruthValue(} \varAt \texttt{,} l \texttt{,} \forget_\bot \texttt{)}$ is, and vice versa, which is stated in rules (AF21) and (AF22).

\begin{lemma}\label{lem:asp-fb-1}
% If $M$ is an answer set of $\pf'(\kb)$ then $( \bigwedge \mathcal{K})[\varAt^{1}_1 \rightarrow \top;\ldots; \varAt^{n}_p \rightarrow \top ; Y^{1}_1 \rightarrow \bot;\ldots; Y^{m}_q \rightarrow \bot] $ is consistent with $T_M=\{\varAt^{1}_1,\ldots,\varAt^{n}_p\}$ and $F_M=\{Y^{1}_1,\ldots,Y^{m}_q\}$.
If $M$ is an answer set of $\pf'(\kb)$ then $( \bigwedge \kb)[\varAt^{1} \forgetting \top,\bot;\ldots; \varAt^{n} \forgetting \top,\bot] $ is consistent with $F_M=\{\varAt^{1},\ldots,\varAt^{n}\}$.
\end{lemma}

\begin{proof}
% As the rules (AF25)--(AF28) state, each answer set of $\pf'(\kb)$ will contain exactly one of \texttt{atomTruthValue(}$\varAt$\texttt{,}$l$\texttt{,}$t$\texttt{)}, \texttt{atomTruthValue(}$\varAt$\texttt{,}$l$\texttt{,}$f$\texttt{)}, \texttt{atomTruthValue(}$\varAt$\texttt{,}$l$\texttt{,}$\forget_\top$\texttt{)}, or \texttt{atomTruthValue(}$\varAt$\texttt{,}$l$\texttt{,}$\forget_\bot$\texttt{)} for each atom occurrence $\varAt^l$.
		
The rules given in (AF10)--(AF15) model classical entailment wrt.\ $\land$, $\lor$, and $\lnot$.
To be specific, the rules (AF10) and (AF11) model that the conjunction of two subformulas is only true if both conjuncts are true, and false if it is not true.
In the same fashion, a disjunction is only false if both of its individual disjuncts are false (modeled by (AF12)). 
It is true, if it is not false (modeled by (AF13)).
Classical negation, i.\,e., $\lnot \varSub$ is true if $\varSub$ is false, and vice versa, is represented by (AF14) and (AF15).

Furthermore, rule (AF9) represents that each atom $\varAt$ is evaluated to either $t$ or $f$.
In other words, since we represent that each atom $\varAt \in \atoms(\kb)$ is assigned a unique truth value, we automatically represent an \textit{interpretation} on $\atoms(\kb)$.
% The rules (AF21) and (AF22) allow us to replace the truth values of some atom occurrences which are given by said interpretation with $\top$ or $\bot$.

% Each atom occurrence $\varAt$ with label $l$ (i,\,e., $\varAt^l$) is represented as $\texttt{atomOcc(}X,l\texttt{)}$, due to rules (AF2) (which represents each subformula $\varSub$ that consists of an atom occurrence $\varAt^l$ as $\texttt{formulaIsAtomOcc(}\varSub,\varAt,l \texttt{)}$) and (AF7) (which extracts $\texttt{atomOcc(}\varAt,l\texttt{)}$ from $\texttt{formulaIsAtomOcc(}\varSub,\varAt,l \texttt{)}$).
Each atom occurrence $\varAt$ with label $l$ (i,\,e., $\varAt^l$) is represented as $\texttt{atomOcc(}X,l\texttt{)}$, due to (AF2) and (AF7).
Further, rule (AF16) models that each $\varAt^l$ could be forgotten by guessing a subset of atom occurrences by means of a cardinality constraint.
More precisely, (AF16) assigns a subset of all atom occurrences $\varAt^l$ as $\texttt{atomOccForgotten(}\varAt,l\texttt{)}$.
Thus, (AF16) determines $F_M$.
%
% We can also see that the atom occurrences $T_M = \{\varAt_1^{1}, \ldots, \varAt_p^{n}\}$ are exactly those that are replaced by $\top$, and the atom occurrences $F_M = \{Y_1^{1}, \ldots, Y_q^{m}\}$ are exactly those that are replaced by $\bot$.
% The replacement of atom occurrences directly influences the evaluation of formulas, as modeled in (AF18) and (AF20).
% To be specific, (AF18) models that an atom occurrence replaced by $\top$ is interpreted as the atom occurrence being evaluated to $t$.
% In the same manner, (AF20) models that an atom occurrence replaced by $\bot$ is interpreted as the atom occurrence being evaluated to $f$.
For each atom occurrence in this set, the truth value of the subformula containing the atom occurrence is reversed by (AF18) and (AF19), respectively.
Given an interpretation $\omega$, this results in setting an atom occurrence $\varAt^l$ with $\omega(\varAt)= f$ to $t$ (AF18), and vice versa (AF19).
If an atom occurrence $\varAt^l$ is not in $F_M$ (i.\,e., if an atom is not forgotten), (AF17) is triggered, and $\varAt^l$ must conform to $\omega$.

Finally, all formulas in the given knowledge base must be true. 
This is ensured by the integrity constraint (AF20) which forbids elements $\varFor \in \kb$ to be evaluated to $f$.
This also implies that each $\varAt^l \in F_M$ must be assigned the opposite truth value of $\varAt$, which is equivalent to replacing it with $\top$ or $\bot$, corresponding on the context.
Thus, since all formulas $\varFor \in \kb$ must be true, the version of $\kb$ modified by the forgetting operation, i.\,e., $( \bigwedge \kb)[\varAt^{1} \forgetting \top,\bot;\ldots; \varAt^{n} \forgetting \top,\bot] $, must be consistent.
% Thus, since all formulas $\varFor \in \kb$ must be true, $\kb$ modified by the forgetting operation, i.e., $( \bigwedge \mathcal{K})[\varAt^{1}_1 \rightarrow \top;\ldots; \varAt^{n}_p \rightarrow \top ; Y^{1}_1 \rightarrow \bot;\ldots; Y^{m}_q \rightarrow \bot] $, must be consistent.
\end{proof}

{
\renewcommand{\thetheorem}{\ref{thm:asp-fb}}
\begin{theorem}
    Let $M_o$ be an optimal answer set of $\pf(\kb)$.
    % Then $|T_{M_o}| + |F_{M_o}| = \iforget(\kb)$.
    Then $|F_{M_o}| = \iforget(\kb)$.
    \begin{proof}
        % From Lemma \ref{lem:asp-fb-1} it follows that the answer sets of $\pf'(\kb)$ correspond exactly to $( \bigwedge \mathcal{K})[\varAt^{1}_1 \rightarrow \top;\ldots; \varAt^{n}_p \rightarrow \top ; Y^{1}_1 \rightarrow \bot;\ldots; Y^{m}_q \rightarrow \bot] $ being consistent with $T_M=\{\varAt^{1}_1,\ldots,\varAt^{n}_p\}$ and $F_M=\{Y^{1}_1,\ldots,Y^{n}_q\}$.
        % The minimize statement (AF28) included in $\pf(\kb)$ then ensures that only a minimal number of instances of \texttt{atomTruthValue(}$\varAt$\texttt{,}$l$\texttt{,}$\forget_\top$\texttt{)} and \texttt{atomTruthValue(}$\varAt$\texttt{,}$l$\texttt{,}$\forget_\bot$\texttt{)} are included in the optimal answer set.
        % Overall, this corresponds to the minimal number of atom occurrences in $\kb$ being forgotten, which is exactly the definition of $\iforget(\kb)$.
        From Lemma \ref{lem:asp-fb-1} it follows that the answer sets of $\pf'(\kb)$ correspond exactly to  $( \bigwedge \kb)[\varAt^{1} \forgetting \top,\bot;\ldots; \varAt^{n} \forgetting \top,\bot] $ being consistent with $F_M=\{\varAt^{1},\ldots,\varAt^{n}\}$.
        The minimize statement (AF21) included in $\pf(\kb)$ then ensures that only a minimal number of instances of \texttt{atomOccForgotten(}$\varAt$\texttt{,}$l$\texttt{)} are included in the optimal answer set.
        Overall, this corresponds to the minimal number of atom occurrences in $\kb$ being forgotten, which is exactly the definition of $\iforget(\kb)$.
    \end{proof}
\end{theorem} 
\addtocounter{theorem}{-1}
}

\subsubsection{The Hitting Set Inconsistency Measure}\label{app:proof-asp-hs}

{
\renewcommand\tablename{Encoding}
\begin{small}
\setlength\extrarowheight{2pt}
\begin{longtable}{| >{\columncolor{gray!20}}L{.05\textwidth} >{\columncolor{gray!40}}L{.76\textwidth} >{\columncolor{gray!40}}R{.08\textwidth} |}
     \hline 
        \rowcolor{gray!20} \multicolumn{3}{|l|}{For every $\varFor \in \kb$: } \\
        & \texttt{kbMember(}$\varFor$\texttt{).} & (AH1) \\
        \rowcolor{gray!20} \multicolumn{3}{|l|}{For every $\varAt\in \atoms(\kb)$: } \\
        & \texttt{atom(}$\varAt$\texttt{).} & (AH2) \\
        \rowcolor{gray!20} \multicolumn{3}{|l|}{Define $|\kb|$ interpretations: } \\
        & \texttt{interpretation(1..}$|\kb|$\texttt{).} & (AH3) \\
        
        \rowcolor{gray!20} \multicolumn{3}{|l|}{For every conjunction $\varSub_c = \varSubSub_{c,1} \land \varSubSub_{c,2}$ appearing in some formula: } \\
            & $\texttt{conjunction(} \varSub_c \texttt{,} \varSubSub_{c,1} \texttt{,} \varSubSub_{c,2} \texttt{).} $ & (AH4) \\
        
        \rowcolor{gray!20} \multicolumn{3}{|l|}{For every disjunction $\varSub_d = \varSubSub_{d,1} \lor \varSubSub_{d,2}$ appearing in some formula: } \\
            & $\texttt{disjunction(} \varSub_d \texttt{,} \varSubSub_{d,1} \texttt{,} \varSubSub_{d,2} \texttt{).}$ & (AH5) \\
        
        \rowcolor{gray!20} \multicolumn{3}{|l|}{For each negation $\varSub_n = \lnot \varSubSub_n$ appearing in some formula: } \\
        & $\texttt{negation(} \varSub_n \texttt{,} \varSubSub_n \texttt{).}$ & (AH6) \\
        \rowcolor{gray!20} \multicolumn{3}{|l|}{For every formula $\varSub_a$ which consists of a single atom $\varAt$: } \\
        & $\texttt{formulaIsAtom(} \varSub_a \texttt{,} \varAt \texttt{).}$ & (AH7) \\
        
        \rowcolor{gray!20} \multicolumn{3}{|l|}{At least one, and at most $|\kb|$ interpretations must be included in the hitting set:} \\
        & \texttt{1\{interpretationActive(I)\,:\,interpretation(I)\}}$|\kb|$\texttt{.} & (AH8) \\[1ex]
        
        \hline
        \rowcolor{gray!20} \multicolumn{3}{|l|}{For the static part, we define:} \\
        
        & $\texttt{tv(} t \texttt{;} f \texttt{).}$ & (AH9) \\[1ex]
        
        % Unique atom evaluation:
        & \texttt{1\{truthValueInt(A,I,T)\,:\,tv(T)\}1\,:-} & \\
            & \qquad \qquad \texttt{atom(A),} & \\
            & \qquad \qquad \texttt{interpretation(I).} & (AH10) \\[1ex]
            
        % Conjunction rules:
        & \texttt{truthValueInt(F,I,}$t$\texttt{) :- } & \\
            & \qquad \qquad \texttt{conjunction(F,G,H),} & \\
            & \qquad \qquad \texttt{interpretation(I),} & \\
            & \qquad \qquad \texttt{truthValueInt(G,I,}$t$\texttt{),} & \\
            & \qquad \qquad \texttt{truthValueInt(H,I,}$t$\texttt{).} & (AH11) \\[1ex]
        & \texttt{truthValueInt(F,I,}$f$\texttt{) :- } & \\
            & \qquad \qquad \texttt{conjunction(F,\_,\_),} & \\
            & \qquad \qquad \texttt{interpretation(I),} & \\
            & \qquad \qquad \texttt{not truthValueInt(F,I,}$t$\texttt{).} & (AH12) \\[1ex]
            
        % Disjunction rules:
        & \texttt{truthValueInt(F,I,}$f$\texttt{) :- } & \\
            & \qquad \qquad \texttt{disjunction(F,G,H),} & \\
            & \qquad \qquad \texttt{interpretation(I),} & \\
            & \qquad \qquad \texttt{truthValueInt(G,I,}$f$\texttt{),} & \\
            & \qquad \qquad \texttt{truthValueInt(H,I,}$f$\texttt{).} & (AH13) \\[1ex]
        & \texttt{truthValueInt(F,I,}$t$\texttt{) :- } & \\
            & \qquad \qquad \texttt{disjunction(F,\_,\_),} & \\
            & \qquad \qquad \texttt{interpretation(I),} & \\
            & \qquad \qquad \texttt{not truthValueInt(F,I,}$f$\texttt{).} & (AH14) \\[1ex]
            
        % Negation rules:
        & \texttt{truthValueInt(F,I,}$t$\texttt{)\,:-} & \\
            & \qquad \qquad \texttt{negation(F,G),} & \\
            & \qquad \qquad \texttt{truthValueInt(G,I,}$f$\texttt{).} & (AH15) \\[1ex]
        & \texttt{truthValueInt(F,I,}$f$\texttt{)\,:-} & \\
            & \qquad \qquad \texttt{negation(F,G),} & \\
            & \qquad \qquad \texttt{truthValueInt(G,I,}$t$\texttt{).} & (AH16) \\[1ex]
            
        % if a formula consists of a single atom:
        & \texttt{truthValueInt(F,I,T)\,:-} & \\
            & \qquad \qquad \texttt{formulaIsAtom(F,G),} &\\ 
            & \qquad \qquad \texttt{truthValueInt(G,I,T),} & \\
            & \qquad \qquad \texttt{interpretation(I),} & \\
            & \qquad \qquad \texttt{tv(T).} & (AH17) \\[1ex]
            
        % Truth values of KB elements:
        & \texttt{truthValue(F,}$t$\texttt{)\,:-} & \\
            & \qquad \qquad \texttt{truthValueInt(F,I,}$t$\texttt{),} & \\
            & \qquad \qquad \texttt{kbMember(F),} & \\ 
            & \qquad \qquad \texttt{interpretation(I),} & \\ 
            & \qquad \qquad \texttt{interpretationActive(I).} & (AH18) \\[1ex]
        & \texttt{truthValue(F,}$f$\texttt{)\,:-} & \\
            & \qquad \qquad \texttt{kbMember(F),} & \\ 
            & \qquad \qquad \texttt{not truthValue(F,}$t$\texttt{).} & (AH19) \\[1ex]

        % Integrity constraint to avoid symmetries:
        & \texttt{:- } & \\
        & \qquad \qquad \texttt{interpretationActive(I), } & \\
        & \qquad \qquad \texttt{1 < I,} & \\
        & \qquad \qquad \texttt{not interpretationActive(I-1).} & (AH20) \\[1ex]
            
        % Integrity constraint:
        & \texttt{:-} & \\
        & \qquad \qquad \texttt{truthValue(F,}$f$\texttt{),} & \\
        & \qquad \qquad \texttt{kbMember(F).} & (AH21) \\[1ex]

        % minimize statement:
        &  \texttt{\#minimize\{1,I\,:\,interpretationActive(I)\}}. & (AH22) \\
        \hline
    \caption{Overview of the ASP encoding $\ph(\kb)$ of the hitting set inconsistency measure $\ihs$.
    $\kb$ is the given knowledge base.}
    \label{tab:asp-hs}
\end{longtable}
\end{small}
}

Let $\kb$ be an arbitrary knowledge base and $\ph(\kb)$ the extended logic program consisting of rules (AH1)--(AH22) listed in Encoding \ref{tab:asp-hs} in Section \ref{sec:asp-algo-hs}.

\begin{lemma}\label{lem:asp-hs-1}
    $\ph(\kb)$ does not have an answer set if and only if $\kb$ contains one or more contradictory formulas. 
    \begin{proof}
        Let some $\varFor^\bot \in \kb$ be contradictory.
        To begin with, the formulas $\varFor \in \kb$ and the atoms $\varAt \in \atoms(\kb)$ are represented in ASP via (AH1) and (AH2).
        Moreover, the operators of each (sub-)formula in $\kb$ is encoded using (AH4)--(AH7).
        Thus, $\ph'(\kb)$ contains an ASP representation of $\kb$ itself.
        Additionally, (AH3) adds $|\kb|$ interpretation constants to $\ph'(\kb)$.
        By means of (AH10), we assign a distinct truth value (defined by (AH9)) to each \texttt{atom/1} for each \texttt{interpretation/1}.
        (AH11)--(AH16) model classical entailment (analogous to (AF11)--(AF16), but with an additional reference to an interpretation).
        Besides, (AH17) models that a formula consisting of an individual atom gets the same truth value as the atom.
        Since every atom representation in ASP is connected to a truth value, and the evaluation of the (sub-)formulas is accomplished via (AH11)--(AH17), every formula  $\varFor\in \kb$ (represented by \texttt{kbMember/1}) evaluates to either $t$ or $f$.
        Since we cannot construct an interpretation which satisfies $\varFor^\bot$, it would always evaluate to $f$.
        However, the integrity constraint (AH21) forbids this, therefore, no answer set can be constructed.
    \end{proof}
\end{lemma}

Let $\ph'(\kb)$ denote the extended logic program $\ph(\kb)$ \emph{without} the minimize statement (AH22).
Let $M$ be an answer set of $\ph'(\kb)$.
Further, let $\omega^\mathrm{asp}_i$, $i\in \{1, \ldots, |\kb|\}$, denote the ASP constants representing interpretations, and let $\omega_i$ denote the corresponding interpretations in classical logic.
We define the set of interpretations $\omega_i$ represented in $M$ as 
$$ \Omega(M) = \{ \omega_i \mid \texttt{interpretationActive(}\omega^{\mathrm{asp}}_i\texttt{)} \in M \}. $$

\begin{lemma}\label{lem:asp-hs-2}
    If $M$ is an answer set of $\ph'(\kb)$ then $\Omega(M)$ is a hitting set of $\kb$.
    \begin{proof}
        $\Omega(M)$ is a hitting set if and only if every formula $\varFor \in \kb$ is satisfied by at least one $\omega_i \in \Omega(M)$. 
        We will first show that for any $\omega_i \in \Omega(M)$ and $\varFor \in \kb$, $\omega_i \models \varFor$ if and only if $\texttt{interpretationActive(}\omega^\mathrm{asp}_i \texttt{)} \in M$.
        
        % The cardinality constraint (AH8) ensures that at least one instance $\texttt{interpretationActive(}\omega^\mathrm{asp}_i \texttt{)}$ 
        % must be included in the answer set. 
        The cardinality constraint (AH8) ensures that the answer set includes at least one instance of $\texttt{interpretationActive(}\omega^\mathrm{asp}_i \texttt{)}$.
        According to (AH8), \texttt{interpretation(}$\omega_i$\texttt{)} must be included in $M$ as well, however, this is always the case due to (AH3).
        Moreover, \texttt{interpretationActive(}$\omega^\mathrm{asp}_i$\texttt{)} must be included in $M$ in order for (AH19) to be satisfied.
        Further, for each $\varFor \in \kb$ there must be at least one interpretation such that $\texttt{truthValue(}\varFor\texttt{,}t\texttt{)}\in M$, because otherwise, $\texttt{truthValue(}\varFor\texttt{,}f\texttt{)}\in M$ (due to (AH19)), and this is forbidden by the integrity constraint (AH21).
        
        As mentioned before, rules (AH11)--(AH17) represent entailment in classical propositional logic wrt.\ a given interpretation.
        Finally, an atom $\varAt \in \kb$ can only be either true or false, but not both, under a given interpretation, as stated in (AH14).
        
        Consequently, for each formula $\varFor \in \kb$ there exists at least one interpretation $\omega_i$ such that $\texttt{truthValueInt(}\varFor\texttt{,}\omega^\mathrm{asp}_i\texttt{,}t\texttt{)}\in M$.
        It follows that for each $\varFor \in \kb$, there is also at least one $\texttt{interpretationActive(}\omega^\mathrm{asp}_i\texttt{)}\in M$.
        Hence, for each formula $\varFor \in \kb$ there exists an $\omega_i \in \Omega(M)$ such that $\omega_i \models \varFor$.
        Thus, $\Omega(M)$ is a hitting set of $\kb$. 
    \end{proof}
\end{lemma}

{
\renewcommand{\thetheorem}{\ref{thm:asp-hs}}
\begin{theorem}
    Let $M_o$ be an optimal answer set of $\ph(\kb)$.
    Then $|\Omega(M_o)| -1 = \ihs(\kb)$.
    If no answer set of $\ph(\kb)$ exists, $\ihs(\kb) = \infty$.
    \begin{proof}
        Lemma \ref{lem:asp-hs-1} shows that $\ph(\kb)$ does not have an answer set if $\kb$ contains at least one contradictory formula, i.\,e., if $\ihs(\kb) = \infty$. 
        
        Further, in the above proof of Lemma \ref{lem:asp-hs-2} we showed that $\Omega(M)$ is a hitting set of $\kb$ wrt.\ some answer set $M$. 
        The cardinality of $\Omega(M)$ is exactly the number of instances $\texttt{interpretationActive(}\omega^\mathrm{asp}_i\texttt{)}$, with $i\in \{1, \ldots, |\kb|\}$, that are included in $M$.
        The minimize statement (AH22) selects the minimal number of such instances.
        This represents the selection of the smallest possible hitting set wrt.\ $\kb$.
        At last, subtracting $1$ from $|\Omega(M_o)|$ corresponds exactly to the definition of $\ihs(\kb)$.
    \end{proof}
\end{theorem} 
\addtocounter{theorem}{-1}
}

\subsubsection{The Max-Distance Inconsistency Measure}\label{app:proof-asp-max-dist}

{
\renewcommand\tablename{Encoding}
\begin{small}
\setlength\extrarowheight{2pt}
\begin{longtable}{| >{\columncolor{gray!20}}L{.05\textwidth}  >{\columncolor{gray!40}}L{.74\textwidth} >{\columncolor{gray!40}}R{.1\textwidth} |}
     \hline 
        \rowcolor{gray!20} \multicolumn{3}{|l|}{For every $\varFor_i \in \kb$ with $i\in \{0,\ldots,|\kb|-1\}$: } \\
        & \texttt{kbMember(}$\varFor$\texttt{,}$i$\texttt{).} & (ADM1) \\
        \rowcolor{gray!20} \multicolumn{3}{|l|}{For every $\varAt \in \atoms(\kb)$: } \\
        & \texttt{atom(}$\varAt$\texttt{).} & (ADM2) \\
        \rowcolor{gray!20} \multicolumn{3}{|l|}{Define $|\kb|+1$ interpretations: } \\
        & \texttt{interpretation(0..}$|\kb|$\texttt{).} & (ADM3) \\
        
        \rowcolor{gray!20} \multicolumn{3}{|l|}{For every conjunction $\varSub_c = \varSubSub_{c,1} \land \varSubSub_{c,2}$ appearing in some formula: } \\
            & $\texttt{conjunction(} \varSub_c \texttt{,} \varSubSub_{c,1} \texttt{,} \varSubSub_{c,2} \texttt{).} $ & (ADM4) \\
        
        \rowcolor{gray!20} \multicolumn{3}{|l|}{For every disjunction $\varSub_d = \varSubSub_{d,1} \lor \varSubSub_{d,2}$ appearing in some formula: } \\
            & $\texttt{disjunction(} \varSub_d \texttt{,} \varSubSub_{d,1} \texttt{,} \varSubSub_{d,2} \texttt{).}$ & (ADM5) \\
        
        \rowcolor{gray!20} \multicolumn{3}{|l|}{For each negation $\varSub_n = \lnot \varSubSub_n$ appearing in some formula: } \\
        & $\texttt{negation(} \varSub_n \texttt{,} \varSubSub_n \texttt{).}$ & (ADM6) \\
        \rowcolor{gray!20} \multicolumn{3}{|l|}{For every formula $\varSub_a$ which consists of a single atom $\varAt$: } \\
        & $\texttt{formulaIsAtom(} \varSub_a \texttt{,} \varAt \texttt{).}$ & (ADM7) \\
        
        % \rowcolor{gray!20} \multicolumn{3}{|l|}{Maximum distance between interpretation $|\kb|$ and the models of $\kb$: } \\
        % & \texttt{dMax(X)\,:-} & \\
        %     & \qquad \qquad \texttt{X = \#max\{Y\,:\,d(I,}$|\kb|$\texttt{,Y), interpretation(I)\},} & \\
        %     & \qquad \qquad \texttt{X >= 0.} & (ADM8) \\

        \rowcolor{gray!20} \multicolumn{3}{|l|}{Distance between interpretation $|\kb|$ and the models of each formula: } \\
        & \texttt{d(X) :- } & \\
        & \qquad \qquad \texttt{diff(X),} & \\
        & \qquad \qquad \texttt{interpretation(I),} & \\
        & \qquad \qquad \texttt{X <= \#count\{A: atom(A), } & \\
        & \qquad \qquad \qquad \qquad \texttt{truthValueInt(A,I,T),} & \\
        & \qquad \qquad \qquad \qquad \texttt{not truthValueInt(A,}$|\kb|$\texttt{,T)\}.} & (ADM8) \\
        
        \hline
        \rowcolor{gray!20} \multicolumn{3}{|l|}{For the static part, we define:} \\
        & $\texttt{tv(} t \texttt{;} f \texttt{).}$ & (ADM9) \\[1ex]
        
        % Unique atom evaluation:
        & \texttt{1\{truthValueInt(A,I,T)\,:\,tv(T)\}1\,:-} & \\
            & \qquad \qquad \texttt{atom(A),} & \\
            & \qquad \qquad \texttt{interpretation(I).} & (ADM10) \\[1ex]
            
        % Conjunction rules:
        & \texttt{truthValueInt(F,I,}$t$\texttt{) :- } & \\
            & \qquad \qquad \texttt{conjunction(F,G,H),} & \\
            & \qquad \qquad \texttt{interpretation(I),} & \\
            & \qquad \qquad \texttt{truthValueInt(G,I,}$t$\texttt{),} & \\
            & \qquad \qquad \texttt{truthValueInt(H,I,}$t$\texttt{).} & (ADM11) \\[1ex]
        & \texttt{truthValueInt(F,I,}$f$\texttt{) :- } & \\
            & \qquad \qquad \texttt{conjunction(F,\_,\_),} & \\
            & \qquad \qquad \texttt{interpretation(I),} & \\
            & \qquad \qquad \texttt{not truthValueInt(F,I,}$t$\texttt{).} & (ADM12) \\[1ex]
            
        % Disjunction rules:
        & \texttt{truthValueInt(F,I,}$f$\texttt{) :- } & \\
            & \qquad \qquad \texttt{disjunction(F,G,H),} & \\
            & \qquad \qquad \texttt{interpretation(I),} & \\
            & \qquad \qquad \texttt{truthValueInt(G,I,}$f$\texttt{),} & \\
            & \qquad \qquad \texttt{truthValueInt(H,I,}$f$\texttt{).} & (ADM13) \\[1ex]
        & \texttt{truthValueInt(F,I,}$t$\texttt{) :- } & \\
            & \qquad \qquad \texttt{disjunction(F,\_,\_),} & \\
            & \qquad \qquad \texttt{interpretation(I),} & \\
            & \qquad \qquad \texttt{not truthValueInt(F,I,}$f$\texttt{).} & (ADM14) \\[1ex]
            
        % Negation rules:
        & \texttt{truthValueInt(F,I,}$t$\texttt{)\,:-} & \\
            & \qquad \qquad \texttt{negation(F,G),} & \\
            & \qquad \qquad \texttt{truthValueInt(G,I,}$f$\texttt{).} & (ADM15) \\[1ex]
        & \texttt{truthValueInt(F,I,}$f$\texttt{)\,:-} & \\
            & \qquad \qquad \texttt{negation(F,G),} & \\
            & \qquad \qquad \texttt{truthValueInt(G,I,}$t$\texttt{).} & (ADM16) \\[1ex]
            
        % if a formula consists of a single atom:
        & \texttt{truthValueInt(F,I,T)\,:-} & \\
            & \qquad \qquad \texttt{formulaIsAtom(F,G),} & \\ 
            & \qquad \qquad \texttt{truthValueInt(G,I,T),} & \\
            & \qquad \qquad \texttt{interpretation(I),} & \\
            & \qquad \qquad \texttt{tv(T).} & (ADM17) \\[1ex]
            
        % Truth values of KB elements:
        % & \texttt{truthValueInt(F,L,I,T)\,:-} & \\
        %     & \qquad \qquad \texttt{kbMember(F,L),} & \\
        %     & \qquad \qquad \texttt{interpretation(I),} & \\
        %     & \qquad \qquad \texttt{tv(T),} & \\
        %     & \qquad \qquad \texttt{truthValueInt(F,I,T).} & (ADM18) \\[1ex]
            
        % Integrity constraint:
        & \texttt{:-} & \\
        % & \qquad \qquad \texttt{truthValueInt(F,L,I,}$f$\texttt{),} & \\
        % & \qquad \qquad \texttt{kbMember(F,L),} & \\
        % & \qquad \qquad \texttt{interpretation(I),} &  \\
        % & \qquad \qquad \texttt{L == I.} & (ADM19) \\[1ex]
        & \qquad \qquad \texttt{truthValueInt(F,I,}$f$\texttt{),} & \\
        & \qquad \qquad \texttt{kbMember(F,I).} &  (ADM18) \\[1ex]

        % distance:
        & \texttt{diff(1..X) :-} & \\
        & \qquad \qquad \texttt{X = \#count\{A: atom(A)\}.} & (ADM19) \\[1ex]

        % minimize statement:
        % & \texttt{\#minimize\{1,X : d(X)\}.} & (ADM20) \\
        & \texttt{\#minimize\{1,X : d(X)\}.} & (ADM20) \\
        
        % % Dalal distance:
        % & \texttt{diff(A,I,J)\,:-} & \\
        %     & \qquad \qquad \texttt{atom(A),} &  \\
        %     & \qquad \qquad \texttt{interpretation(I),} & \\
        %     & \qquad \qquad \texttt{interpretation(J),} & \\
        %     & \qquad \qquad \texttt{truthValueInt(A,I,T),} & \\
        %     & \qquad \qquad \texttt{truthValueInt(A,J,U),} & \\
        %     & \qquad \qquad \texttt{T != U.} & (ADM19) \\[1ex]
        % & \texttt{d(I,J,X)\,:-} & \\
        %     & \qquad \qquad \texttt{interpretation(I),} & \\
        %     & \qquad \qquad \texttt{interpretation(J),} & \\
        %     & \qquad \qquad \texttt{X = \#count\{A\,:\,diff(A,I,J), atom(A)\}.} & (ADM20) \\[1ex]

        % % minimize statement:
        % &  \texttt{\#minimize\{X\,:\,dMax(X)\}}.  & (ADM21) \\
        \hline
    \caption{Overview of the ASP encoding $\pmax(\kb)$ of the max-distance inconsistency measure $\imdalal$.
    $\kb$ is the given knowledge base.}
    \label{tab:asp-dmax}
\end{longtable}
\end{small}
}

Let $\kb$ be an arbitrary knowledge base and $\pmax(\kb)$ the extended logic program consisting of rules (ADM1)--(ADM22) listed in Encoding \ref{tab:asp-dmax} in Section \ref{sec:asp-algo-dmax}.

\begin{lemma}\label{lem:asp-max-dist-1}
    $\pmax(\kb)$ does not have an answer set if and only if $\kb$ contains one or more contradictory formulas. 
    \begin{proof}
        The knowledge base, i.\,e., its formulas, subformulas, and atoms are modelled by (ADM1)--(ADM2) and (ADM4)--(ADM7).
        Note that each $\varFor \in  \kb$ gets a label $i$, and is represented as $\texttt{kbMember(}\varFor, i \texttt{)}$ (ADM1).
        The two truth values are represented by (ADM9).
        Moreover, via (ADM3) we represent $|\kb|+1$ interpretations in $\pmax(\kb)$.
        We denote these interpretations as follows (note that we represent $\omega_i$ simply as $i$ in ASP):
        \begin{align*}
            \omega_i(\varAt) = \begin{cases} 
            t & \quad \texttt{truthValue(}\varAt, i, t \texttt{)} \in M \\
            f & \quad \texttt{truthValue(}\varAt, i, f \texttt{)} \in M
            \end{cases}
        \end{align*}

        Analogously to (AH10) (wrt.\ $\ihs$), (ADM10) guesses all $|\kb|+1$ interpretations by assigning each atom in each interpretation a distinct truth value.
        (ADM11)--(ADM17) correspond exactly to (AH11)--(AH17) (wrt.\ $\ihs$).
        Thus, these rules model classical entailment.
        % In addition, (ADM18) handles the evaluation of formulas $\varFor_i \in \kb$ as a special case, since the label $i$ needs to be taken into account.
        % % To be precise, $\texttt{truthValueInt(}\varSub, \omega_i, \theta \texttt{)}$ (with $i \in \{0, \ldots, |\kb|\}$ %, and $j \in \{0, \ldots, |\kb|\}$) 
        % is contained in an answer set if the corresponding formula $\varSub$ evaluates to $\theta \in \{t,f\}$, i.\,e., if $\texttt{kbMember(}\varSub, i \texttt{)}$, $\texttt{interpretation(}\omega_j \texttt{)}$, and $\texttt{tv(}\theta \texttt{)}$, are contained in the answer set.

        Each interpretation $\omega_i$ is supposed to be a model for formula $\varFor_i \in \kb$, $i\in \{0, \ldots, |\kb|-1\}$.
        To this end, the integrity constraint (ADM18) ensures that no formula $\varFor_i$ evaluates to $f$ under $\omega_i$.
        If a contradictory formula $\varFor^\bot_i$ is included in $\kb$, there exists no interpretation $\omega_i$ that would satisfy $\varFor^\bot_i$.
        Hence, the answer set would need to include $\texttt{truthValueInt(}\varFor^\bot, i, f \texttt{)}$.
        However, (ADM18) does not allow this, so no answer set can be derived.
    \end{proof}
\end{lemma}

Let $\pmax{'}(\kb)$ denote the extended logic program $\pmax(\kb)$ \emph{without} the minimize statement (ADM20).
Let $D^{\max}_M = \{n \mid \texttt{d(}n\texttt{)} \in M, n\in \mathbb{N}\}$ wrt.\ an answer set $M$.

\begin{lemma}\label{lem:asp-max-dist-2}
    Let $M$ be an answer set of $\pmax{'}(\kb)$.
    % If $\texttt{d(}\omega_i, \omega_j, D_M \texttt{)} \in M$, then $D_M$ corresponds to the Dalal distance between $\omega_i$ and $\omega_j$ (i.\,e., $D_M = d_{dalal}(\omega_i, \omega_j)$), with $i,j \in \{0\, \ldots, |\kb|\}$.
    % TODO:
    Each $n \in D^{\max}_M$ corresponds to a value smaller than or equal to the distance between some interpretation $\omega_{|\kb|}$ and one of the models of the formulas in $\kb$, i.\,e., $\omega_i$ with $i \in \{0, \ldots, |\kb|-1\}$.
    \begin{proof}
        If an atom $\varAt \in \atoms(\kb)$ evaluates to a truth value $\theta \in \{t,f\}$ under an interpretation $\omega_i$, and not to truth value $\theta$ under interpretation $\omega_{|\kb|}$, with $i \in \{0, \ldots, |\kb|-1\}$, %(and $\omega_i, \omega_j$ as defined in the Proof of Lemma \ref{lem:asp-max-dist-1}), 
        then $\texttt{truthValueInt(}\varAt, i, \theta \texttt{)} \in M$ and $\texttt{truthValueInt(}\varAt, |\kb|, \theta \texttt{)} \notin M$.
        We also have $\texttt{atom(}\varAt \texttt{)} \in M$, $\texttt{interpretation(}i \texttt{)} \in M$, and $\texttt{interpretation(}|\kb| \texttt{)} \in M$.
        Consequently, if for $n$ cases of atoms $\varAt$ wrt.\ interpretations $\omega_i$ and $\omega_{|\kb|}$ we have different truth values, i.\,e., we have $\texttt{truthValueInt(}\varAt, i, \theta \texttt{)} \in M$, but $\texttt{truthValueInt(}\varAt, |\kb|, \theta \texttt{)} \notin M$, this is counted by means of the \texttt{\#count} aggregate in (ADM8).
        Note that this also represents the Dalal distance between $\omega_i$ and $\omega_{|\kb|}$.
        Moreover, $\{\texttt{diff(0),} \ldots, \texttt{diff(}|\atoms(\kb)|\texttt{)}\} \subset M$ (ADM19).
        Since the Dalal distance between $\omega_i$ and $\omega_{|\kb|}$ cannot exceed $|\atoms(\kb)|$ (except for the $\infty$ case, see Lemma \ref{lem:asp-max-dist-1}), $\texttt{diff(}n\texttt{)} \in M$.
        Moreover, via (ADM8), we do not only derive $\texttt{d(}n\texttt{)}$, but also $\texttt{d(}m\texttt{)}$ with $m \in \{1, \ldots, n-1\}$.
        Thus, we conclude that $\{ \texttt{d(1)}, \ldots, \texttt{d(}n\texttt{)}\} \subset M$ if the Dalal distance between $\omega_i$ and $\omega_{|\kb|}$ is $n$.
        % -> n ist dann Ergebnis des counts; wenn = X, dann ist die "Bedingung" erfüllt 
        Further, each interpretation $\omega_i$ represents a model of the $i$-th formula, due to the integrity constraint (ADM18).
        Hence, each element in $D^{\max}_M$ always represents a value smaller than or equal to the Dalal distance between a model of a formula in $\kb$ and $\omega_{|\kb|}$.
        % we also have $\texttt{diff(}\varAt, \omega_i, \omega_j \texttt{)} \in M$ (ADM20).
        % Otherwise, i.\,e., if $\theta_i = \theta_2$, then $\texttt{diff(}\varAt, \omega_i, \omega_j \texttt{)} \not\in M$.
        % Consequently, due to (ADM21), for each pair of interpretations $(\omega_i, \omega_j)$, the number of elements in $\{\varAt \mid \texttt{diff(}\varAt, \omega_i, \omega_j \texttt{)} \in M\}$ are counted by means of a \texttt{\#count} aggregate.
        % Thus, $|\{\varAt \mid \texttt{diff(}\varAt, \omega_i, \omega_j \texttt{)} \in M\}|$ is the number of atoms $\varAt \in \atoms(\kb)$ which are assigned different truth values under two interpretations, which corresponds to the definition of the Dalal distance between two interpretations.
        % Further, $|\{\varAt \mid \texttt{diff(}\varAt, \omega_i, \omega_j \texttt{)} \in M\}| = D_M$, and thus $\texttt{d(}\omega_1, \omega_2, D_M\texttt{)} \in M$.
        % In summary, $D_M = d_{dalal}(\omega_1, \omega_2) = |\{\varAt \mid \texttt{diff(}\varAt, \omega_i, \omega_j \texttt{)} \in M\}|$, and $\texttt{d(}\omega_1, \omega_2, D_M\texttt{)} \in M$.
    \end{proof}
\end{lemma}

{
\renewcommand{\thetheorem}{\ref{thm:asp-maxdalal}}
\begin{theorem}
    Let $M_o$ be an optimal answer set of $\pmax(\kb)$.
    % Then $D^{\max}_{M_o} = \imdalal(\kb)$ with \texttt{dMax(}$D^{\max}_{M_o}$\texttt{)} $\in M_o$.
    % If no answer set of $\pmax(\kb)$ exists, $\imdalal(\kb) = \infty$.
    Then $|D^{\max}_{M_o}| = \imdalal(\kb)$.
    If no answer set of $\pmax(\kb)$ exists, $\imdalal(\kb) = \infty$.
    \begin{proof}
        Lemma \ref{lem:asp-max-dist-1} shows that $\pmax(\kb)$ does not have an answer set if $\kb$ contains at least one contradictory formula, i.\,e., if $\imdalal(\kb) = \infty$. 

        Furthermore, Lemma \ref{lem:asp-max-dist-2} shows that the elements in $D^{\max}_{M}$ represent the Dalal distances between an interpretation $\omega_{|\kb|}$ and a model of one of the formulas, represented by $\omega_i$, $i \in \{0, \ldots, |\kb|-1 \}$.
        %
        % Furthermore, Lemma \ref{lem:asp-max-dist-2} shows that for a pair of interpretations $(\omega_i, \omega_j)$, the Dalal distance $D_M = d_{dalal}(\omega_i, \omega_j)$ corresponds to $\texttt{d(}\omega_1, \omega_2, D_M\texttt{)} \in M$, with $M$ being an answer set of $\pmax{'}(\kb)$.
        % Thus, (ADM8) considers the Dalal distance between every pair $\omega_i, \omega_{|\kb|}$, with $i \in \{0, \ldots, |\kb|\}$.
        % Moreover, according to the integrity constraint (ADM19), each interpretation $\omega_i$ must be a model of the corresponding formula $\varFor_i$.
        % Therefore, (ADM8) represents the maximum Dalal distance (by means of a \texttt{\#max} aggregate) between interpretation $\omega_{|\kb|}$ (which represents the ``optimal'' interpretation) and the models of each formula $\varFor \in \kb$.
        % We denote this maximum Dalal distance as $D^{\max}_M$.
        % It follows that $\texttt{dMax(}D^{\max}_M\texttt{)} \in M$.
        %
        % Finally, the minimization statement in (ADM22) retrieves the minimal possible value of $D^{\max}_M$.
        % the minimal possible value of $D^{\max}_M$.
        % Let $M_o$ be an optimal answer set of $\pmax(\kb)$.
        % Then $\texttt{dMax(}D^{\max}_{M_o}\texttt{)} \in M_o$, and $D^{\max}_{M_o}$ corresponds to the minimal maximum Dalal distance between an ``optimal'' interpretation and the models of the formulas in $\kb$.
        Let $M_o$ be an optimal answer set of $\pmax(\kb)$.
        The minimize statement (ADM20) finds the interpretation ($\omega_{|\kb|}$) with the ``optimal'' distance to the models of the formulas of $\kb$.
        More precisely, the objective of the minimize statement is to find as few different distance values as possible, since it ``counts'' each distance. 
        Note that each distinct distance can only be derived once (due to (ADM19)), and for every $\texttt{d(}n\texttt{)} \in M$ with $n \in D^{\max}_{M_o}$, we also have $\texttt{d(}m\texttt{)} \in M$ with $m \in \{1, \ldots, n-1 \}$.
        In other words, every distance that is smaller than $n$ (but at least 1, since $\texttt{diff(0)} \notin M$ (ADM19)) is also counted, since (ADM8) derives $\texttt{d(}n\texttt{)}$ for each value that is smaller than or equal to the Dalal distance (see Lemma \ref{lem:asp-max-dist-2}).
        % Hence, we have exactly as many elements in $\{1, \ldots, n_{\max}\}$ in $D^{\max}_{M_o}$, with $n_{max}$ corresponding to the maximal value in $D^{\max}_{M_o}$.
        Hence, we have $D^{\max}_{M_o} = \{1, \ldots, n_{\max}\}$ with $n_{max}$ corresponding to the maximal value in $D^{\max}_{M_o}$.
        Due to the minimization in (ADM20), $n_{\max}$ corresponds exactly to the minimal maximum distance.
        It is easy to see that $|D^{\max}_{M_o}| = n_{\max}$.
        Therefore, $|D^{\max}_{M_o}| = \imdalal(\kb)$.
        % Finally, the minimize statement in (ADM20) finds the interpretation with the ``optimal'' distance ($\omega_{|\kb|}$) and therefore retrieves the minimal maximum distance between $\omega_{|\kb|}$ and any $\omega_i$.
    \end{proof}
\end{theorem} 
\addtocounter{theorem}{-1}
}

\subsubsection{The Sum-Distance Inconsistency Measure}\label{app:proof-asp-sum-dist}

{
\renewcommand\tablename{Encoding}
\begin{small}
\setlength\extrarowheight{2pt}
\begin{longtable}{| >{\columncolor{gray!20}}L{.05\textwidth}  >{\columncolor{gray!40}}L{.74\textwidth} >{\columncolor{gray!40}}R{.1\textwidth}|}
     \hline 
        \rowcolor{gray!20} \multicolumn{3}{|l|}{For every $\varFor_i \in \kb$ with $i\in \{0,\ldots,|\kb|-1\}$: } \\
        & \texttt{kbMember(}$\varFor$\texttt{,}$i$\texttt{).} & (ADS1) \\
        \rowcolor{gray!20} \multicolumn{3}{|l|}{For every $\varAt \in \atoms(\kb)$: } \\
        & \texttt{atom(}$\varAt$\texttt{).} & (ADS2) \\
        \rowcolor{gray!20} \multicolumn{3}{|l|}{Define $|\kb|+1$ interpretations: } \\
        & \texttt{interpretation(0..}$|\kb|$\texttt{).} & (ADS3) \\
        
        \rowcolor{gray!20} \multicolumn{3}{|l|}{For every conjunction $\varSub_c = \varSubSub_{c,1} \land \varSubSub_{c,2}$ appearing in some formula: } \\
            & $\texttt{conjunction(} \varSub_c \texttt{,} \varSubSub_{c,1} \texttt{,} \varSubSub_{c,2} \texttt{).} $ & (ADS4) \\
        
        \rowcolor{gray!20} \multicolumn{3}{|l|}{For every disjunction $\varSub_d = \varSubSub_{d,1} \lor \varSubSub_{d,2}$ appearing in some formula: } \\
            & $\texttt{disjunction(} \varSub_d \texttt{,} \varSubSub_{d,1} \texttt{,} \varSubSub_{d,2} \texttt{).}$ & (ADS5) \\
        
        \rowcolor{gray!20} \multicolumn{3}{|l|}{For each negation $\varSub_n = \lnot \varSubSub_n$ appearing in some formula: } \\
        & $\texttt{negation(} \varSub_n \texttt{,} \varSubSub_n \texttt{).}$ & (ADS6) \\
        \rowcolor{gray!20} \multicolumn{3}{|l|}{For every formula $\varSub_a$ which consists of a single atom $\varAt$: } \\
        & $\texttt{formulaIsAtom(} \varSub_a \texttt{,} \varAt \texttt{).}$ & (ADS7) \\
        
        % % sum distance: 
        % \rowcolor{gray!20} \multicolumn{3}{|l|}{Sum of distances between interpretation $|\kb|$ and the models of $\kb$: } \\
        % & \texttt{dSum(X)\,:-} & \\
        %     & \qquad \qquad \texttt{X = \#sum\{Y,I\,:\,d(I,}$|\kb|$\texttt{,Y),} & \\
        %     & \qquad \qquad \qquad \qquad \texttt{interpretation(I)\},} & \\
        %     & \qquad \qquad \texttt{X >= 0.} & (ADS8) \\

        % minimize statement:
        \rowcolor{gray!20} \multicolumn{3}{|l|}{Minimize statement: } \\
        & \texttt{\#minimize\{1,A,I: atom(A), truthValueInt(A,I,T),} & \\
            & \qquad \qquad \texttt{not truthValueInt(A,}$|\kb|$\texttt{,T)\}.} & (ADS8) \\
        
        \hline
        \rowcolor{gray!20} \multicolumn{3}{|l|}{For the static part, we define:} \\
        
        & $\texttt{tv(} t \texttt{;} f \texttt{).}$ & (ADS9) \\[1ex]
        
        % Unique atom evaluation:
        & \texttt{1\{truthValueInt(A,I,T)\,:\,tv(T)\}1\,:-} & \\
            & \qquad \qquad \texttt{atom(A),} & \\
            & \qquad \qquad \texttt{interpretation(I).} & (ADS10) \\[1ex]
            
        % Conjunction rules:
        & \texttt{truthValueInt(F,I,}$t$\texttt{) :- } & \\
            & \qquad \qquad \texttt{conjunction(F,G,H),} & \\
            & \qquad \qquad \texttt{interpretation(I),} & \\
            & \qquad \qquad \texttt{truthValueInt(G,I,}$t$\texttt{),} & \\
            & \qquad \qquad \texttt{truthValueInt(H,I,}$t$\texttt{).} & (ADS11) \\[1ex]
        & \texttt{truthValueInt(F,I,}$f$\texttt{) :- } & \\
            & \qquad \qquad \texttt{conjunction(F,\_,\_),} & \\
            & \qquad \qquad \texttt{interpretation(I),} & \\
            & \qquad \qquad \texttt{not truthValueInt(F,I,}$t$\texttt{).} & (ADS12) \\[1ex]
            
        % Disjunction rules:
        & \texttt{truthValueInt(F,I,}$f$\texttt{) :- } & \\
            & \qquad \qquad \texttt{disjunction(F,G,H),} & \\
            & \qquad \qquad \texttt{interpretation(I),} & \\
            & \qquad \qquad \texttt{truthValueInt(G,I,}$f$\texttt{),} & \\
            & \qquad \qquad \texttt{truthValueInt(H,I,}$f$\texttt{).} & (ADS13) \\[1ex]
        & \texttt{truthValueInt(F,I,}$t$\texttt{) :- } & \\
            & \qquad \qquad \texttt{disjunction(F,\_,\_),} & \\
            & \qquad \qquad \texttt{interpretation(I),} & \\
            & \qquad \qquad \texttt{not truthValueInt(F,I,}$f$\texttt{).} & (ADS14) \\[1ex]
            
        % Negation rules:
        & \texttt{truthValueInt(F,I,}$t$\texttt{)\,:-} & \\
            & \qquad \qquad \texttt{negation(F,G),} & \\
            & \qquad \qquad \texttt{truthValueInt(G,I,}$f$\texttt{).} & (ADS15) \\[1ex]
        & \texttt{truthValueInt(F,I,}$f$\texttt{)\,:-} & \\
            & \qquad \qquad \texttt{negation(F,G),} & \\
            & \qquad \qquad \texttt{truthValueInt(G,I,}$t$\texttt{).} & (ADS16) \\[1ex]
            
        % if a formula consists of a single atom:
        & \texttt{truthValueInt(F,I,T)\,:-} & \\
            & \qquad \qquad \texttt{formulaIsAtom(F,G),} & \\ 
            & \qquad \qquad \texttt{truthValueInt(G,I,T),} & \\
            & \qquad \qquad \texttt{interpretation(I),} & \\
            & \qquad \qquad \texttt{tv(T).} & (ADS17) \\[1ex]
            
        % Truth values of KB elements:
        % & \texttt{truthValueInt(F,L,I,T)\,:-} & \\
        %     & \qquad \qquad \texttt{kbMember(F,L),} & \\
        %     & \qquad \qquad \texttt{interpretation(I),} & \\
        %     & \qquad \qquad \texttt{tv(T),} & \\
        %     & \qquad \qquad \texttt{truthValueInt(F,I,T).} & (ADS18) \\[1ex]
            
        % Integrity constraint:
        & \texttt{:-} & \\
        % & \qquad \qquad \texttt{truthValueInt(F,L,I,}$f$\texttt{),} & \\
        % & \qquad \qquad \texttt{kbMember(F,L),} & \\
        % & \qquad \qquad \texttt{interpretation(I),} & \\
        % & \qquad \qquad \texttt{L == I.} & (ADS19) \\[1ex]
            & \qquad \qquad \texttt{truthValueInt(F,I,}$f$\texttt{),} & \\
            & \qquad \qquad \texttt{kbMember(F,I).} &  (ADS18) \\
        
        % % Dalal distance:
        % & \texttt{diff(A,I,J)\,:-} & \\
        %     & \qquad \qquad \texttt{atom(A),} & \\
        %     & \qquad \qquad \texttt{interpretation(I),} & \\
        %     & \qquad \qquad \texttt{interpretation(J),} & \\
        %     & \qquad \qquad \texttt{truthValueInt(A,I,T),} & \\
        %     & \qquad \qquad \texttt{truthValueInt(A,J,U),} & \\
        %     & \qquad \qquad \texttt{T != U.} & (ADS19) \\[1ex]
        % & \texttt{d(I,J,X)\,:-} & \\
        %     & \qquad \qquad \texttt{interpretation(I),} & \\
        %     & \qquad \qquad \texttt{interpretation(J),} & \\
        %     & \qquad \qquad \texttt{X = \#count\{A\,:\,diff(A,I,J), atom(A)\}.} & (ADS20) \\[1ex]

        % % minimize statement:
        % &  \texttt{\#minimize\{X\,:\,dSum(X)\}}. & (ADS21) \\
        \hline
    \caption{Overview of the ASP encoding $\psum(\kb)$ of the sum-distance inconsistency measure $\isdalal$.
    $\kb$ is the given knowledge base.}
    \label{tab:asp-dsum}
\end{longtable}
\end{small}
}

Let $\kb$ be an arbitrary knowledge base and $\psum(\kb)$ the extended logic program consisting of rules (ADS1)--(ADS22) listed in Encoding \ref{tab:asp-dsum} in Section \ref{sec:asp-algo-dsum}.

\begin{lemma}\label{lem:asp-sum-dist-1}
    $\psum(\kb)$ does not have an answer set if and only if $\kb$ contains one or more contradictory formulas. 
    \begin{proof}
        Analogous to the proof of Lemma \ref{lem:asp-max-dist-1}.
    \end{proof}
\end{lemma}

% Let $\psum{'}(\kb)$ denote the extended logic program $\psum(\kb)$ \emph{without} the minimize statement (ADS22).
% Further, let $M$ be an answer set of $\psum'$, and let $\theta \in \{t, f\}$.
We define 
$$D^\Sigma_M = \{ \mathrm{diff}_M(\varAt,i) \mid \texttt{truthValueInt(}\varAt, i, \theta\texttt{)} \in M, \texttt{truthValueInt(}\varAt, |\kb|, \theta\texttt{)} \notin M \}.$$

% \begin{lemma}\label{lem:asp-sum-dist-2}
%     Let $M$ be an answer set of $\psum{'}(\kb)$.
%     If $\texttt{d(}\omega_i, \omega_j, D_M \texttt{)} \in M$, then $D_M$ corresponds to the Dalal distance between $\omega_i$ and $\omega_j$ (i.\,e., $D_M = d_{dalal}(\omega_i, \omega_j)$), with $i,j \in \{0\, \ldots, |\kb|\}$.
%     \begin{proof}
%         Analogous to the proof of Lemma \ref{lem:asp-max-dist-2}.
%     \end{proof}
% \end{lemma}

{
\renewcommand{\thetheorem}{\ref{thm:asp-sumdalal}}
\begin{theorem}
    % Let $M_o$ be an optimal answer set of $\psum(\kb)$.
    % Then $D^{\Sigma}_{M_o} = \isdalal(\kb)$ with \texttt{dSum(}$D^{\Sigma}_{M_o}$\texttt{)} $\in M_o$.
    % If no answer set of $\psum(\kb)$ exists, $\isdalal(\kb) = \infty$.
    Let $M_o$ be an optimal answer set of $\psum(\kb)$.
    Then $|D^{\Sigma}_{M_o}| = \isdalal(\kb)$.
    If no answer set of $\psum(\kb)$ exists, $\isdalal(\kb) = \infty$.
    \begin{proof}
        % erst sagen, dass \omega_i Modelle sind (Verweis auf max-distance)
        % dann erklären wir, dass diff_M(X,i) die Dalal-Distanz zwischen zwei Interpretationen repräsentiert
        %   -> siehe Inhalt des Minimize-Statements
        %       -> genau genommen zwischen den Modellen der Formeln (s.o.) und omega_|K|
        %       -> omega_K soll unsere optimale Interpr. sein 
        % Wir zählen pro Interpretation, wie viele Atome unterschiedlich sind
        %   -> also insgesamt Summierung über die Distanzen zwischen je allen omega_i und omega_K
        % -> Minimierung => wir bekommen die minimale Summe 
        % q.e.d.
    
        Lemma \ref{lem:asp-sum-dist-1} shows that $\psum(\kb)$ does not have an answer set if $\kb$ contains at least one contradictory formula, i.\,e., if $\isdalal(\kb) = \infty$. 

        Each formula $\varFor \in \kb$ is connected to one specific interpretation.
        More precisely, interpretation $\omega_i$ is a model of formula $\varFor_i$ (ADS18) (see also the proof for Lemma \ref{lem:asp-max-dist-2}).
        Further, the inner part of the minimize statement (ADS8) represents the calculation and summation of the Dalal distance between $\omega_{|\kb|}$ and each interpretation $\omega_i$ with $i \in \{0, \ldots, |\kb|-1\}$.
        To be precise, each case in which an atom $\varAt$ evaluates to $\theta \in \{t,f\}$ under interpretation $\omega_i$, but $\omega_{|\kb|} \neq \theta$, i.\,e., if $\texttt{truthValue(}\varAt, i, \theta \texttt{)} \in M$, but $\texttt{truthValue(}\varAt, |\kb|, \theta \texttt{)} \notin M$, is counted as $1$.
        Thus, for any answer set $M$, $|D^{\Sigma}_{M}|$ corresponds to the sum of Dalal distances between the models of the formulas $\omega_i$ and $\omega_{|\kb|}$.
        For an optimal answer set $M_o$, yielded by the $\#minimize$ statement in (ADS8), $\omega_{|\kb|}$ becomes the ``optimal'' interpretation, i.\,e., the one with the lowest sum of distances.
        Hence, $|D^{\Sigma}_{M_o}| = \isdalal(\kb)$.
        %
        % Furthermore, Lemma \ref{lem:asp-sum-dist-2} shows that for a pair of interpretations $(\omega_i, \omega_j)$, the Dalal distance $D_M = d_{dalal}(\omega_i, \omega_j)$ corresponds to $\texttt{d(}\omega_1, \omega_2, D_M\texttt{)} \in M$, with $M$ being an answer set of $\pmax{'}(\kb)$.
        % Thus, (ADS8) considers the Dalal distance between every pair $\omega_i, \omega_{|\kb|}$, with $i \in \{0, \ldots, |\kb|\}$.
        % Moreover, according to the integrity constraint (ADS19), each interpretation $\omega_i$ must be a model of the corresponding formula $\varFor_i$.
        % Therefore, (ADS8) represents the sum of Dalal distances (by means of a \texttt{\#sum} aggregate) between interpretation $\omega_{|\kb|}$ (which represents the ``optimal'' interpretation) and the models of each formula $\varFor \in \kb$.
        % We denote this sum of Dalal distances as $D^{\Sigma}_M$.
        % It follows that $\texttt{dSum(}D^{\Sigma}_M\texttt{)} \in M$.
        %
        % Finally, the minimization statement in (ADS22) retrieves the minimal possible value of $D^{\Sigma}_M$.
        % To be precise, let $M_o$ be an optimal answer set of $\psum(\kb)$.
        % Then $\texttt{dSum(}D^{\Sigma}_{M_o}\texttt{)} \in M_o$, and $D^{\Sigma}_{M_o}$ corresponds to the minimal sum of Dalal distances between an ``optimal'' interpretation and the models of the formulas in $\kb$.
        % Hence, $D^{\Sigma}_{M_o} = \isdalal(\kb)$. 
    \end{proof}
\end{theorem} 
\addtocounter{theorem}{-1}
}

\subsubsection{The Hit-Distance Inconsistency Measure}\label{app:proof-asp-hit-dist}

{
\renewcommand\tablename{Encoding}
\begin{small}
\setlength\extrarowheight{2pt}
\begin{longtable}{| >{\columncolor{gray!20}}L{.05\textwidth} >{\columncolor{gray!40}}L{.74\textwidth} >{\columncolor{gray!40}}R{.1\textwidth} |}
     \hline 
        \rowcolor{gray!20} \multicolumn{3}{|l|}{For every $\varFor \in \kb$: } \\
        & \texttt{kbMember(}$\varFor$\texttt{).} & (ADH1) \\
        \rowcolor{gray!20} \multicolumn{3}{|l|}{For every $\varAt \in \atoms(\kb)$: } \\
        & \texttt{atom(}$\varAt$\texttt{).} & (ADH2) \\
        
        \rowcolor{gray!20} \multicolumn{3}{|l|}{For every conjunction $\varSub_c = \varSubSub_{c,1} \land \varSubSub_{c,2}$ appearing in some formula: } \\
            & $\texttt{conjunction(} \varSub_c \texttt{,} \varSubSub_{c,1} \texttt{,} \varSubSub_{c,2} \texttt{).} $ & (ADH3) \\

        \rowcolor{gray!20} \multicolumn{3}{|l|}{For every disjunction $\varSub_d = \varSubSub_{d,1} \lor \varSubSub_{d,2}$ appearing in some formula: } \\
            & $\texttt{disjunction(} \varSub_d \texttt{,} \varSubSub_{d,1} \texttt{,} \varSubSub_{d,2} \texttt{).}$ & (ADH4) \\
        
        \rowcolor{gray!20} \multicolumn{3}{|l|}{For each negation $\varSub_n = \lnot \varSubSub_n$ appearing in some formula: } \\
        & $\texttt{negation(} \varSub_n \texttt{,} \varSubSub_n \texttt{).}$ & (ADH5) \\
        \rowcolor{gray!20} \multicolumn{3}{|l|}{For every formula $\varSub_a$ which consists of a single atom $\varAt$: } \\
        & $\texttt{formulaIsAtom(} \varSub_a \texttt{,} \varAt \texttt{).}$ & (ADH6) \\
        \hline
        
        \rowcolor{gray!20} \multicolumn{3}{|l|}{For the static part, we define:} \\
        
        & $\texttt{tv(} t \texttt{;} f \texttt{).}$ & (ADH7) \\[1ex]
        
        % Unique atom evaluation:
        & \texttt{1\{truthValue(A,T)\,:\,tv(T)\}1\,:-} & \\
            & \qquad \qquad \texttt{atom(A).} & (ADH8) \\[1ex]
            
        % Conjunction rules:
        & \texttt{truthValue(F,}$t$\texttt{) :- } & \\
            & \qquad \qquad \texttt{conjunction(F,G,H),} & \\
            & \qquad \qquad \texttt{truthValue(G,}$t$\texttt{),} & \\
            & \qquad \qquad \texttt{truthValue(H,}$t$\texttt{).} & (ADH9) \\[1ex]
        & \texttt{truthValue(F,}$f$\texttt{) :- } & \\
            & \qquad \qquad \texttt{conjunction(F,\_,\_),} & \\
            & \qquad \qquad \texttt{not truthValue(F,}$t$\texttt{).} & (ADH10) \\[1ex]
            
        % Disjunction rules:
        & \texttt{truthValue(F,}$f$\texttt{) :- } & \\
            & \qquad \qquad \texttt{disjunction(F,G,H),} & \\
            & \qquad \qquad \texttt{truthValue(G,}$f$\texttt{),} & \\
            & \qquad \qquad \texttt{truthValue(H,}$f$\texttt{).} & (ADH11) \\[1ex]
        & \texttt{truthValue(F,}$t$\texttt{) :- } & \\
            & \qquad \qquad \texttt{disjunction(F,\_,\_),} & \\
            & \qquad \qquad \texttt{not truthValue(F,}$f$\texttt{).} & (ADH12) \\[1ex]
            
        % Negation rules:
        & \texttt{truthValue(F,}$t$\texttt{)\,:-} & \\
            & \qquad \qquad \texttt{negation(F,G),} & \\
            & \qquad \qquad \texttt{truthValue(G,}$f$\texttt{).} & (ADH13) \\[1ex]
        & \texttt{truthValue(F,}$f$\texttt{)\,:-} & \\
            & \qquad \qquad \texttt{negation(F,G),} & \\
            & \qquad \qquad \texttt{truthValue(G,}$t$\texttt{).} & (ADH14) \\[1ex]
            
        % if a formula consists of a single atom:
        & \texttt{truthValue(F,T)\,:-} & \\
            & \qquad \qquad \texttt{formulaIsAtom(F,G),} & \\ 
            & \qquad \qquad \texttt{truthValue(G,T),} & \\
            & \qquad \qquad \texttt{tv(T).} & (ADH15) \\[1ex]
            
        % Truth values of KB elements:
        & \texttt{truthValueKbMember(F,T)\,:-} & \\
            & \qquad \qquad \texttt{kbMember(F),} & \\
            & \qquad \qquad \texttt{tv(T),} & \\
            & \qquad \qquad \texttt{truthValue(F,T).} & (ADH16) \\[1ex]

        % minimize statement:
        &  \texttt{\#minimize\{1,F\,:\,truthValueKbMember(F,}$f$\texttt{)\}.} & (ADH17) \\
        \hline
    \caption{Overview of the ASP encoding $\phit(\kb)$ of the hit-distance inconsistency measure $\ihdalal$.
    $\kb$ is the given knowledge base.}
    \label{tab:asp-dhit}
\end{longtable}
\end{small}
}

Let $\kb$ be an arbitrary knowledge base and $\phit(\kb)$ the extended logic program consisting of rules (ADH1)--(ADH17) listed in Encoding \ref{tab:asp-dhit} in Section \ref{sec:asp-algo-dhit}.
Let $\phit{'}(\kb)$ denote the extended logic program $\phit(\kb)$ without the minimize statement (ADH17).

For an answer set $M$ of $\phit{'}(\kb)$ we define:
    $$ K_M = \{ \varFor^\mathrm{asp} \mid \varFor \in \kb, \texttt{truthValueKbMember(}\varFor^\mathrm{asp},f \texttt{)} \in M \} $$
with $\varFor^\mathrm{asp}$ being an ASP representation of $\varFor$.

\begin{lemma}\label{lem:asp-hit-distance}
    Each $\varFor^\mathrm{asp} \in K_M$ corresponds to a formula $\varFor \in \kb$ being evaluated to $f$.
    \begin{proof}
        The knowledge base, i.\,e., its formulas, subformulas, and atoms are modeled by (ADH1)--(ADH6). 
        (This corresponds exactly to (AC1)--(AC6) wrt.\ $\icont$.)
        The two truth values are represented by (ADH7).
        Analogously to (AF10) ($\iforget$), (ADH8) guesses an interpretation by assigning each atom a distinct truth value.
        We denote this interpretation as follows:
        \begin{align*}
            \omega_M(\varAt) = \begin{cases} 
            t & \quad \texttt{truthValue(}\varAt \texttt{,} t \texttt{)} \in M \\
            f & \quad \texttt{truthValue(}\varAt \texttt{,} f \texttt{)} \in M
            \end{cases}
        \end{align*}
        (ADH9)--(ADH14) model the evaluation of conjunctions, disjunctions, and negations wrt.\ $\omega_M$ (this corresponds exactly to (AF11)--(AF16) wrt.\ $\iforget$). 
        In addition, (ADH15) models that a formula consisting of an individual atom gets the same truth value as that atom.
        Moreover, (ADH16) essentially extracts the truth values of the formulas $\varFor \in \kb$.
        More precisely, if the atom valuations determined by (ADH8) eventually lead to a formula being evaluated to a truth value $\theta \in \{t,f\}$---i.\,e., if $\texttt{truthValue(}\varFor^\mathrm{asp}, \theta \texttt{)} \in M$ and $\varFor$ is an element of $\kb$ (i.\,e., $\texttt{kbMember(}\varFor^\mathrm{asp} \texttt{)} \in M $), then $\texttt{truthValueKbMember(}\varFor^\mathrm{asp}, \theta \texttt{)} \in M $.
        Hence, each $\varFor^\mathrm{asp} \in K_M$ corresponds to a formula $\varFor \in \kb$ with $\omega_M(\varFor) = f$.
    \end{proof}
\end{lemma}

{
\renewcommand{\thetheorem}{\ref{thm:asp-hitdalal}}
\begin{theorem}
    Let $M_o$ be an optimal answer set of $\phit(\kb)$.
    Then $|K_{M_o}| = \ihdalal(\kb)$.
    \begin{proof}
        From Lemma \ref{lem:asp-hit-distance} we know that for an answer set $M$ of $\phit{'}(\kb)$, each $\varFor^\mathrm{asp} \in K_M$ corresponds to a formula $\varFor \in \kb$ being evaluated to false (i.\,e., with $\omega_M(\varFor) = f$).
        In $\phit(\kb)$, the minimization statement (ADH17) ensures that the resulting optimal answer set $M_o$ contains a minimal number of instances of $\texttt{truthvalueKbMember(}\varFor^\mathrm{asp}, f \texttt{)}$ (with $\varFor \in \kb$), i.\,e., the cardinality of $|K_{M_o}|$ gets minimized. 
        Hence, we are computing the minimal number of formulas in $\kb$ which evaluate to $f$ (and which would need to be removed in order to render $\kb$ consistent), which meets the definition of $\ihdalal$.
    \end{proof}
\end{theorem} 
\addtocounter{theorem}{-1}
}

\section{Additional Data} \label{app:data}

This section includes additional data and visualizations regarding the experimental evaluation (Section \ref{sec:evaluation}).

% \subsection{Cumulative Runtimes} % and Exact Numbers of Timeouts

\subsection{Runtime Measurements}\label{app:cactus}

% Figure \ref{fig:overall-runtime-all-RAN} 
The following figures display cactus plots\footnote{A cactus plot is comprised of time measurements ordered from low to high.
Hence, we measure the runtime of an algorithm wrt.\ each knowledge base of a given data set, and subsequently order the measurements from low to high and plot them.
If an inconsistency value wrt.\ a given knowledge base could not be computed within the given time limit (i.e., if it timed out) or if a memory error occurred, the corresponding time measurement is not plotted.}, each corresponding to a comparison between the three approaches wrt.\ a specific inconsistency measure and data set.

% SRS
\begin{figure}
    \begin{subfigure}[t]{.5\textwidth}
        \includegraphics[width=\textwidth]{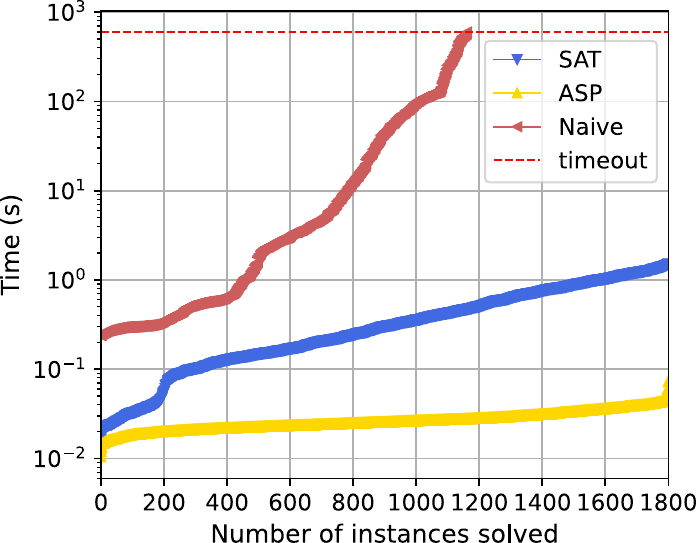}
        \caption{Contension inconsistency measure ($\icont$)}
    \end{subfigure}%
    \begin{subfigure}[t]{.5\textwidth}
        \includegraphics[width=\textwidth]{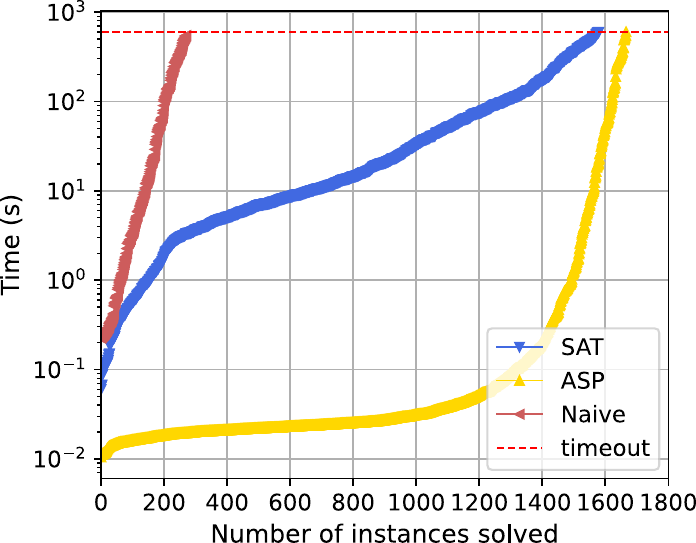}
        \caption{Forgetting-based inconsistency measure ($\iforget$)}
    \end{subfigure}\\[1ex]
    
    \begin{subfigure}[t]{.5\textwidth}
        \includegraphics[width=\textwidth]{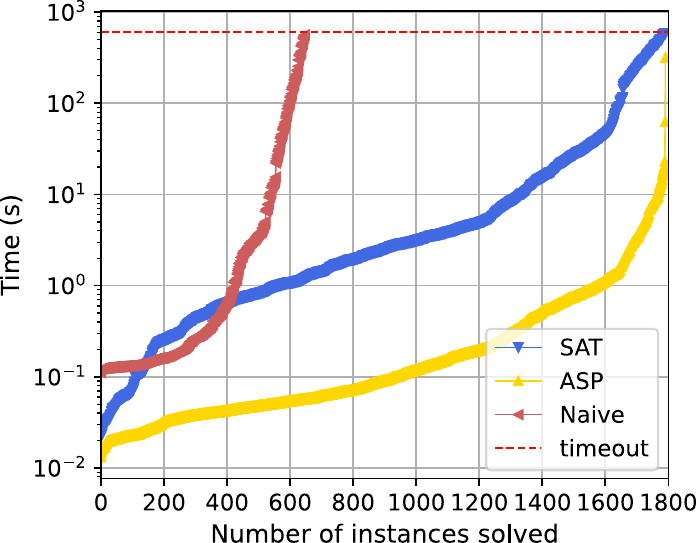}
        \caption{Hitting Set inconsistency measure ($\ihs$)}
    \end{subfigure}%
    \begin{subfigure}[t]{.5\textwidth}
        \includegraphics[width=\textwidth]{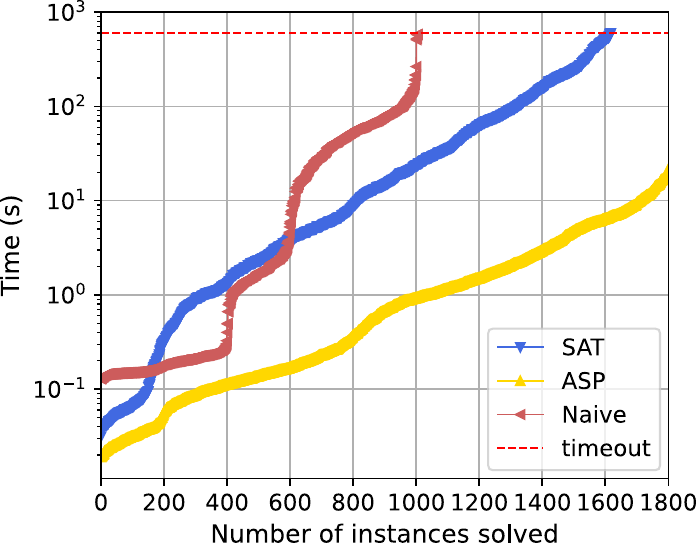}
        \caption{Max-distance inconsistency measure ($\imdalal$)}
    \end{subfigure}\\[1ex]
    
    \begin{subfigure}[t]{.5\textwidth}
        \includegraphics[width=\textwidth]{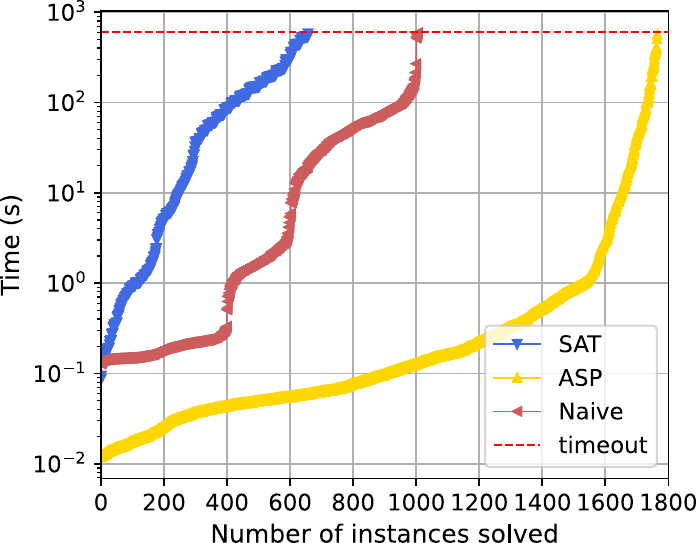}
        \caption{Sum-distance inconsistency measure ($\isdalal$)}
        \label{fig:overall-runtime-all-RAN-sum-dalal}
    \end{subfigure}%
    \begin{subfigure}[t]{.5\textwidth}
        \includegraphics[width=\textwidth]{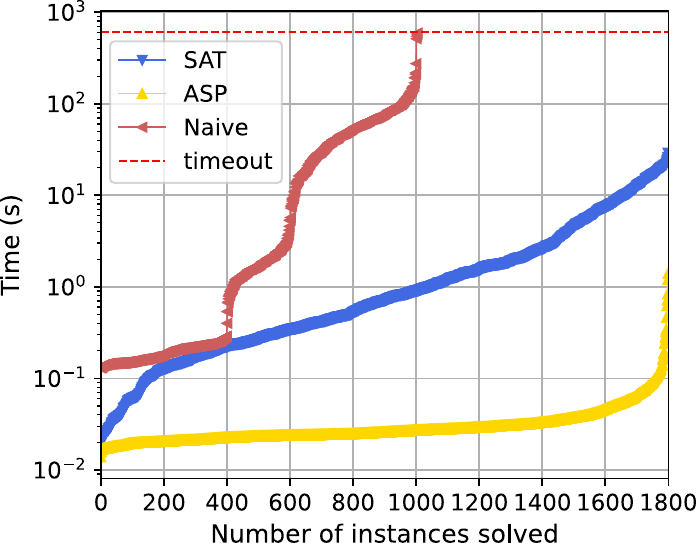}
        \caption{Hit-distance inconsistency measure ($\ihdalal$)}
    \end{subfigure}
    \caption{Runtime comparison of the ASP-based, SAT-based, and naive approaches on the SRS data set. Timeout: $10$ minutes.}
    \label{fig:overall-runtime-all-RAN}
\end{figure}

% ML
\begin{figure}
    \begin{subfigure}[t]{.5\textwidth}
        \includegraphics[width=.98\textwidth]{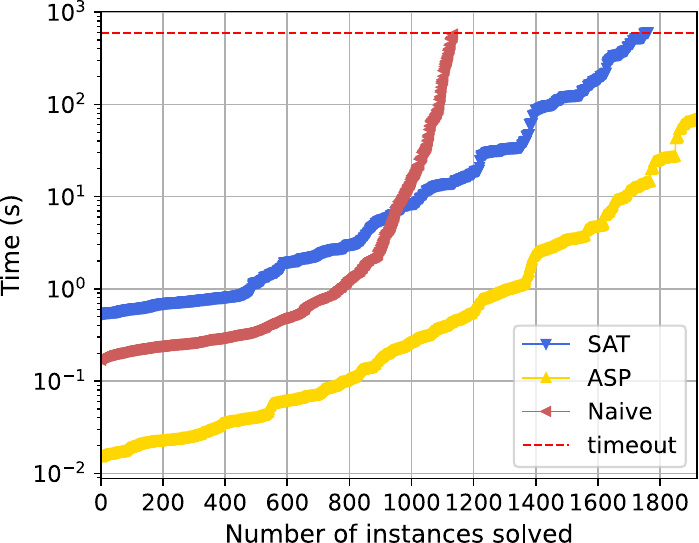}
        \caption{Contension inconsistency measure ($\icont$)}
    \end{subfigure}%
    \hfill%
    \begin{subfigure}[t]{.5\textwidth}
        \includegraphics[width=.98\textwidth]{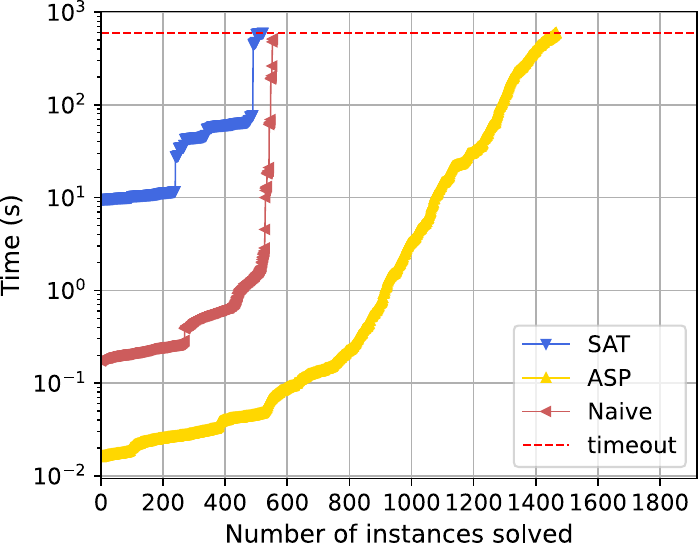}
        \caption{Forgetting-based inconsistency measure ($\iforget$)}
    \end{subfigure}\\[1ex]
    
    \begin{subfigure}[t]{.5\textwidth}
        \includegraphics[width=.98\textwidth]{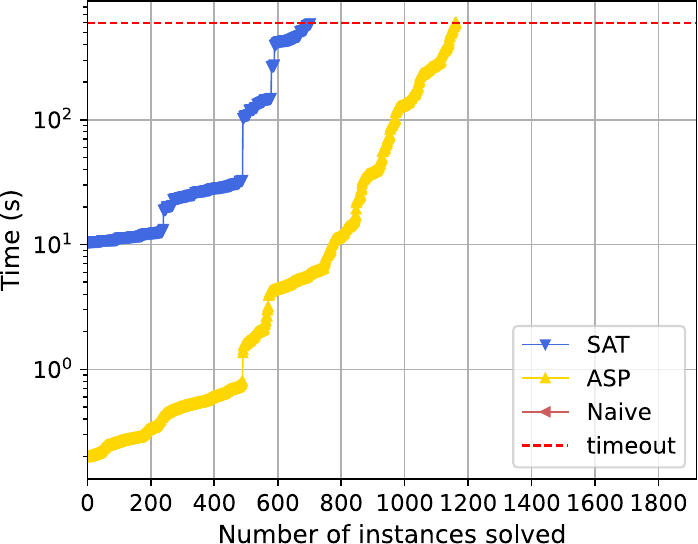}
        \caption{Hitting Set inconsistency measure ($\ihs$)}
    \end{subfigure}%
    \begin{subfigure}[t]{.5\textwidth}
        \includegraphics[width=.98\textwidth]{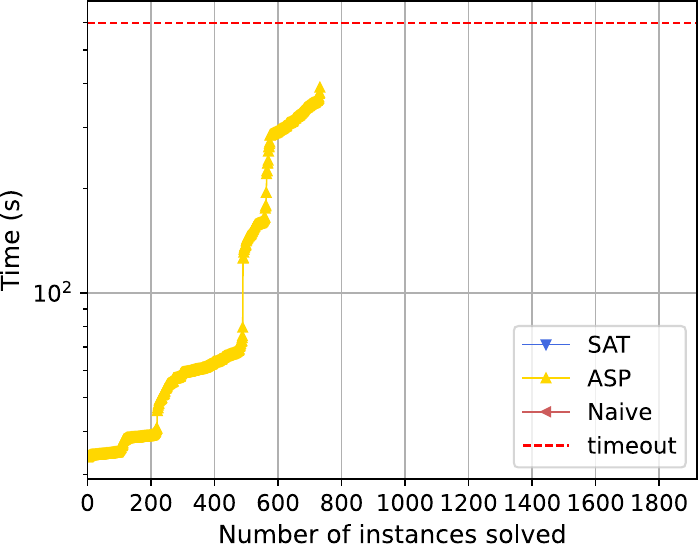}
        \caption{Max-distance inconsistency measure ($\imdalal$)}
        \label{fig:overall-runtime-all-ML-maxdalal}
    \end{subfigure}\\[1ex]
    
    \begin{subfigure}[t]{.5\textwidth}
        \includegraphics[width=.98\textwidth]{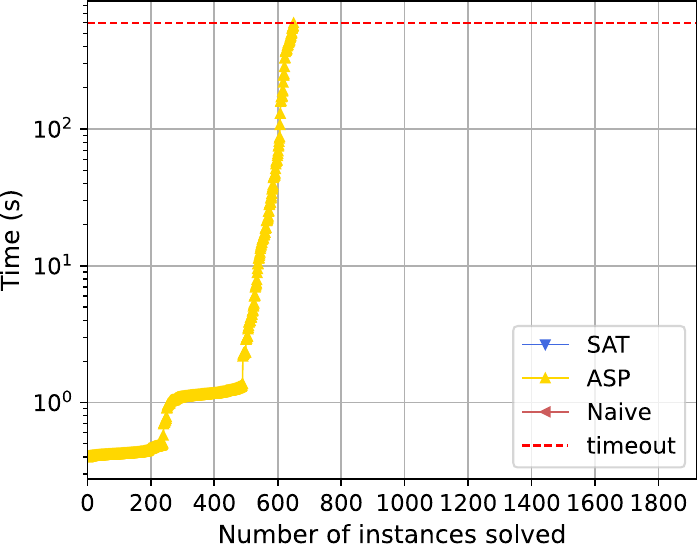}
        \caption{Sum-distance inconsistency measure ($\isdalal$)}
        \label{fig:overall-runtime-all-ML-sumdalal}
    \end{subfigure}%
    \begin{subfigure}[t]{.5\textwidth}
        \includegraphics[width=.98\textwidth]{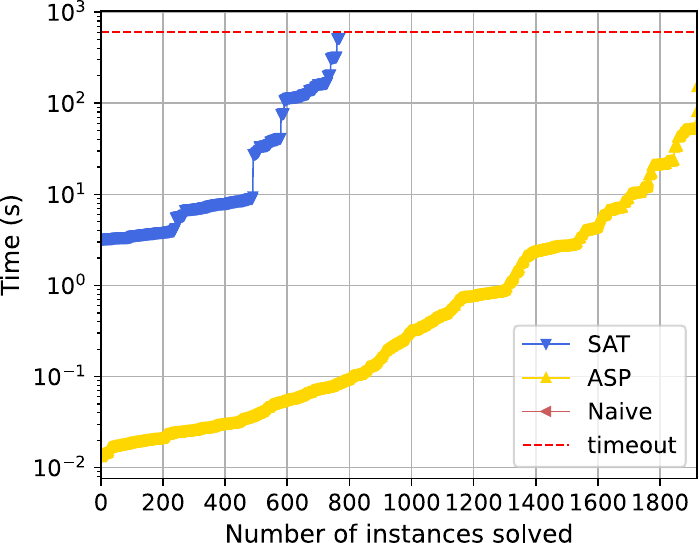}
        \caption{Hit-distance inconsistency measure ($\ihdalal$)}
        \label{fig:overall-runtime-all-ML-hitdalal}
    \end{subfigure}
    \caption{Runtime comparison of the ASP-based, SAT-based, and naive approaches on the ML data set. Timeout: $10$ minutes.}
    \label{fig:overall-runtime-all-ML}
\end{figure}

% ARG
\begin{figure}
    \begin{subfigure}[t]{.5\textwidth}
        \includegraphics[width=.98\textwidth]{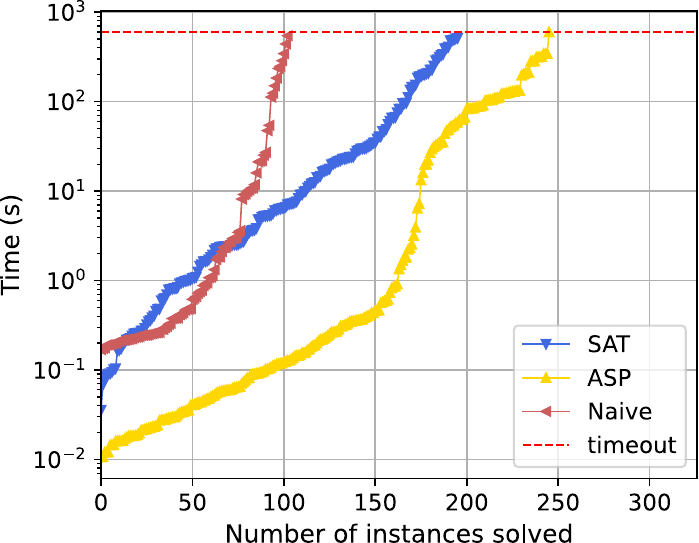}
        \caption{Contension inconsistency measure ($\icont$)}
    \end{subfigure}%
    \hfill%
    \begin{subfigure}[t]{.5\textwidth}
        \includegraphics[width=.98\textwidth]{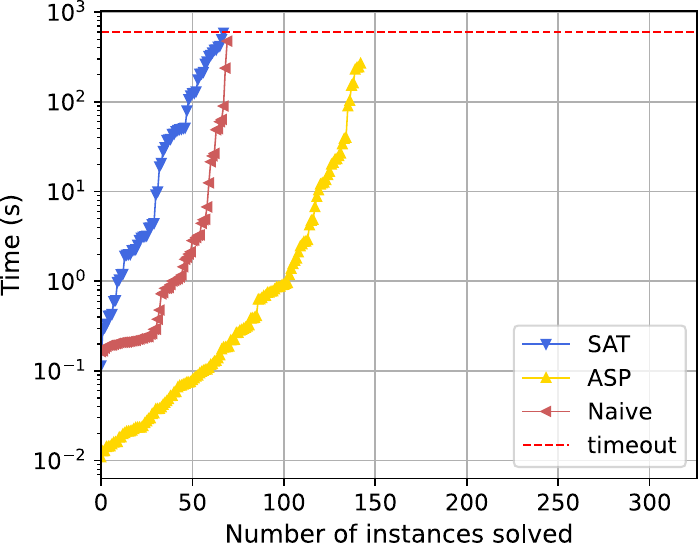}
        \caption{Forgetting-based inconsistency measure ($\iforget$)}
    \end{subfigure}\\[1ex]
    
    \begin{subfigure}[t]{.5\textwidth}
        \includegraphics[width=.98\textwidth]{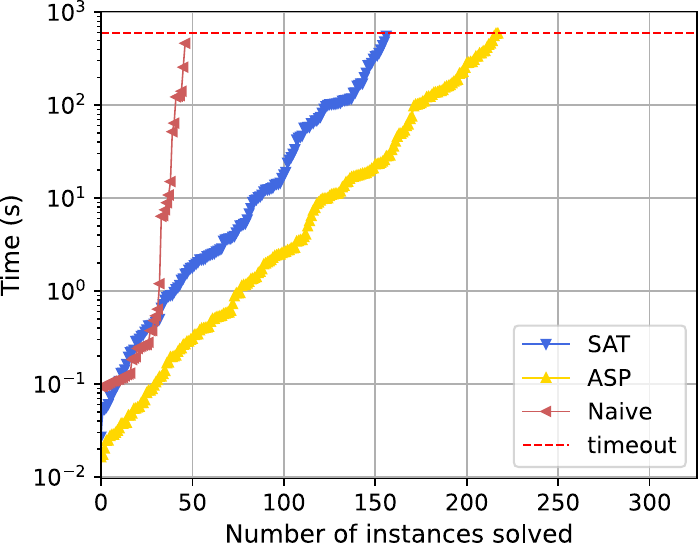}
        \caption{Hitting Set inconsistency measure ($\ihs$)}
    \end{subfigure}%
    \hfill%
    \begin{subfigure}[t]{.5\textwidth}
        \includegraphics[width=.98\textwidth]{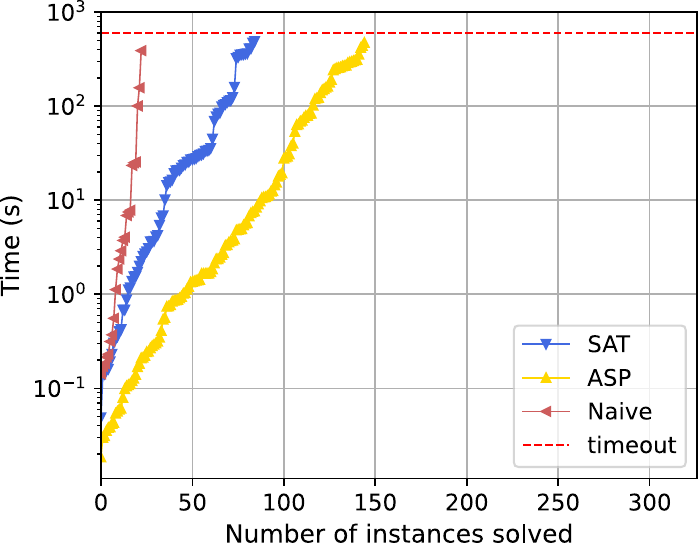}
        \caption{Max-distance inconsistency measure ($\imdalal$)}
    \end{subfigure}\\[1ex]
    
    \begin{subfigure}[t]{.5\textwidth}
        \includegraphics[width=.98\textwidth]{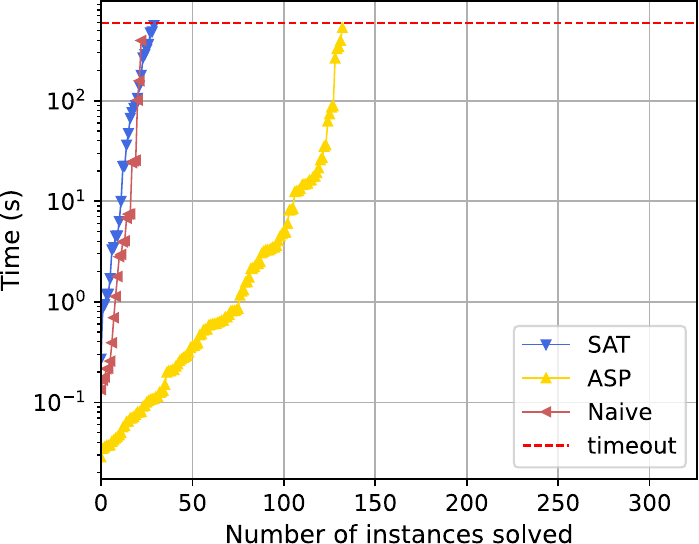}
        \caption{Sum-distance inconsistency measure ($\isdalal$)}
    \end{subfigure}%
    \hfill%
    \begin{subfigure}[t]{.5\textwidth}
        \includegraphics[width=.98\textwidth]{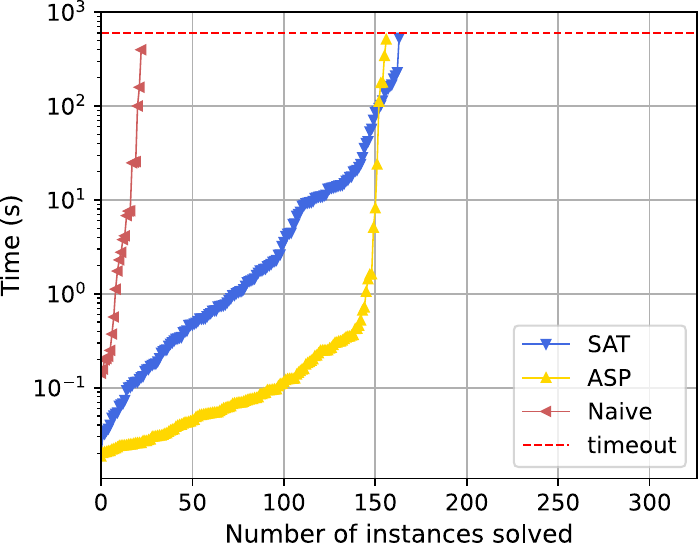}
        \caption{Hit-distance inconsistency measure ($\ihdalal$)}
    \end{subfigure}
    \caption{Runtime comparison of the ASP-based, SAT-based, and naive approaches on the ARG data set. Timeout: $10$ minutes.}
    \label{fig:overall-runtime-all-ARG}
\end{figure}

% SC
\begin{figure}[t]
    \centering
    \includegraphics[width=.55\textwidth]{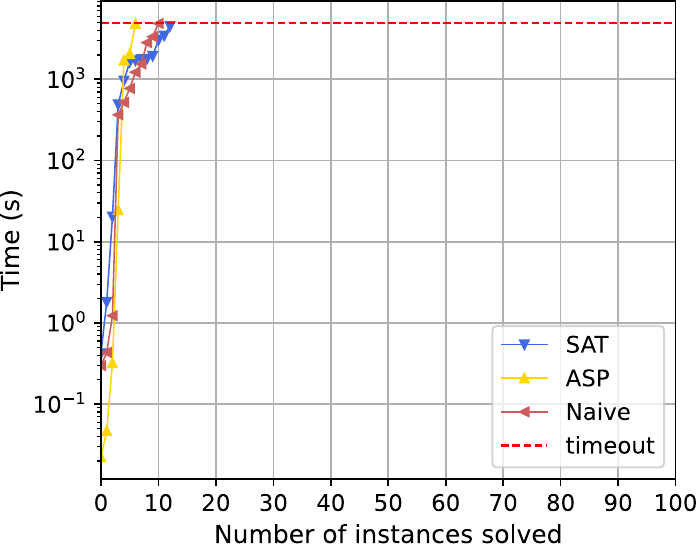}
    \caption{Runtime comparison of the ASP-based, SAT-based, and naive approaches for the contension inconsistency measure on the SC data set. Timeout: $5000$ seconds ($83.\overline{3}$ minutes).}
    \label{fig:overall-runtime-SAT}
\end{figure}

% LP
\begin{figure}
    \begin{subfigure}[t]{.5\textwidth}
        \includegraphics[width=\textwidth]{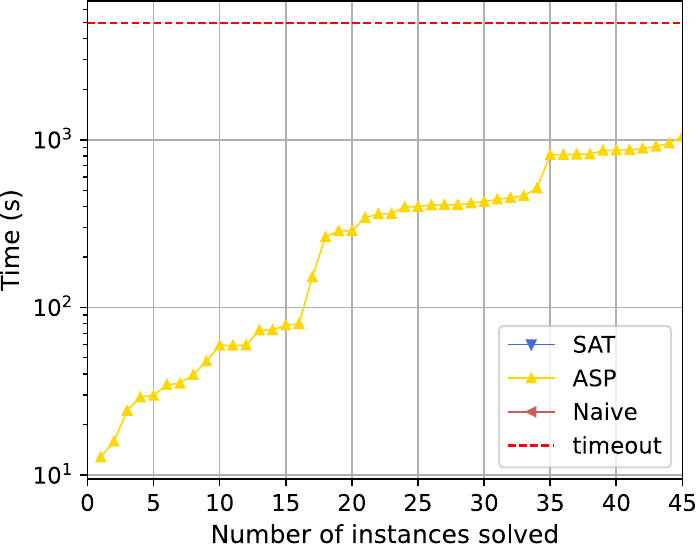}
        \caption{Contension inconsistency measure ($\icont$)}
    \end{subfigure}%
    \begin{subfigure}[t]{.5\textwidth}
        \includegraphics[width=\textwidth]{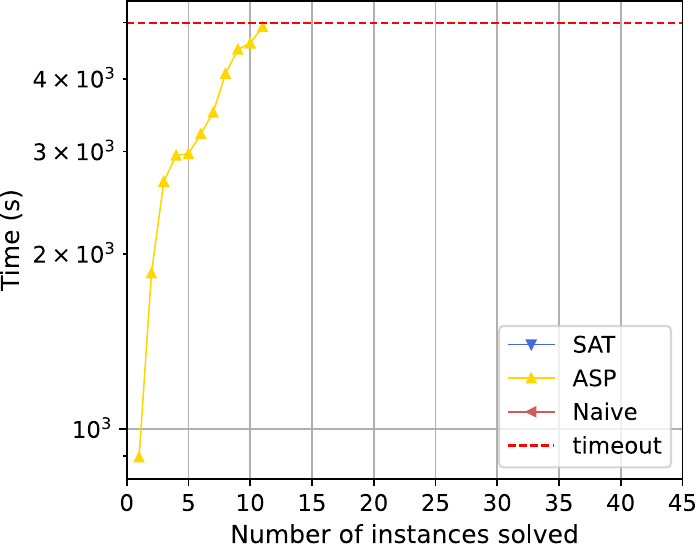}
        \caption{Forgetting-based inconsistency measure ($\iforget$)}
    \end{subfigure}\\[1ex]
    
    \centering
    \begin{subfigure}{.5\textwidth}
        \includegraphics[width=\textwidth]{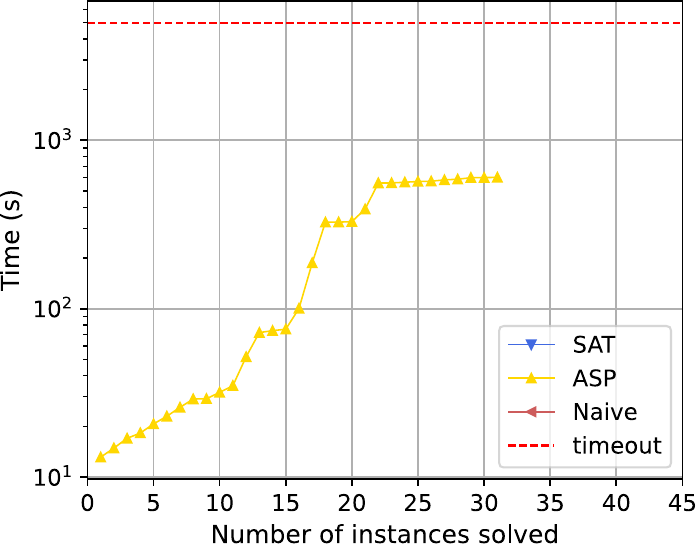}
        \caption{Hit-distance inconsistency measure ($\ihdalal$)}
    \end{subfigure}
    \caption{Runtime comparison of the ASP-based, SAT-based, and naive approaches on the LP data set for the contension and forgetting-based inconsistency measures. Timeout: $10$ minutes.}
    \label{fig:overall-runtime-all-ASP}
\end{figure}

% linear vs. binary search
\begin{figure}
    \begin{subfigure}[t]{.5\textwidth}
        \includegraphics[width=\textwidth]{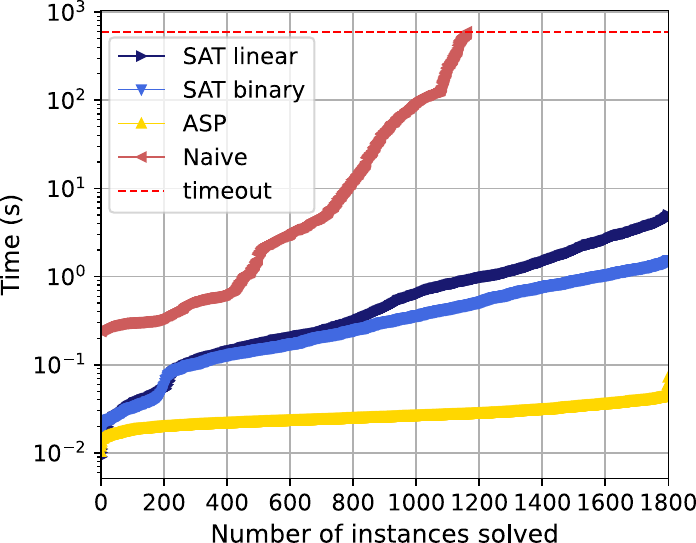}
        \caption{Contension inconsistency measure ($\icont$)}
    \end{subfigure}%
    \begin{subfigure}[t]{.5\textwidth}
        \includegraphics[width=\textwidth]{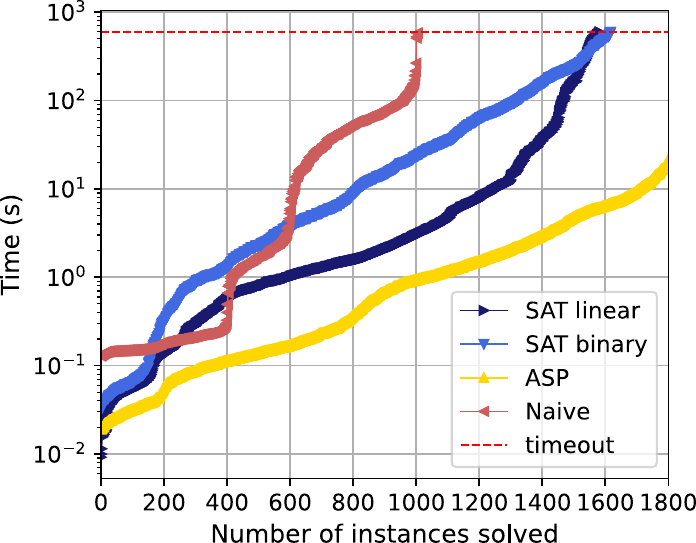}
        \caption{Max-distance inconsistency measure ($\imdalal$)}
    \end{subfigure}\\[1ex]
    
    \centering
    \begin{subfigure}[t]{.5\textwidth}
        \includegraphics[width=\textwidth]{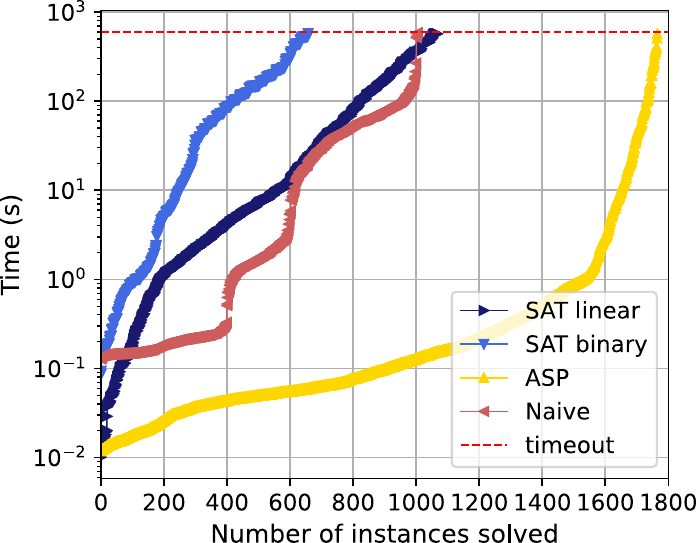}
        \caption{Sum-distance inconsistency measure ($\isdalal$)}
    \end{subfigure}
    \caption{Runtime comparison of the SAT-based approaches based on linear search (for $\icont$, $\imdalal$, and $\isdalal$) and the corresponding binary search versions, as well as the ASP-based versions, on the SRS data set. Timeout: $10$ minutes.}
    \label{fig:cactus-binary-vs-linear-search}
\end{figure}

% ASP: usc vs. bb
% \begin{figure}
%     \centering
%     \includegraphics[width=.5\textwidth]{img/asp_usc/cactus_ASP_usc-bb.pdf}
%     \caption{Runtime comparison of the ASP approach for $\ihs$ with core-guided vs.\ with branch-and-bound optimization (default) on the ML data set. Timeout: 10 minutes.}
%     \label{fig:asp-usc-cactus}
% \end{figure}
\begin{figure}
    \centering
    \begin{subfigure}[t]{.49\textwidth}
        \includegraphics[width=\textwidth]{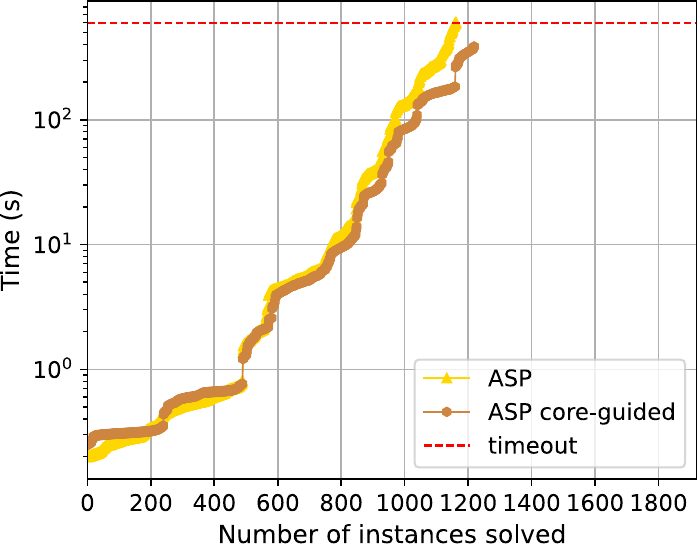}
        \caption{ML / $\ihs$}
    \end{subfigure}\\[1ex]
    \begin{subfigure}[t]{.49\textwidth}
        \includegraphics[width=\textwidth]{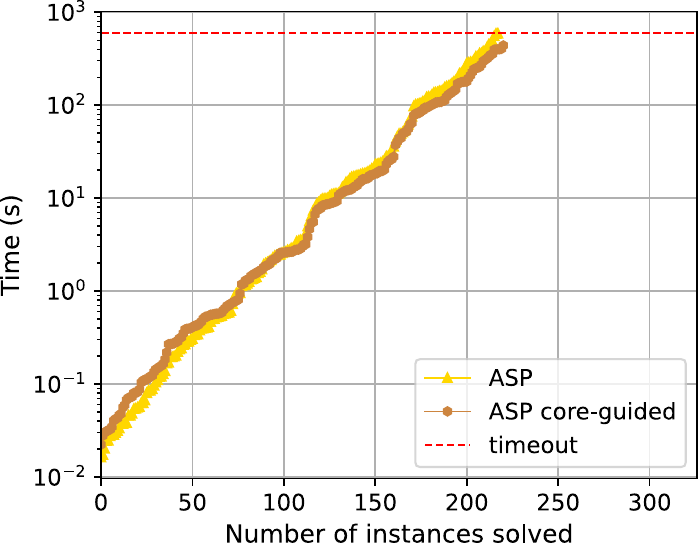}
        \caption{ARG / $\ihs$}
    \end{subfigure}%
    \hfill%
    \begin{subfigure}[t]{.49\textwidth}
        \includegraphics[width=\textwidth]{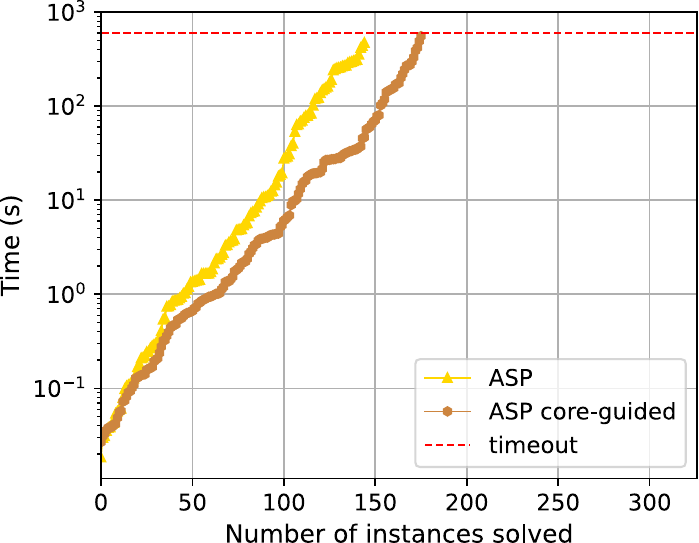}
        \caption{ARG / $\imdalal$}
    \end{subfigure}
    \caption{Runtime comparison of the ASP approach with core-guided vs.\ with branch-and-bound optimization (default) on % selected combinations of data sets and inconsistency measures. 
    the ML data set wrt.\ the hitting set measure ($\ihs$) and on the ARG data set wrt.\ $\ihs$ and the max-distance measure ($\imdalal$).
    Timeout: 10 minutes.}
    \label{fig:asp-usc-cactus}
\end{figure}

% MaxSAT
\begin{figure}
    \centering
    \includegraphics[width=.5\textwidth]{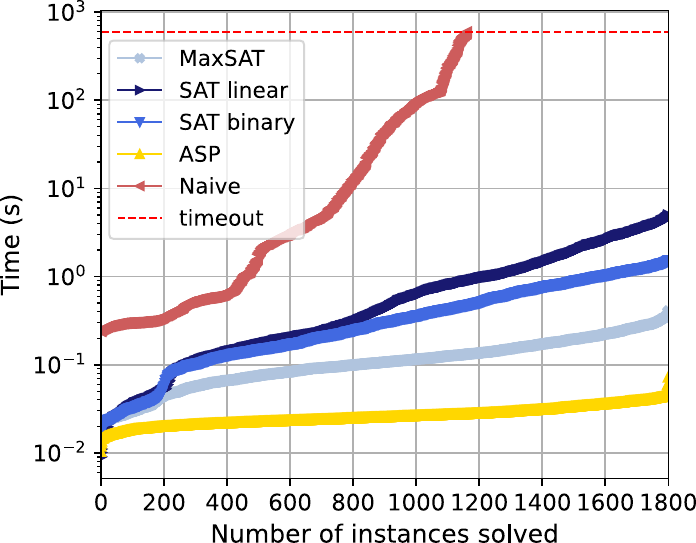}
    \caption{Runtime comparison of the MaxSAT approach for $\icont$ and the previously introduced approaches on the SRS data set. Timeout: 10 minutes.}
    \label{fig:maxsat-cactus}
\end{figure}

% ASP comparison (contension)
\begin{figure}
    \centering
    \includegraphics[width=.5\textwidth]{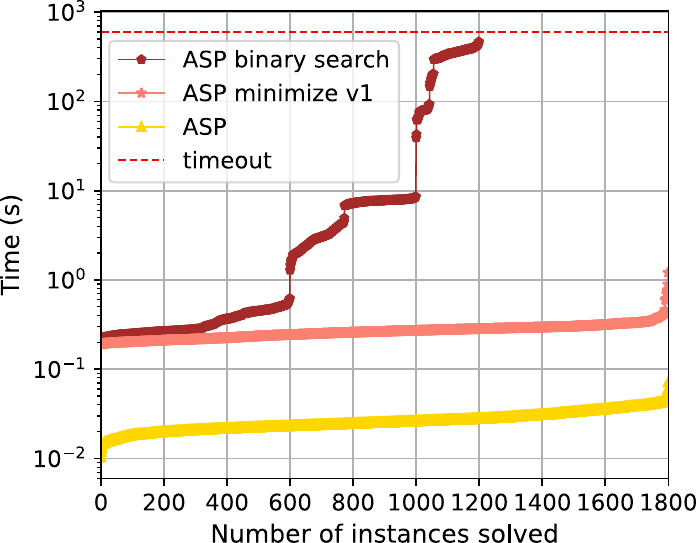}
    \caption{Runtime comparison of the different versions of the ASP approach for $\icont$ on the SRS data set. ``ASP binary search'' refers to the version from \cite{kuhlmann2020algorithm}, ``ASP minimize v1'' to the version from \cite{kuhlmann2021algorithms}, and ``ASP'' to the new one. Timeout: 10 minutes.}
    \label{fig:asp-cactus}
\end{figure}

% \newpage
\clearpage

\subsection{Inconsistency values}\label{app:inc-values}

In this section, we present histograms over the inconsistency values regarding all inconsistency measures and all data sets considered in this work. 
Note that we could only include inconsistency values wrt.\ instances for which at least one approach did not time out.
To be precise, the histograms for the SRS data set are included in Figure \ref{fig:histos-RAN}, those for the ML data set in Figure \ref{fig:histos-ML}, those for the ARG data set in Figure \ref{fig:histos-ARG}, those for the SC data set in Figure \ref{fig:histo-SAT}, and those for the LP data set in Figure \ref{fig:histo-ASP}.  

Furthermore, the following plots only include non-$\infty$ values. 
With regard to those measures that allow for the value $\infty$ (i.e., $\ihs$, $\imdalal$, and $\isdalal$), the number of instances resulting in that value are mentioned in the corresponding captions.

% SRS
\begin{figure}
    \begin{subfigure}[t]{.485\textwidth}
        \includegraphics[width=.98\textwidth]{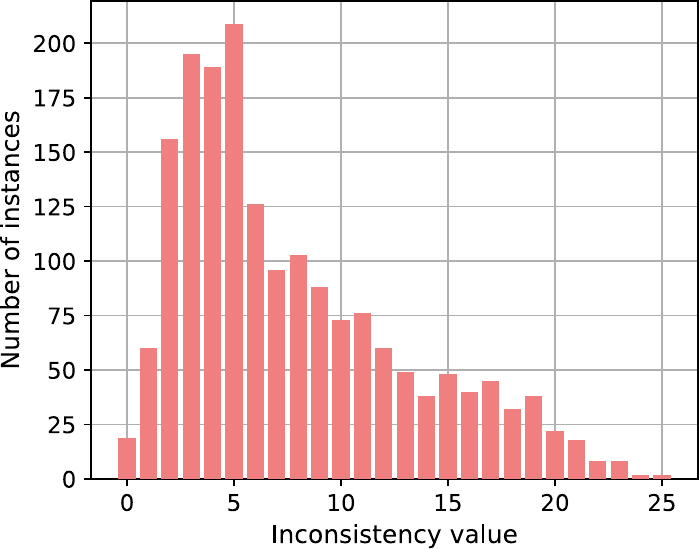}
        \caption{Contension inconsistency measure ($\icont$)}
    \end{subfigure}%
    \hfill%
    \begin{subfigure}[t]{.485\textwidth}
        \includegraphics[width=.98\textwidth]{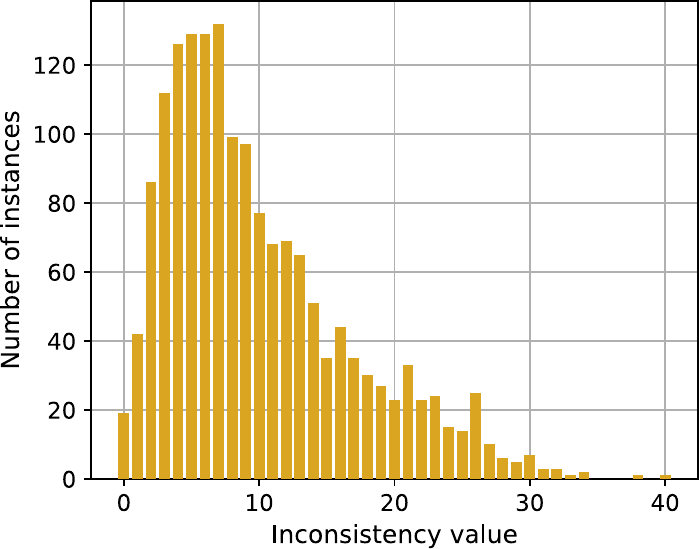}
        \caption{Forgetting-based inconsistency measure ($\iforget$)}
    \end{subfigure}\\[1ex]
    
    \begin{subfigure}[t]{.485\textwidth}
        \includegraphics[width=.98\textwidth]{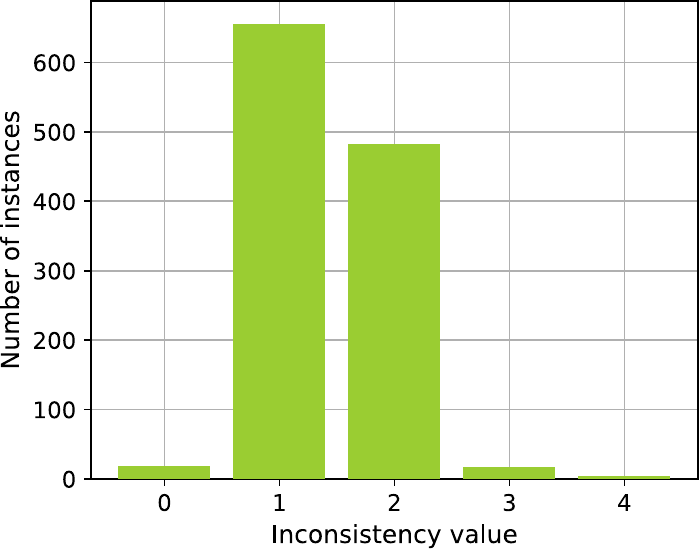}
        \caption{\raggedright Hitting Set inconsistency measure ($\ihs$); number of $\infty$ cases: $621$}
    \end{subfigure}%
    \hfill%
    \begin{subfigure}[t]{.485\textwidth}
        \includegraphics[width=.98\textwidth]{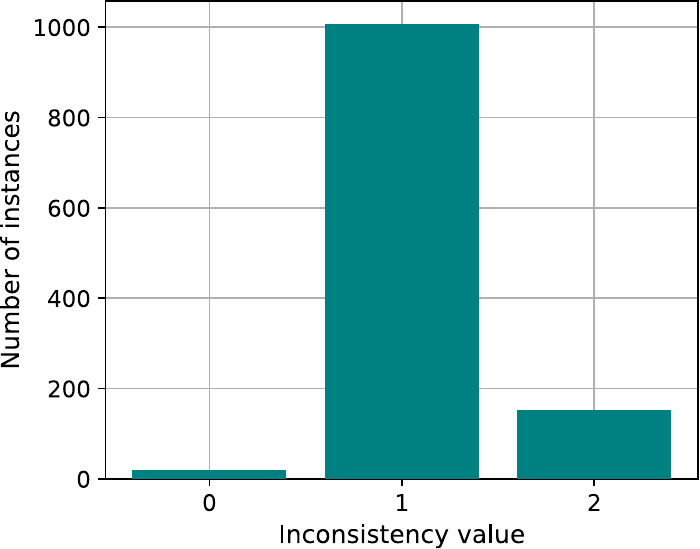}
        \caption{Max-distance inconsistency measure ($\imdalal$); number of $\infty$ cases: $621$}
    \end{subfigure}\\[1ex]
    
    \begin{subfigure}[t]{.485\textwidth}
        \includegraphics[width=.98\textwidth]{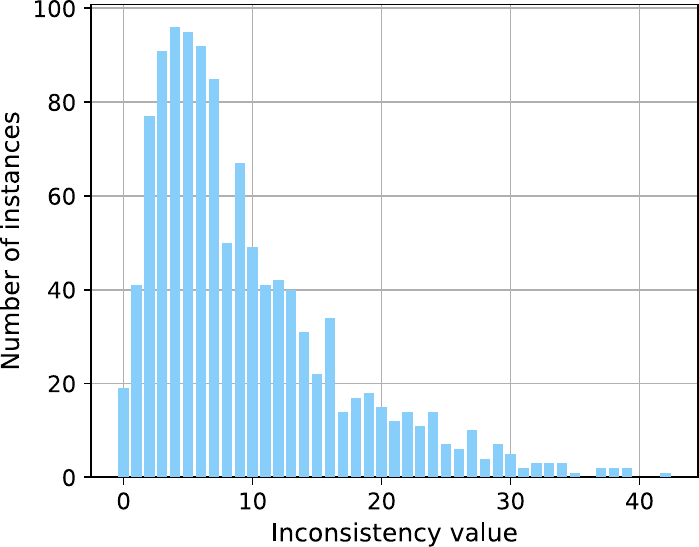}
        \caption{Sum-distance inconsistency measure ($\isdalal$); number of $\infty$ cases: $621$}
    \end{subfigure}%
    \hfill%
    \begin{subfigure}[t]{.485\textwidth}
        \includegraphics[width=.98\textwidth]{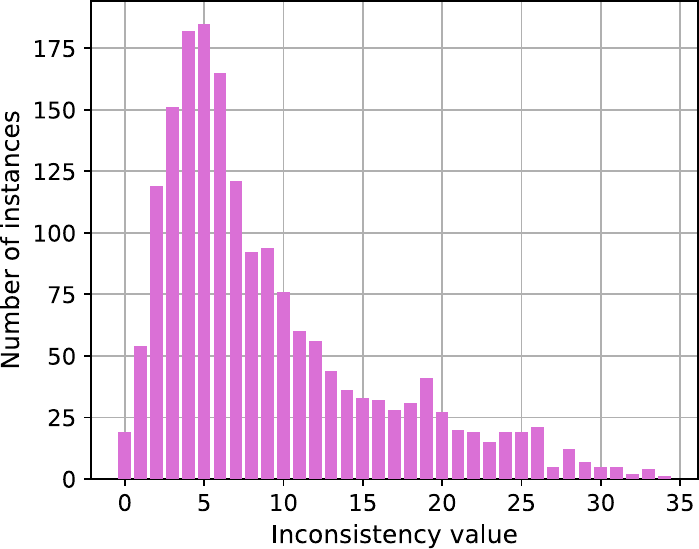}
        \caption{Hit-distance inconsistency measure ($\ihdalal$)}
    \end{subfigure}
    \caption{Histograms of the inconsistency values of the SRS data set wrt.\ all six inconsistency measures.}
    \label{fig:histos-RAN}
\end{figure}

% ML
\begin{figure}
    \begin{subfigure}[t]{.485\textwidth}
        \includegraphics[width=.98\textwidth]{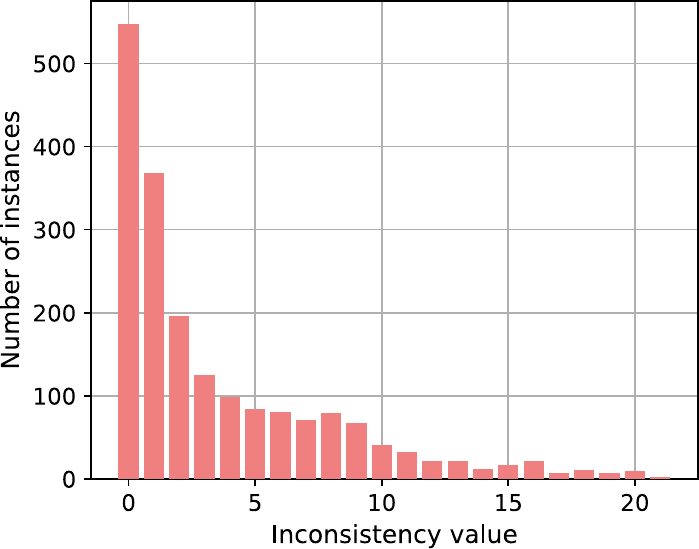}
        \caption{Contension inconsistency measure ($\icont$)}
    \end{subfigure}%
    \hfill%
    \begin{subfigure}[t]{.485\textwidth}
        \includegraphics[width=.98\textwidth]{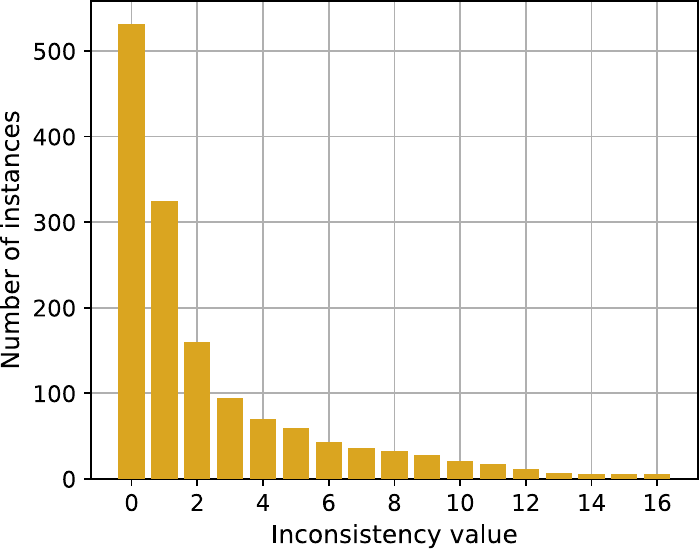}
        \caption{Forgetting-based inconsistency measure ($\iforget$)}
    \end{subfigure}\\[1ex]
    
    \begin{subfigure}[t]{.485\textwidth}
        \includegraphics[width=.98\textwidth]{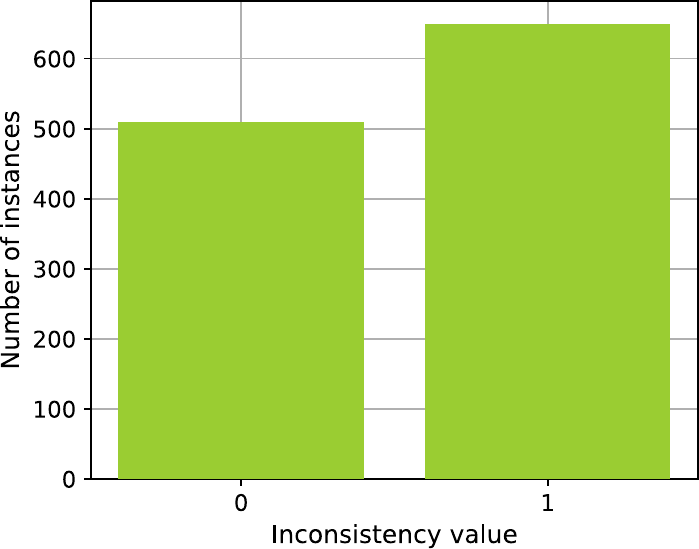}
        \caption{\raggedright Hitting Set inconsistency measure ($\ihs$); number of $\infty$ cases: $0$}
    \end{subfigure}%
    \hfill%
    \begin{subfigure}[t]{.485\textwidth}
        \includegraphics[width=.98\textwidth]{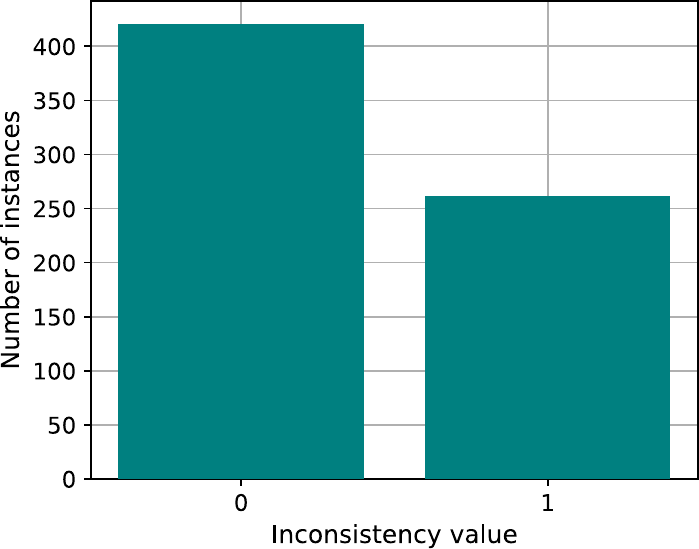}
        \caption{Max-distance inconsistency measure ($\imdalal$); number of $\infty$ cases: $0$}
    \end{subfigure}\\[1ex]
    
    \begin{subfigure}[t]{.485\textwidth}
        \includegraphics[width=.98\textwidth]{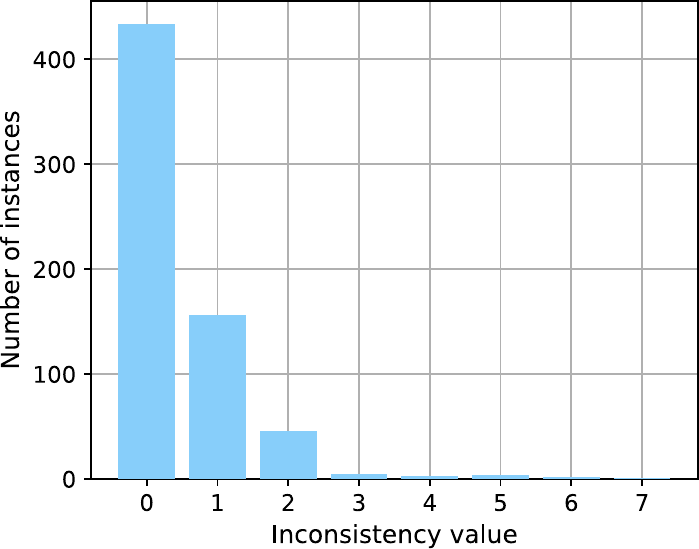}
        \caption{Sum-distance inconsistency measure ($\isdalal$); number of $\infty$ cases: $0$}
    \end{subfigure}%
    \hfill%
    \begin{subfigure}[t]{.485\textwidth}
        \includegraphics[width=.98\textwidth]{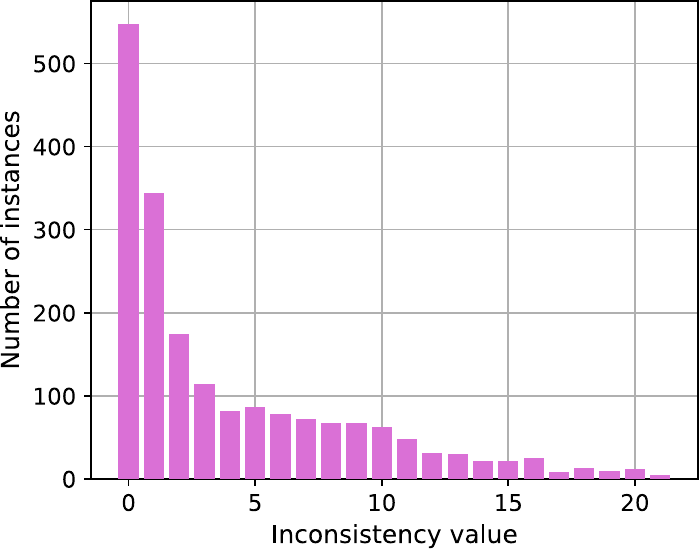}
        \caption{Hit-distance inconsistency measure ($\ihdalal$)}
    \end{subfigure}
    \caption{Histograms of the inconsistency values of the ML data set wrt.\ all six inconsistency measures
    % $\icont$, $\iforget$, $\imdalal$, and $\ihdalal$. 
    % Since no instances could be solved for $\isdalal$, no histogram could be compiled.
   }
    \label{fig:histos-ML}
\end{figure}

% ARG:
\begin{figure}
    \begin{subfigure}[t]{.485\textwidth}
        \includegraphics[width=.98\textwidth]{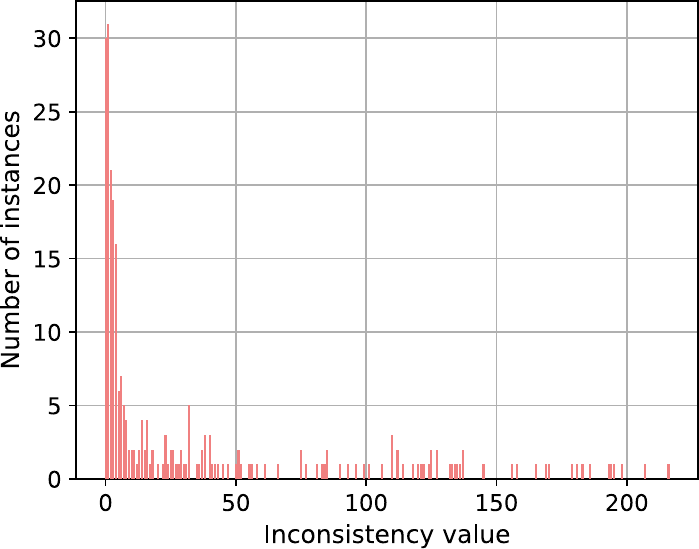}
        \caption{Contension inconsistency measure ($\icont$)}
    \end{subfigure}%
    \hfill%
    \begin{subfigure}[t]{.485\textwidth}
        \includegraphics[width=.98\textwidth]{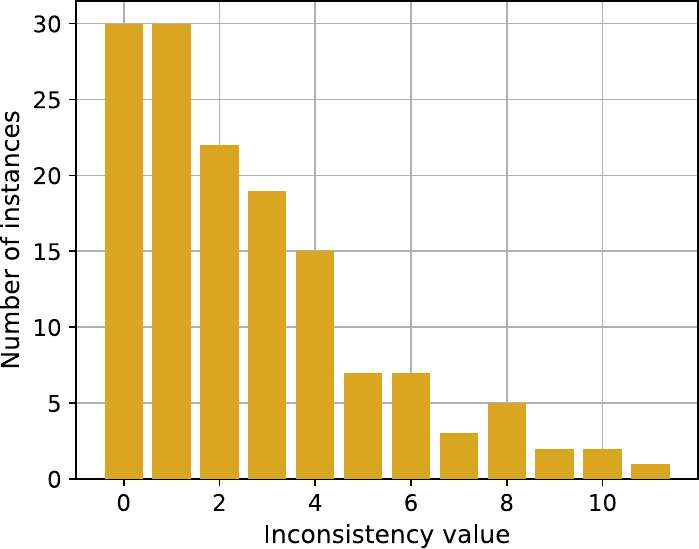}
        \caption{Forgetting-based inconsistency measure ($\iforget$)}
    \end{subfigure}\\[1ex]
    
    \begin{subfigure}[t]{.485\textwidth}
        \includegraphics[width=.98\textwidth]{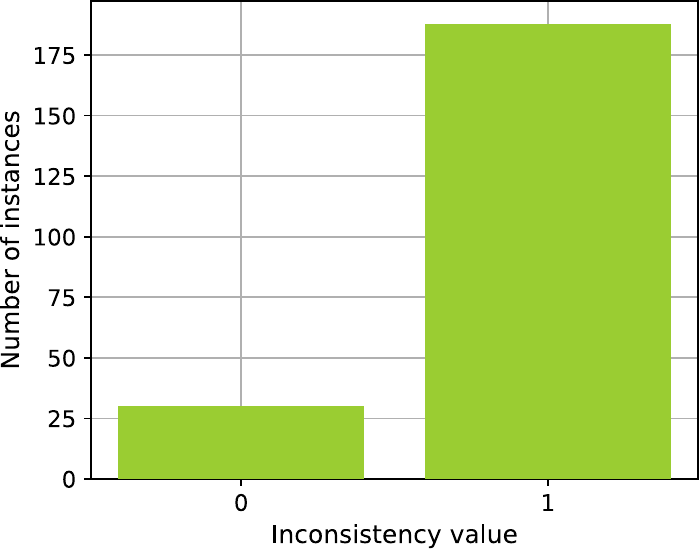}
        \caption{\raggedright Hitting Set inconsistency measure ($\ihs$); number of $\infty$ cases: $0$}
    \end{subfigure}%
    \hfill%
    \begin{subfigure}[t]{.485\textwidth}
        \includegraphics[width=.98\textwidth]{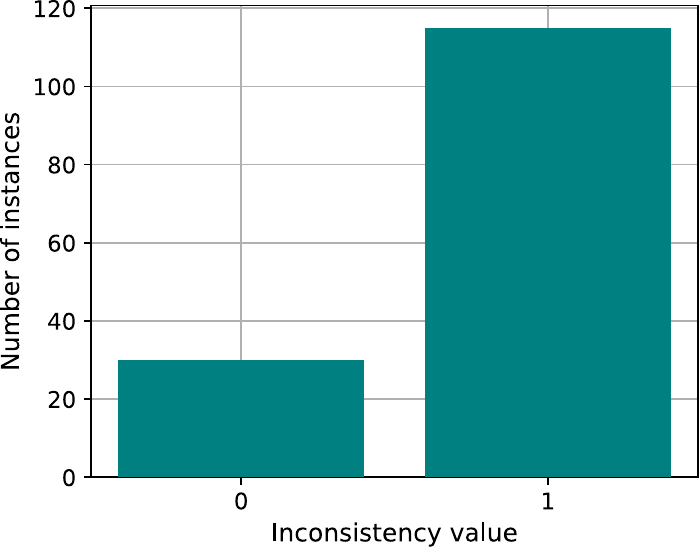}
        \caption{Max-distance inconsistency measure ($\imdalal$); number of $\infty$ cases: $0$}
    \end{subfigure}\\[1ex]
    
    \begin{subfigure}[t]{.485\textwidth}
        \includegraphics[width=.98\textwidth]{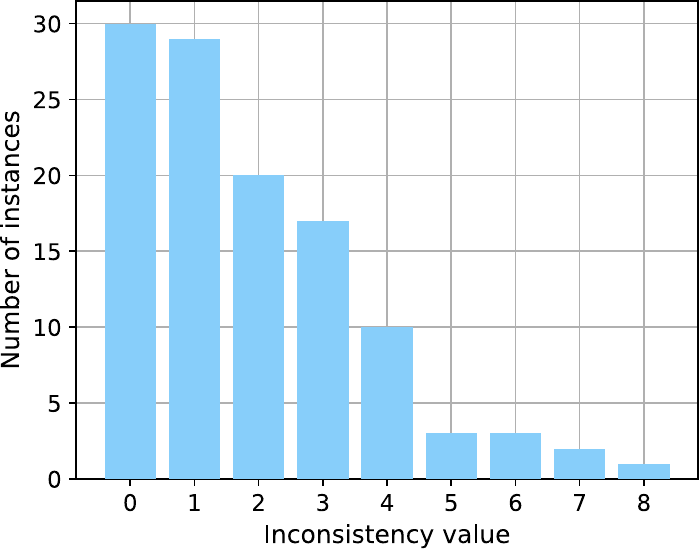}
        \caption{Sum-distance inconsistency measure ($\isdalal$); number of $\infty$ cases: $0$}
    \end{subfigure}%
    \hfill%
    \begin{subfigure}[t]{.485\textwidth}
        \includegraphics[width=.98\textwidth]{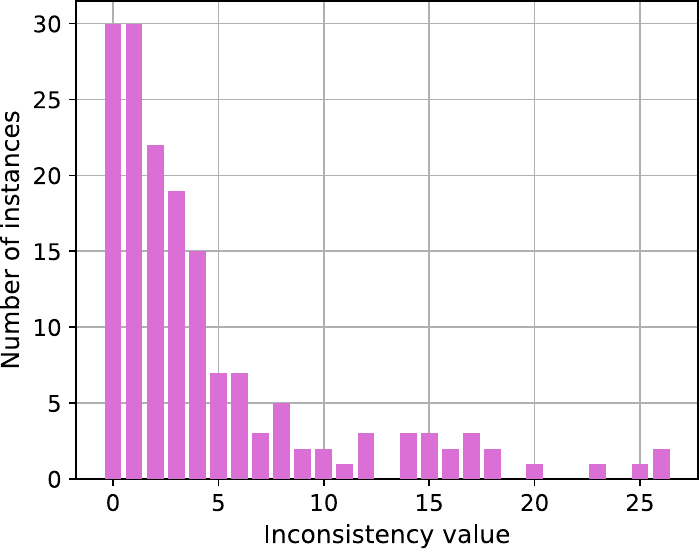}
        \caption{Hit-distance inconsistency measure ($\ihdalal$)}
    \end{subfigure}
    \caption{Histograms of the inconsistency values of the ARG data set wrt.\ all six inconsistency measures}
    \label{fig:histos-ARG}
\end{figure}

% SAT: 
\begin{figure}
    \centering
    \includegraphics[width=.5\textwidth]{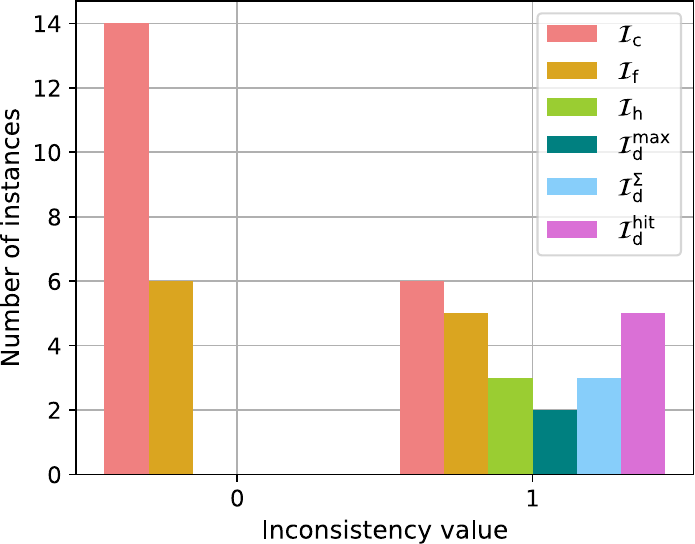}
    \caption{Histogram of the inconsistency values of the SC data set wrt.\ all six inconsistency measures.}
    \label{fig:histo-SAT}
\end{figure}

% ASP: 
\begin{figure}
    \centering
    \includegraphics[width=.5\textwidth]{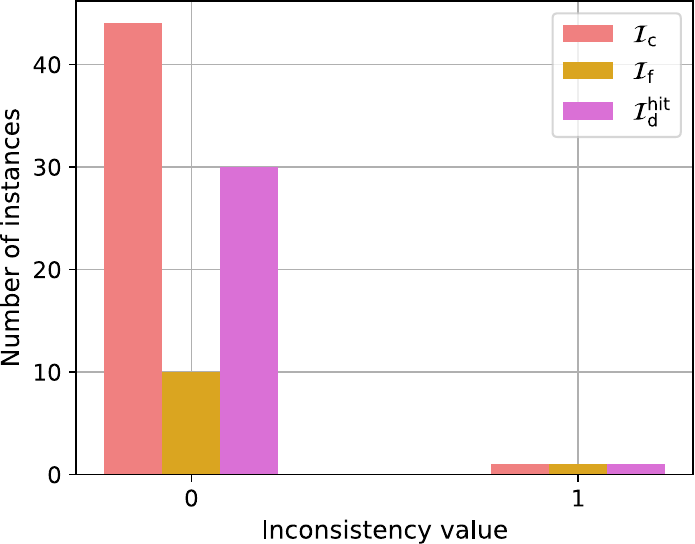}
    \caption{Histogram of the inconsistency values of the LP data set wrt.\ $\icont$, $\iforget$, and $\ihdalal$. Wrt.\ the other measures ($\ihs$, $\imdalal$, $\isdalal$), none of the instances from the ASP data set could be solved.}
    \label{fig:histo-ASP}
\end{figure}

\newpage

\subsection{Scatter Plots}\label{app:scatter-plots}

In this section we present scatter plots which compare each pair of approaches wrt.\ each inconsistency measure and each data set in terms of runtime per instance.
Figures \ref{fig:scatter-RAN-asp-sat}, \ref{fig:scatter-RAN-asp-naive}, and \ref{fig:scatter-RAN-sat-naive} refer to the SRS data set, Figures \ref{fig:scatter-ML-asp-sat}, \ref{fig:scatter-ML-asp-naive}, and \ref{fig:scatter-ML-sat-naive} refer to the ML data set, and Figures \ref{fig:scatter-ARG-asp-sat}, \ref{fig:scatter-ARG-asp-naive}, \ref{fig:scatter-ARG-sat-naive} refer to the ARG data set, and Figures \ref{fig:scatter-SAT-asp-sat}, \ref{fig:scatter-SAT-asp-naive}, \ref{fig:scatter-SAT-sat-naive} refer to the SC data set.
The scatter plots regarding the LP data set are omitted, since both the naive and the SAT-based approaches exclusively produced timeouts.

Furthermore, Figures \ref{fig:scatter-SAT-binary-vs-linear-1} and \ref{fig:scatter-SAT-binary-vs-linear-2} depict the scatter plots regarding the comparison between a linear search variant of the SAT approach and the other approaches, as described in Section \ref{sec:linear-search}.
Moreover, the scatter plots relating to the comparison of the MaxSAT approach for the contension inconsistency measure and the corresponding other approaches (see Section \ref{sec:maxsat}) are shown in Figure \ref{fig:scatter-maxsat}. 

% SRS:
\begin{figure}
    \begin{subfigure}[t]{.5\textwidth}
        \includegraphics[width=\textwidth]{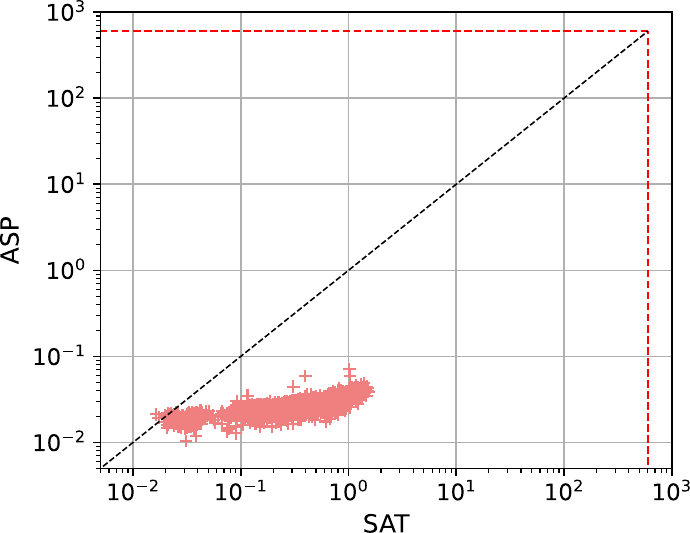}
        \caption{Contension inconsistency measure ($\icont$)}
    \end{subfigure}%
    \begin{subfigure}[t]{.5\textwidth}
        \includegraphics[width=\textwidth]{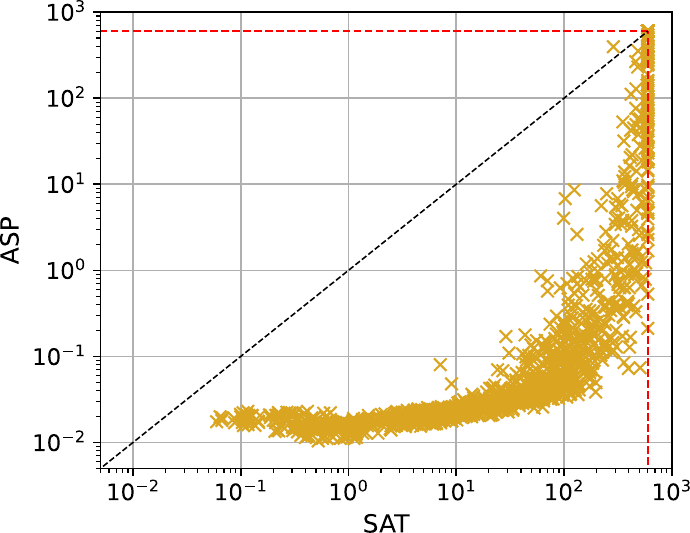}
        \caption{Forgetting-based inconsistency measure ($\iforget$)}
    \end{subfigure}\\[1ex]
    
    \begin{subfigure}[t]{.5\textwidth}
        \includegraphics[width=\textwidth]{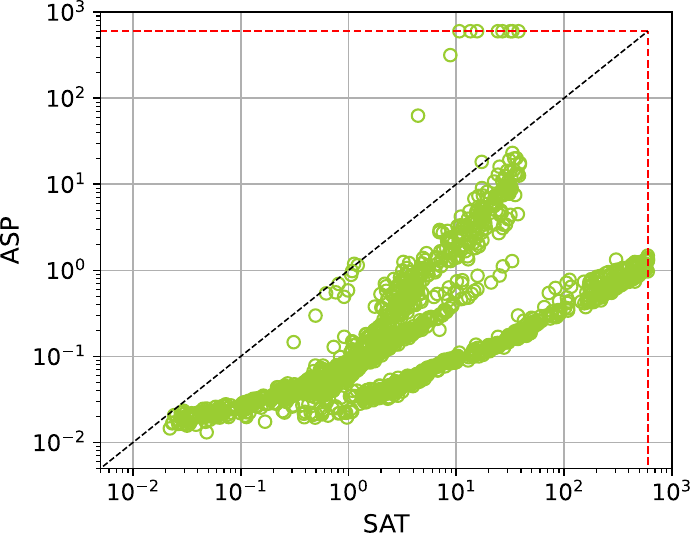}
        \caption{Hitting Set inconsistency measure ($\ihs$)}
    \end{subfigure}%
    \begin{subfigure}[t]{.5\textwidth}
        \includegraphics[width=\textwidth]{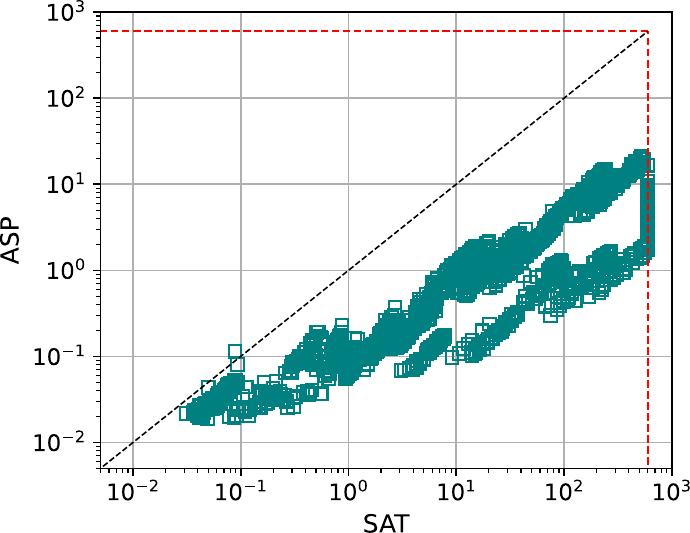}
        \caption{Max-distance inconsistency measure ($\imdalal$)}
    \end{subfigure}\\[1ex]
    
    \begin{subfigure}[t]{.5\textwidth}
        \includegraphics[width=\textwidth]{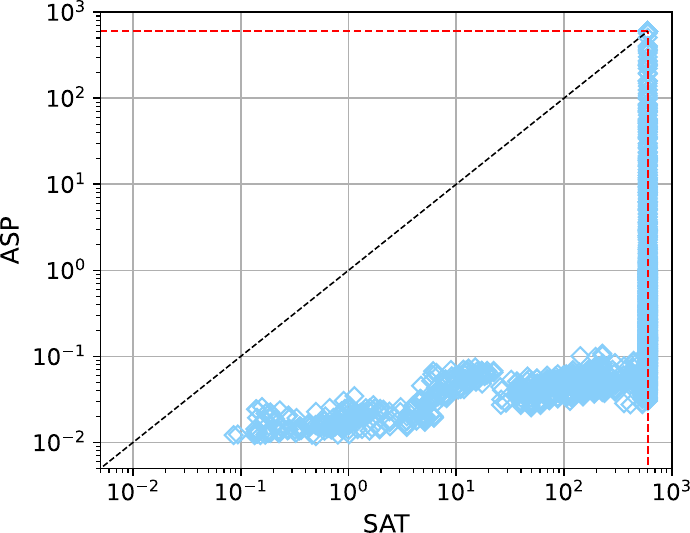}
        \caption{Sum-distance inconsistency measure ($\isdalal$)}
    \end{subfigure}%
    \begin{subfigure}[t]{.5\textwidth}
        \includegraphics[width=\textwidth]{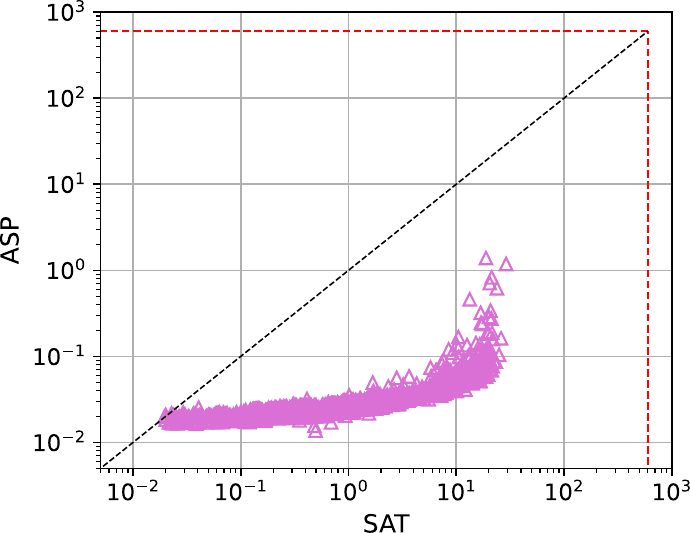}
        \caption{Hit-distance inconsistency measure ($\ihdalal$)}
    \end{subfigure}
    \caption{Runtime comparison of the ASP-based and SAT-based approaches on the SRS data set. Timeout: $10$ minutes.}
    \label{fig:scatter-RAN-asp-sat}
\end{figure}

\begin{figure}
    \begin{subfigure}[t]{.5\textwidth}
        \includegraphics[width=\textwidth]{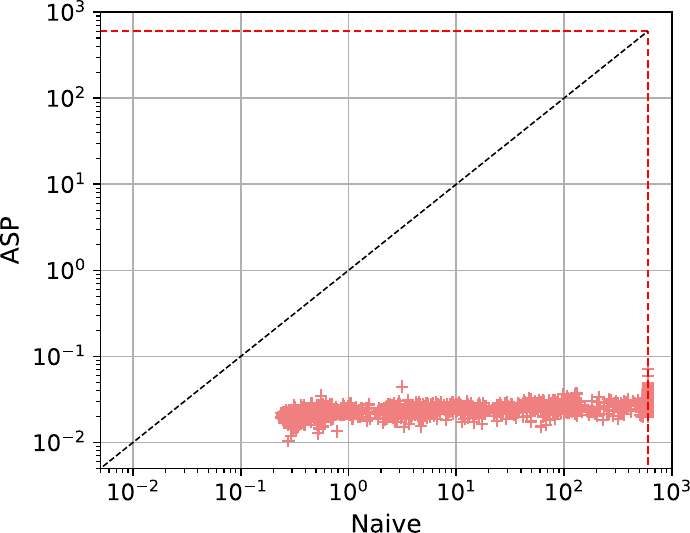}
        \caption{Contension inconsistency measure ($\icont$)}
    \end{subfigure}%
    \begin{subfigure}[t]{.5\textwidth}
        \includegraphics[width=\textwidth]{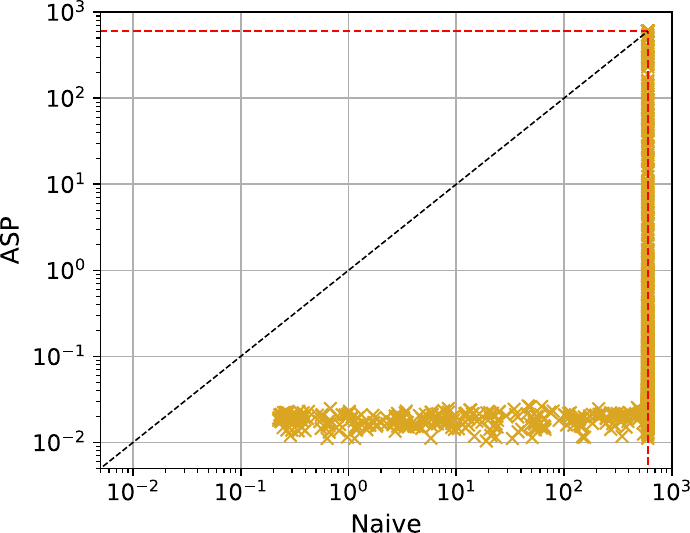}
        \caption{Forgetting-based inconsistency measure ($\iforget$)}
    \end{subfigure}\\[1ex]
    
    \begin{subfigure}[t]{.5\textwidth}
        \includegraphics[width=\textwidth]{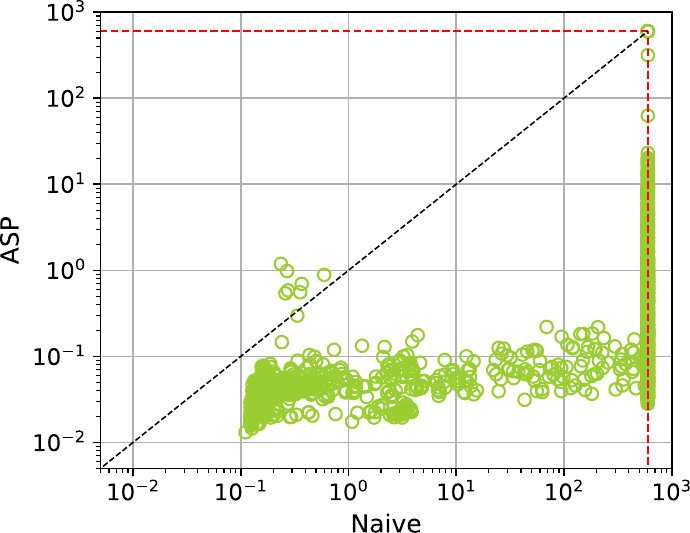}
        \caption{Hitting Set inconsistency measure ($\ihs$)}
    \end{subfigure}%
    \begin{subfigure}[t]{.5\textwidth}
        \includegraphics[width=\textwidth]{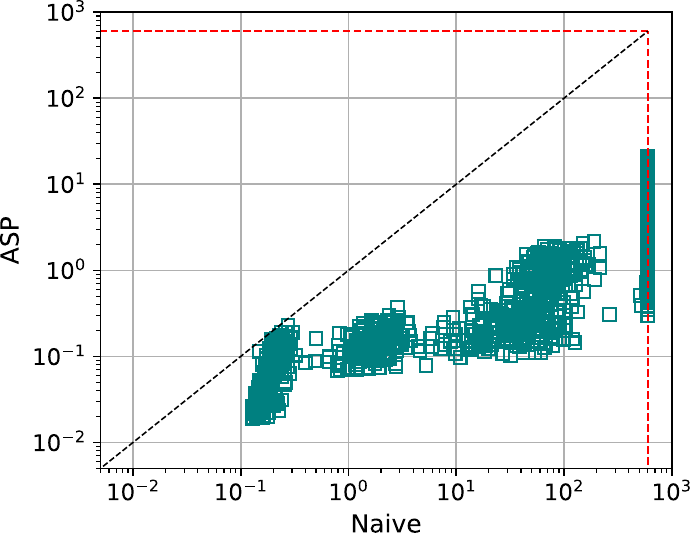}
        \caption{Max-distance inconsistency measure ($\imdalal$)}
    \end{subfigure}\\[1ex]
    
    \begin{subfigure}[t]{.5\textwidth}
        \includegraphics[width=\textwidth]{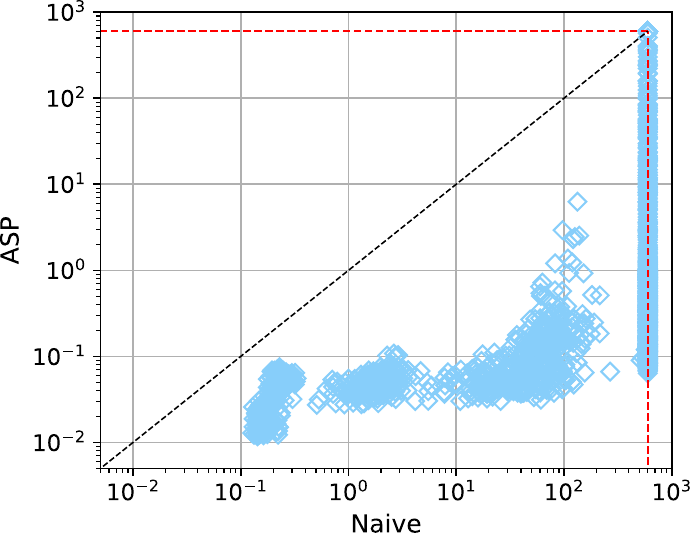}
        \caption{Sum-distance inconsistency measure ($\isdalal$)}
    \end{subfigure}%
    \begin{subfigure}[t]{.5\textwidth}
        \includegraphics[width=\textwidth]{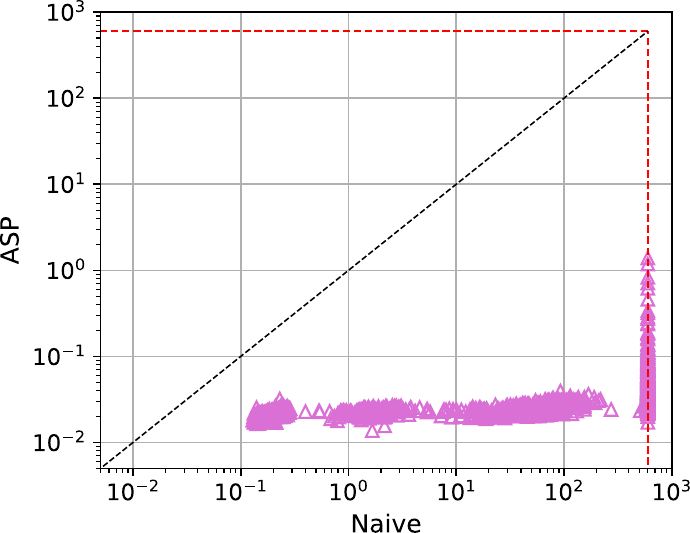}
        \caption{Hit-distance inconsistency measure ($\ihdalal$)}
    \end{subfigure}
    \caption{Runtime comparison of the ASP-based and naive approaches on the SRS data set. Timeout: $10$ minutes.}
    \label{fig:scatter-RAN-asp-naive}
\end{figure}

\begin{figure}
    \begin{subfigure}[t]{.5\textwidth}
        \includegraphics[width=\textwidth]{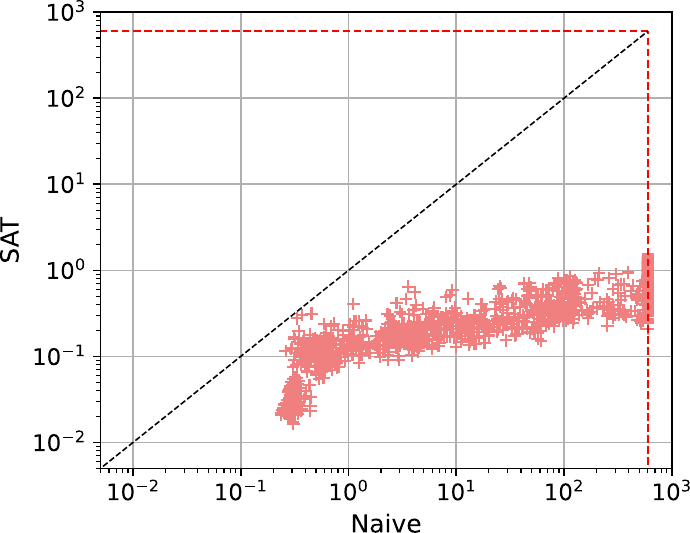}
        \caption{Contension inconsistency measure ($\icont$)}
    \end{subfigure}%
    \begin{subfigure}[t]{.5\textwidth}
        \includegraphics[width=\textwidth]{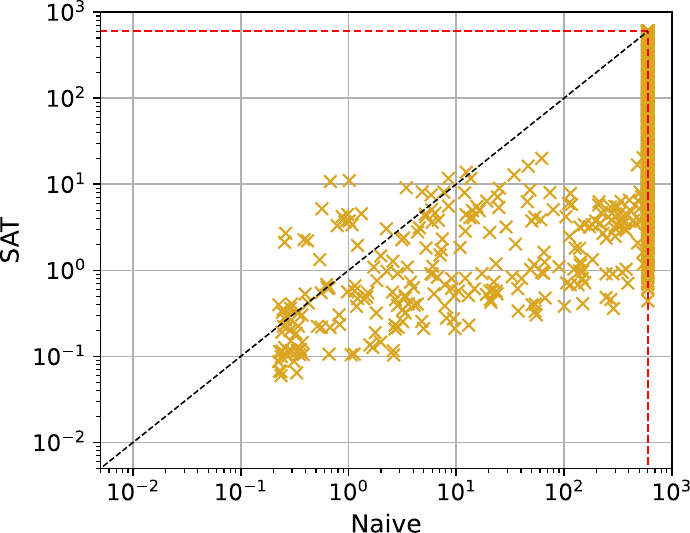}
        \caption{Forgetting-based inconsistency measure ($\iforget$)}
    \end{subfigure}\\[1ex]
    
    \begin{subfigure}[t]{.5\textwidth}
        \includegraphics[width=\textwidth]{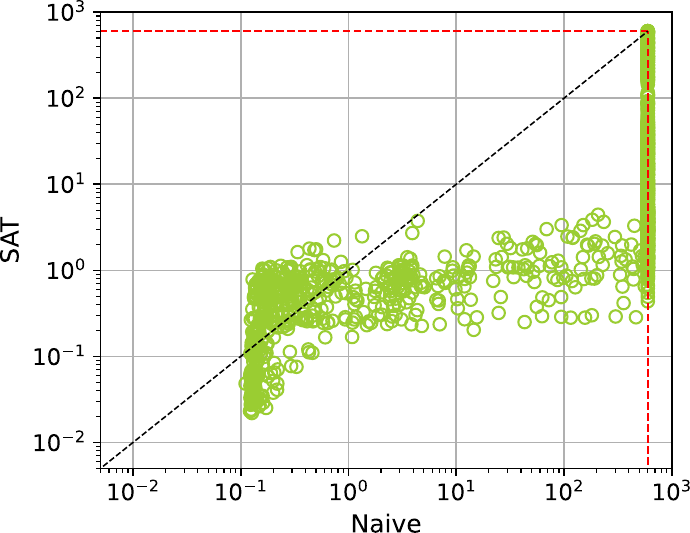}
        \caption{Hitting Set inconsistency measure ($\ihs$)}
    \end{subfigure}%
    \begin{subfigure}[t]{.5\textwidth}
        \includegraphics[width=\textwidth]{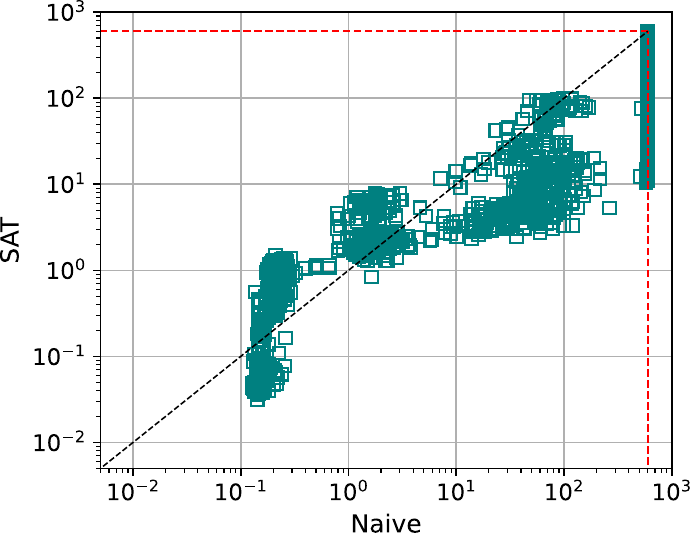}
        \caption{Max-distance inconsistency measure ($\imdalal$)}
    \end{subfigure}\\[1ex]
    
    \begin{subfigure}[t]{.5\textwidth}
        \includegraphics[width=\textwidth]{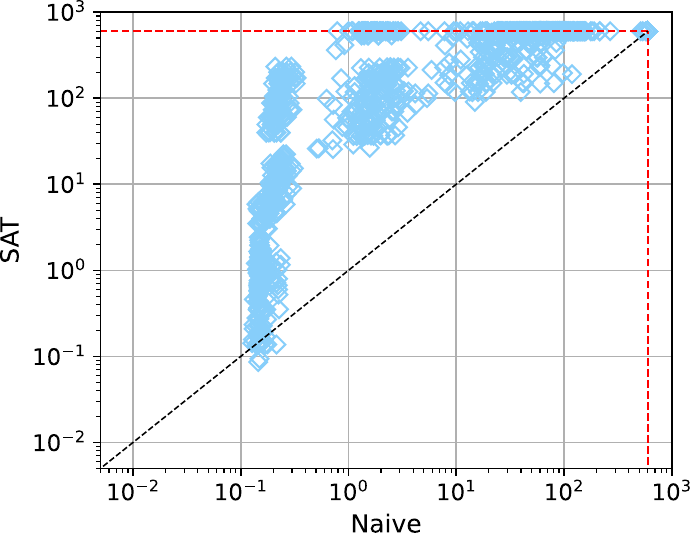}
        \caption{Sum-distance inconsistency measure ($\isdalal$)}
    \end{subfigure}%
    \begin{subfigure}[t]{.5\textwidth}
        \includegraphics[width=\textwidth]{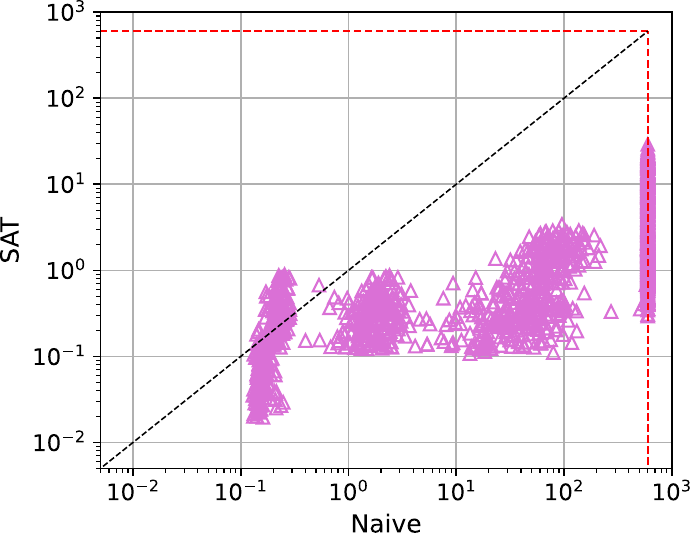}
        \caption{Hit-distance inconsistency measure ($\ihdalal$)}
    \end{subfigure}
    \caption{Runtime comparison of the SAT-based and naive approaches on the SRS data set. Timeout: $10$ minutes.}
    \label{fig:scatter-RAN-sat-naive}
\end{figure}

% ML
\begin{figure}
    \begin{subfigure}[t]{.5\textwidth}
        \includegraphics[width=\textwidth]{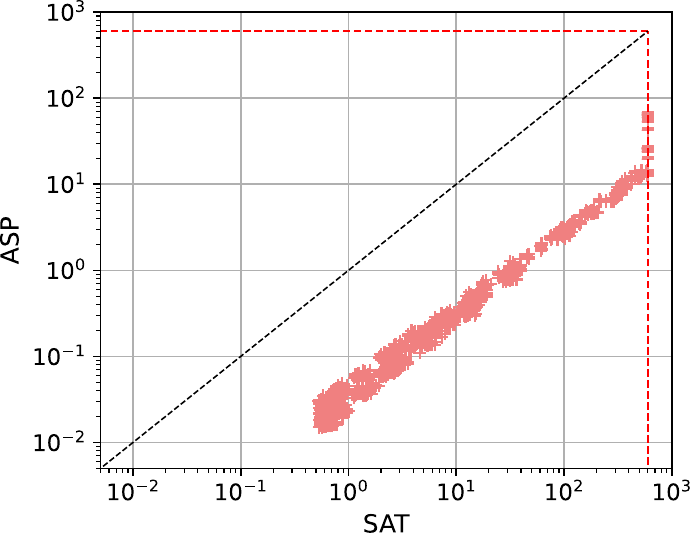}
        \caption{Contension inconsistency measure ($\icont$)}
    \end{subfigure}%
    \begin{subfigure}[t]{.5\textwidth}
        \includegraphics[width=\textwidth]{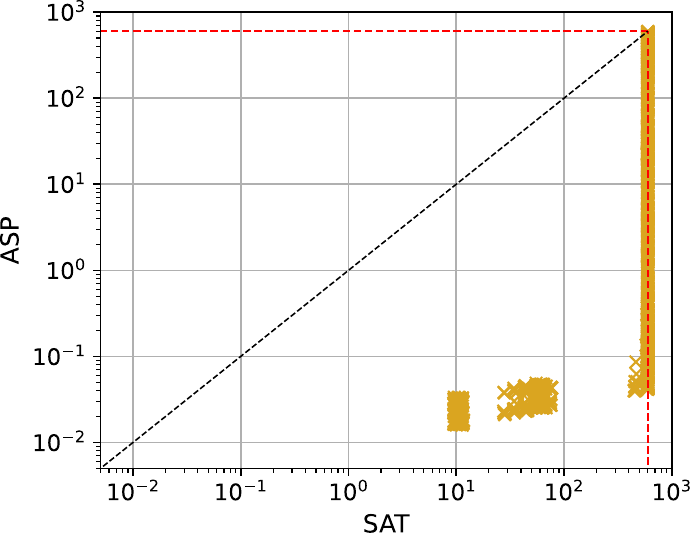}
        \caption{Forgetting-based inconsistency measure ($\iforget$)}
    \end{subfigure}\\[1ex]
    
    \begin{subfigure}[t]{.5\textwidth}
        \includegraphics[width=\textwidth]{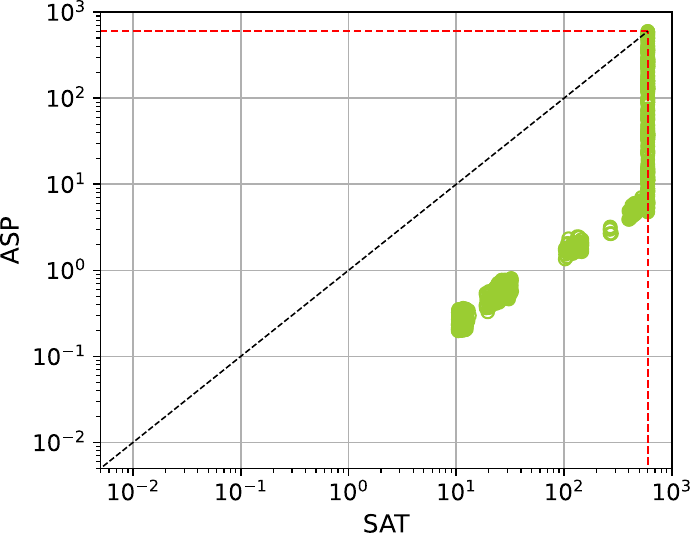}
        \caption{Hitting Set inconsistency measure ($\ihs$)}
    \end{subfigure}%
    \begin{subfigure}[t]{.5\textwidth}
        \includegraphics[width=\textwidth]{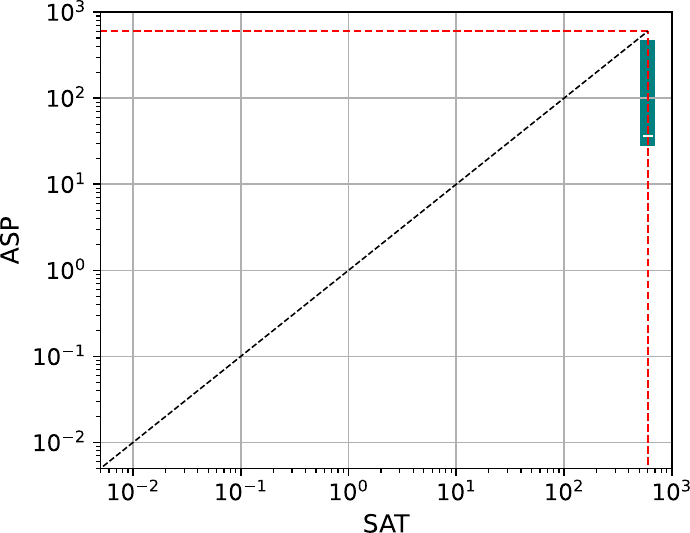}
        \caption{Max-distance inconsistency measure ($\imdalal$)}
    \end{subfigure}\\[1ex]
    
    \begin{subfigure}[t]{.5\textwidth}
        \includegraphics[width=\textwidth]{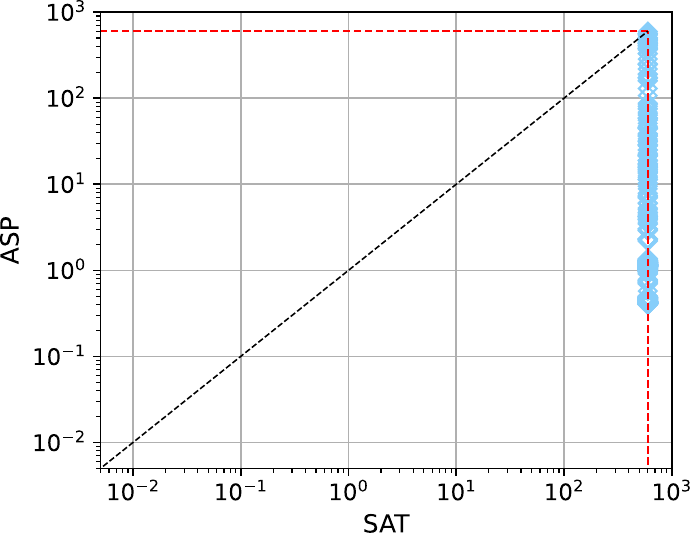}
        \caption{Sum-distance inconsistency measure ($\isdalal$)}
    \end{subfigure}%
    \begin{subfigure}[t]{.5\textwidth}
        \includegraphics[width=\textwidth]{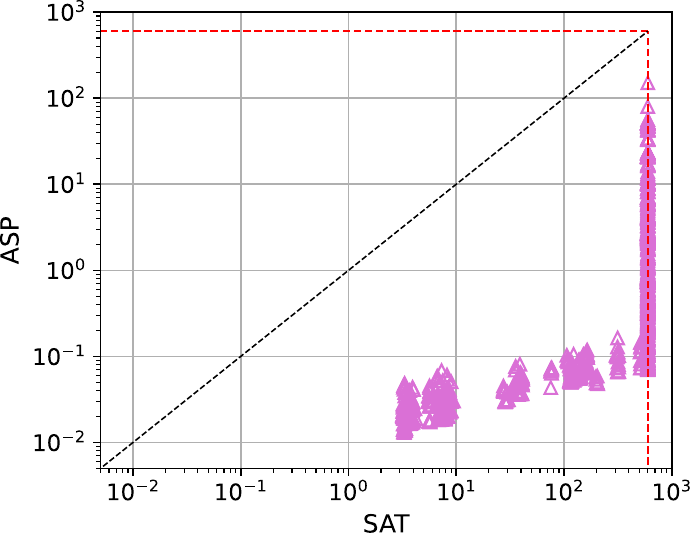}
        \caption{Hit-distance inconsistency measure ($\ihdalal$)}
    \end{subfigure}
    \caption{Runtime comparison of the ASP-based and SAT-based approaches on the ML data set. Timeout: $10$ minutes.}
    \label{fig:scatter-ML-asp-sat}
\end{figure}

\begin{figure}
    \begin{subfigure}[t]{.5\textwidth}
        \includegraphics[width=\textwidth]{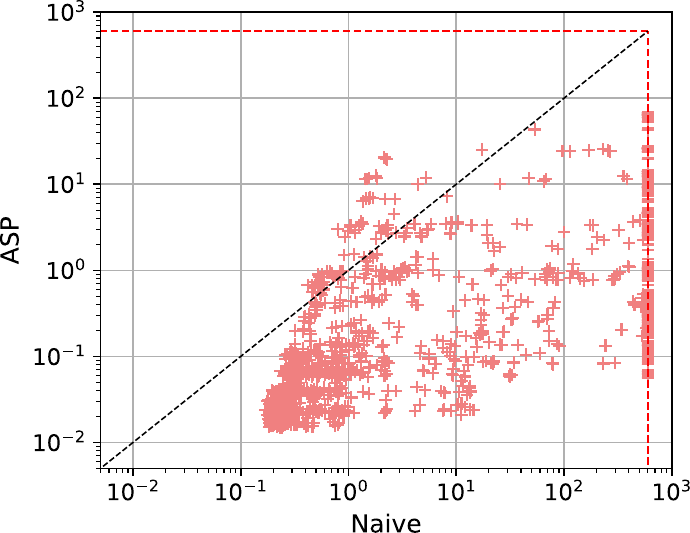}
        \caption{Contension inconsistency measure ($\icont$)}
    \end{subfigure}%
    \begin{subfigure}[t]{.5\textwidth}
        \includegraphics[width=\textwidth]{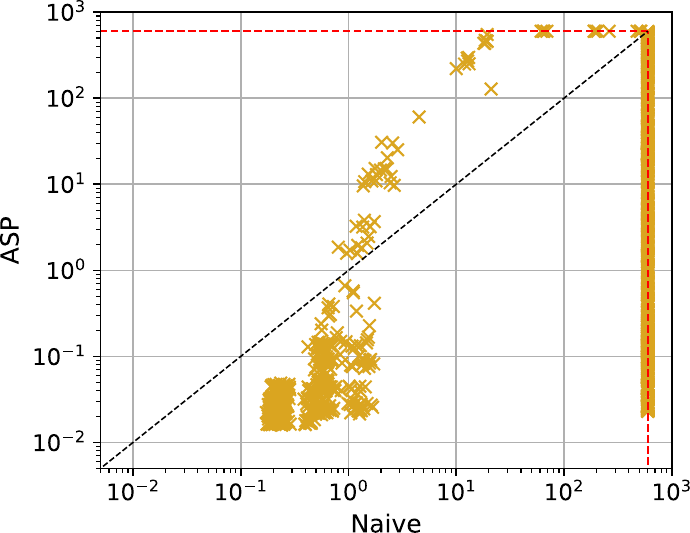}
        \caption{Forgetting-based inconsistency measure ($\iforget$)}
    \end{subfigure}\\[1ex]
    
    \begin{subfigure}[t]{.5\textwidth}
        \includegraphics[width=\textwidth]{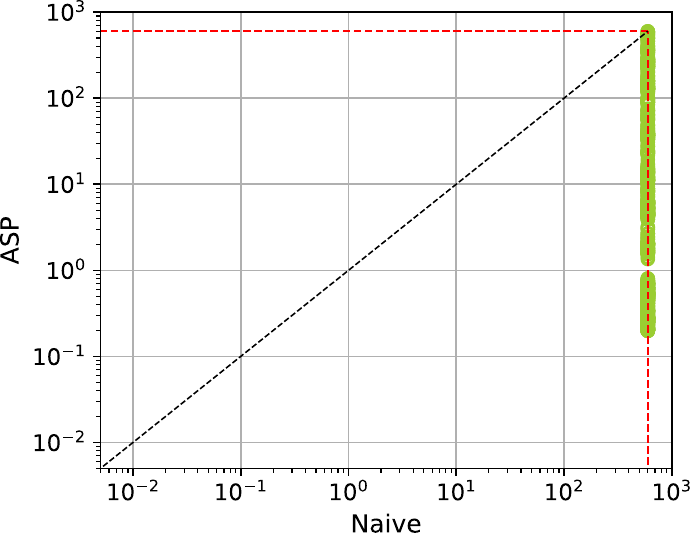}
        \caption{Hitting Set inconsistency measure ($\ihs$)}
    \end{subfigure}%
    \begin{subfigure}[t]{.5\textwidth}
        \includegraphics[width=\textwidth]{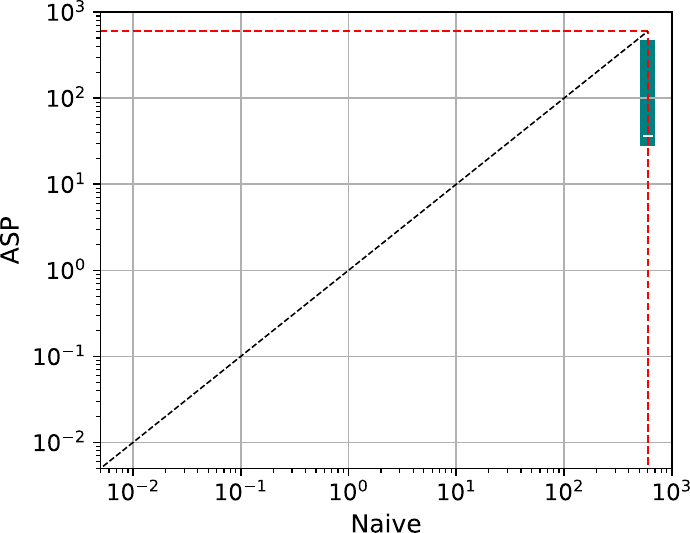}
        \caption{Max-distance inconsistency measure ($\imdalal$)}
    \end{subfigure}\\[1ex]
    
    \begin{subfigure}[t]{.5\textwidth}
        \includegraphics[width=\textwidth]{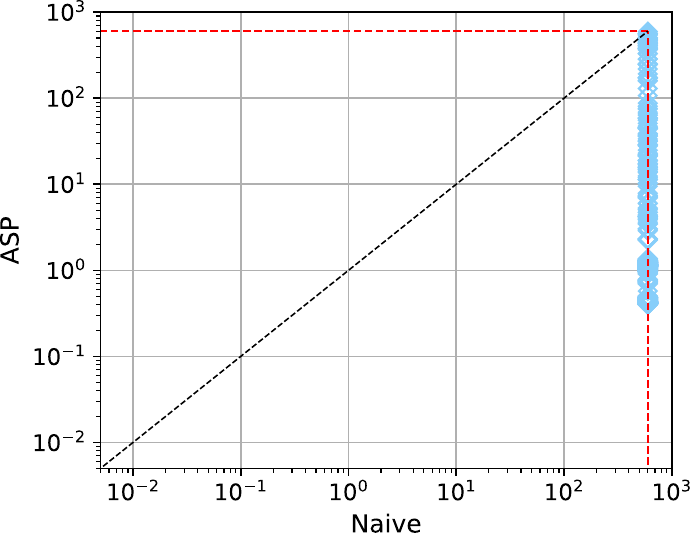}
        \caption{Sum-distance inconsistency measure ($\isdalal$)}
    \end{subfigure}%
    \begin{subfigure}[t]{.5\textwidth}
        \includegraphics[width=\textwidth]{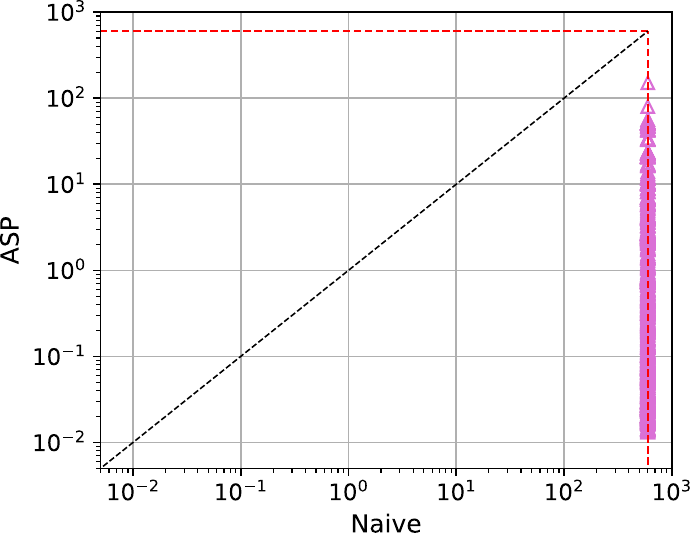}
        \caption{Hit-distance inconsistency measure ($\ihdalal$)}
    \end{subfigure}
    \caption{Runtime comparison of the ASP-based and naive approaches on the ML data set. Timeout: $10$ minutes.}
    \label{fig:scatter-ML-asp-naive}
\end{figure}

\begin{figure}
    \begin{subfigure}[t]{.5\textwidth}
        \includegraphics[width=\textwidth]{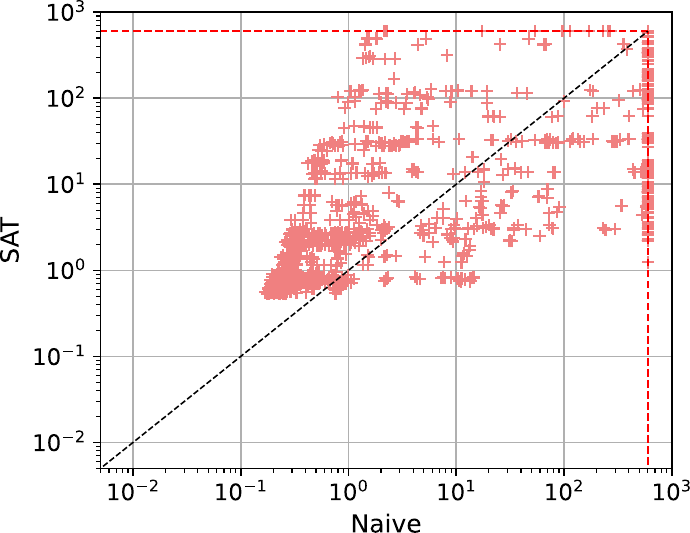}
        \caption{Contension inconsistency measure ($\icont$)}
    \end{subfigure}%
    \begin{subfigure}[t]{.5\textwidth}
        \includegraphics[width=\textwidth]{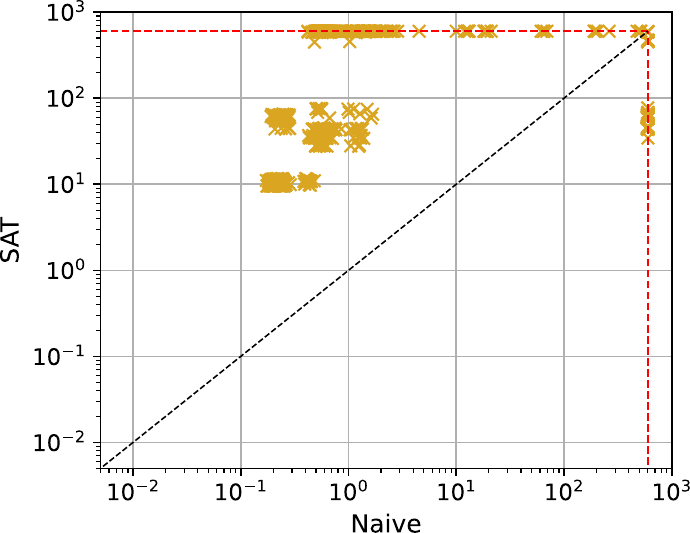}
        \caption{Forgetting-based inconsistency measure ($\iforget$)}
    \end{subfigure}\\[1ex]
    
    \begin{subfigure}[t]{.5\textwidth}
        \includegraphics[width=\textwidth]{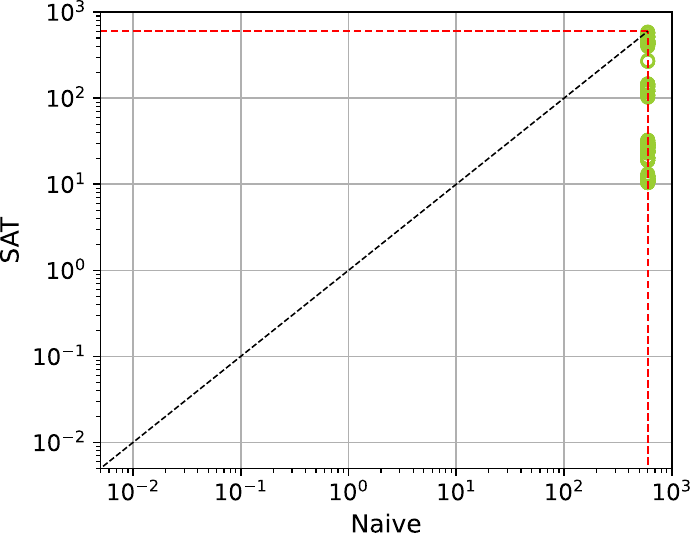}
        \caption{Hitting Set inconsistency measure ($\ihs$)}
    \end{subfigure}%
    \begin{subfigure}[t]{.5\textwidth}
        \includegraphics[width=\textwidth]{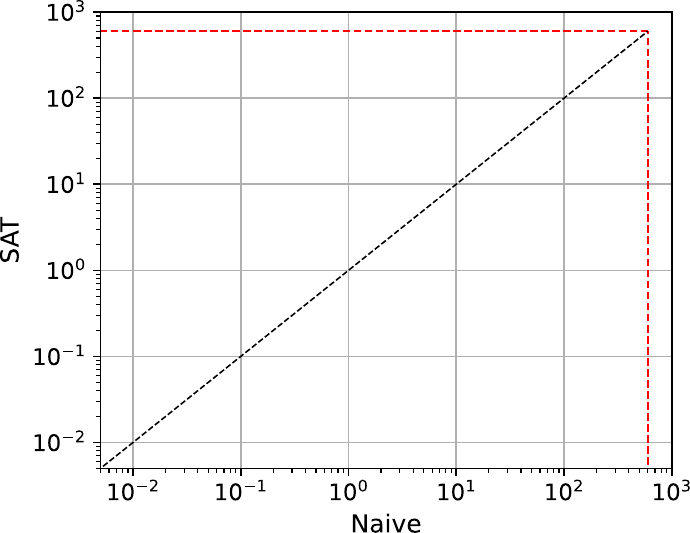}
        \caption{Max-distance inconsistency measure ($\imdalal$)}
    \end{subfigure}\\[1ex]
    
    \begin{subfigure}[t]{.5\textwidth}
        \includegraphics[width=\textwidth]{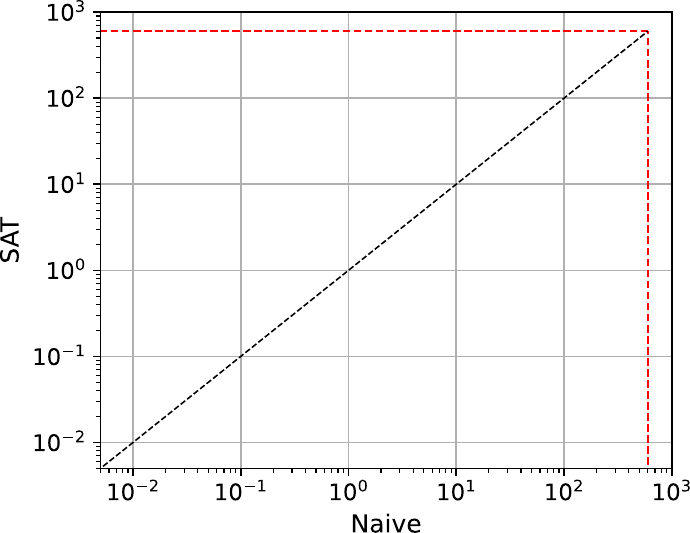}
        \caption{Sum-distance inconsistency measure ($\isdalal$)}
    \end{subfigure}%
    \begin{subfigure}[t]{.5\textwidth}
        \includegraphics[width=\textwidth]{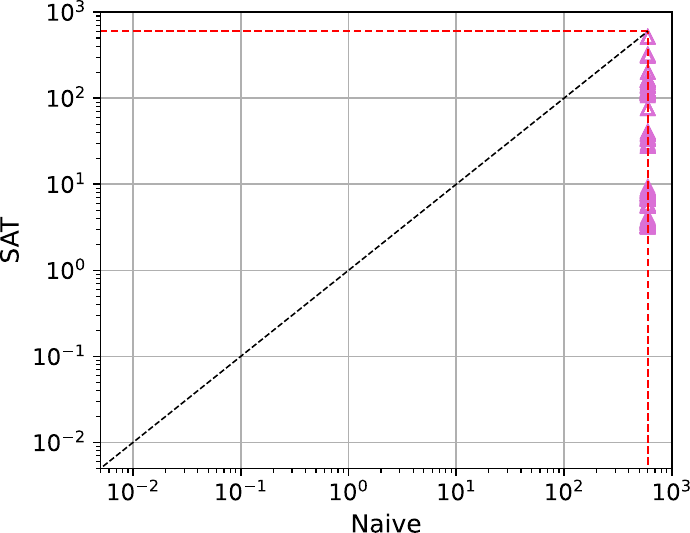}
        \caption{Hit-distance inconsistency measure ($\ihdalal$)}
    \end{subfigure}
    \caption{Runtime comparison of the SAT-based and naive approaches on the ML data set. Timeout: $10$ minutes.}
    \label{fig:scatter-ML-sat-naive}
\end{figure}

% ARG
\begin{figure}
    \begin{subfigure}[t]{.5\textwidth}
        \includegraphics[width=\textwidth]{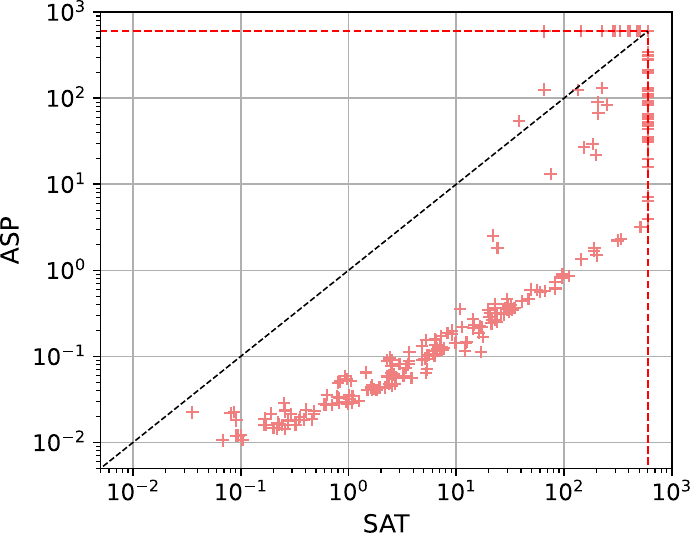}
        \caption{Contension inconsistency measure ($\icont$)}
    \end{subfigure}%
    \begin{subfigure}[t]{.5\textwidth}
        \includegraphics[width=\textwidth]{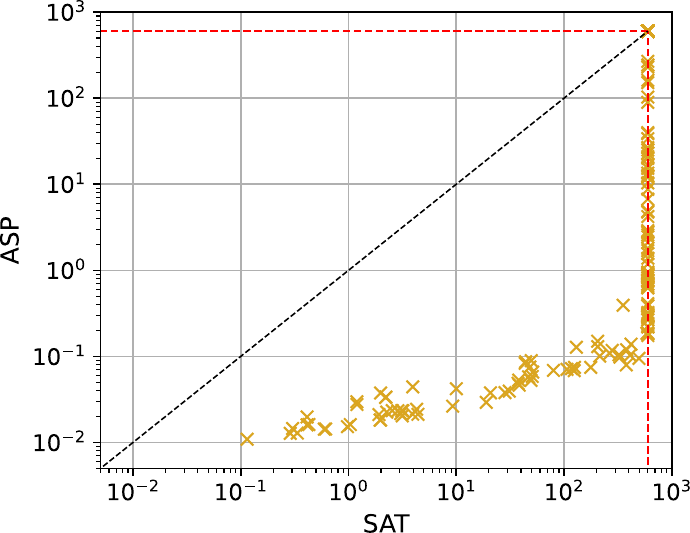}
        \caption{Forgetting-based inconsistency measure ($\iforget$)}
    \end{subfigure}\\[1ex]
    
    \begin{subfigure}[t]{.5\textwidth}
        \includegraphics[width=\textwidth]{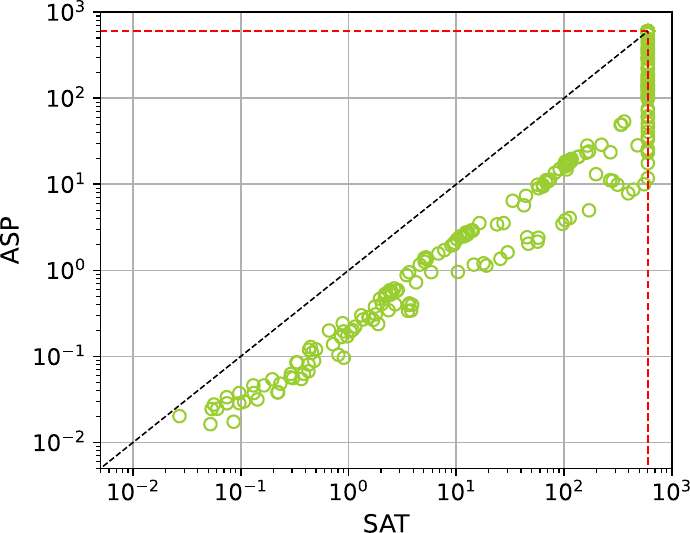}
        \caption{Hitting Set inconsistency measure ($\ihs$)}
        \label{fig:scatter-ARG-asp-sat-hs}
    \end{subfigure}%
    \begin{subfigure}[t]{.5\textwidth}
        \includegraphics[width=\textwidth]{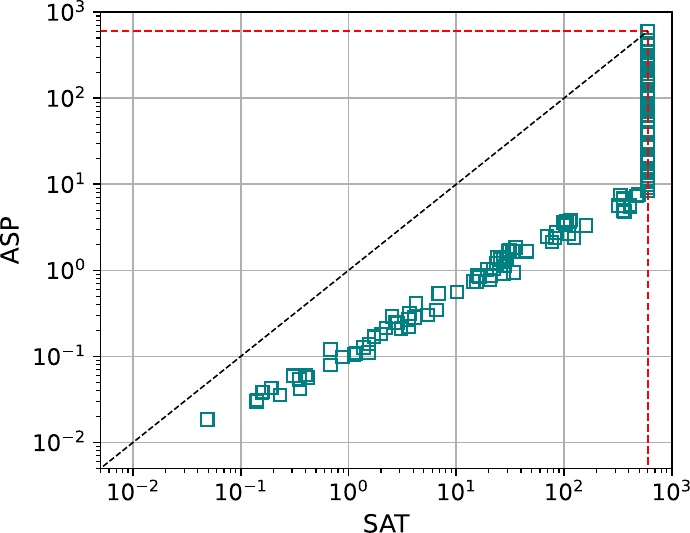}
        \caption{Max-distance inconsistency measure ($\imdalal$)}
    \end{subfigure}\\[1ex]
    
    \begin{subfigure}[t]{.5\textwidth}
        \includegraphics[width=\textwidth]{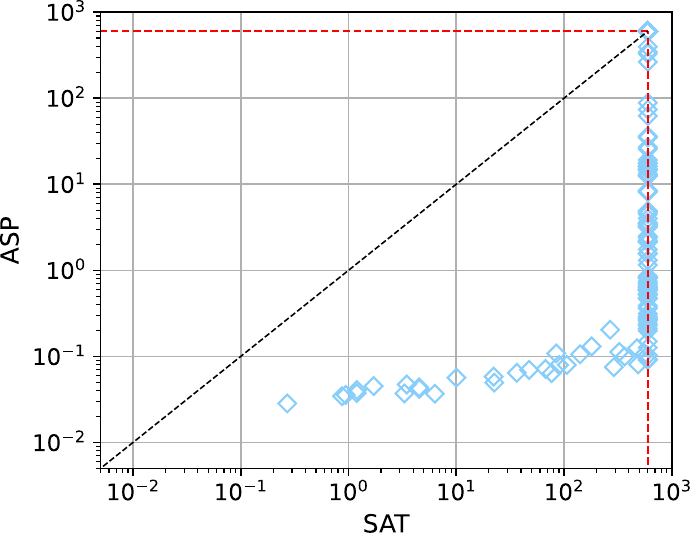}
        \caption{Sum-distance inconsistency measure ($\isdalal$)}
    \end{subfigure}%
    \begin{subfigure}[t]{.5\textwidth}
        \includegraphics[width=\textwidth]{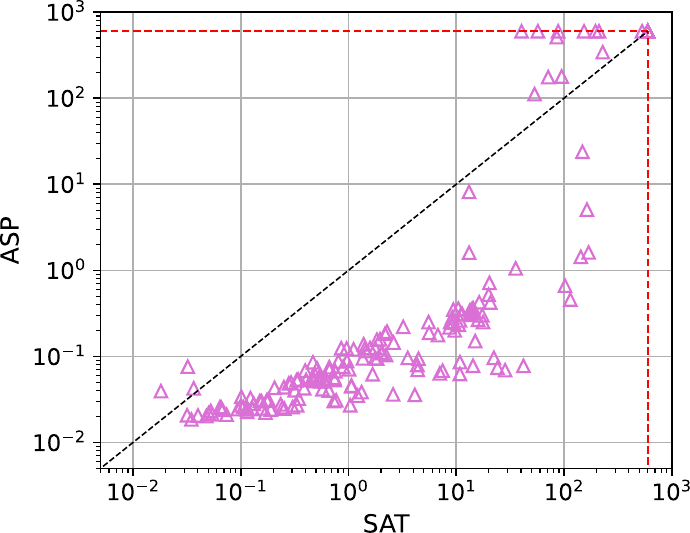}
        \caption{Hit-distance inconsistency measure ($\ihdalal$)}
        \label{fig:scatter-ARG-asp-sat-hitdalal}
    \end{subfigure}
    \caption{Runtime comparison of the ASP-based and SAT-based approaches on the ARG data set. Timeout: $10$ minutes.}
    \label{fig:scatter-ARG-asp-sat}
\end{figure}

\begin{figure}
    \begin{subfigure}[t]{.5\textwidth}
        \includegraphics[width=\textwidth]{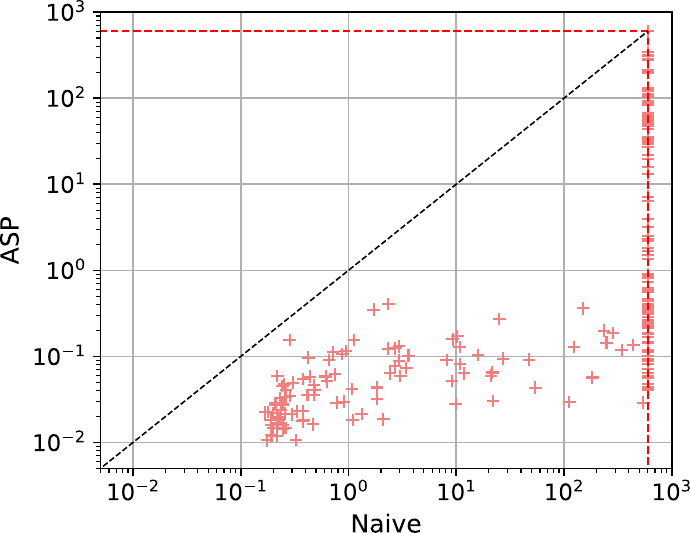}
        \caption{Contension inconsistency measure ($\icont$)}
    \end{subfigure}%
    \begin{subfigure}[t]{.5\textwidth}
        \includegraphics[width=\textwidth]{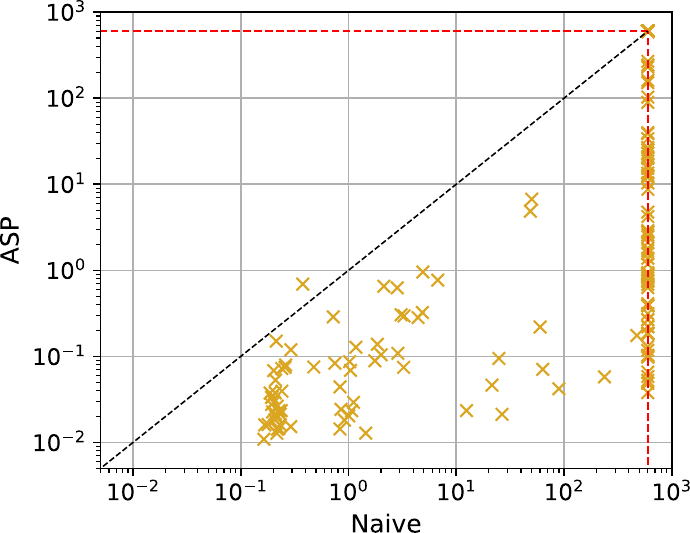}
        \caption{Forgetting-based inconsistency measure ($\iforget$)}
    \end{subfigure}\\[1ex]
    
    \begin{subfigure}[t]{.5\textwidth}
        \includegraphics[width=\textwidth]{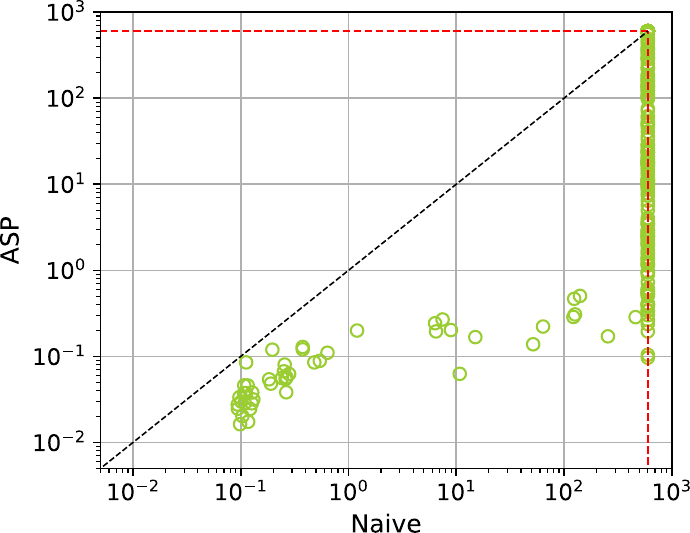}
        \caption{Hitting Set inconsistency measure ($\ihs$)}
    \end{subfigure}%
    \begin{subfigure}[t]{.5\textwidth}
        \includegraphics[width=\textwidth]{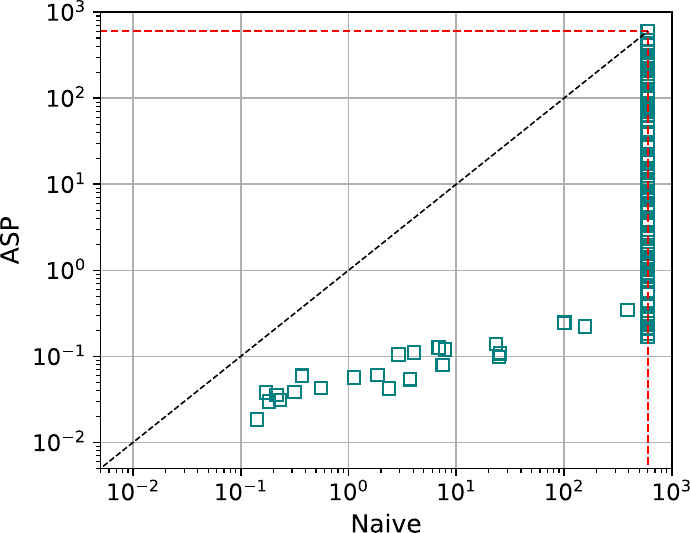}
        \caption{Max-distance inconsistency measure ($\imdalal$)}
    \end{subfigure}\\[1ex]
    
    \begin{subfigure}[t]{.5\textwidth}
        \includegraphics[width=\textwidth]{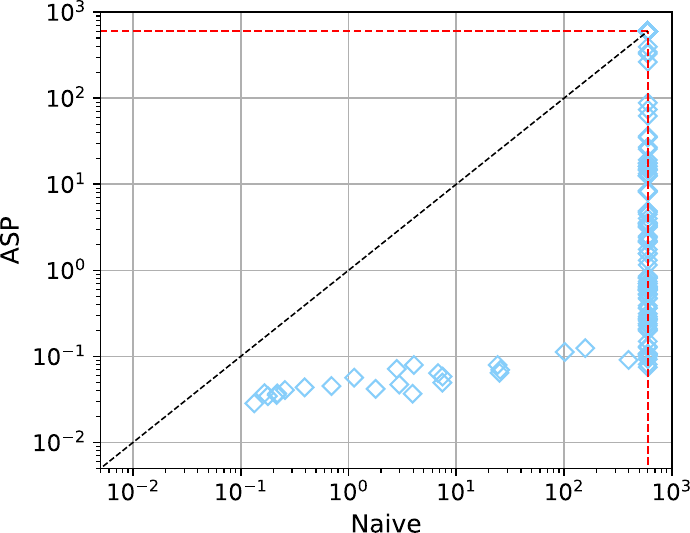}
        \caption{Sum-distance inconsistency measure ($\isdalal$)}
    \end{subfigure}%
    \begin{subfigure}[t]{.5\textwidth}
        \includegraphics[width=\textwidth]{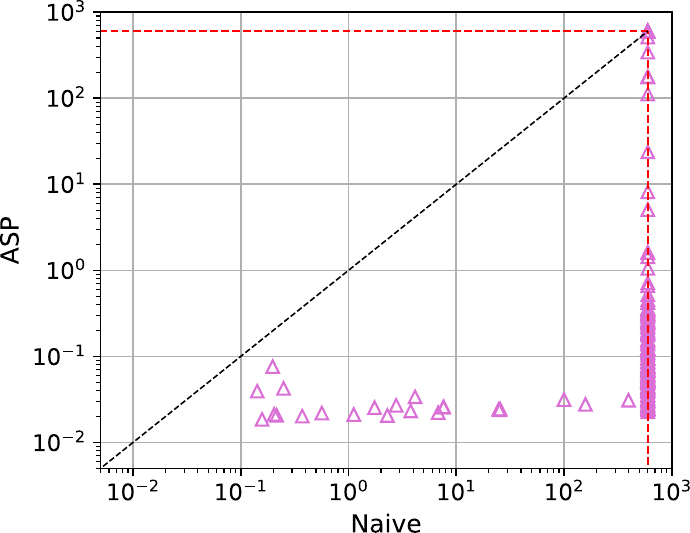}
        \caption{Hit-distance inconsistency measure ($\ihdalal$)}
    \end{subfigure}
    \caption{Runtime comparison of the ASP-based and naive approaches on the ARG data set. Timeout: $10$ minutes.}
    \label{fig:scatter-ARG-asp-naive}
\end{figure}

\begin{figure}
    \begin{subfigure}[t]{.5\textwidth}
        \includegraphics[width=\textwidth]{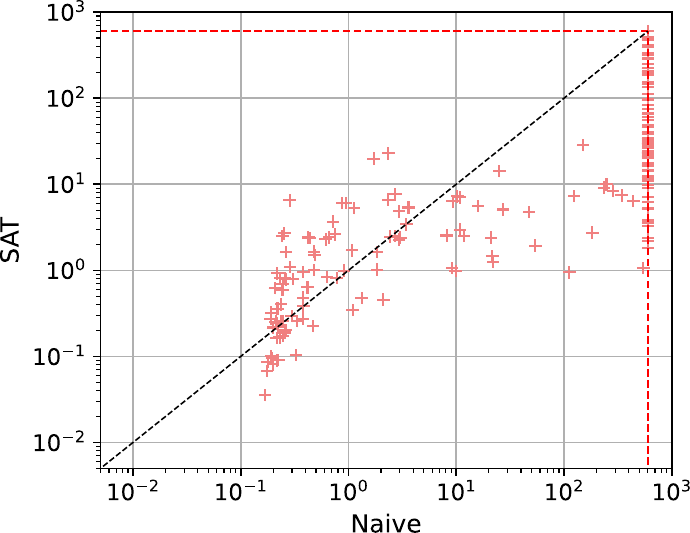}
        \caption{Contension inconsistency measure ($\icont$)}
    \end{subfigure}%
    \begin{subfigure}[t]{.5\textwidth}
        \includegraphics[width=\textwidth]{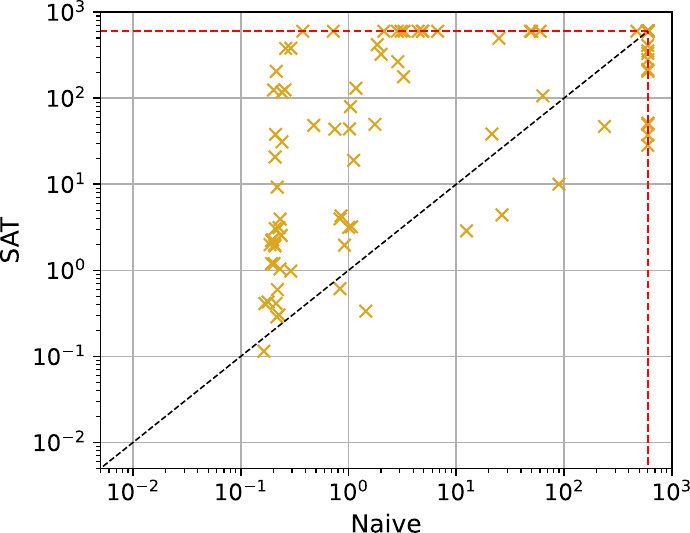}
        \caption{Forgetting-based inconsistency measure ($\iforget$)}
    \end{subfigure}\\[1ex]
    
    \begin{subfigure}[t]{.5\textwidth}
        \includegraphics[width=\textwidth]{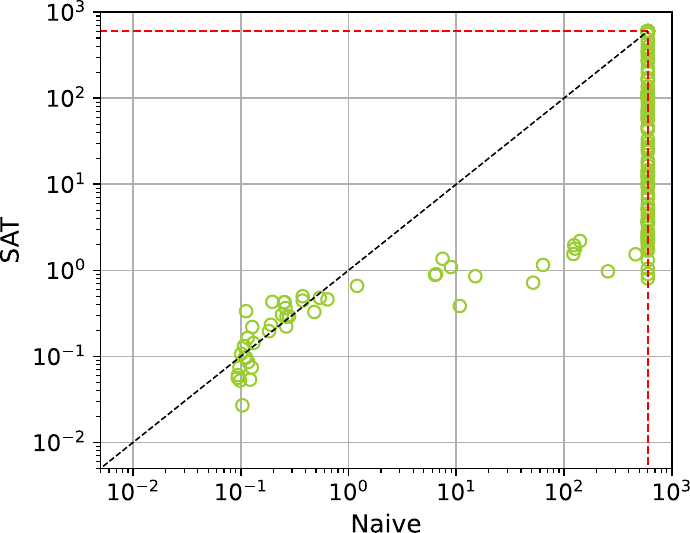}
        \caption{Hitting Set inconsistency measure ($\ihs$)}
    \end{subfigure}%
    \begin{subfigure}[t]{.5\textwidth}
        \includegraphics[width=\textwidth]{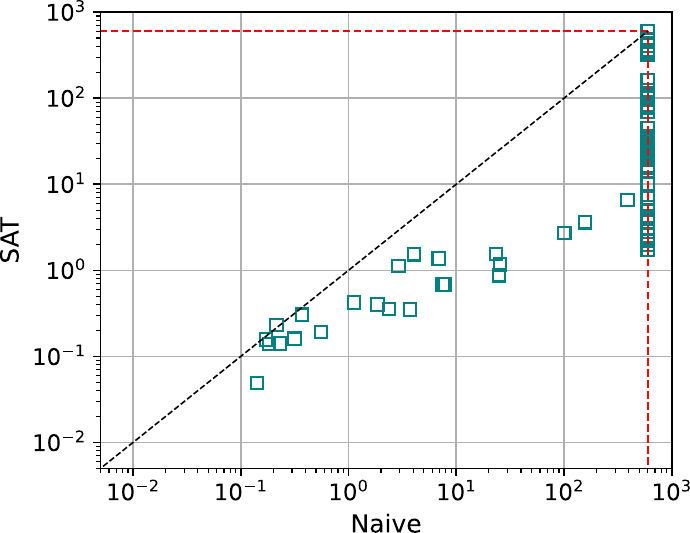}
        \caption{Max-distance inconsistency measure ($\imdalal$)}
    \end{subfigure}\\[1ex]
    
    \begin{subfigure}[t]{.5\textwidth}
        \includegraphics[width=\textwidth]{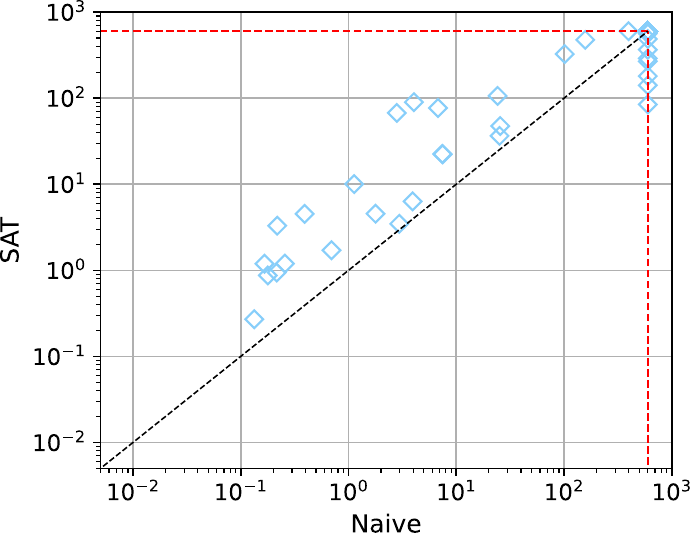}
        \caption{Sum-distance inconsistency measure ($\isdalal$)}
    \end{subfigure}%
    \begin{subfigure}[t]{.5\textwidth}
        \includegraphics[width=\textwidth]{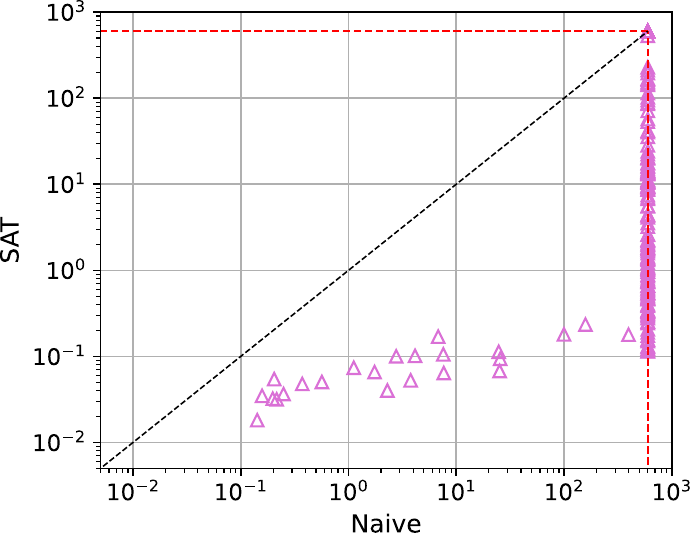}
        \caption{Hit-distance inconsistency measure ($\ihdalal$)}
    \end{subfigure}
    \caption{Runtime comparison of the SAT-based and naive approaches on the ARG data set. Timeout: $10$ minutes.}
    \label{fig:scatter-ARG-sat-naive}
\end{figure}

% SAT
\begin{figure}
    \begin{subfigure}[t]{.5\textwidth}
        \includegraphics[width=\textwidth]{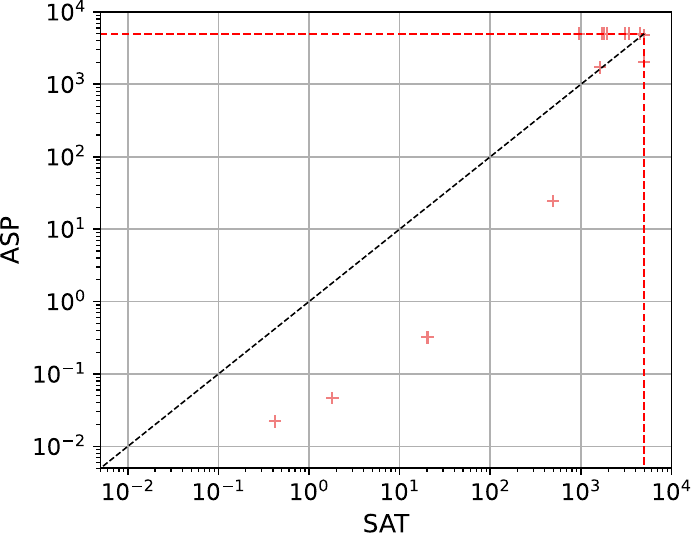}
        \caption{Contension inconsistency measure ($\icont$)}
    \end{subfigure}%
    \begin{subfigure}[t]{.5\textwidth}
        \includegraphics[width=\textwidth]{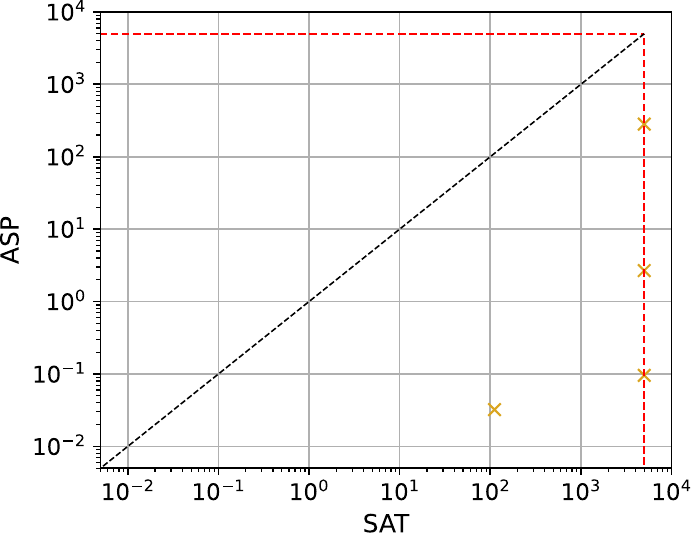}
        \caption{Forgetting-based inconsistency measure ($\iforget$)}
    \end{subfigure}\\[1ex]
    
    \begin{subfigure}[t]{.5\textwidth}
        \includegraphics[width=\textwidth]{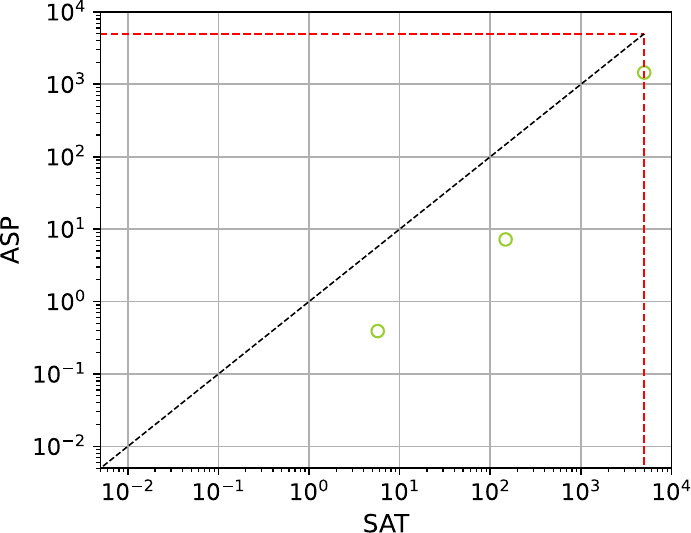}
        \caption{Hitting Set inconsistency measure ($\ihs$)}
    \end{subfigure}%
    \begin{subfigure}[t]{.5\textwidth}
        \includegraphics[width=\textwidth]{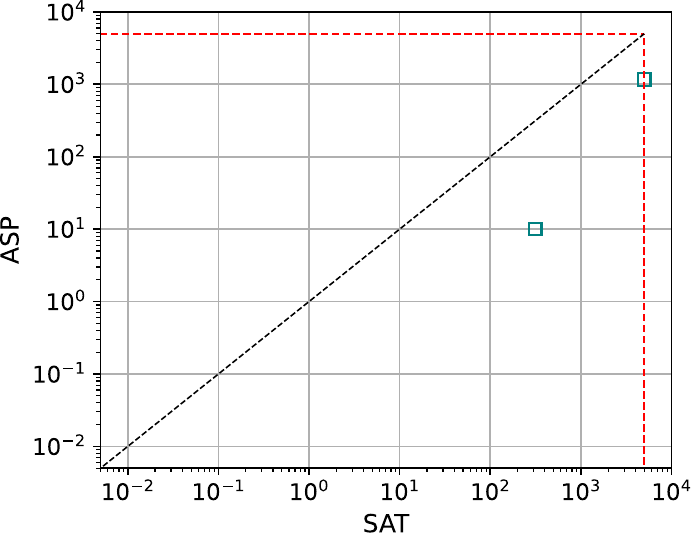}
        \caption{Max-distance inconsistency measure ($\imdalal$)}
    \end{subfigure}\\[1ex]
    
    \begin{subfigure}[t]{.5\textwidth}
        \includegraphics[width=\textwidth]{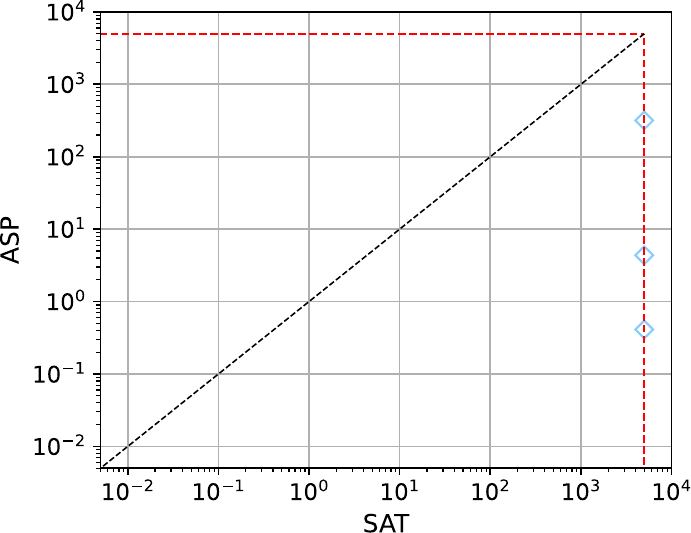}
        \caption{Sum-distance inconsistency measure ($\isdalal$)}
    \end{subfigure}%
    \begin{subfigure}[t]{.5\textwidth}
        \includegraphics[width=\textwidth]{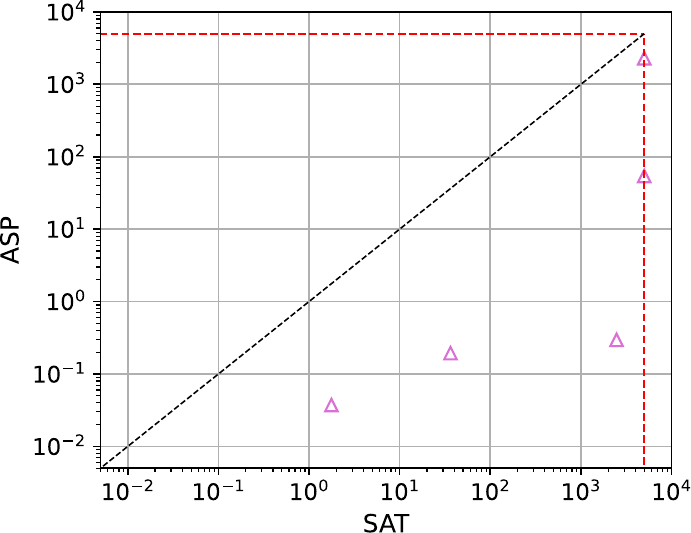}
        \caption{Hit-distance inconsistency measure ($\ihdalal$)}
    \end{subfigure}
    \caption{Runtime comparison of the ASP-based and SAT-based approaches on the SC data set. Timeout: $5000$ seconds.}
    \label{fig:scatter-SAT-asp-sat}
\end{figure}

\begin{figure}
    \begin{subfigure}[t]{.5\textwidth}
        \includegraphics[width=\textwidth]{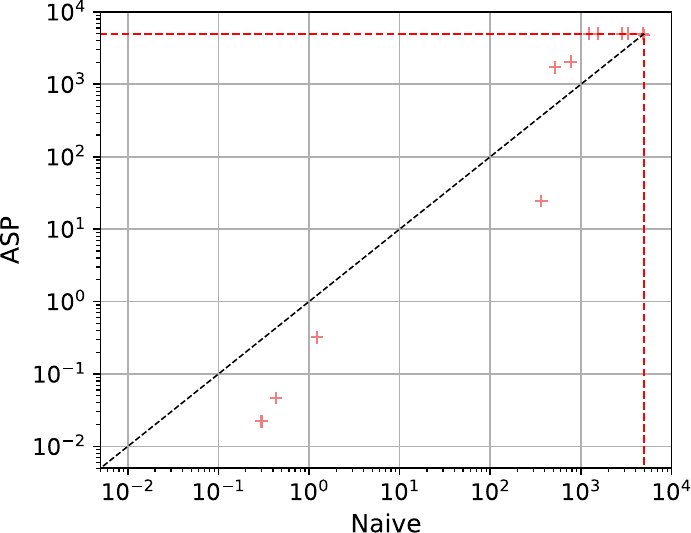}
        \caption{Contension inconsistency measure ($\icont$)}
    \end{subfigure}%
    \begin{subfigure}[t]{.5\textwidth}
        \includegraphics[width=\textwidth]{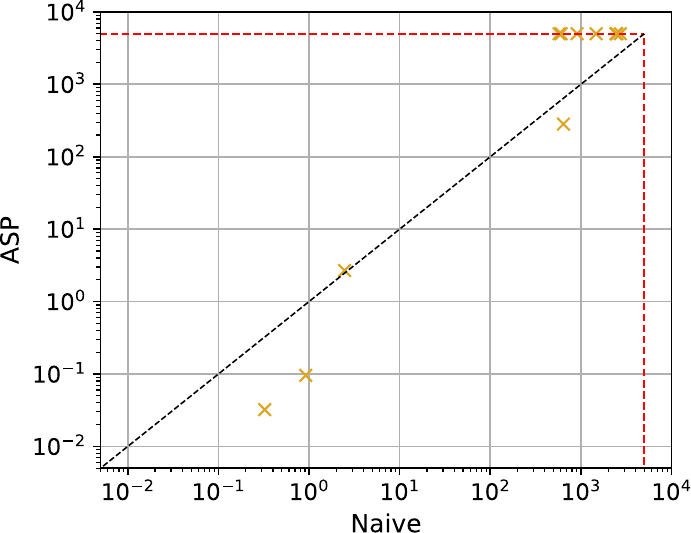}
        \caption{Forgetting-based inconsistency measure ($\iforget$)}
    \end{subfigure}\\[1ex]
    
    \begin{subfigure}[t]{.5\textwidth}
        \includegraphics[width=\textwidth]{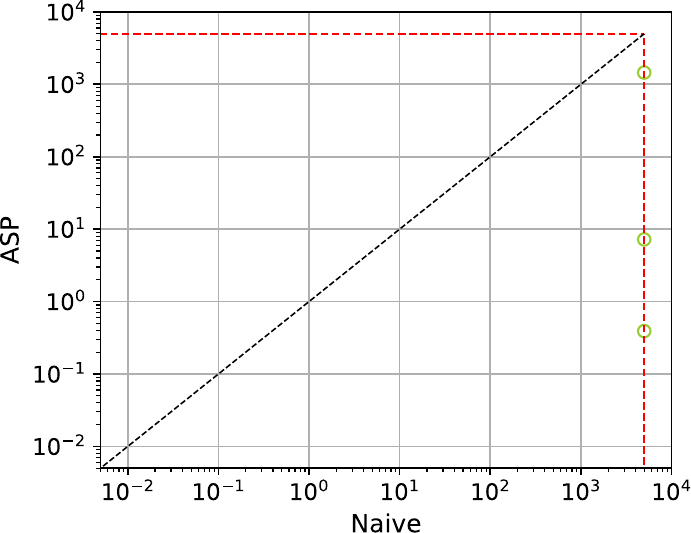}
        \caption{Hitting Set inconsistency measure ($\ihs$)}
    \end{subfigure}%
    \begin{subfigure}[t]{.5\textwidth}
        \includegraphics[width=\textwidth]{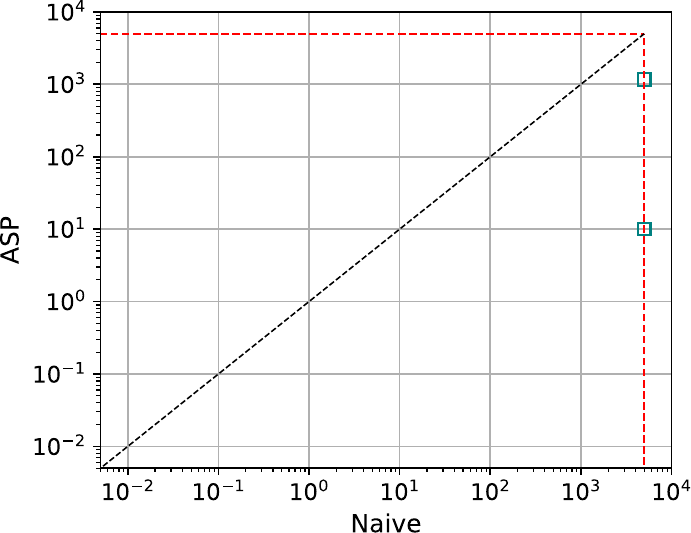}
        \caption{Max-distance inconsistency measure ($\imdalal$)}
    \end{subfigure}\\[1ex]
    
    \begin{subfigure}[t]{.5\textwidth}
        \includegraphics[width=\textwidth]{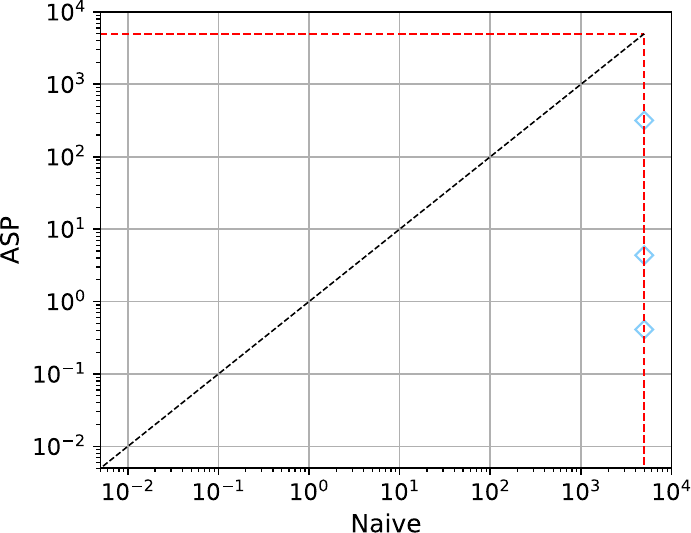}
        \caption{Sum-distance inconsistency measure ($\isdalal$)}
    \end{subfigure}%
    \begin{subfigure}[t]{.5\textwidth}
        \includegraphics[width=\textwidth]{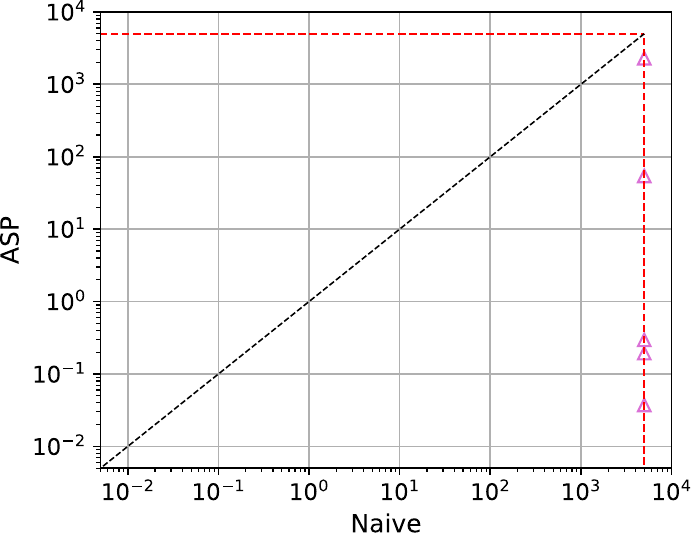}
        \caption{Hit-distance inconsistency measure ($\ihdalal$)}
    \end{subfigure}
    \caption{Runtime comparison of the ASP-based and naive approaches on the SC data set. Timeout: $5000$ seconds.}
    \label{fig:scatter-SAT-asp-naive}
\end{figure}

\begin{figure}
    \begin{subfigure}[t]{.5\textwidth}
        \includegraphics[width=\textwidth]{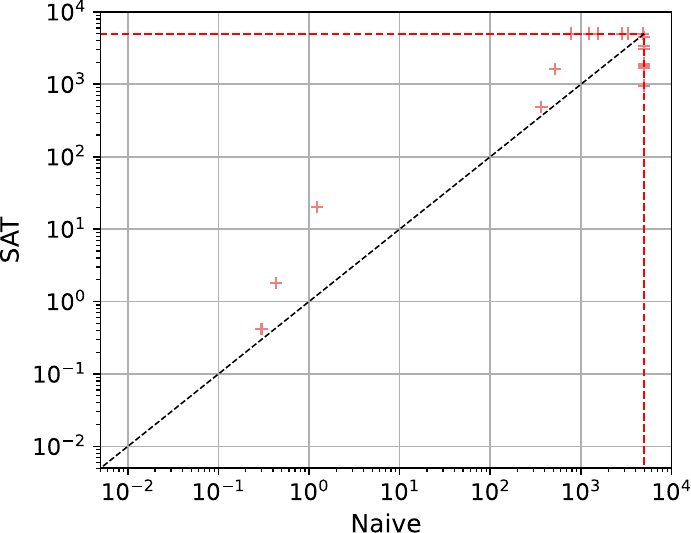}
        \caption{Contension inconsistency measure ($\icont$)}
    \end{subfigure}%
    \begin{subfigure}[t]{.5\textwidth}
        \includegraphics[width=\textwidth]{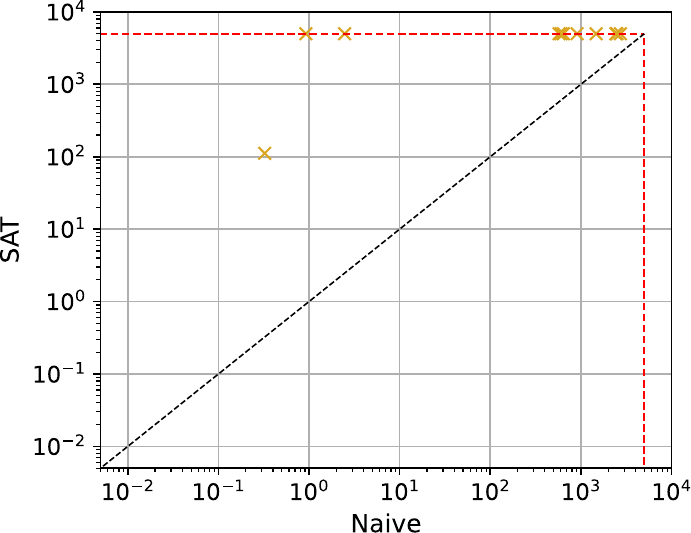}
        \caption{Forgetting-based inconsistency measure ($\iforget$)}
    \end{subfigure}\\[1ex]
    
    \begin{subfigure}[t]{.5\textwidth}
        \includegraphics[width=\textwidth]{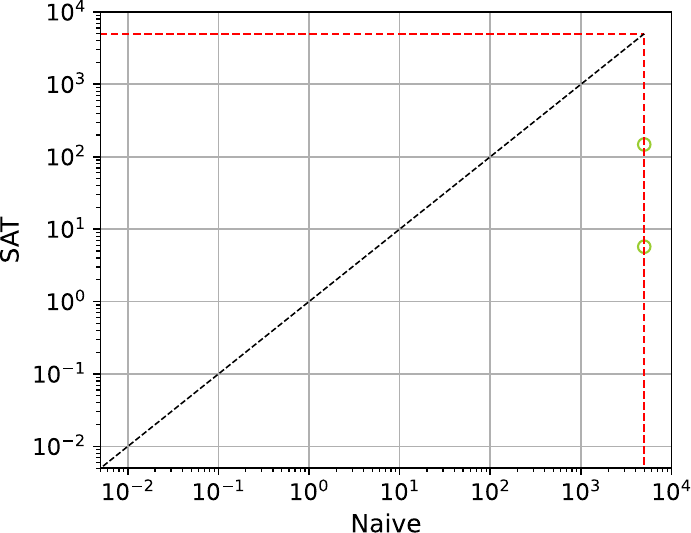}
        \caption{Hitting Set inconsistency measure ($\ihs$)}
    \end{subfigure}%
    \begin{subfigure}[t]{.5\textwidth}
        \includegraphics[width=\textwidth]{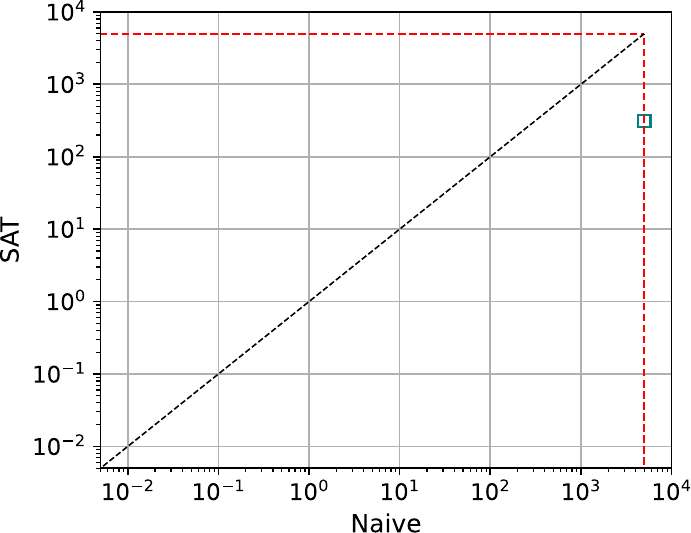}
        \caption{Max-distance inconsistency measure ($\imdalal$)}
    \end{subfigure}\\[1ex]
    
    \begin{subfigure}[t]{.5\textwidth}
        \includegraphics[width=\textwidth]{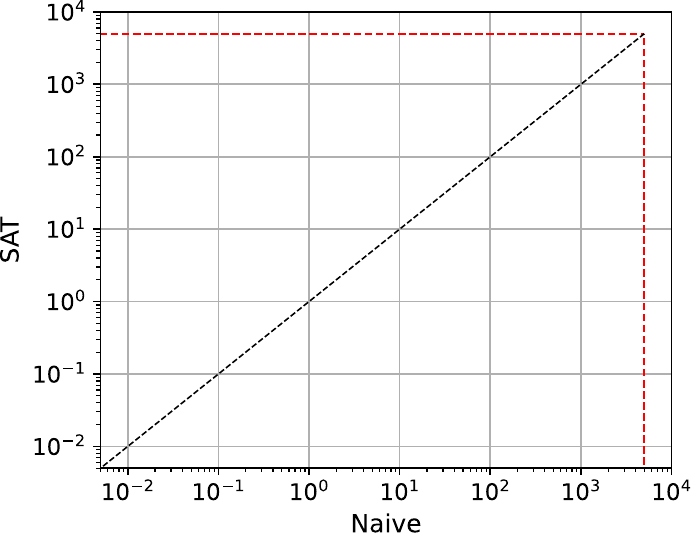}
        \caption{Sum-distance inconsistency measure ($\isdalal$)}
    \end{subfigure}%
    \begin{subfigure}[t]{.5\textwidth}
        \includegraphics[width=\textwidth]{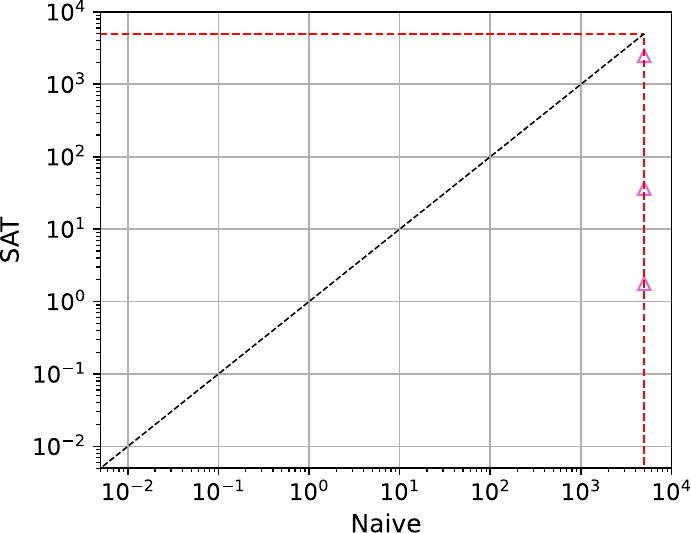}
        \caption{Hit-distance inconsistency measure ($\ihdalal$)}
    \end{subfigure}
    \caption{Runtime comparison of the SAT-based and naive approaches on the SC data set. Timeout: $5000$ seconds.}
    \label{fig:scatter-SAT-sat-naive}
\end{figure}

% ASP
% \ik{LP komplett weglassen? Falls ja, kurz im Text erklären.}

% SAT-linar vs. SAT-binary:
\begin{figure}
    \begin{subfigure}[t]{.5\textwidth}
        \includegraphics[width=\textwidth]{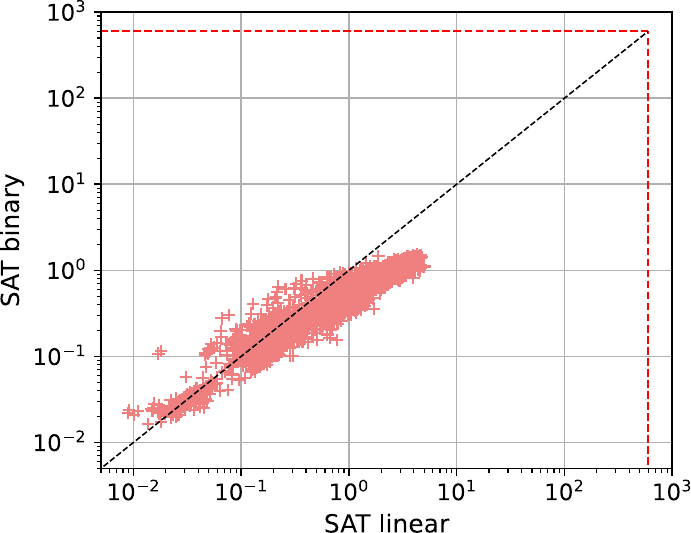}
        \caption{Contension inconsistency measure ($\icont$)}
    \end{subfigure}%
    \begin{subfigure}[t]{.5\textwidth}
        \includegraphics[width=\textwidth]{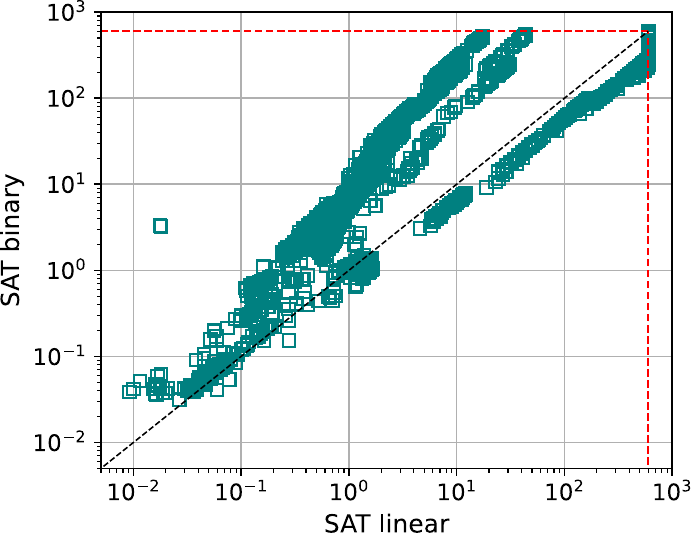}
        \caption{Max-distance inconsistency measure ($\imdalal$)}
    \end{subfigure}\\[1ex]
    
    \begin{subfigure}[t]{.5\textwidth}
        \includegraphics[width=\textwidth]{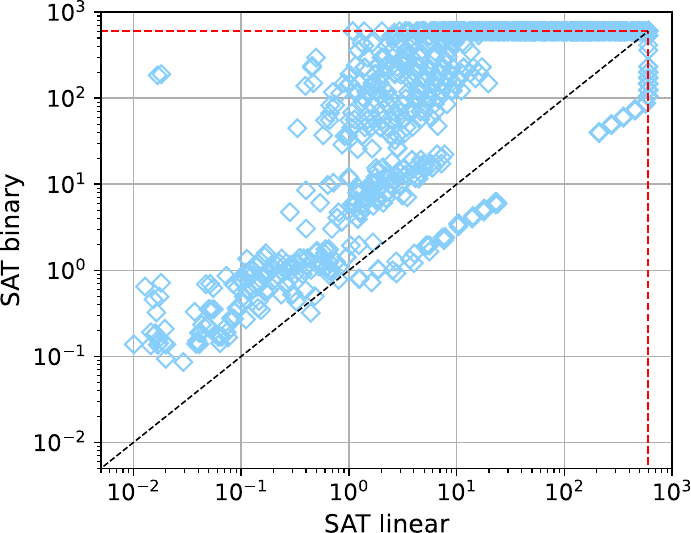}
        \caption{Sum-distance inconsistency measure ($\isdalal$)}
    \end{subfigure}%
    \begin{subfigure}[t]{.5\textwidth}
        \includegraphics[width=\textwidth]{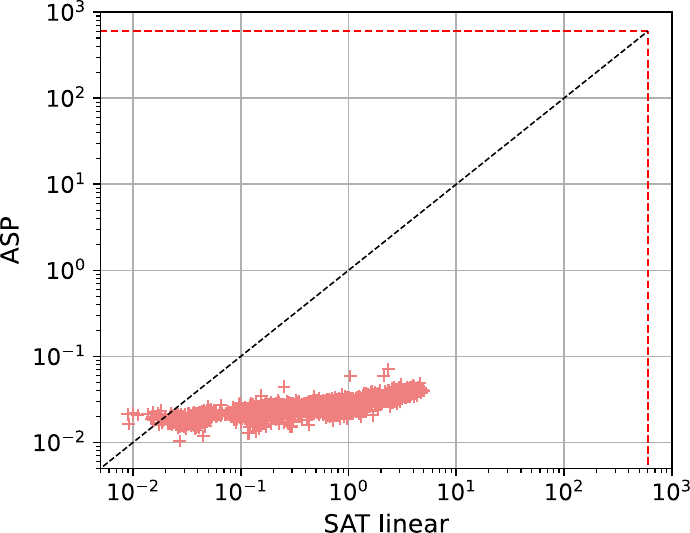}
        \caption{Contension inconsistency measure ($\icont$)}
    \end{subfigure}\\[1ex]
    
    \begin{subfigure}[t]{.5\textwidth}
        \includegraphics[width=\textwidth]{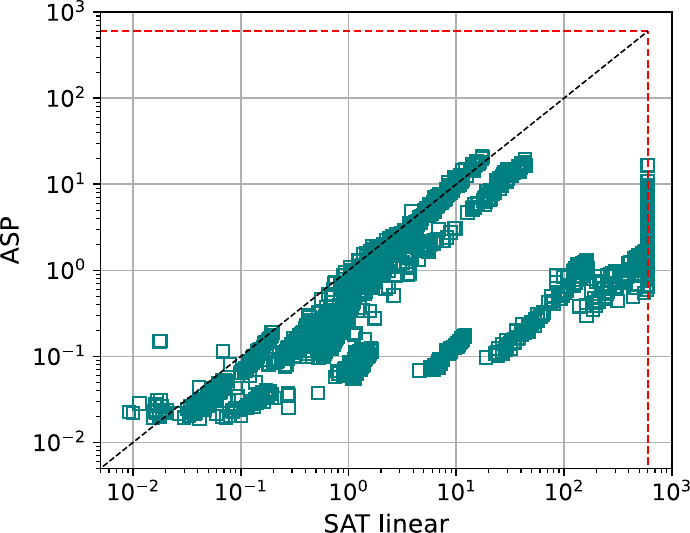}
        \caption{Max-distance inconsistency measure ($\imdalal$)}
    \end{subfigure}%
    \begin{subfigure}[t]{.5\textwidth}
        \includegraphics[width=\textwidth]{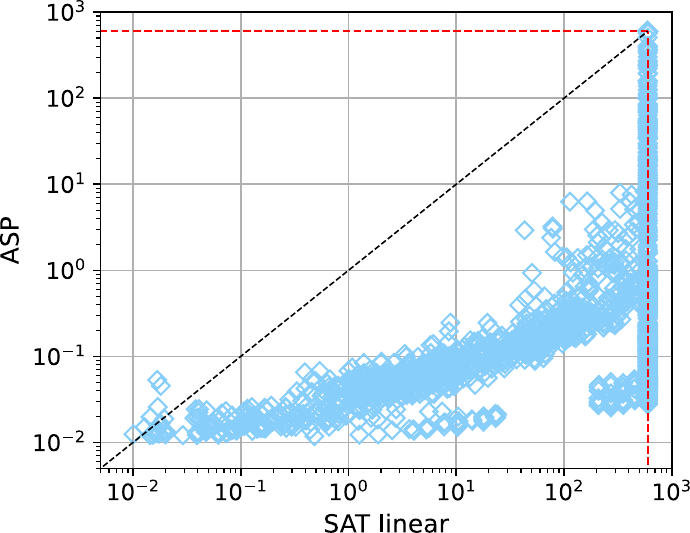}
        \caption{Sum-distance inconsistency measure ($\isdalal$)}
        % \caption{Sum-distance inconsistency measure ($\isdalal$)}
    \end{subfigure}
    \caption{Runtime comparison of the SAT-based approaches based on linear search (for $\icont$, $\imdalal$, and $\isdalal$) and the corresponding binary search versions, as well as the ASP-based versions, on the SRS data set. Timeout: $600$ seconds.}
    \label{fig:scatter-SAT-binary-vs-linear-1}
\end{figure}

\begin{figure}
    \begin{subfigure}[t]{.5\textwidth}
        \includegraphics[width=\textwidth]{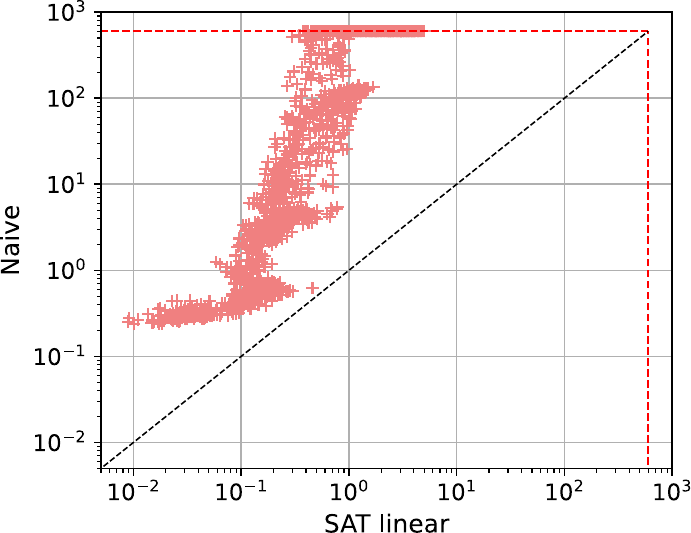}
        \caption{Contension inconsistency measure ($\icont$)}
    \end{subfigure}%
    \begin{subfigure}[t]{.5\textwidth}
        \includegraphics[width=\textwidth]{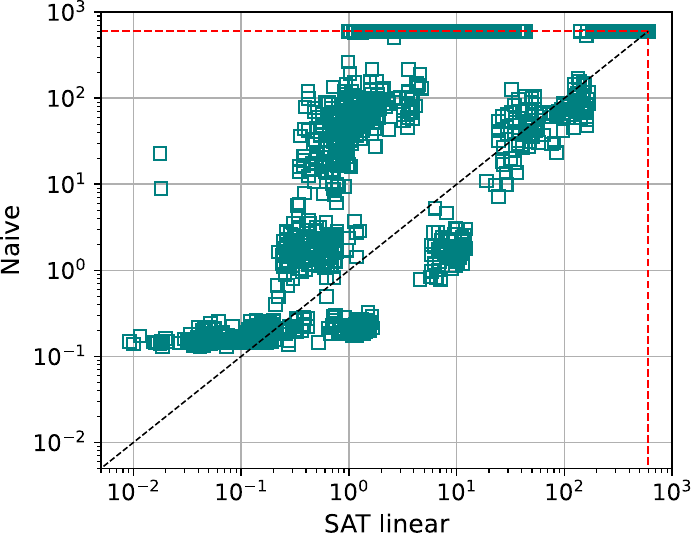}
        \caption{Max-distance inconsistency measure ($\imdalal$)}
    \end{subfigure}\\[1ex]
    
    \begin{subfigure}[t]{.5\textwidth}
        \includegraphics[width=\textwidth]{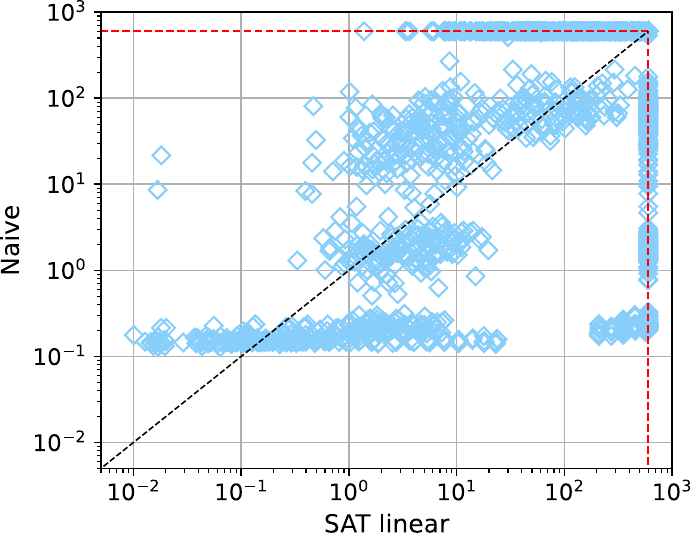}
        \caption{Sum-distance inconsistency measure ($\isdalal$)}
    \end{subfigure}%
    \hfill

    \caption{Runtime comparison of the SAT-based approaches based on linear search (for $\icont$, $\imdalal$, and $\isdalal$) and the corresponding naive methods, on the SRS data set. Timeout: $600$ seconds.}
    \label{fig:scatter-SAT-binary-vs-linear-2}
\end{figure}

% MaxSAT:
\begin{figure}
    \begin{subfigure}[t]{.5\textwidth}
        \includegraphics[width=\textwidth]{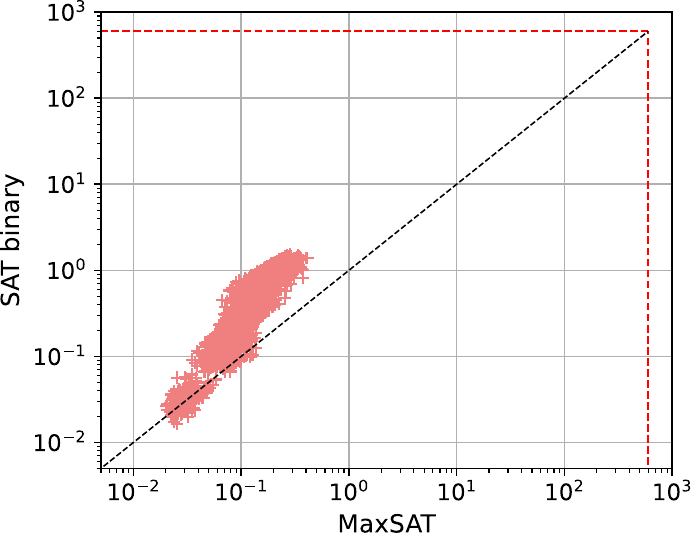}
        % \caption{Contension inconsistency measure ($\icont$)}
    \end{subfigure}%
    \begin{subfigure}[t]{.5\textwidth}
        \includegraphics[width=\textwidth]{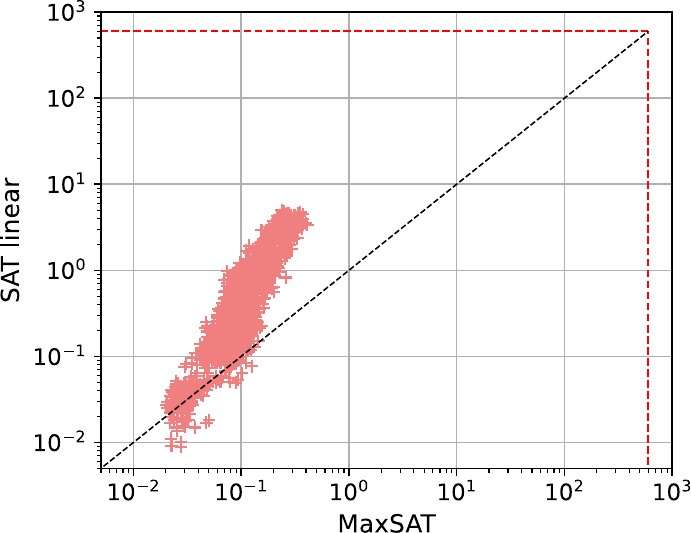}
        % \caption{Forgetting-based inconsistency measure ($\iforget$)}
    \end{subfigure}\\[1ex]
    
    \begin{subfigure}[t]{.5\textwidth}
        \includegraphics[width=\textwidth]{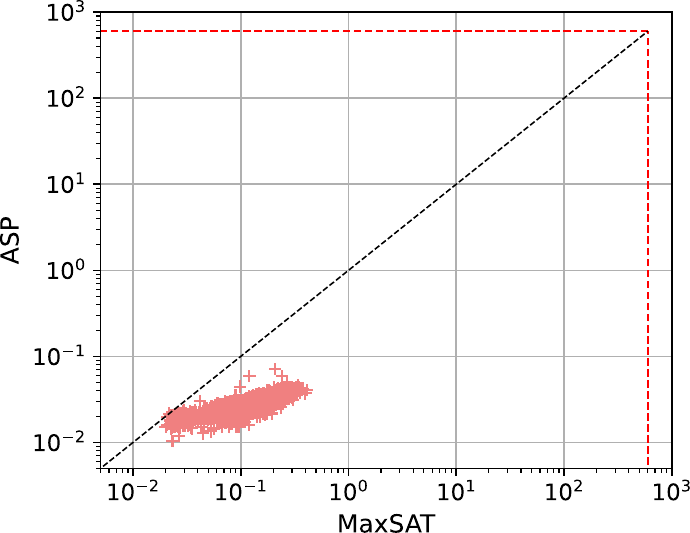}
        % \caption{Hitting Set inconsistency measure ($\ihs$)}
    \end{subfigure}%
    \begin{subfigure}{.5\textwidth}
        \includegraphics[width=\textwidth]{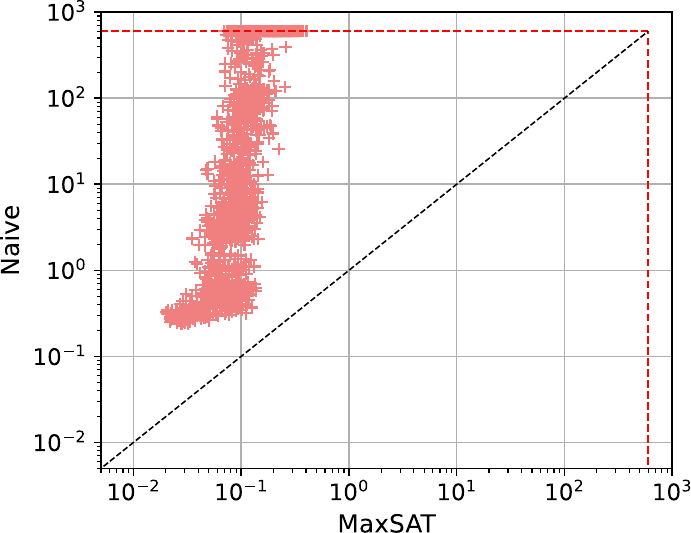}
        % \caption{Max-distance inconsistency measure ($\imdalal$)}
    \end{subfigure}
    \caption{Runtime comparison of the MaxSAT approach for $\icont$ and the corresponding other approaches on the SRS data set. Timeout: $600$ seconds.}
    \label{fig:scatter-maxsat}
\end{figure}

% ASP usc vs. bb:
% \begin{figure}
%     \centering
%     \includegraphics[width=.5\textwidth]{img/asp_usc/scatter_ASP_usc-bb.pdf}
%     \caption{Runtime comparison of the ASP approach for $\ihs$ with core-guided vs.\ with branch-and-bound optimization (default) on the ML data set. Timeout: $600$ seconds.}
%     \label{fig:scatter-asp-usc-bb}
% \end{figure}
\begin{figure}
    \centering
    \begin{subfigure}[t]{.49\textwidth}
        \includegraphics[width=\textwidth]{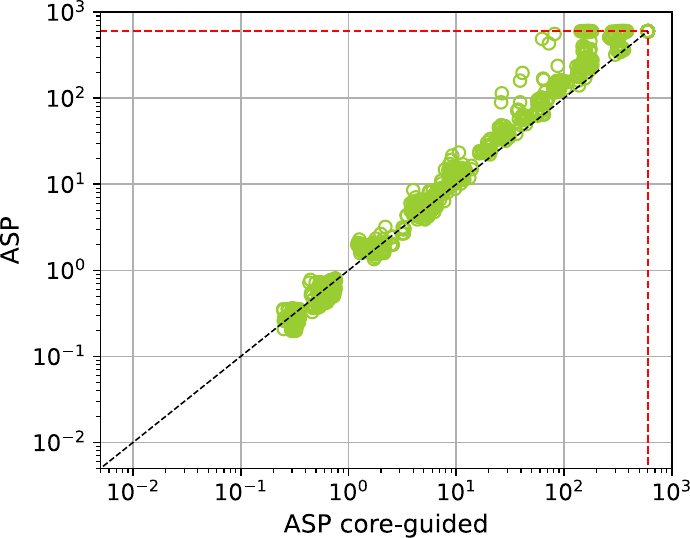}
        \caption{ML / $\ihs$}
    \end{subfigure}\\[1ex]
    \begin{subfigure}[t]{.49\textwidth}
        \includegraphics[width=\textwidth]{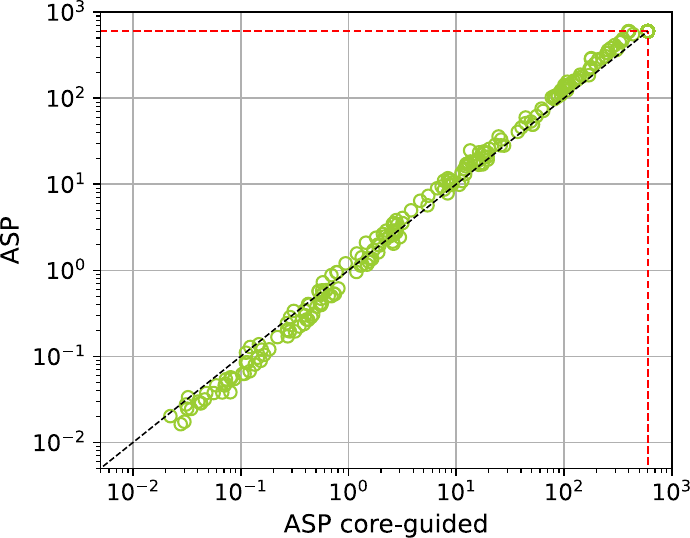}
        \caption{ARG / $\ihs$}
    \end{subfigure}%
    \hfill%
    \begin{subfigure}[t]{.49\textwidth}
        \includegraphics[width=\textwidth]{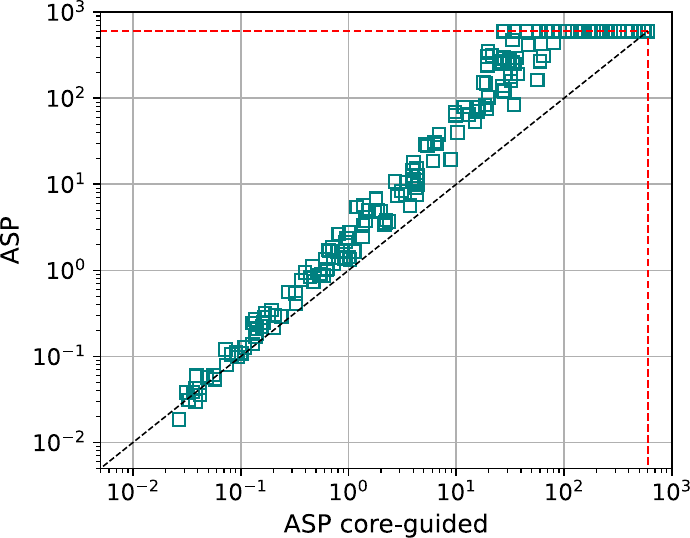}
        \caption{ARG / $\imdalal$}
    \end{subfigure}
    \caption{Runtime comparison of the ASP approach with core-guided vs.\ with branch-and-bound optimization (default) on % selected combinations of data sets and inconsistency measures. 
    the ML data set wrt.\ the hitting set measure ($\ihs$) and on the ARG data set wrt.\ $\ihs$ and the max-distance measure ($\imdalal$).
    Timeout: 10 minutes.}
    \label{fig:scatter-asp-usc-bb}
\end{figure}

\newpage

\subsection{Runtime Composition}

This section contains additional bar plots visualizing the average runtime composition of the SAT-based and ASP-based approaches wrt.\ each measure and the SRS data set (Figure~\ref{fig:runtime-composition-SRS}), as well as the ARG data set (Figure~\ref{fig:runtime-composition-ARG}).

\begin{figure}
    \centering
    \begin{subfigure}[t]{\textwidth}
        \includegraphics[width=.93\textwidth]{img/runtime_composition/runtime_composition_SRS_contension.pdf}
        \caption{Contension inconsistency measure ($\icont$)}
    \end{subfigure}\\[1ex]
    
    \begin{subfigure}[t]{\textwidth}
        \includegraphics[width=.93\textwidth]{img/runtime_composition/runtime_composition_SRS_fb.pdf}
        \caption{Forgetting-based inconsistency measure ($\iforget$)}
    \end{subfigure}\\[1ex]
    
    \begin{subfigure}[t]{\textwidth}
        \includegraphics[width=.93\textwidth]{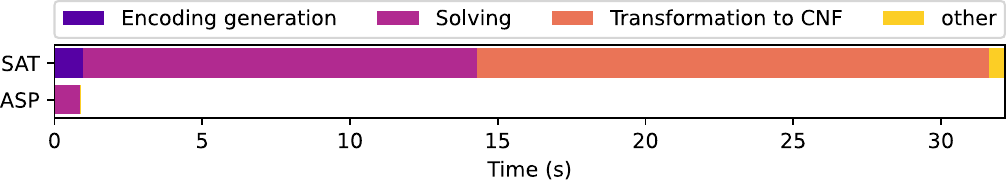}
        \caption{Hitting Set inconsistency measure ($\ihs$)}
    \end{subfigure}\\[1ex]
    
    \begin{subfigure}[t]{\textwidth}
        \includegraphics[width=.93\textwidth]{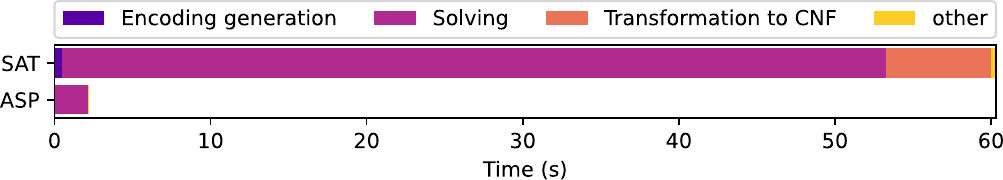}
        \caption{Max-distance inconsistency measure ($\imdalal$)}
    \end{subfigure}\\[1ex]
    
    \begin{subfigure}[t]{\textwidth}
        \includegraphics[width=.93\textwidth]{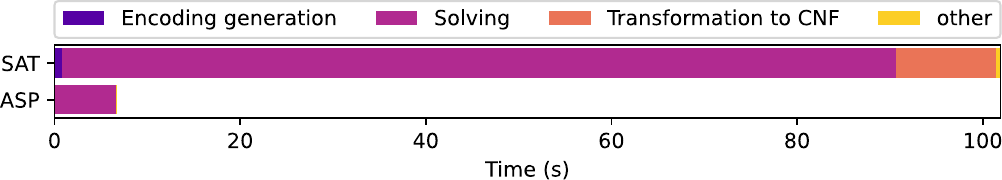}
        \caption{Sum-distance inconsistency measure ($\isdalal$)}
    \end{subfigure}\\[1ex]
    
    \begin{subfigure}[t]{\textwidth}
        \includegraphics[width=.93\textwidth]{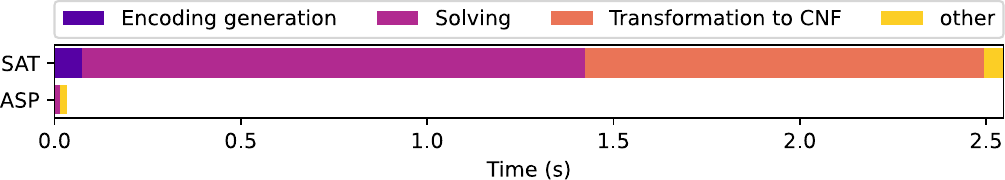}
        \caption{Hit-distance inconsistency measure ($\ihdalal$)}
    \end{subfigure}
    \caption{Comparison between the ASP-based and SAT-based approaches in terms of runtime composition wrt.\ the SRS data set.}
    \label{fig:runtime-composition-SRS}
\end{figure}

\begin{figure}
    \centering
    \begin{subfigure}[t]{.93\textwidth}
        \includegraphics[width=\textwidth]{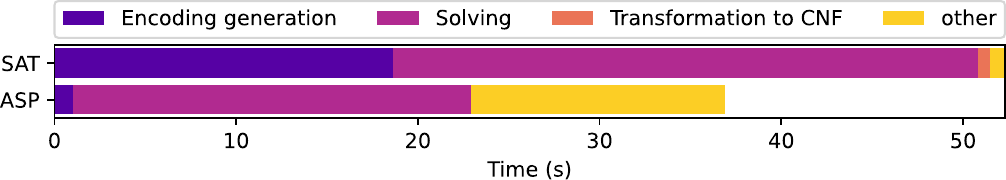}
        \caption{Contension inconsistency measure ($\icont$)}
    \end{subfigure}\\[1ex]
    
    \begin{subfigure}[t]{.93\textwidth}
        \includegraphics[width=\textwidth]{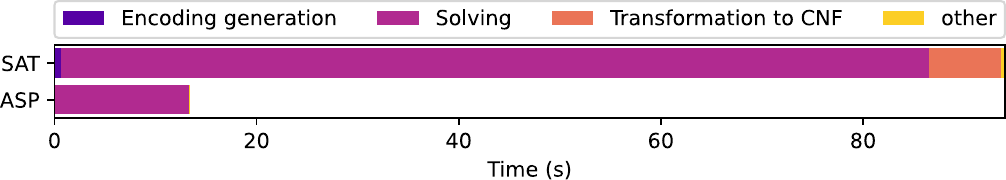}
        \caption{Forgetting-based inconsistency measure ($\iforget$)}
    \end{subfigure}\\[1ex]
    
    \begin{subfigure}[t]{.93\textwidth}
        \includegraphics[width=\textwidth]{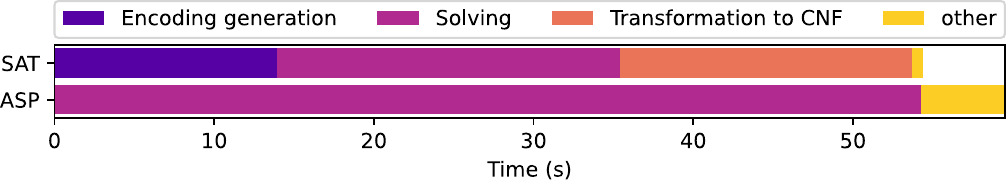}
        \caption{Hitting Set inconsistency measure ($\ihs$)}
    \end{subfigure}\\[1ex]
    
    \begin{subfigure}[t]{.93\textwidth}
        \includegraphics[width=\textwidth]{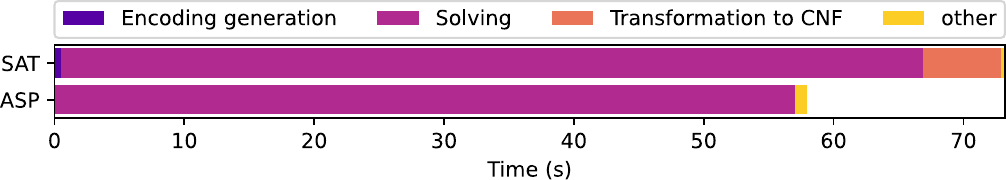}
        \caption{Max-distance inconsistency measure ($\imdalal$)}
    \end{subfigure}\\[1ex]
    
    \begin{subfigure}[t]{.93\textwidth}
        \includegraphics[width=\textwidth]{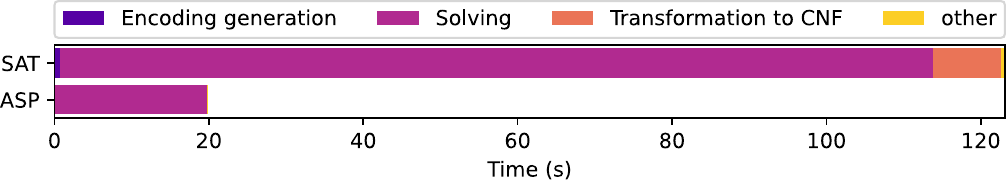}
        \caption{Sum-distance inconsistency measure ($\isdalal$)}
    \end{subfigure}\\[1ex]
    
    \begin{subfigure}[t]{.93\textwidth}
        \includegraphics[width=\textwidth]{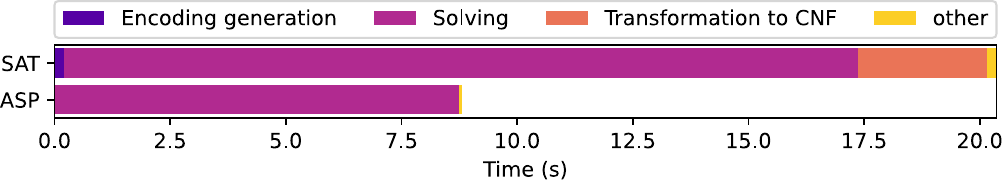}
        \caption{Hit-distance inconsistency measure ($\ihdalal$)}
    \end{subfigure}
    \caption{Comparison between the ASP-based and SAT-based approaches in terms of runtime composition wrt.\ the ARG data set.}
    \label{fig:runtime-composition-ARG}
\end{figure}

% \newpage
\clearpage
% \vskip 0.2in
\bibliography{references}
\bibliographystyle{theapa}

\end{document}